\documentclass{IOS-Book-Article}

\usepackage{mathptmx}
\usepackage{soul}\setuldepth{article}
\def\hb{\hbox to 11.5 cm{}}

\usepackage{amsmath}
\usepackage{amsthm}
\usepackage{csquotes}
\usepackage{xspace}
\usepackage{todonotes}

\usepackage{amssymb}
\usepackage{mathtools}

\usepackage{tikz}
\usetikzlibrary{arrows,arrows.meta,positioning,shapes.misc,shapes.geometric,shapes.symbols,patterns}
\usetikzlibrary{arrows, decorations.markings}

\usepackage{comment}

\begin{document}
\newtheorem{example}{Example}
\newtheorem{theorem}{Theorem}
\newtheorem{definition}{Definition}
\newtheorem{proposition}{Proposition}
\newtheorem{corollary}{Corollary}
\newtheorem{lemma}{Lemma}
\newtheorem{illustration}{Illustration}

\tikzset{
    >=stealth',
    args/.style={circle,draw=black, thick, minimum size=6mm},
    multiLine/.style={double, thick, line width=1.4pt},
    suppNode/.style={inner sep=0, white, minimum size=0mm},
}
\newcommand{\tip}{{Latex[length=3mm]}}
\tikzstyle{multiEnd} = [
	thick,
	decoration={markings,mark=at position 1 with {\arrow[semithick]{triangle 60}}},
	double distance=1.4pt,
	shorten >= 5.5pt,
	preaction = {decorate},
	postaction = {draw,line width=1.4pt, white,shorten >= 4.5pt}]
\tikzstyle{multiBegin} = [
	thick,
	line width = 1.4pt,
	double distance = 1.4pt]
\tikzstyle{innerWhite} = [semithick, white,line width=1.4pt, shorten >= 4.5pt]

\newcommand{\tab}{\hspace{0.4cm}}
\newcommand{\stab}{\hspace{0.2cm}}
\newcommand{\ie}{i.e.\ }
\newcommand{\suchthat}{s.t.\ }
\newcommand{\wrt}{w.r.t.\ }
\newcommand{\eg}{e.g.\ }
\newcommand{\etal}{et~al.~}
\newcommand{\generality}{w.l.o.g.\ }
\newcommand{\Generality}{W.l.o.g.\ }
\newcommand{\etc}{etc.\ }
\newcommand{\cf}{cf.\ }
\newcommand{\AS}{$AS=(R_s, R_d, n,\leq_{r})$\xspace}
\newcommand{\DAOM}{\textnormal{Deductive ASPIC}$^{\ominus}$\xspace}
\newcommand{\DAOMItallic}{Deductive ASPIC$^{\ominus}$\xspace}
\newcommand{\LogLang}{$\mathcal{L}=(\mathcal{F},Atoms,\leq_l)$\xspace}
\newcommand{\isModelOf}{\vDash_M}
\newcommand{\entails}{\vDash_C}
\newcommand{\ClassLog}{$\mathcal{CL}=(\mathcal{F},Atoms,\leq_{\mathcal{F}},\mathcal{I},\vDash_M)$\xspace}
\newcommand{\DefTheo}{$DT=(\mathcal{CL},AX, R_d,n,\leq_r)$\xspace}
\newcommand{\JSBAF}{$\mathcal{J}=(\mathcal{A},{\rightarrow},{\Rightarrow},{\preceq})$\xspace}
\newcommand{\JSBAFOne}{$\mathcal{J}_1=(\mathcal{A}_1,{\rightarrow_1},{\Rightarrow_1},{\preceq_1})$\xspace}
\newcommand{\JSBAFTwo}{$\mathcal{J}_2=(\mathcal{A}_2,{\rightarrow_1},{\Rightarrow_2},{\preceq_2})$\xspace}
\newcommand{\JSBAFPlus}{$\mathcal{J}_+=(\mathcal{A}_+,{\rightarrow_+},{\Rightarrow_+},{\preceq_+})$\xspace}
\newcommand{\F}{$\mathcal{F}$\xspace}
\newcommand{\AF}{$\mathcal{F}=(Args,Att)$\xspace}
\newcommand{\J}{$\mathcal{J}$\xspace}
\newcommand{\f}{$\mathcal{AF}$\xspace}
\newcommand{\In}{\textit{IN}\xspace}
\newcommand{\Out}{\textit{OUT}\xspace}
\newcommand{\Undec}{\textit{UNDEC}\xspace}
\newcommand{\SIM}{strict including minimal\xspace}
\newcommand{\JDA}{$\mathcal{J}\in\mathfrak{J}_{DA^{\ominus}}$\xspace}
\newcommand{\JSBAFOneTwo}{$\mathcal{J}_1=(Args_1,Att_1,Supp_1)$ and $\mathcal{J}_2=(Args_2,Att_2,Supp_2)$\xspace}
\newcommand{\BeginJSBAFSingular}{Let \JSBAF be a JSBAF \st $\mathcal{J}\in\mathfrak{J}_{\DA}$.\xspace}
\newcommand{\DefTheoOne}{$DT_1=(\mathcal{CL},AX_1, R_{d_1},n_1)$\xspace}
\newcommand{\DefTheoTwo}{$DT_2=(\mathcal{CL},AX_2, R_{d_2},n_2)$\xspace}
\newcommand{\ASOne}{$AS_1=(R_{s_1},R_{d_1},n_1,{\leq_{r_1}})$\xspace}
\newcommand{\ASTwo}{$AS_2=(R_{s_2},R_{d_2},n_2,{\leq_{r_2}})$\xspace}
\newcommand{\ASPlus}{$AS_+=(R_s^+,R_d^+,n^+,{\leq_r^+})$\xspace}
\newcommand{\BeginArgSysDual}{Let \ArgSysOne and \ArgSysTwo be two AS \st $AS_1$ and $AS_2$ are syntactically disjoint.\xspace}
\newcommand{\BeginArgSysDualForLemma}{Let \ArgSysOne and \ArgSysTwo be two syntactically disjoint AS.\xspace}
\newcommand{\BeginMixedDual}{Let \ArgSysOne and \ArgSysTwo be two AS which are syntactically disjoint.
Let \JSBAFOne, \JSBAFTwo and \JSBAFPlus be the JSBAF's corresponding to $AS_1$, $AS_2$ and $AS_1\uplus AS_2$.\xspace}
\newcommand{\chainN}{$\big{\{}(S_0,B_0),...,(S_n,B_n)\big{\}}$\xspace}
\newcommand{\JSBAFGR}{$\mathcal{J}=(Args,Att,Supp)$\xspace}


\newcommand{\Tom}[1]{\todo[color=blue!50!white, inline=true]{#1}}
\newcommand{\Marcos}[1]{\todo[color=green!50!white, inline=true]{#1}}
\newcommand{\WIP}[1]{\todo[color=yellow, inline=true]{#1}}
\newcommand{\Rewriting}[1]{\todo[color=red!50!white, inline=true]{#1}}


\includecomment{technicalReportOnly}
\excludecomment{conferencePaperOnly}
\excludecomment{submissionOnly}


\pagestyle{headings}
\def\thepage{}
\begin{frontmatter}              

\title{Satisfying Rationality Postulates of Structured Argumentation Through Deductive Support -- Technical Report}

\markboth{}{April 2026\hb}

\author[A]{\fnms{Marcos} \snm{Cramer}\orcid{0000-0002-9461-1245}}
and
\author[B]{\fnms{Tom} \snm{Friese}
\thanks{Corresponding Author: Tom Friese, tom.friese@tu-dresden.de.}}
\address[A]{secunet Security Networks AG}
\address[B]{TU Dresden}

\begin{abstract}
ASPIC-style structured argumentation frameworks provide a formal basis for reasoning in artificial intelligence by combining internal argument structure with abstract argumentation semantics. A key challenge in these frameworks is ensuring compliance with five critical rationality postulates: closure, direct consistency, indirect consistency, non-interference, and crash-resistance. Recent approaches, including \mbox{$\textnormal{ASPIC}^{\ominus}$} and \mbox{Deductive ASPIC$-$}, have made significant progress but fall short of meeting all postulates simultaneously under a credulous semantics (e.g.\ \emph{preferred}) in the presence of undercuts. This paper introduces Deductive ASPIC$^{\ominus}$, a novel framework that integrates gen-rebuttals from \mbox{$\textnormal{ASPIC}^{\ominus}$} with the Joint Support Bipolar Argumentation Frameworks (JSBAFs) of \mbox{Deductive ASPIC$-$}, incorporating preferences. We show that \mbox{Deductive ASPIC$^{\ominus}$} satisfies all five rationality postulates under a version of preferred semantics. This work opens new avenues for further research on robust and logically sound structured argumentation systems.
\end{abstract}

\end{frontmatter}
\markboth{April 2026\hb}{April 2026\hb}

\section{Introduction}
Formal argumentation has become a fruitful field of research within AI~\cite{ArgumentationInAI}. 
Dung~\cite{Dung1995:OnTheAcceptabilityOfArguments} introduced \emph{argumentation frameworks} (\emph{AF}), directed graphs where nodes represent \emph{arguments} and edges represent \emph{attacks}.
Acceptance of arguments is decided by applying \emph{argumentation semantics} to AFs.
Two prominent semantics are grounded semantics and preferred semantics.
In \emph{structured argumentation}, a logical language builds the basis for the internal structure of arguments, which in turn gives rise to the attack relation.
The result is an argumentation framework, to which argumentation semantics can be applied.

We consider ASPIC-style frameworks for structured argumentation, in which arguments are built inductively from strict and defeasible inference rules~\cite{Amgoud2006:FinalReviewReportFormalArgumentation,Prakken2010:AbstractFramework4ArgumentationWithStructureArguments}.
Attacks in ASPIC can target conclusions of arguments (rebuttals) or applications of defeasible rules (undercuts).
Preferences over the defeasible rules are used to decide which attacks result in defeats.
Arguments and defeats are then interpreted as abstract argumentation frameworks and the acceptance of arguments is decided by argumentation semantics.

There are variants of ASPIC, such as ASPIC+ and ASPIC$-$, that differ in some details, e.g.\ the distinction between \emph{restricted} and \emph{unrestricted rebuttals}~\cite{Prakken2010:AbstractFramework4ArgumentationWithStructureArguments,Caminada2014:PreferencesUnrestrictedRebut}.
Restricted rebuttals can only target conclusions derived from defeasible rules, which can lead to counter-intuitive results when considering argumentation in a dialectical context, as discussed by Caminada~\etal~\cite{Caminada2014:PreferencesUnrestrictedRebut}.
On the other hand, unrestricted rebuttals can target conclusions derived from either defeasible or strict rules, provided the attacked argument is defeasible.
An even more generalized approach, \emph{gen-rebuttals}, can target multiple sub-arguments simultaneously, providing a powerful mechanism for reasoning~\cite{Heyninck2017:RevisitingUnrestrictedRebutPreferences}.

To ensure logical soundness and coherence, a robust argumentation framework should satisfy key rationality postulates. The postulates of \emph{closure}, \emph{direct consistency} and \emph{indirect consistency}, which were introduced by Caminada and Amgoud~\cite{Caminada2007:OnEvaluationOfArgumentationFormalisms}, ensure that the accepted arguments do not conflict with one another and that whenever the antecedents of a strict rule are accepted, its conclusion is also accepted.
On the other hand, the postulates of \emph{non-interference} and \emph{crash-resistance}, introduced by Caminada~\etal~\cite{Caminada2012:SemiStableSemantics}, ensure that, when combining two argumentation frameworks that do not overlap in the knowledge they model, neither of them influences the other.

Meeting all five postulates under various semantics has proven challenging.
Heyninck and Straßer's $\textnormal{ASPIC}^{\ominus}$~\cite{Heyninck2017:RevisitingUnrestrictedRebutPreferences} introduced gen-rebuttals and demonstrated compliance with all five postulates by ignoring undercuts and focusing on grounded semantics, which only has singular solutions.
Cramer and Bhadra's Deductive ASPIC$-$~\cite{Cramer2020:DeductiveJointSupportForRationalUnrestrictedRebuttal} addressed these issues using \emph{Joint Support Bipolar Argumentation Frameworks} (JSBAFs), which track the application of strict rules to enforce closure.
However, an error in one of their proofs means the closure postulate can be violated by using strict rules without antecedents.

This work introduces Deductive ASPIC$^{\ominus}$, a novel framework combining the strengths of $\textnormal{ASPIC}^{\ominus}$ and Deductive ASPIC$-$.
Our approach leverages gen-rebuttals and preference-enhanced JSBAFs to provide a unified solution that satisfies all five rationality postulates under a version of preferred semantics for JSBAFs and while considering preferences.
What's more, our preferred semantics aligns with that of Dung's abstract argumentation framework.
To our knowledge, this is the first ASPIC-style framework for structured argumentation that satisfies all five above-mentioned rationality postulates in a credulous semantics like preferred while considering both rebuttals and undercuts and while avoiding the downsides of restricted rebuttal.

The rest of this paper is structured as follows:
Section \ref{Sec:RelatedWork} outlines related versions of ASPIC.
Section~\ref{Sec:Motviation} explains key motivations for our semantics.
Section~\ref{Sec:JSBAF} defines JSBAFs and their preferred semantics.
Section~\ref{Sec:DAOM} introduces \DAOM and presents the results that preferred semantics of \DAOM satisfies all five rationality postulates.
\begin{technicalReportOnly}
Section~\ref{Sec:Grounded} introduces our version of grounded semantics and contains proofs of the existence and uniqueness of grounded labelings.
\end{technicalReportOnly}
Section~\ref{Sec:FutureWork} gives an outlook on avenues for future work and Section~\ref{Sec:Conclusion} concludes this paper.
\begin{conferencePaperOnly}
The detailed proofs of the results in this paper are in a technical report~\cite{cramer2026:SatisfyingRationalityPostulatesTechnicalReport}.
\end{conferencePaperOnly}

\section{Related Work}\label{Sec:RelatedWork}
Modgil and Prakken~\cite{Modgil2013:GeneralAccountOfArgumentationWithPreferences} have shown that ASPIC+ (which uses restricted rebuttals and incorporates preferences) satisfies the rationality postulates of closure, direct consistency and indirect consistency if the strict rules are closed under transposition. Heyninck and Straßer~\cite{Heyninck2021:RationalityMaximalConsistentSetsForFragmentOfASPICPlus} showed that \mbox{ASPIC+} without undercuts and with a total preference relation additionally satisfies non-interference.

In \cite{Caminada2014:PreferencesUnrestrictedRebut}, the same three rationality postulates were studied for ASPIC$-$, which uses unrestricted rebuttals. The postulates are satisifed only under grounded semantics and the limiting assumptions of a total preference order and closure under transpositions.

Cramer and Bhadra~\cite{Cramer2020:DeductiveJointSupportForRationalUnrestrictedRebuttal} propose \emph{Deductive ASPIC$-$} which, just like \DAOM of the current paper, is based on \emph{Joint Support Bipolar Argumentation Frameworks} (JSBAFs), that track not only attacks but also the support relation between antecedents and conclusions of strict rules. In their approach, JSBAFs are \emph{flattened} into classical argumentation frameworks using auxiliary arguments. They claim their approach satisfies closure, direct consistency, and indirect consistency under classical semantics, but their proof for Lemma 3.17 fails to account for strict rules with no antecedents. The current paper aims to address this issue and to also satisfy the rationality postulates of non-interference and crash-resistance, while incorporating preferences.

Wu and Podlaszewski~\cite{Wu2015:ImplementingCrashResistanceAndNonInterferenceInLogicBasedArgumentation} defined \mbox{\emph{ASPIC Lite}}, where the postulates for non-interference, crash resistance, closure and consistency are satisfied by deleting inconsistent arguments. The approach is not extended to preferences and they present a counterexample to the closure postulate with preferences lifted through the last-link principle.

Heyninck and Straßer~\cite{Heyninck2017:RevisitingUnrestrictedRebutPreferences} introduce the notion of gen-rebuttals (which we also use in the current paper) as a generalization of the unrestricted rebut, giving rise to $\textnormal{ASPIC}^{\ominus}$.
Their version of ASPIC satisfies closure, consistency and non-interference under grounded semantics and when lifting preferences with the weakest-link principle. However, their approach does not consider attacks resulting from undercuts.

To the best of our knowledge, there does not yet exist an ASPIC-style formalism that satisfies all five rationality postulates introduced in the introduction in a credulous semantics like \emph{preferred semantics} while incorporating preferences and undercuts and avoiding the limitations of restricted rebuttals. The current paper aims at filling this gap.

\section{Conceptual Motivation}\label{Sec:Motviation}

This paper introduces Deductive ASPIC$^{\ominus}$, a new ASPIC-style framework for structured argumentation. Like in other such frameworks, arguments are constructed from strict and defeasible rules. Strict rules encode deductive (i.e. exceptionless) reasoning patterns, whereas defeasible rules encode reasoning patterns that allow for exceptions. Like in other ASPIC-style frameworks, the acceptability of arguments in Deductive ASPIC$^{\ominus}$ is determined based on the relation between arguments,
with the difference that not only the attacks between arguments are considered, but also an additional \emph{deductive support}, taken from~\cite{Cramer2020:DeductiveJointSupportForRationalUnrestrictedRebuttal}. A deductive support from argument set $S$ to argument $b$ encodes the information that $b$ was constructed from $S$ using a strict rule, so that the conclusion of $b$ is a deductive consequence of the conclusions of the arguments in $S$.

Note that there is a difference between an un-supported argument and an argument supported by the empty set. When the empty set supports an argument $b$, the conclusion of $b$ deductively follows from the empty set of premises, i.e.\ it is a logical tautology.

The principle of closure under strict rules says that if certain formulas $\phi_1, \dots, \phi_n$ are accepted and there is a strict rule $\phi_1, \dots, \phi_n \rightarrow \psi$, then also $\psi$ should be accepted. In Deductive ASPIC$^{\ominus}$, if such a strict rule exists and there are arguments for the conclusions $\phi_1, \dots, \phi_n$, they will deductively support an argument for the conclusion $\psi$. So all that is needed for a semantics to satisfy closure in Deductive ASPIC$^{\ominus}$ is ensuring acceptance of a supported argument whenever all supporting arguments are accepted.

As in abstract argumentation, attacks in our approach can cause rejection of arguments. But there is another reason to reject an argument, namely by supporting an argument that for some other reason needs to be rejected. This relates to the principle of contraposition and its generalization, the principle of transposition, in classical logic: When $\phi_1, \dots, \phi_n$ logically entail $\psi$, then $\phi_1, \dots, \phi_{i-1}, \neg \psi, \phi_{i+1}, \dots, \phi_n$ logically entail $\neg \phi_i$. We generalize this mechanism for rejecting arguments through supports to a mechanism that applies not only to arguments that are clearly rejected, but also to arguments with an undecided acceptance status. This mechanism for propagating the non-acceptance of arguments through supports is the only function that the supports play in our argumentation semantics, and its enough to ensure that the closure principle is satisfied. 

Note that this mechanism for propagating the non-acceptance of arguments works in the opposite direction of the support direction, due to its connection to the principles of contraposition and transposition.
Furthermore, this mechanism is meant to propagate the non-acceptance of arguments towards further arguments, but it is not meant to bring about the non-acceptance of arguments without some initiating attack.

\section{Joint Support Bipolar Argumentation Frameworks}\label{Sec:JSBAF}
\subsection{Syntax}\label{Sec:JSBAF:Syntax}
The core of our approach are \emph{Joint Support Bipolar Argumentation Framework} (JSBAF) from Cramer~and~Bhadra~\cite{Cramer2020:DeductiveJointSupportForRationalUnrestrictedRebuttal}, extended with preferences between arguments:
%
%
\begin{definition}\label{def:jsbaf}
A JSBAF is a tuple \JSBAF, where:
\begin{itemize}
	\item {$\mathcal{A}$ is a set of \emph{arguments}.}
	\item {${\rightarrow}\subseteq \mathcal{A}\times \mathcal{A}$ is a set of \emph{attacks}.}
	\item {${\Rightarrow}\subseteq 2^{\mathcal{A}}\times \mathcal{A}$ is a set of \emph{supports}.}
	\item{${\preceq}\subseteq \mathcal{A}\times \mathcal{A}$ is a total \emph{preference ordering}.\footnote{
	We use $a\prec b$ as an abbreviation for $a\preceq b$ and $b\not\preceq a$.}
	}
\end{itemize}
\end{definition}
For $(a,b)\in {\rightarrow}$ we say that $a$ attacks $b$.
If no $a$ attacks $b$, we call $b$ unattacked.
For $(S,b)\in {\Rightarrow}$ we say that the set of arguments $S$ supports the argument $b$.
We will use JSBAFs to model arguments and their relations in \DAOM (cf.\ Section~\ref{Sec:ASPIC2JSBAF}).
A support $(S,b)$ will model the case that the argument $b$ was constructed with a strict rule that derives the conclusion of $b$ from the conclusions of the arguments in $S$.
We define a \emph{support chain} as a non-empty set of supports $\{(S_0,b_0),...,(S_n,b_n)\}\subseteq {\Rightarrow}$ \suchthat for each $0\leq i\leq n-1$ we have $b_i\in S_{i+1}$.
We recursively define an argument $a$ as strict iff there is $(S,a)\in {\Rightarrow}$ and we either have $S=\emptyset$ or for all $b\in S$, $b$ is strict.
\begin{technicalReportOnly}
We denote the set of all strict arguments of a JSBAF \J by $STR_{\mathcal{J}}$.
\end{technicalReportOnly}
Strict arguments will model logical proofs of a tautology in \DAOM.
The preference ordering between arguments in a JSBAF will correspond to the preference ordering between arguments in an Argumentation System (cf.\ Section~\ref{Sec:ASPIC2JSBAF}).
%

%

%
In \DAOM, an argument encodes one particular way of reaching a conclusion from a finite set of premises.
Strict arguments are special in that they are derivations of tautologies, making them immune against counterarguments.
To adequately capture arguments and their relations as constructible in \DAOM, we assume the following restrictions on JSBAFs:
\begin{itemize}
	\item{
	For all support chains $\{(S_0,b_0),\dots,$ $(S_n,b_n)\}$, $b_n\not\in S_0$.
	}
	\item{
	For $S,S'\subseteq \mathcal{A}$ and $b\in \mathcal{A}$, $(S,b),(S',b)\in {\Rightarrow}$ implies $S=S'$.
	}
	\item{
	$|S|<\infty$ for all $(S,b)\in {\Rightarrow}$.
	}
	\item{
	If $a\in\mathcal{A}$ is strict, then it is unattacked.
	}
	\item{
	If $a,b\in\mathcal{A}$ are both strict, then $a\preceq b$ and $b\preceq a$.
	}
	\item{
	If $a,b\in\mathcal{A}$ \suchthat $b$ is strict and $a$ is not strict, then $a\prec b$.
	}
\end{itemize}

\subsection{Semantics}\label{Sec:JSBAF:Semantics}
We use a labeling-based approach for our semantics.
For a JSBAF \JSBAF, a \emph{labeling} $L$ is a mapping $L:\mathcal{A}\mapsto\{\In,\Out,\Undec\}$ with the usual meaning of labels: \In denotes accepted arguments, \Out denotes rejected arguments and \Undec denotes arguments whose status is undecided.
%
%
The sets $in(L)$, $out(L)$ and $undec(L)$ denote all arguments labeled \In, \Out and \Undec resprectively.
Our semantics are based on \emph{legal} labelings.
We define whether an argument is legally \In, legally \Out or legally \Undec given the labels of other arguments that are connected to it through attacks or supports.
The key ideas for our definitions of legal labelings are as follows:
First, the preferred semantics for our version of ASPIC will be based on our preferred semantics for JSBAFs (cf.\ Section~\ref{Sec:ASPIC2JSBAF}).
Our ASPIC semantics needs to satisfy closure under strict rules (cf.\ Section~\ref{Sec:RationalityPostulates}) and we enforce this already on the level of our semantics for JSBAFs.
To do this, we utilize the support relation ${\Rightarrow}$:
Because a support $(S,b)$ models the application of a strict rule, satisfying closure under strict rules means that an argument supported only by accepted arguments also needs to be accepted, \ie that for any labeling $L$ and every support $(S,b)$, $S\subseteq in(L)$ implies $b \in in(L)$.
Secondly, our preferred semantics for JSBAFs should be an extension of the preferred semantics for AFs.
More precisely, in the absence of supports, preferred semantics of JSBAFs and AFs should coincide.
Thirdly, it is helpful to think of a labeling as being constructed iteratively, where arguments labeled \Undec have the \emph{potential} to be labeled either \Out or \In, while the labels \In and \Out remain fixed.
Lastly, preferences are only relevant when choosing between arguments which jointly support another one.
For attacks, we will instead consider preferences on the level of our ASPIC formalism, where we delete attacks from weaker arguments to stronger ones, before applying our JSBAF semantics.
%
%

%
Based on these key notions, we first motivate our definition of legally \Out:
As in AFs, an argument is legally \Out if it is attacked by an accepted argument.
However, in a JSBAF an argument can also be legally \Out due to a support:
Suppose we have a support $(S,b)$ with $a \in S$ and a labeling in which $b$ is rejected, while all arguments in $S\setminus\{a\}$ are accepted.
Accepting all arguments in $S$ violates closure, thus we need to reject $a$.
However, we never want to reject an argument because of accepting a weaker one, so we also need to ensure that the other arguments in $S$ are at least as strong as $a$.
The considerations so far could motivate the following definition:
An argument $a$ is legally \Out iff it is attacked by an accepted argument or if it is contained in a set $S$ that supports a rejected argument with all arguments in $S \setminus \{a\}$ being accepted and being at least as strong as $a$.
However, this definition would cause problems in the case of an infinite support chain $\mathcal{C}=\{(S_0,b_0),(S_1,b_1),...\}$, where each $S_i$ contains only a single argument, while every $b_i$ is unattacked.\footnote{
Strict inference rules will be based on the consequence relation of an underlying logical language (cf.\ see Section~\ref{Sec:ASPIC2JSBAF}), thus such a chain could correspond to a statement like $\phi\vDash\phi\vDash\dots$.
}
It would be consistent with the proposed definition to label all $b_i$ in this infinite chain \Out, although there is no attack originating from an accepted argument which contributes to this rejection of the $b_i$.
As explained in Section~\ref{Sec:Motviation}, this is not intended.
Thus, in order to avoid anomalies caused by infinite support chains, we define an argument $a$ as legally \Out iff it is attacked by an accepted argument, or if there is a finite support chain $\mathcal{C}=\{(S_0,b_0),\dots,(S_n,b_n)\}$ where $a\in S_0$, $b_n$ is attacked by an accepted argument, each $b_i$ is rejected, $S_0 \setminus \{a\}$ as well as every set $S_{i+1} \setminus \{b_i\}$ contains only accepted arguments, and $S_0 \setminus \{a\}$ contains no arguments weaker than $a$.
Note that $\mathcal{C}$ may be part of a (possibly infinite) support chain $\mathcal{C}'=\{(S_0,b_0),\dots,(S_n,b_n),(S_{n+1},b_{n+1}),\dots\}$ but the finite support chain $\mathcal{C}$ is required to reject $a$.
Now for our definition of legally \In:
As in AFs, for an argument $a$ to be legally \In, all its attackers must be rejected.
However, we also need to consider the supports $(S,c)$ for which we have $a\in S$.
If $c$ is accepted, than this support can never constitute a reason to label $a$ \Out or \Undec.
However, if $c$ is not accepted, then we need to ensure that at least one argument in $S$ is also not accepted, in order to satisfy the closure postulate.
%
So if $c$ is not accepted and all elements of $S \setminus \{a\}$ are accepted, this constitutes a reason for $a$ not to be legally \In.
Similarly to the case of legally \Out, we never want to reject an argument because of accepting a weaker argument, so we need to add the additional constraint that none of the arguments in $S \setminus \{a\}$ is weaker than $a$.
These considerations motivate the following first approach for a definition:
Given some labeling $L$, an argument $a$ is legally \In \wrt $L$ iff all its attackers are \Out and if there is no support $(S,c)\in{\Rightarrow}$ with $a\in S$, $a \preceq b$ for all $b \in S \setminus \{a\}$, $c\notin in(L)$ and $S\setminus\{a\}\subseteq in(L)$.
For the case that the supported argument $c$ is labeled \Undec, this approach works well, but for the case that $c$ is labeled \Out, we have one last problematic case to consider.
Suppose we have a labeling $L$ and a support $(S,c)$ with $a\in S$, $L(c)=\Out$, $S\setminus in(L)=\{a,b\}$, $a\neq b$, $a \preceq b$, $b \preceq a$ and $L(a)=L(b)=\Undec$, \ie the supported argument $c$ is rejected, there are precisely two arguments in $S$ which are not accepted and these two arguments are of equal strength and labeled \Undec.
In line with the above approach, $a$ would be legally \In.
However, since $b$ is labeled \Undec, it still has the potential to be labeled \In.
But changing the label of $b$ to \In means $a$ is not legally \In anymore.
In essence, considering one of the arguments $a$ or $b$ as legally \In in such a case deprives the other one of their potential to be labeled \In.
Therefore, if the supported argument is labeled \Out while there are no \Out-labeled arguments in the supporting set, we require not one but at least two arguments in the supporting set to be labeled \Undec in order for any other argument in the supporting set to be considered legally \In.
We capture all these ideas in the following definition:
\begin{definition}\label{Def:JSBAF:LegalLabeling}
For any labeling $L$, the argument $a$ is:
\begin{enumerate}
	\item{
	\emph{legally \In} \wrt $L$ iff for all attackers $b$ of $a$ we have $L(b)=\Out$ and for all supports $(S,c)$ with $a\in S$ and $a \preceq b$ for all $b\in S\setminus\{a\}$, one of $a-c$ holds:
	\begin{enumerate}
		\item{
		$L(c)=\In$, or
		}
		\item{
		$L(c)=\Undec$ and there is $b\in S\setminus\{a\}$ with $L(b)\in\{\Out,\Undec\}$, or
		}
		\item{
		$L(c)=\Out$ and there is either $b\in S\setminus\{a\}$ with $L(b)=\Out$ or there are $b_1,b_2\in S\setminus\{a\}$ with $b_1\neq b_2$ and $L(b_1)=L(b_2)=\Undec$.
		}
	\end{enumerate}
	}
	\item{
	\emph{legally \Out} \wrt $L$ iff there exists $(b,a)\in {\rightarrow}$ with $L(b)=\In$ or a support chain $\mathcal{C}=\{(S_0,b_0),...,(S_n,b_n)\}\subseteq {\Rightarrow}$ \suchthat all of $a-e$ hold:
	\begin{enumerate}
		\item{
		$a\in S_0$ and $S_0 \setminus \{a\} \subseteq in(L)$
		}
		\item{
		$(c,b_n)\in {\rightarrow}$ for some $c$ with $L(c)=\In$
		}
		\item{
		for $0\leq i\leq n$, $b_i\in out(L)$
		}
		\item{
		for $0 < i\leq n$,  $S_{i} \setminus \{b_{i-1}\} \subseteq in(L)$
		}
		\item{
		$a\preceq d$ for all $d\in S\setminus\{a\}$
		}
	\end{enumerate}
	}
	\item{
	\emph{legally \Undec} \wrt $L$ iff $a$ is neither legally \In nor legally \Out \wrt $L$
	}
\end{enumerate}
\end{definition}

Admissible and preferred labelings are based on legal labelings, the idea that strict arguments should always be labeled \In and the standard notion from labeling-based argumentation semantics that preferred labelings are subset maximal:
\begin{definition}
Let \J be a JSBAF and $L$ a labeling of \J.
$L$ is an \emph{admissible} labeling iff:
\begin{itemize}
	\item{
	$a\in in(L)$ implies $a$ is legally \In \wrt $L$.
	}
	\item{
	$a\in out(L)$ iff $a$ is legally \Out \wrt $L$.
	}
	\item{
	If $a$ is a strict argument, then $L(a)=\In$.
	}
\end{itemize}
$L$ is a \emph{preferred} labeling of \J iff $L$ is a maximal (\wrt sub-set inclusion of $in(L)$) admissible labeling.
\begin{conferencePaperOnly}
We denote the set of all preferred labelings of \J by $pr(\mathcal{J})$.
\end{conferencePaperOnly}
\begin{technicalReportOnly}
We denote the set of all admissible labelings of \J by $adm(\mathcal{J})$ and the set of all preferred labelings of \J by $pr(\mathcal{J})$.
\end{technicalReportOnly}
\end{definition}
Below, we give an example JSBAF.
Here, nodes represent arguments, single arrows represent attacks between arguments and double arrows represent supports.
%
%
\begin{example}\label{Exmp:JSBAF}
	JSBAF $\mathcal{J}_1$
	
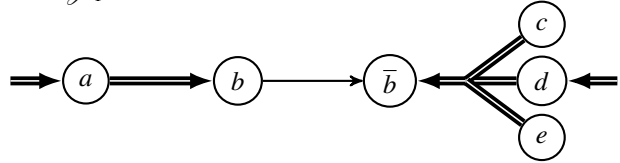
\vspace{-4mm}
	\leavevmode
	\begin{center}
		\begin{tikzpicture}
		\node[args] (A) at (-4, 0) {$a$};
		\node[args] (B) at (-2, 0) {$b$};
		\node[args] (notB) at (0, 0) {$\overline{b}$};
		\node[args] (C) at (2,0.75) {$c$};
		\node[args] (D) at (2,0) {$d$};
		\node[args] (E) at (2,-0.75) {$e$};
		\node[suppNode] (suppA) at (-5,0) {};
		\node[suppNode] (suppNotB) at (1, 0) {};
		\node[suppNode] (suppD) at (3, 0) {};
		
		\path [->, thick]
		(B) edge (notB);
		
		\draw[multiLine, -\tip] (suppA) to (A);
		\draw[multiLine, -\tip] (A) to (B);
		\draw[multiLine] (D) to (suppNotB);
		\draw[multiLine] (C) to (suppNotB);
		\draw[multiLine] (E) to (suppNotB);
		\draw[multiLine, -\tip] (suppNotB) to (notB);				
		\draw[multiLine, -\tip] (suppD) to (D);						
		\end{tikzpicture}
	\end{center}
\end{example}
$\mathcal{J}_1$ has one attack, from $b$ to $\overline{b}$, the supports, $(\emptyset,a)$, $(\{a\},b)$, $(\{c,d,e\},\overline{b})$ and $(\emptyset,d)$, and the strict arguments $a$, $b$ and $d$.
Assume the following preference ordering between the arguments of $\mathcal{J}_1$, by which strict arguments are preferred to non-strict arguments, but there are no other preferences: ${\preceq} = \{a,b,\overline{b},c,d,e\}^2 \setminus \big(\{a,b,d\} \times \{\overline{b},c,e\}\big)$.
Let us first consider the admissible labelings of $\mathcal{J}_1$:
Since $a$,$b$ and $d$ are strict arguments, they are labeled \In in every admissible labeling.
The argument $a$ is unattacked and only supports the accepted argument $b$, therefore $a$ is legally \In.
The argument $b$ is unattacked and doesn't support any argument, therefore $b$ is also legally \In.
The argument $d$ is unattacked and for the support $(\{c,d,e\},\overline{b})$, we have $d\not\preceq c$ and $d\not\preceq e$.
Thus we don't have to consider this support when checking the legality of the label of $d$, meaning $d$ is also legally \In.
Lastly, because $\overline{b}$ is attacked by $b$, it is legally \Out in every admissible labeling and thus needs to be labeled \Out.
Labeling $c$ and $d$ \Undec results in a first admissible labeling $L_1$ with $in(L_1)=\{a,b,d\}$, $out(L_1)=\{\overline{b}\}$ and $undec(L_1)=\{c,e\}$.
Now for the remaining admissible labelings:
In order to label $c$ or $e$ \Out, they need to be legally \Out.
Both $c$ and $e$ are unattacked -- therefore they cannot be legally \Out due to an attack -- but they are contained in the support chain $\big{\{}(\{c,d,e\},\overline{b})\big{\}}$.
Note that $c\preceq d$, $c\preceq e$, $e\preceq c$ and $e\preceq d$.
Thus, for this support chain, $c$ is legally \Out if $e$ is labeled \In and $e$ is legally \Out if $c$ is labeled \In.
On the other hand, because $d$ is always labeled \In in an admissible labeling, if $e$ is labeled \Out, then $c$ is legally \In and if $c$ is labeled \Out, then $e$ is legally \In.
This results in the remaining two admissible labelings $L_2$ and $L_3$ with $in(L_2)=\{a,b,c,d\}$, $out(L_2)=\{\overline{b},e\}$ and $undec(L_2)=\emptyset$, while $in(L_3)=\{a,b,d,e\}$, $out(L_3)=\{\overline{b},c\}$ and $undec(L_3)=\emptyset$.
$L_1$ is not preferred, because $in(L_1) \subsetneq in(L_2)$. Since neither of $in(L_2)$ and $in(L_3)$ is a subset of the other, both of them are preferred labelings, i.e.\ $pr(\mathcal{J}_1)=\{L_2,L_3\}$.
\begin{conferencePaperOnly}
We can show that for every JSBAF, there always exists at least one admissible labeling and at least one preferred labeling (even in JSBAFs with infinitely many arguments).
\begin{theorem}\label{Theo:PreferredExistence:Conference}
Let \J be a JSBAF.
Then $pr(\mathcal{J})\neq\emptyset$.
\end{theorem}
\end{conferencePaperOnly}
\begin{technicalReportOnly}
\subsection{Existence of admissible and preferred labelings}

In this section, we show that there always exists at least one admissilbe and at least one preferred labeling for every JSBAF, even in cases with an infinite amount of arguments.

\subsubsection{Existence of admissible labelings}

We first show that for every JSBAF, there exists at least one admissible labeling.
To this end, we define a family of labelings which labels precisely the strict arguments of a JSBAF as \In, labels all arguments \Out that are legally \Out as a result from accepting the strict arguments and labels all the remaining arguments \Undec.
We call this the \emph{strict including minimal labeling} (SIM):
\begin{definition}\label{Def:JSBAF:SIM}
For any JSBAF \JSBAF, the labeling $SIM_{\mathcal{J}}$ is defined as:
\begin{enumerate}
	\item {$in(SIM_{\mathcal{J}})=STR_{\mathcal{J}}$}
	\item {
	\(\begin{aligned}[t]
	O_0
	&=\big{\{ }a\in \mathcal{A}\mid(b,a)\in {\rightarrow},b\in in(SIM_{\mathcal{J}})\big{\} }\\
	O_{n+1}
	&=	O_n \cup \big{\{ }a\in \mathcal{A}\mid (S,b)\in {\Rightarrow}, a\in S, b\in O_n,\\
	&\hspace{2.55cm}S\setminus\{a\}\subseteq in(SIM_{\mathcal{J}})\big{\} }\end{aligned}
	\\
	out(SIM_{\mathcal{J}})=\bigcup\limits_{i\geq 0} O_i
		$}
	\item {$undec(SIM_{\mathcal{J}})=\mathcal{A}\setminus\big(in(SIM_{\mathcal{J}})\cup out(SIM_{\mathcal{J}})\big)$}
\end{enumerate}
\end{definition}
We now prove that $SIM_{\mathcal{J}}$ is an admissible labeling.
To make the proof more accessible, we have divided it into several parts.
We show (in this order) that $SIM_{\mathcal{J}}$ is a labeling, that the accepted arguments are legally \In, that the rejected arguments are exactly those which are legally \Out and, finally, that $SIM_{\mathcal{J}}$ is an admissible labeling.
\begin{proposition}\label{Prop:JSBAF:SIMLabeling}
Let \JSBAF\ be a JSBAF.
Then $SIM_{\mathcal{J}}$ is a \emph{labeling}.
\end{proposition}
\begin{proof}
From the definition of $undec(SIM_{\mathcal{J}})$, it is clear that every argument gets assigned some label and that $undec(SIM_{\mathcal{J}})\cap in(SIM_{\mathcal{J}})=\emptyset$ as well as $undec(SIM_{\mathcal{J}})\cap out(SIM_{\mathcal{J}})=\emptyset$.
Thus, we have left to show that $in(SIM_{\mathcal{J}})\cap out(SIM_{\mathcal{J}})=\emptyset$ also holds.
For this, we will prove by induction over $n\in\mathbb{N}$ that $in(SIM_{\mathcal{J}})\cap O_{n}=\emptyset$.
Induction start $n=0$:
By construction of $SIM_{\mathcal{J}}$, if $a\in in(SIM_{\mathcal{J}})$, then $a$ is strict and therefore unattacked in $\mathcal{J}$.
If $a\in O_0$, then there exists $(b,a)\in{\rightarrow}$, which contradicts that $a$ is unattacked.
We conclude that $in(SIM_{\mathcal{J}})\cap O_0=\emptyset$ holds.
Inductin step $n\rightarrow n+1$:
We concentrate on $O'=O_{n+1}\setminus O_{n}$.
Suppose there is $a\in in(SIM_{\mathcal{J}})\cap O'$.
Because $a\in in(SIM_{\mathcal{J}})$, we know that $a$ is strict.
By $a\in O'$, we know there is $(S,b)\in Supp$ with $a\in S$, $b\in O_{n}$ and $S\setminus\{a\}\subseteq in(SIM_{\mathcal{J}})=STR_{\mathcal{J}}$.
This means that all arguments in $S$ are strict.
Now $b$ must also be strict, therefore $b\in in(SIM_{\mathcal{J}})$ by construction of $SIM_{\mathcal{J}}$.
But then $b\in in(SIM_{\mathcal{J}})\cap O_n$, contradicting the induction hypothesis $in(SIM_{\mathcal{J}})\cap O_n=\emptyset$.
We conclude that $in(SIM_{\mathcal{J}})\cap O_{n+1}=\emptyset$ holds, therefore $in(SIM_{\mathcal{J}})\cap out(SIM_{\mathcal{J}})=\emptyset$ as required.
\end{proof}
\begin{proposition}\label{Prop:JSBAF:SIMIn}
Let \JSBAF be a JSBAF and $a\in \mathcal{A}$ an argument.
If $SIM_{\mathcal{J}}(a)=\In$, then $a$ is legally \In \wrt $SIM_{\mathcal{J}}$.
\end{proposition}
\begin{proof}
By construction of $in(SIM_{\mathcal{J}})$, $a$ is strict and therefore unattacked.
Now consider some support $(S,c)$ with $a\in S$.
By the restrictions for JSBAFs mentioned in Section~\ref{Sec:JSBAF:Syntax}, if $a\preceq b$ for all $b\in S\setminus\{a\}$, then all arguments in $S$ must be strict.
From this we can infer that $c$ is also a strict argument and therefore $c\in in(SIM_{\mathcal{J}})$ by construction of $SIM_{\mathcal{J}}$.
Thus $a$ is legally \In \wrt $SIM_{\mathcal{J}}$ as required.
\end{proof}
\begin{proposition}\label{Prop:JSBAF:SIMOut}
Let \JSBAF be a JSBAF and $a\in \mathcal{A}$ an argument.
We have $SIM_{\mathcal{J}}(a)=\Out$ iff $a$ is legally \Out \wrt $SIM_{\mathcal{J}}$.
\end{proposition}
\begin{proof}
$\rightarrow$:
$SIM_{\mathcal{J}}(a)=\Out$ implies there is $O_n$ \suchthat $a\in O_n$.
We show by induction over $n\in\mathbb{N}$ that, if $a\in O_{n}$, then $a$ is legally \Out \wrt $SIM_{\mathcal{J}}$.
Induction start $n=0$:
By construction of $SIM_{\mathcal{J}}$, $a\in O_0$ implies there is $b\in in(SIM_{\mathcal{J}})$ \suchthat $(b,a)\in {\rightarrow}$.
This means $a$ is legally \Out, as required.
Induction step $n\rightarrow n+1$:
We focus on $O'=O_{n+1}\setminus O_{n}$.
By construction of $SIM_{\mathcal{J}}$ we know there is $(S,b)\in {\Rightarrow}$ with $a\in S$, $b\in O_{n}$ and $S\setminus\{a\}\subseteq in(SIM_{\mathcal{J}})=STR_{\mathcal{J}}$.
Note that, by the restrictions on JSBAFs mentioned in Section~\ref{Sec:JSBAF:Syntax}, $S\setminus\{a\}\subseteq STR_{\mathcal{J}}$ implies $a\prec d$ for all $d\in S\setminus\{a\}$.
By the induction hypothesis, $b$ is legally \Out \wrt $SIM_{\mathcal{J}}$.
Suppose first that $b$ is legally \Out because of an attack $(c,b)\in {\rightarrow}$.
Then $\{(S,b)\}$ is a support chain \suchthat $a\in S$, $(c,b)\in {\rightarrow}$ with $c\in in(SIM_{\mathcal{J}})$, $b\in out(SIM_{\mathcal{J}})$, $S\setminus\{a\}\subseteq in(SIM_{\mathcal{J}})$ and $a\preceq d$ for all $d\in S\setminus\{a\}$.
Thus $a$ is legally \Out \wrt $SIM_{\mathcal{J}}$.
Next, suppose that $b$ is legally \Out because of a support chain $\{(S_0,b_0),\dots,(S_n,b_n)\}$ which satisfies the conditions of item two of Definition~\ref{Def:JSBAF:LegalLabeling}.
%
%
Then we construct the support chain $\{(S,b),(S_0,b_0),\dots,(S_n,b_n)\}$ which satisfies those same conditions \wrt $a\in S$.
By Definition~\ref{Def:JSBAF:LegalLabeling}, $a$ is now legally \Out \wrt $SIM_{\mathcal{J}}$, as required.
We conclude that $a\in out(SIM_{\mathcal{J}})$ implies $a$ is legally \Out \wrt $SIM_{\mathcal{J}}$.
$\leftarrow$:
Suppose first that $a$ is legally \Out \wrt $SIM_{\mathcal{J}}$ because of an attack $(b,a)\in {\rightarrow}$ with $b\in in(SIM_{\mathcal{J}})$.
By construction of $SIM_{\mathcal{J}}$, we infer that $b$ must be a strict argument, thus $a\in O_0\subseteq out(SIM_{\mathcal{J}})$.
%
%
%
%
%
Next, suppose that $a$ is legally \Out \wrt $SIM_{\mathcal{J}}$ due to a support chain $\{(S_0,b_0),\dots,(S_n,b_n)\}$.
Then in particular $a\in S_0$, $b_0\in out(SIM_{\mathcal{J}})$ and $S_0\setminus\{a\}\subseteq in(SIM_{\mathcal{J}})$.
By construction of $SIM_{\mathcal{J}}$, there needs to be some $m\in\mathbb{N}$ \suchthat $b_0\in O_m$.
Now $a\in O_{m+1}\subseteq out(SIM_{\mathcal{J}})$ as required.
\end{proof}
\begin{lemma}\label{Lem:JSBAF:SIMIsAdm}
Let \JSBAF be a JSBAF.
Then $SIM_{\mathcal{J}}$ is an admissible labeling.
\end{lemma}
\begin{proof}
By Proposition~\ref{Prop:JSBAF:SIMLabeling} we know that $SIM_{\mathcal{J}}$ is a labeling.
By Proposition~\ref{Prop:JSBAF:SIMIn} we know that all $a\in in(SIM_{\mathcal{J}})$ are legally \In \wrt $SIM_{\mathcal{J}}$.
By Proposition~\ref{Prop:JSBAF:SIMOut} we know that $a\in out(SIM_{\mathcal{J}})$ iff $a$ is legally \Out \wrt $SIM_{\mathcal{J}}$.
Finally, $STR_{\mathcal{J}}\subseteq in(SIM_{\mathcal{J}})$ is trivially satisfied by construction of $SIM_{\mathcal{J}}$.
\end{proof}
\begin{corollary}\label{Cor:JSBAF:ADMNonEmpty}
Let \JSBAF be a JSBAF.
Then we have $adm(\mathcal{J})\neq\emptyset$.
\end{corollary}

\subsubsection{Existence of preferred labelings}
To prove that every JSBAF has at least one preferred labeling, we have left to show that there always exists a \emph{maximal} (\wrt sub-set inclusion) admissible labeling.
For JSBAFs with an infinite amount of arguments, it is not immediately clear that this is always the case.
Our proof will rely on the Lemma of Zorn.\footnote{
Zorn's Lemma states that, if $X$ is a partially ordered set \suchthat every chain in $X$ has an upper bound, then $X$ contains a maximal element.
}
We begin by defining the following partial order on admissible labelings:
\begin{definition}\label{Def:JSBAF:AdmissiblePartialOrder}
Let \JSBAF be a JSBAF.
We define the relation ${\leq_{\mathfrak{L}}}\subseteq adm(\mathcal{J})\times adm(\mathcal{J})$ as follows:
$L_1\leq_{\mathfrak{L}} L_2$ iff $in(L_1)\subseteq in(L_2)$ and $out(L_1)\subseteq out(L_2)$.
\end{definition}
It is easy to verify that ${\leq_{\mathfrak{L}}}$ is a partial order and we omit a proof.
Based on this relation, we now define the \emph{supremum} \wrt a chain of admissible labelings:
\begin{definition}\label{Def:JSBAF:KMax}
Let \JSBAF be a JSBAF.
Furthermore, let $\mathcal{K}\subseteq adm(\mathcal{J})$ be a chain of admissible labelings \wrt ${\leq_{\mathfrak{L}}}$.
We define the \emph{supremum of $\mathcal{K}$}, denoted by $L_{\mathcal{K}}$, as follows:
\begin{itemize}
	\item {$in(L_{\mathcal{K}})=\bigcup\limits_{L\in \mathcal{K}} in(L)$}
	\item {$out(L_{\mathcal{K}})=\bigcup\limits_{L\in \mathcal{K}} out(L)$}
	\item {$undec(L_{\mathcal{K}})=\mathcal{A}\setminus\big(in(L_{\mathcal{K}})\cup out(L_{\mathcal{K}})\big)$}
\end{itemize}
\end{definition}

We now show that this labeling $L_{\mathcal{K}}$ is an admissible labeling.
To make it more accessible, we have split the proof in several parts and show (in this order) that $L_{\mathcal{K}}$ is a labeling, that the accepted arguments are legally \In, that the rejected arguments are exactly those which are legally \Out and, finally, that $L_{\mathcal{K}}$ is an admissible labeling.
\begin{proposition}\label{Prop:JSBAF:KMaxLabeling}
Let \JSBAF be a JSBAF.
Furthermore, let $\mathcal{K}\subseteq adm(\mathcal{J})$ be a chain of admissible labelings \wrt $\leq_{\mathfrak{L}}$ and let $L_{\mathcal{K}}$ be the supremum of $\mathcal{K}$.
Then $L_{\mathcal{K}}$ is a labeling.
\end{proposition}
\begin{proof}
It is clear from the definition of $L_{\mathcal{K}}$ that every argument gets assigned some label and that $undec(L_{\mathcal{K}})\cap in(L_{\mathcal{K}})=\emptyset$ as well as $undec(L_{\mathcal{K}})\cap out(L_{\mathcal{K}})=\emptyset$.
Thus, we have left to show that $in(L_{\mathcal{K}})\cap out(L_{\mathcal{K}})=\emptyset$ also holds.
Towards a contradiction, let $a\in in(L_{\mathcal{K}})\cap out(L_{\mathcal{K}})$.
By construction of $L_{\mathcal{K}}$, this means there are $L_1,L_2\in {\mathcal{K}}$ with $a\in in(L_1)$ and $a\in out(L_2)$.
Because $\mathcal{K}$ was a chain, we know that either $L_1\leq_{\mathfrak{L}} L_2$ or $L_2\leq_{\mathfrak{L}} L_1$ holds.
First, lets assume $L_1\leq_{\mathfrak{L}} L_2$:
Then $a\in in(L_1)\subseteq in(L_2)$.
Now $a$ is labeled \In and \Out by $L_2$, contradicting that $L_2$ is a labeling.
Next, assume $L_2\leq_{\mathfrak{L}} L_1$:
Then $a\in out(L_2)\subseteq out(L_1)$.
Now $a$ is labeled \Out and \In by $L_1$, contradicting that $L_1$ is a labeling.
We conclude that $in(L_{\mathcal{K}})\cap out(L_{\mathcal{K}})=\emptyset$ as required.
\end{proof}
\begin{proposition}\label{Prop:JSBAF:KMaxIn}
Let \JSBAF be a JSBAF.
Furthermore, let $\mathcal{K}\subseteq adm(\mathcal{J})$ be a chain of admissible labelings \wrt $\leq_{\mathfrak{L}}$ and let $L_{\mathcal{K}}$ be the supremum of $\mathcal{K}$.
If $a\in in(L_{\mathcal{K}})$, then $a$ is legally \In \wrt $L_{\mathcal{K}}$.
\end{proposition}
\begin{proof}
Let $a\in in(L_{\mathcal{K}})$.
We first show that any attacker of $a$ is labeled \Out:
Let $L\in\mathcal{K}$ \suchthat $L(a)=\In$.
By construction of $L_{\mathcal{K}}$, such a labeling $L$ must exist.
Now assume that there is $b\in \mathcal{A}$ \suchthat $(b,a)\in {\rightarrow}$.
Because $L$ is an admissible labeling, we have $L(b)=\Out$.
By construction of $L_{\mathcal{K}}$, we can now infer $L_{\mathcal{K}}(b)=\Out$ as required.
Now take some support $(S,b)\in {\Rightarrow}$ with $a\in S$ and $a\preceq c$ for all $c\in S\setminus\{a\}$.
We have to show that the conditions of item one of Definition~\ref{Def:JSBAF:LegalLabeling} are satisfied.
If $L_{\mathcal{K}}(b)=\In$, then this trivially holds.
Thus we assume $L_{\mathcal{K}}(b)\neq\In$.
We first show that this implies $S\not\subseteq in(L_{\mathcal{K}})$:
Towards a contradiction suppose that this is not the case.
Let $S=\{a_0,\dots,a_m\}$ and let $L_{\In}=\{L_0,..,L_m\}\subseteq \mathcal{K}$ \suchthat $a_i\in in(L_i)$ for all $0\leq i\leq m$.
We know that $S$ is finite (by the restrictions mentioned in Section~\ref{Sec:JSBAF:Syntax}).
Because $\mathcal{K}$ is a chain \wrt $\leq_{\mathfrak{L}}$, we know that there is $L_{max}\in L_{\In}$ \suchthat $L_i\leq_{\mathfrak{L}} L_{max}$ for all $0\leq i\leq m$.
Now we have $in(L_i)\subseteq in(L_{max})$ for all $0\leq i\leq m$.
This means $S\subseteq in(L_{max})$.
We know that $L_{max}$ is an admissible labeling.
Thus, for $a\in S$ in particular, we have that $a$ is legally \In \wrt $L_{max}$.
Now we can infer that $L_{max}(b)=\In$ must hold.
By construction of $L_{\mathcal{K}}$, this implies $L_{\mathcal{K}}(b)=\In$, contradicting our assumption $L_{\mathcal{K}}(b)\neq\In$.
Now for the remaining options of $L_{\mathcal{K}}(b)$:
We first consider the case $L_{\mathcal{K}}(b)=\Undec$.
We have argued that $S\subseteq in(L_{\mathcal{K}})$ cannot hold.
This means there is at least one $c\in S$ with $L_{\mathcal{K}}(c)=\Out$ or $L_{\mathcal{K}}(c)=\Undec$, as required.
Lastly, lets consider the case $L_{\mathcal{K}}(b)=\Out$:
As before, we know $S\not\subseteq in(L_{\mathcal{K}})$.
Let $S_{\In}=\{a_0,\dots,a_m\}=S\cap in(L_{\mathcal{K}})$ and let $L_{\In}=\{L_0,\dots,L_m\}\subseteq \mathcal{K}$ \suchthat $a_i\in L_i$ for all $0\leq i\leq m$.
Furthermore, let $L_{\Out}\in\mathcal{K}$ \suchthat $L_{\Out}(b)=\Out$.
We take $L_{max}\in L_{\In}\cup \{L_{\Out}\}$ \suchthat $L_i\leq_{\mathfrak{L}} L_{max}$ and $L_{\Out}\leq_{\mathfrak{L}} L_{max}$.
Then $S_{\In}\subseteq in(L_{max})$ and $b\in out(L_{\max})$.
Now consider the set $S'=S\setminus S_{\In}$:
We either have $|S'|=1$ or $|S'|>1$.
Assume first that $S'=\{c\}$.
Then we can infer from the admissibility of $L_{max}$, that $c\in out(L_{max})\subseteq out(L_{\mathcal{K}})$ must hold.
Now $a$ is legally \In \wrt $L_{\mathcal{K}}$, as required.
Next, assume that $|S'|>1$.
Then there either exists an argument $c\in S'$ \suchthat $c\in out(L_{max})\subseteq out(L_{\mathcal{K}})$, or there are arguments $c_1,c_2\in S'$ \suchthat $c_1\neq c_2$ and $c_1,c_2\in undec(L_{max})$.
In the first case, we can immediately infer that $a$ is legally \In \wrt $L_{\mathcal{K}}$.
In the second case, we either have $c_1,c_2\in undec(L_{\mathcal{K}})$, or $\{c_1,c_2\}\cap out(L_{\mathcal{K}})\neq\emptyset$.
Either way, we can again infer that $a$ is legally \In \wrt $L_{\mathcal{K}}$, as required.
\end{proof}

\begin{proposition}\label{Prop:JSBAF:KMaxOut}
Let \JSBAF be a JSBAF.
Furthermore, let $\mathcal{K}\subseteq adm(\mathcal{J})$ be a chain of admissible labelings \wrt $\leq_{\mathfrak{L}}$ and let $L_{\mathcal{K}}$ be the supremum of $\mathcal{K}$.
Then for any $a\in \mathcal{A}$, $a\in out(L_{\mathcal{K}})$ iff $a$ is legally \Out \wrt $L_{\mathcal{K}}$.
\end{proposition}
\begin{proof}
$\rightarrow$:
Take $a\in out(L_{\mathcal{K}})$.
By construction of $L_{\mathcal{K}}$, there exists $L\in \mathcal{K}$ \suchthat $a\in out(L)$.
As $L$ was an admissible labeling, we know that $a$ is legally \Out \wrt $L$.
Regardless of whether $a$ is legally \Out \wrt $L$ due to an attack or due to a support chain, we can use $in(L)\subseteq in(L_{\mathcal{K}})$ and $out(L)\subseteq out(L_{\mathcal{K}})$ to infer that $a$ is legally \Out \wrt $L_{\mathcal{K}}$ due to the same attacker or support chain.
$\leftarrow$:
Suppose first that $a$ is legally \Out \wrt $L_{\mathcal{K}}$ because of an attack, \ie there is $(b,a)\in {\rightarrow}$ with $b\in in(L_{\mathcal{K}})$.
Let $L\in\mathcal{K}$ \suchthat $b\in in(L)$.
As $L$ is an admissible labeling, we can infer $a\in out(L)$ and by construction of $L_{\mathcal{K}}$ we have $a\in out(L_{\mathcal{K}})$ as required.
Next, suppose that $a$ is legally \Out \wrt $L_{\mathcal{K}}$ because of a support chain $\{(S_0,b_0),\dots,(S_n,b_n)\}$ with $a\in S_0$.
Then in particular $b_0\in out(L_{\mathcal{K}})$ and for $S_0\setminus\{a\}=\{a_0,\dots,a_m\}$, we have $\{a_0,\dots,a_m\}\subseteq in(L_{\mathcal{K}})$.
By construction of $L_{K}$, we again have a set of labelings $L_{\In}=\{L_0,\dots,L_m\}\subseteq \mathcal{K}$ 
\suchthat $a_i\in in(L_i)$ and we have a labeling $L_{\Out}\in\mathcal{K}$ \suchthat $b\in out(L_{\Out})$.
We again take the maximal labeling $L_{max}\subseteq L_{\In}\cup\{L_{\Out}\}$ for which we have $L_i\leq_{\mathfrak{L}} L_{\max}$ for all $L_i\in L_{\In}$, as well as $L_{\Out}\leq_{\mathfrak{L}} L_{max}$.
Because $L_{max}$ is admissible, we can again infer that $a\in out(L_{max})$, otherwise none of the $a_i$ would be legally \In \wrt $L_{max}$.
By construction of $L_{\mathcal{K}}$ we now have $a\in out(L_{\mathcal{K}})$ as required.
\end{proof}
\begin{proposition}\label{Prop:JSBAF:KMaxAdmissble}
Let \JSBAF be a JSBAF.
Furthermore, let $\mathcal{K}\subseteq adm(\mathcal{J})$ be a chain of admissible labelings \wrt $\leq_{\mathfrak{L}}$ and let $L_{\mathcal{K}}$ be the supremum of $\mathcal{K}$.
Then $L_{\mathcal{K}}$ is an admissible labeling.
\end{proposition}
\begin{proof}
From Proposition \ref{Prop:JSBAF:KMaxLabeling} we know that $L_{\mathcal{K}}$ is a labeling of $\mathcal{J}$.
By Proposition \ref{Prop:JSBAF:KMaxIn} we know that every $a\in in(L_{\mathcal{K}})$ is legally \In \wrt $L_{\mathcal{K}}$.
By Proposition \ref{Prop:JSBAF:KMaxOut} we know that $a\in out(L_{\mathcal{K}})$ iff $a$ is legally \Out \wrt $L_{\mathcal{K}}$.
Finally, because all labelings $L\in \mathcal{K}$ are admissible labelings, we can infer that $STR_{\mathcal{J}}\subseteq in(L)$ holds.
By construction of $L_{\mathcal{K}}$, we now have $STR_{\mathcal{J}}\subseteq in(L_{\mathcal{K}})$.
We conclude that $Lab_{\mathcal{K}}$ is an admissible labeling.
\end{proof}

With the construction of $Lab_{\mathcal{K}}$, we can now argue that a maximal admissible labeling exists even for JSBAFs with an infinite number of arguments.
This means that for each JSBAF, there is always at least one preferred labeling:

\begin{theorem}\label{Theo:JSBAF:Theo:PreferredExistence}
Let \J be a JSBAF.
Then $pr(\mathcal{J})\neq\emptyset$.
\end{theorem}
\begin{proof}
By Corollary \ref{Cor:JSBAF:ADMNonEmpty} we know that $adm(\mathcal{J})$ is non-empty.
The elements of $adm(\mathcal{J})$ form a partially ordered set \wrt $\leq_{\mathfrak{L}}$.
Now assume that $\mathcal{K}\subseteq adm(\mathcal{J})$ is a chain of admissible labelings.
We take the supremum of $\mathcal{K}$ as constructed  in Definition \ref{Def:JSBAF:KMax}, \ie $L_{\mathcal{K}}$.
By Proposition \ref{Prop:JSBAF:KMaxAdmissble}, we know that $L_{\mathcal{K}}$ is itself an admissible labeling.
From the construction of $L_{\mathcal{K}}$ it is clear that $L_{\mathcal{K}}$ is an upper bound of $\mathcal{K}$ \wrt $\leq_{\mathfrak{L}}$.
By Zorn's Lemma we can now infer that $adm(\mathcal{J})$ contains a maximal element $L_{max}$.
Since $L_{max}$ is maximal \wrt $\leq_{\mathfrak{L}}$, there is no admissible labeling $L'\in adm(\mathcal{J})$ \suchthat $in(L_{max})\subset in(L')$.
Therefore $L_{max}$ is a preferred labeling.
\end{proof}

\end{technicalReportOnly}

\section{Deductive ASPIC\boldmath$^{\ominus}$}\label{Sec:DAOM}

\subsection{DeductiveASPIC\boldmath$^{\ominus}$ and JSBAFs}\label{Sec:ASPIC2JSBAF}
We are assuming some underlying logical language for which we have a set of well-founded formulas $\mathcal{F}$ and a function $Atoms:\mathcal{F}\mapsto 2^{\mathcal{F}}$ mapping formulas of $\mathcal{F}$ to the atomic formulas occurring in them.\footnote{
One can think of Propositional Logic or First Order Logic to get an idea of the logical languages we want to cover.
}
We say that formulas $\phi,\psi\in\mathcal{F}$ are syntactically disjoint, written $\phi||\psi$, iff $Atoms(\phi)\cap Atoms(\psi)=\emptyset$.
Similarly, we say that sets of formulas $\Gamma,\Delta\subseteq\mathcal{F}$ are syntactically disjoint, written $\Gamma||\Delta$, iff $\phi||\psi$ for all $\phi\in\Gamma$ and $\psi\in\Delta$.
We do not make any assumptions regarding the specific syntax of the logical language we are considering, but we assume that it is closed under negation $\lnot$ and conjunction $\wedge$.
As a shorthand notation, we use $\bigwedge\{\phi_0,\dots,\phi_n\}$ to abbreviate $\phi_0\wedge\dots\wedge\phi_n$ and we write $\phi=-\psi$ to indicate that either $\phi=\lnot\psi$ or $\lnot\phi=\psi$.
To define the strict rules of our version of ASPIC, we make a few assumptions on the semantics of the underlying logical language:
First, we assume a set of interpretations $\mathcal{I}$ and some model-relation ${\vDash_M}$ for which $\lnot$ and $\wedge$ behave in the usual way, \ie for any interpretation $I$, we have $I\isModelOf\phi$ iff not $I\isModelOf\lnot\phi$, and we have $I\isModelOf\phi\wedge\psi$ iff $I\isModelOf\phi$ and $I\isModelOf\psi$.
Secondly, we assume that for any two sets of formulas $\Gamma$ and $\Delta$ with $\Gamma||\Delta$, if there exist interpretations $I_1,I_2$ \suchthat $I_1\isModelOf\phi$ for all $\phi\in\Gamma$ and $I_2\isModelOf\psi$ for all $\psi\in\Delta$, then there exists some interpretation $I$ for which we have $I\isModelOf\phi$ for all $\phi\in\Gamma$ and $I\isModelOf\psi$ for all $\psi\in\Delta$.
We call a set of formulas $\Gamma$ satisfiable, iff there exists an interpretation $I$ \suchthat $I\isModelOf\phi$ for all $\phi\in\Gamma$.
\begin{technicalReportOnly}
As an abbreviation, we write $I\isModelOf\Gamma$ to indicate $I\isModelOf\phi$ for all $\phi\in\Gamma$.
We call a set of formulas $\Gamma$ tautological, iff for all $I\in\mathcal{I}$ we have $I\isModelOf\Gamma$.
Lastly, we assume that there are no unsatisfiable or tautological atoms, \ie for all $\phi\in\mathcal{F}$, for all $\psi\in Atoms(\phi)$, there exist $I_1$, $I_2$ \suchthat $I_1\isModelOf\psi$ and $I_2\isModelOf \lnot\psi$.
\end{technicalReportOnly}
Using this model-relation, we now define a consequence relation ${\entails}\subseteq 2^{\mathcal{F}}\times\mathcal{F}$ as follows:
$\Gamma\entails\psi$ iff for all $I\in\mathcal{I}$, $I\isModelOf\Gamma$ implies $I\isModelOf\psi$.
\begin{technicalReportOnly}
\footnote{
It is easy to see that this consequence relation is monotone, \ie $\Gamma\entails\phi$ implies $\Gamma\cup\{\psi\}\entails\phi$.
}
\end{technicalReportOnly}
\emph{$\textnormal{Deductive ASPIC}^{\ominus}$} now uses such an underlying logical language as a basis.
\begin{definition}
Let $\mathcal{F}$ be a set of well-founded formulas and let ${\entails}$ be a consequence relation associated with $\mathcal{F}$.
An \emph{Argumentation System} (AS) according to \DAOMItallic is a tuple \AS, where:
\begin{itemize}
	\item{
	$R_s=R_s^{AX}\cup R_s^{\vDash}$ is a set of \emph{strict rules} consisting of axiomatic rules $R_s^{AX}=\{\rightarrow\phi\mid\phi\in AX\}$ for some satisfiable set of \emph{axioms} $AX\subseteq\mathcal{F}$ and consequence-based rules $R_s^{\vDash}=\big{\{}\phi_0,\dots,\phi_{m-1}\rightarrow\phi_m\mid \phi_0, \dots, \phi_m \in \mathcal{F}$ and $\{\phi_0,\dots,\phi_{m-1}\}\entails\phi_m\big{\}}$.
	}
	\item{
	$R_d$ is a set of \emph{defeasible rules} of the form $\phi_0,\dots,\phi_{m-1}\Rightarrow\phi_m$ where $\phi_i\in\mathcal{F}$ for all $0\leq i\leq m$.
	}
	\item{
	$n:R_d\rightharpoonup\mathcal{F}$ is a partial function called \emph{naming function}.
	}
	\item{
	${\leq_r}\subseteq R_d\times R_d$ is a total preorder called \emph{preferences order}.
	}
\end{itemize}
\end{definition}
\begin{technicalReportOnly}
In cases where the distinction between strict and defeasible rules is not relevant, we will use $\rightsquigarrow$ and mean by that some rule $r\in R_s\cup R_d$.
\end{technicalReportOnly}
Arguments based on $AS$ are defined recursively:
$a$ is an argument with conclusion $a^C=\phi$ iff $a$ is of the form $a:a_0,\dots,a_m\rightarrow\phi$, where $a_0,\dots,a_m$ are arguments based on $AS$ and $a_0^C,\dots,a_m^C\rightarrow\phi$ is a strict rule, or if $a$ is of the form $a:a_0,\dots,a_m\Rightarrow\phi$, where $a_0,\dots,a_m$ are arguments based on $AS$ and $a_0^C,\dots,a_m^C\Rightarrow\phi$ is a defeasible rule.
We say $a_0^C,\dots,a_m^C\rightarrow\phi$ (respectively $a_0^C,\dots,a_m^C\Rightarrow\phi$) is the top-rule of $a$, denoted $TR(a)$.
The defeasible rules of $a$ are $DR(a)=DR(a_0)\cup \dots\cup DR(a_m)\cup \big(\{TR(a)\}\cap R_d\big)$ and the sub-arguments of $a$ are $Sub(a)=Sub(a_0)\cup\dots\cup Sub(a_m)\cup\{a\}$.
%
%
We denote the arguments that can be constructed based on some $AS$ by $\mathcal{A}(AS)$.
We call $a$ defeasible if $DR(a)\neq\emptyset$ and strict otherwise.
We call $AS$ inconsistent if there are strict arguments $a,b\in\mathcal{A}(AS)$ \suchthat $a^C=-b^C$ and consistent otherwise.
In this paper we only consider consistent $AS$.
For a set of arguments $A$, we define its conclusions as $A^C=\{a^C\mid a\in A\}$.
\begin{technicalReportOnly}
We call an argument $a$ inconsistent if there is $\Gamma\subseteq Sub(a)^C$ \suchthat $\Gamma$ is unsatisfiable, otherwise we call $a$ consistent.
\end{technicalReportOnly}

We say that $a$ undercuts $b$ iff there is $b'\in Sub(b)$ with $TR(b')=r\in R_d$ and $a^C=-n(r)$.
We use the notion of \emph{gen-rebuts} taken from Heyninck~and~Straßer~\cite{Heyninck2017:RevisitingUnrestrictedRebutPreferences} and say that $a$ gen-rebuts $b$ iff $b$ is defeasible and there is $\Gamma\subseteq Sub(b)^C$ \suchthat $a^C=\lnot\bigwedge\Gamma$.
We consider the elitist weakest link principle to lift preferences between defeasible rules to preferences over arguments, denoted by the relation ${\preceq_{ewl}}\subseteq \mathcal{A}(AS)\times \mathcal{A}(AS)$.\footnote{
We write $a\prec_{ewl} b$ to abbreviate $a\preceq_{ewl} b$ and $b\not\preceq_{ewl} a$.}
If arguments $a$ and $b$ are both strict, then we define $a\preceq_{ewl} b$ and $b \preceq_{ewl} a$, otherwise $a\preceq_{ewl} b$ iff $\exists r_a\in DR(a)$ \suchthat $\forall r_b\in DR(b)$, $r_a\leq_r r_b$.
Note that $\preceq_{ewl}$ is a total pre-order between arguments.
We say that $a$ \emph{defeats} $b$ if $a$ undercuts $b$ or if $a$ gen-rebuts $b$ and $a\not\prec_{ewl} b$.
Semantics of \DAOM are defined via JSBAFs with the translation:
\begin{definition}
Let \AS be some AS.
Then the JSBAF corresponding to $AS$  is \JSBAF with:
\begin{itemize}
	\item{$\mathcal{A}=\mathcal{A}(AS)$}
	\item{$(a,b)\in{\rightarrow}$ iff $a$ defeats $b$.}
	\item{$(S,b)\in{\Rightarrow}$ iff $S=\{b_0,\dots,b_m\}$ and $b$ is of the form $b:b_0,\dots,b_m\rightarrow b^C$}
	\item{$a\preceq b$ iff $a\preceq_{ewl} b$}
\end{itemize}
\end{definition}
Note that the JSBAFs we can construct based on some AS satisfy the additional constraints on JSBAFs that we introduced in Section~\ref{Sec:JSBAF} after Definition~\ref{def:jsbaf}.
The preferred semantics of \DAOM are now based on the preferred semantics of JSBAFs:
\begin{definition}
Let \AS be some AS and \J the corresponding JSBAF.
$E\subseteq\mathcal{A}(AS)$ is a \emph{preferred extension} of $AS$ iff $E= in(L)$ for some $L\in pr(\mathcal{J})$.
The set of all preferred extensions of $AS$ is $pr(AS)$.
The \emph{preferred conclusions of $AS$} are $\mathcal{C}_{pr}(AS)=\{E^C\mid E\in pr(AS)\}$
\end{definition}
Note that $pr(AS)$ is a set of sets of arguments and $\mathcal{C}_{pr}(AS)$ is a set of sets of formulas.

Below, we give an example for an AS and a translation to a JSBAF.
For the sake of finiteness of the resulting framework, we don't base the strict rules on the consequence relation of some underlying logical language, but rather give here a specific set of strict rules.
Nevertheless, the results presented in this paper are of course also valid for argumentation systems with an infinite number of arguments.
\begin{example}\label{Exmp:DAOM}
Argumentation System $AS_1=(R_s,R_d,n,\leq)$:
\begin{itemize}
	\item{
	\(\begin{aligned}[t]
	R_{s}=\big{\{}
	&
	r_{s_1}:\rightarrow \alpha;
	\:\:\:\:
	r_{s_2}:\alpha\rightarrow \lnot(\gamma \land \delta \land \epsilon);
	\:\:\:\:
	r_{s_3}:\rightarrow \delta;
	\:\:\:\:
	r_{s_4}:\,\gamma,\delta,\epsilon\rightarrow(\gamma \land \delta\land \epsilon)
	\big{\}}
	\end{aligned}\)
	}
	\item{
	$R_{d}=\{
	r_{d_1}:\Rightarrow \gamma;
	\:\:\:\:
	r_{d_2}:\Rightarrow \epsilon
	\}$
	}
	\item{
	$n=\emptyset$ and $r_{d_1}\leq r_{d_2}\leq r_{d_1}$
	}
\end{itemize}
\end{example}
Based on $AS_1$ we construct the arguments $a:\rightarrow \alpha$, $b: a\rightarrow\lnot(\gamma \land \delta \land \epsilon)$, $c:\Rightarrow \gamma$, $d:\rightarrow \delta$, $e:\Rightarrow \epsilon$ and $\overline{b}:c,d,e\rightarrow(\gamma \land \delta \land  \epsilon)$.
We have ${\preceq_{ewl}} = \{a,b,\overline{b},c,d,e\}^2 \setminus \big(\{a,b,d\} \times \{\overline{b},c,e\}\big)$.
Argument $b$ gen-rebuts $\overline{b}$ on its conclusion, but $\overline{b}$ does not gen-rebut $b$ since $b$ is strict.
Thus $b$ defeats $\overline{b}$.
The JSBAF corresponding to $AS_1$ is $\mathcal{J}_1$ of Example~\ref{Exmp:JSBAF}.
In this JSBAF, the strict rules $r_{s_1}$, $r_{s_2}$, $r_{s_3}$ and $r_{s_4}$ are translated to the supports $(\emptyset,a)$, $(\{a\},b)$, $(\{c,d,e\},\overline{b})$ and $(\emptyset,d)$ respectively.
We have $pr(AS_1)=\big{\{}\{a,b,c,d\},\{a,b,d,e\}\big{\}}$

\subsection{Rationality Postulates}\label{Sec:RationalityPostulates}
The first three postulates we want to cover are those of \emph{closure}, \emph{direct consistency} and \emph{indirect consistency}, which were defined by Caminada and Amgoud~\cite{Caminada2007:OnEvaluationOfArgumentationFormalisms}.
They state that no extension should contain contradictory arguments (consistency) and for any strict rule $r$, if all antecedents of $r$ are accepted, then the conclusion of $r$ should also be accepted (closure).
For \DAOM we define these postulates as follows:
\begin{definition}
A \DAOMItallic semantics $\sigma$ satisfies:
\begin{itemize}
	\item{\emph{Direct consistency} iff for all $AS$, for all $E\in \sigma(AS)$, for all $\phi,\psi\in E^C$, $\phi\neq-\psi$.}
	\item{\emph{Closure} iff for all $AS$ and for all $E\in\sigma(AS)$, $E^C=CL_{R_s}(E^C)$, where for any set $S$, $CL_{R_s}(S)$ is the smallest set \suchthat $S\subseteq CL_{R_s}(S)$ and for all $r\in R_s$, if $r$ is of the form $r:\phi_0,\dots\phi_m\rightarrow\psi$ and $\phi_0,\dots,\phi_m\in CL_{R_s}(S)$, then $\psi\in CL_{R_s}(S)$.}
	\item{\emph{Indirect consistency} iff it satisfies closure and direct consistency.}
\end{itemize}
\end{definition}

We can show preferred semantics of \DAOM satisfies these postulates.
%
%
\begin{theorem}
Preferred semantics of \DAOMItallic satisfies closure, direct consistency and indirect consistency.
\end{theorem}
\begin{technicalReportOnly}
\begin{proof}
Let \AS be some AS and let \JSBAF be the JSBAF corresponding to $AS$ according to \DAOM.
We begin with the closure postulate:
Let $E\in pr(AS)$ be some preferred extension of $AS$ and let $L\in pr(\mathcal{J})$ be the labeling that corresponds to $E$, \ie $E=in(L)$.
Note that we have $L\in adm(\mathcal{J})$.
Assume that there is a strict rule $r=\phi_0,...,\phi_n\rightarrow\psi$ for which we have $\phi_0,...,\phi_n\in E^C$.
Then there are arguments $a_0,...,a_n\in E$ \suchthat for each $0\leq i\leq n$, we have $a_i^C=\phi_i$.
Because of the strict rule $r$, there now also exists an argument $a:a_0,...,a_n\rightarrow\psi$.
By construction of $\mathcal{J}$ from $AS$ we know that there is a support $\big(\{a_0,...,a_n\},a\big)$.
Because $a_0,...,a_n\in E=in(L)$ and by admissibility of $L$, we infer that $a\in E$ must also hold.
Therefore we have $\psi\in E^C$ as required.
Now for direct consistency:
As before, let $E\in pr(AS)$ be some preferred extension of AS and let $L\in pr(\mathcal{J})$ be the labeling that corresponds to $E$.
Towards a contradiction, assume that there are $\phi,\psi\in E^C$ \suchthat $\phi=-\psi$.
Then we have arguments $a,b\in E$ \suchthat $a^C=\phi$ and $b^C=\psi$.
Because $AS$ is consistent, at least one of these arguments has to be defeasible.
\Generality we assume that $b$ is defeasible, meaning $a$ gen-rebuts $b$.
Now the argument $a$ can be strict or not strict.
If $a$ is strict, $DR(a)=\emptyset$ and we have $a\not\preceq_{ewl} b$, meaning $a\not\prec_{ewl} b$.
This means $a$ defeats $b$ and by construction of $\mathcal{J}$ from $AS$ we have $(a,b)\in{\rightarrow}$.
Now $a,b\in in(L)$ contradicts the admissibility of $L$.
On the other hand, if $a$ is not strict, then both $a$ and $b$ gen-rebut each other on their respective conclusions.
Furthermore, either $a\not\prec_{ewl} b$ or $b\not\prec_{ewl} a$ must hold, because otherwise we have $a\preceq_{ewl} b$ and $b\not\preceq_{ewl} a$, while simultaneously $b\preceq_{ewl} a$ and $a\not\preceq_{ewl} b$.
Thus either $(a,b)\in{\rightarrow}$ or $(b,a)\in{\rightarrow}$, which again means $a,b\in in(L)$ cannot hold by admissibility of $L$.
Lastly, indirect consistency under preferred semantics follows directly from the fact that \DAOM satisfies closure and direct consistency under preferred semantics.
\end{proof}
\end{technicalReportOnly}
The remaining postulates are those of \emph{non-interference} and \emph{crash-resistance}.
They were first defined by Caminada~\etal~\cite{Caminada2012:SemiStableSemantics} and cover the behavior of two distinct AS when they are combined to a single AS.
We begin by introducing some additional notation.
Given two AS, \ASOne and \ASTwo, we define a union of $AS_1$ and $AS_2$ as \ASPlus, where $R_{s_+}=R_s^{AX_1}\cup R_s^{AX_2}\cup R_s^{\vDash}$, $R_{d_+}=R_{d_1}\cup R_{d_2}$, $n_+=n_1\cup n_2$ and ${\leq_r^+}\subseteq R_d^+\times R_d^+$ is a total pre-order \suchthat ${\leq_{r_1}}\subseteq {\leq_r^+}$ and ${\leq_{r_2}}\subseteq {\leq_r^+}$.
\begin{technicalReportOnly}
Given any rule $r\in R_{s}^{AX}\cup R_d$ of the form $r=\phi_0,...,\phi_{m-1}\rightsquigarrow\phi_m$, we define $Atoms(r)=\bigcup\limits_{0\leq i\leq m} Atoms(\phi_i)$.
The atoms of the defeasible and axiomatic rules of an $AS$ are $Atoms(R_d)=\bigcup\limits_{r\in R_d}Atoms(r)$ and $Atoms(R_{s}^{AX})=\bigcup\limits_{r\in R_{s}^{AX}} Atoms(r)$ respectively.
For the name of a defeasible rule $r$, the atoms occurring in $n(r)=\phi$ are $Atoms\big(n(r)\big)=Atoms(\phi)$.
The atoms occurring in the naming function $n$ are $Atoms(n)=\bigcup\limits_{r\in R'} Atoms\big(n(r)\big)$, where $R'\subseteq R_d$ is the set of defeasible rules $r$ for which $n(r)$ is defined.
Now the atoms of an $AS$ are defined as $Atoms(AS)=Atoms(R_{d})\cup Atoms(R_s^{AX})\cup Atoms(n)$.
\end{technicalReportOnly}
\begin{conferencePaperOnly}
For any $AS$, the atoms occurring in $AS$, denoted $Atoms(AS)$, are the atomic formulas occurring in all its axiomatic rules, its defeasible rules and in the naming function for its defeasible rules.
\end{conferencePaperOnly}
We say $AS_1$ and $AS_2$ are syntactically disjoint, written $AS_1|| AS_2$, iff $Atoms(AS_1)||Atoms(AS_2)$.
\begin{technicalReportOnly}

Next, we define the atoms of an argument $a$ inductively:
Let $a$ be of the form $a:a_0,\dots,a_{n}\rightsquigarrow\phi$.
If $TR(a)\in R_s^{AX}$, then $Atoms(a)=Atoms(\phi)$.
If $TR(a)\in R_{s}^{\vDash}$, then $Atoms(a)=\emptyset\cup Atoms(a_0)\cup\dots\cup Atoms(a_n)$.
Lastly, if $TR(a)=r\in R_d$, then $Atoms(a)=\big(Atoms(a_0)\cup Atoms(a_0^C)\big)\cup\dots\cup\big(Atoms(a_n)\cup Atoms(a_n^C)\big)\cup Atoms(\phi)$.
%

%
\end{technicalReportOnly}
Given some set of sets of formulas $\Gamma\subseteq 2^{\mathcal{F}}$, and a set of atoms $\Delta\subseteq Atoms(\mathcal{F})$, the restriction of $\Gamma$ to $\Delta$ is defined as the set $\Gamma\!\!\mid\!\!_{\Delta}=\big{\{}\Gamma'\!\!\mid\!\!_{\Delta}\mid\Gamma'\in\Gamma, \Gamma'\!\!\mid\!\!_{\Delta}=\{\phi\in\Gamma'\mid Atoms(\phi)\subseteq\Delta\}\big{\}}$.
The nestedness of this definition is due to the fact that $\Gamma$ is a set of sets of formulas (not simply a set of formulas).
Non-interference states that, when combining $AS_1$ with a syntactically disjoint $AS_2$, $AS_1$ does not influence the behavior of a semantics $\sigma$ \wrt the atoms of $AS_2$ and crash-resistance states that $AS_1$ does not render the $\sigma$-consequences of $AS_2$ irrelevant.
\begin{definition}
A \DAOMItallic semantics $\sigma$ satisfies:
\begin{itemize}
	\item{\emph{Non-interference} iff for any $AS_1$, $AS_2$ with $AS_1||AS_2$, $\mathcal{C}_{\sigma}(AS_i)\!\!\mid\!\!_{Atoms(AS_i)}=\mathcal{C}_{\sigma}(AS_+)\!\!\mid\!\!_{Atoms(AS_i)}$, where $i\in\{1,2\}$ and $AS_+$ is a union of $AS_1$ and $AS_2$.}
	\item{\emph{Crash-resistance} iff there is no contaminating $AS_1$ for $\sigma$, where $AS_1$ is contaminating iff $Atoms(AS_1)\subsetneq Atoms(\mathcal{F})$ and for any $AS_2$ with $AS_1||AS_2$, $\mathcal{C}_{\sigma}(AS_1)=\mathcal{C}_{\sigma}(AS_+)$ with $AS_+$ being a union of $AS_1$ and $AS_2$.}
\end{itemize}
\end{definition}
\begin{conferencePaperOnly}
We can show that our semantics satisfy both Non-interference and crash-resistance.
%
\end{conferencePaperOnly}
\begin{technicalReportOnly}
Caminada~\etal~\cite{Caminada2012:SemiStableSemantics} showed that a logical formalism that satisfies non-interference also satisfies crash-resistance, given that the formalism in question is \emph{non-trivial}.
To this end, we first define non-triviality in the context of \DAOM:
\begin{definition}\label{Def:ASPIC:NonTrivial}
Let $\mathcal{F}$ be some set of well-founded formulas and let $\entails$ be a consequence relation associated with $\mathcal{F}$.
Furthermore, let $\Gamma\subseteq Atoms\big(\bigcup\limits_{\phi\in\mathcal{F}}Atoms(\phi)\big)$ be a non-empty set of atoms.
We say that \DAOMItallic is non-trivial under semantics $\sigma$ iff there exist \ASOne and \ASTwo, \suchthat $AS_1||AS_2$, $\Gamma=Atoms(AS_1)=Atoms(AS_2)$ and we have $\mathcal{C}_{\sigma}(AS_1)\!\!\mid\!\!_{\Gamma}\neq \mathcal{C}_{\sigma}(AS_2)\!\!\mid\!\!_{\Gamma}$.
\end{definition}
Next, we show that \DAOM under preferred semantics satisfies non-triviality. For this, we have adapted a proof of Wu and Podlaszewski~\cite{Wu2015:ImplementingCrashResistanceAndNonInterferenceInLogicBasedArgumentation}:
\begin{proposition}\label{Prop:ASPIC:DAOMNonTrivial}
\DAOMItallic is non-trivial under preferred semantics.
\end{proposition}
\begin{proof}
Let $\mathcal{F}$ be a set of well-founded formulas, let $\entails$ be a consequence-relation associated with $\mathcal{F}$ and let $\Gamma\subseteq \bigcup\limits_{\phi\in\mathcal{F}}Atoms(\phi)$ be a non-empty set of atoms.
We define \ASOne and \ASTwo as follows:
$AX_1=AX_2=\emptyset$ and $n_1=n_2=\emptyset$ while for $\Gamma=\{\phi_0,\dots,\phi_n\}$, $R_{d_1}=\{\Rightarrow\phi_0;\dots;\Rightarrow\phi_n\}$ and $R_{d_2}=\{\phi_0\Rightarrow\phi_0;\dots;\phi_n\Rightarrow\phi_n\}$.
Furthermore, we define ${\leq_{r_1}}=R_{d_1}\times R_{d_2}$ and ${\leq_{r_2}}=R_{d_2}\times R_{d_2}$.
We first note that for any $\phi_i\in\Gamma$, $\phi_i$ is neither unsatisfiable nor tautological.
From this, we can infer two things:
Firstly, for no $\phi_i\in\Gamma$ is $\phi_i$ or $\lnot\phi_i$ the consequence of a strict argument and secondly, no elements $\phi_i,\phi_j\in\Gamma$ with $\phi_i\neq \phi_j$ can contradict each other.
It is obvious that the first claim holds.
To show the second claim, assume towards a contradiction that it does not hold.
Then there exists interpretations $I_1,I_2$ \suchthat $I_1\isModelOf\phi_i$, $I_2\isModelOf\phi_j$.
Because $\phi_i\neq\phi_j$, $\phi_i$ and $\phi_j$ are syntactically disjoint.
Therefore, there exists an interpretation $I$ \suchthat $I\isModelOf\phi_i$ and $I\isModelOf\phi_j$, but now $I\isModelOf\phi$ and $I\isModelOf\lnot\phi$, a contradiction.
Using these two claims, we can infer that in $AS_1$ there do not exist strict arguments $\overline{a_i}$ with $\overline{a_i}^C=\lnot\phi_i$ and in $AS_2$, there do not exist strict arguments $a_i$ with $a_i^C=\phi_i$ for $0\leq i\leq n$.
We conclude:
There exists $\Gamma_1\in \mathcal{C}_{pr}(AS_1)\!\!\mid\!\!_{\mathcal{A}}$ with $\{\phi_1,...,\phi_n\}\subseteq\Gamma_1$ while there is no $\Gamma_2\in \mathcal{C}_{pr}(AS_2)\!\!\mid\!\!_{\mathcal{A}}$ \suchthat $\{\phi_1,...,\phi_n\}\subseteq\Gamma_2$.
Therefore $\mathcal{C}_{pr}(AS_1)\!\!\mid\!\!_{\mathcal{A}}\neq \mathcal{C}_{pr}(AS_2)\!\!\mid\!\!_{\mathcal{A}}$ as required.
\end{proof}

Based on this notion of non-triviality, we now show that, if \DAOM satisfies non-interference under preferred semantics, then it also satisfies crash resistance under preferred semantics.
We adapt a proof from Wu and Podlaszewski~\cite{Wu2015:ImplementingCrashResistanceAndNonInterferenceInLogicBasedArgumentation} for this:
\begin{proposition}\label{Prop:ASPIC:NonInterferenceImpliesCrashResistance}
Let $\mathcal{F}$ be a set of well-founded formulas and let $\entails$ be a consequence relation associated with $\mathcal{F}$.
If \DAOMItallic satisfies non-interference under preferred semantics, then \DAOMItallic satisfies crash-resistance under preferred semantics.
\end{proposition}
\begin{proof}
Towards a contradiction, assume that the claim does not hold.
Then there exists some \AS \suchthat $Atoms(AS)\subsetneq\bigcup\limits_{\phi\in\mathcal{F}} Atoms(\phi)$ and $AS$ is contaminating under preferred semantics.
Let $\Gamma=\big(\bigcup\limits_{\phi\in\mathcal{F}} Atoms(\phi)\big)\setminus Atoms(AS)$.
By Proposition~\ref{Prop:ASPIC:DAOMNonTrivial}, we know that \DAOM is non-trivial.
Thus there exist two further AS, \ASOne and \ASTwo, for which we have $\Gamma=Atoms(AS_1)=Atoms(AS_2)$ and $\mathcal{C}_{pr}(AS_1)\!\!\mid\!\!_{\Gamma}\neq \mathcal{C}_{pr}(AS_2)\!\!\mid\!\!_{\Gamma}$.
Note that we have $Atoms(AS_1)\cap Atoms(AS)=\emptyset$ and $Atoms(AS_2)\cap Atoms(AS)=\emptyset$.
Let $AS_{12}$ be a union of $AS_1$ and $AS_2$, $AS_{1*}$ be a union of $AS_1$ and $AS$ and let $AS_{*2}$ be a union of $AS_2$ and $AS$.
Because \DAOM satisfies non-interference under preferred semantics by assumption, we have $\mathcal{C}_{pr}(AS_1)\!\!\mid\!\!_{\Gamma}=\mathcal{C}_{pr}(AS_{1*})\!\!\mid\!\!_{\Gamma}$ and $\mathcal{C}_{pr}(AS_2)\!\!\mid\!\!_{\Gamma}=\mathcal{C}_{pr}(AS_{*2})\!\!\mid\!\!_{\Gamma}$.
However, because $AS$ is contaminating, we now have $\mathcal{C}_{pr}(AS)\!\!\mid\!\!_{\Gamma}=\mathcal{C}_{pr}(AS_{1*})\!\!\mid\!\!_{\Gamma}$ and $\mathcal{C}_{pr}(AS)\!\!\mid\!\!_{\Gamma}=\mathcal{C}_{pr}(AS_{*2})\!\!\mid\!\!_{\Gamma}$.
Now we have $\mathcal{C}_{pr}(AS_1)\!\!\mid\!\!_{\Gamma}=\mathcal{C}_{pr}(AS_2)\!\!\mid\!\!_{\Gamma}$, contradicting our assumption that \DAOM is non-trivial.
\end{proof}
\end{technicalReportOnly}
\begin{technicalReportOnly}
These proofs utilize three key insights:
First, it is sufficient to concentrate on four specific \enquote{edge cases}, which can -- in a simplified view -- be specified as follows by reference to the argument set $\mathcal{A}'_+=\mathcal{A}_+\setminus(\mathcal{A}_i\cup\mathcal{A}_j)$:
An argument $a\in\mathcal{A}_j$ attacking an argument $b\in\mathcal{A}_i$.
An argument $a\in\mathcal{A}'_+$ attacking an argument $b\in\mathcal{A}_i$.
An argument $a\in\mathcal{A}_i$ attacking an argument $b\in\mathcal{A}'_+$, where $b$ is supported only by arguments in $\mathcal{A}_i$ and $\mathcal{A}_j$ (\ie not by arguments in $\mathcal{A}'_+$).
An argument $a\in\mathcal{A}'_+$ attacking an argument $b\in\mathcal{A}'_+$, where $b$ is supported only by arguments in $\mathcal{A}_i$ and $\mathcal{A}_j$.
\end{technicalReportOnly}
\begin{technicalReportOnly}
For better understanding, we have depicted these \enquote{edge cases} in Illustration~\ref{Illus:Post:ASOneVsASTwoVsASPlus} below.
In this illustration, the JSBAFs resulting from $AS_i$ and $AS_j$ are depicted with clouds on the left and right, while the space between the dashed lines shows the arguments in $A_+'$.
The arrows labeled one to four correspond to the cases described above.
\end{technicalReportOnly}

\begin{technicalReportOnly}
The second key insight is that, when considering these \enquote{edge cases}, we can further restrict ourselves to \emph{reduced} versions of arguments \wrt either $AS_1$ or $AS_2$.
%
%
Essentially, a reduced version of $a$ \wrt $AS_i$ can be created by \enquote{removing} all applications of defeasible or axiomatic rules from $AS_j$ during the construction of $a$.
Because reduced versions of arguments have only atoms of $AS_i$, they are contained in $\mathcal{J}_i$.
Furthermore, we can show that for any preferred labeling $L$, if $a'$ is legally \Out \wrt $L$ then $a$ must be legally \Out \wrt $L$ and if $a$ is legally \In \wrt $L$ then $a'$ must be legally \In \wrt $L$.
Using this insight, we can make statements about arguments $a\in \mathcal{A}_+\setminus(\mathcal{A}_i\cup\mathcal{A}_j)$ by considering their reduced versions $a'$.
\end{technicalReportOnly}
\begin{technicalReportOnly}
The last key insight is that we can utilize the support relation between arguments and our definition of legal labelings:
For an attack $(a,b)\in{\rightarrow}$ that is the result of a gen-rebut, we first construct a new argument $a'$.
This argument $a'$ will only have the argument $a$ as its direct sub-argument and the conclusion of $a'$ will be the negation of the conjunction of all conclusions of sub-arguments of $b$, which are created via axiomatic or defeasible rules.
Since $DR(a)=DR(a')$, we also have $(a',b)\in{\rightarrow}$.
Accepting $a$ would then mean $a'$ also needs to be accepted, whereas rejecting $a'$ would mean $a$ also needs to be rejected.
The conclusion of $a'$ has only atoms which are contained in $AS_i$ or $AS_j$, thus for our purposes the attack originating from $a'$ is easier to deal with.
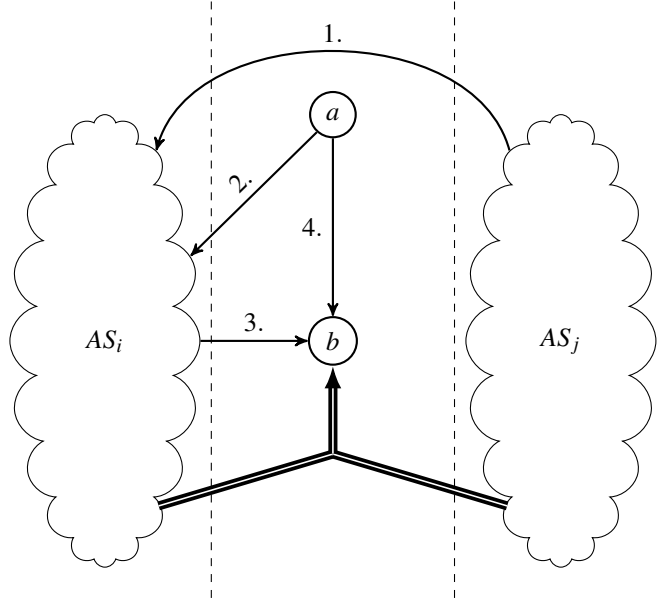
\begin{illustration}\label{Illus:Post:ASOneVsASTwoVsASPlus}
\leavevmode
\begin{center}
\begin{tikzpicture}

	\node[cloud, draw, minimum width=2.5cm, minimum height=6cm, cloud puffs=20] at (-3, 0) (ASOne) {$AS_i$};
	\node[cloud, draw, minimum width=2.5cm, minimum height=6cm, cloud puffs=20] at (3, 0) (ASTwo) {$AS_j$};
	
	\node[args] at (0, 3) (A) {$a$};
	\node[args] at (0, 0) (B) {$b$};
	\node[suppNode] at (0, -1.5) (SuppB) {};
	\node[suppNode] at (-2.3, -2.19) (BStartOne) {};
	\node[suppNode] at (2.3, -2.19) (BStartTwo) {};
	
	\draw[dashed] (-1.6,4.5) -- (-1.6,-3.5);
	\draw[dashed] (1.6,4.5) -- (1.6, -3.5);

	\path[thick, bend right=75, ->]
		(ASTwo) edge node[midway, above, sloped] {$1.$} (ASOne);
		
	\path[thick, ->]
		(A) edge node[midway, above, sloped] {$2.$} (ASOne);
	
	\path[thick, ->]
		(ASOne) edge node[midway, above, sloped] {$3.$} (B);	
	
	\path[thick, ->]
		(A) edge node[midway, left] {$4.$} (B);
	
	\draw[multiLine] (BStartOne) to (SuppB);
	\draw[multiLine] (BStartTwo) to (SuppB);
	\draw[multiLine, -\tip] (SuppB) to (B);

\end{tikzpicture}
\end{center}
\end{illustration}

%
By combining all of these key insights we can show:
\end{technicalReportOnly}
\begin{theorem}\label{Theo:Post:NonInterCrashResist}
Preferred semantics of \DAOM satisfies non-interference and crash-resistance.
\end{theorem}
\begin{technicalReportOnly}

In the remainder of this section, we will prove Theorem~\ref{Theo:Post:NonInterCrashResist}.
Note that from here on out, we will somewhat interweave an AS with its JSBAF counterpart.
For example, we will talk about \enquote{syntactically disjoint JSBAFs $\mathcal{J}_1$ and $\mathcal{J}_2$}, by which we mean that $\mathcal{J}_1$ and $\mathcal{J}_2$ are JSBAFs that correspond to some $AS_1$ and $AS_2$ \suchthat $AS_1$ and $AS_2$ are syntactically disjoint.
Similarly, for a given JSBAF \JSBAF, and an argument $a\in\mathcal{A}$, we will talk about the \enquote{sub-arguments of $a$}, by which we mean the arguments $Sub(a)\subseteq \mathcal{A}(AS)$ for $AS$ being the AS that $\mathcal{J}$ corresponds to.
We will always assume a consistent naming between an AS and its JSBAF counterpart, \ie for some \ASOne, the corresponding JSBAF will always be \JSBAFOne, while for \ASTwo, the corresponding JSBAF will be \JSBAFTwo and for \ASPlus the corresponding JSBAF will be \JSBAFPlus.

\subsection{Non-Interference and Crash-Resistance for Preferred Semantics}

\subsubsection{Axiomatic and defeasible sub-arguments}
In order to prove Theorem~\ref{Theo:Post:NonInterCrashResist}, we begin by taking a more detailed look at the third key insight mentioned in our proof overview.
Remember that, for a given attack $(a,b)\in{\rightarrow}$ which results from a gen-rebut, we wanted to leverage our definition of the support-relation between arguments and our notion of legal labelings, by constructing new arguments $a'$ and $b'$.
The argument $a'$ will have the form $a':a\rightarrow\lnot\bigwedge\Gamma$, while $b'$ will have the form $b':b_0,\dots,b_m\rightarrow\bigwedge\Gamma$, where $b_0,\dots b_m$ are precisely those sub-arguments of $b$ which were created by using an axiomatic or defeasible top-rule and $\Gamma$ is the set of all their conclusions.
For this, we use the following definition:
\begin{definition}\label{Def:Post:ADSubCSub}
Let \ASOne and \ASTwo be two AS \suchthat $AS_1||AS_2$.
Furthermore, let \ASPlus be the union of $AS_1$ and $AS_2$.
The mapping $ADSub:\mathcal{A}(AS_+)\rightarrow 2^{\mathcal{A}(AS_+)}$ maps arguments to their \emph{axiomatic and defeasible sub-arguments}.
For any $a\in\mathcal{A}(AS_+)$ we define:\\
$ADSub(a)=\big{\{} a'\in Sub(a)\mid TR(a')\in R_s^{AX_+}\cup R_{d_+}\big{\}}$.
We define the \emph{restriction} of $ADSub$ to $AS_i$ $(i\in\{1,2\})$ as:\\
$ADSub(a)\!\!\mid\!\!_{AS_i}=\big{\{}a'\in ADSub(a)\mid Atoms(a'^C)\subseteq Atoms(AS_i)\big{\}}$.
\end{definition}
With some abuse of notation, we will also use the mapping $ADSub$ for sets of arguments.
More precisely, for some $S\subseteq\mathcal{A}(AS_+)$, we define $ADSub(S)=\bigcup\limits_{a\in S} ADSub(a)$.
Next, we show that for any argument $a$, any conclusion of a sub-argument $a'$ of $a$ can be derived from $ADSub(a)^C$.
\begin{proposition}\label{Prop:Post:ADSubImpliesSubConcs}
Let \AS be some AS and let $a$ be some argument.
For every non-empty $\Delta\subseteq Sub(a)^C$ we have $ADSub(a)^C\entails \psi$ for all $\psi\in\Delta$.
\end{proposition}
\begin{proof}
We show the claim via structured induction over the construction of the argument $a$.
Let $\Gamma=ADSub(a)^C$ and let $\Delta\subseteq Sub(a)^C$ be non-empty.

Base case:
Let $a$ be of the form $a:\rightsquigarrow\phi$.
Then $\Delta=\{\phi\}$.
If we have $TR(a)\in R_s^{AX}\cup R_d$, then $\Gamma=\{\phi\}$.
Obviously $\{\phi\}\entails\phi$, therefore the claim holds.
On the other hand, if $TR(a)\in R_s^{\vDash}$, then we have $\Gamma=\emptyset$ and by the way that arguments are constructed, $\emptyset\entails\phi$.
Thus we again have $\Gamma\entails\psi$ for all $\psi\in\Delta$.
Induction step:
Let $a$ be of the form $a:a_0,\dots, a_m\rightsquigarrow\phi$.
For each $0\leq i \leq m$, let $\Gamma_i=\bigcup\limits_{a'\in ADSub(a_i)} a'^C$.
Furthermore, let $\Delta_i=\Delta\cap Sub(a_i)^C$.
Then we have $\Delta\neq\bigcup\limits_{i\in\{0,...,m\}}\Delta_i$ iff $\Delta=\{\phi\}\cup\bigcup\limits_{i\in\{0,...,m\}}\Delta_i$.
Similarly, we have $\Gamma\neq\bigcup\limits_{i\in\{0,...,m\}}\Gamma_i$ iff $\Gamma=\{\phi\}\cup\bigcup\limits_{i\in\{0,...,m\}}\Gamma_i$.
Take an arbitrary $\psi\in\Delta$.
We have to show $\Gamma\entails\psi$.
First, we assume that there is some $0 \leq i\leq m$, \suchthat $\psi\in\Delta_i$.
By our induction hypothesis, we now have $\Gamma_i\entails\psi$.
Since $\Gamma_i\subseteq\Gamma$, we can infer $\Gamma\entails\psi$ by monotonicity of $\entails$.
Now assume that there is no such $i\in\{0,...,m\}$.
Then we must have $\Delta=\{\psi\}\cup\bigcup\limits_{i\in\{0,...,m\}}\Delta_i$, meaning $\psi=\phi$.
Now we can either have $\phi\in\Gamma$ or $\phi\not\in\Gamma$, depending on whether or not $TR(a)\in R_d\cup R_s^{AX}$ ($\phi\in\Gamma$) or $TR(a)\in R_s^{\vDash}$ ($\phi\not\in\Gamma$).
Suppose first that $\phi\in\Gamma$.
Then trivially $\Gamma\entails\phi=\psi$.
On the other hand, if $\phi\not\in\Gamma$, then we have $\{a_0^C,...,a_m^C\}\entails\phi$ by the way that arguments are constructed.
For all $0\leq i\leq m$, we have $\Gamma_i\entails a_i^C$ by the induction hypothesis.
Because $\Gamma_i\subseteq\Gamma$, we now also have $\Gamma\entails a_i^C$ for each $a_i$.
By transitivity of $\entails$, we again infer $\Gamma\entails\phi=\psi$, as required.
\end{proof}
In particular, we can use this proposition to infer that the conclusion of any argument is entailed by the conclusions of its axiomatic and defeasible sub-arguments:
\begin{corollary}\label{Cor:Post:ADSubImpliesConc}
Let \AS be some AS and let $a$ be some argument.
Then $ADSub(a)^C\entails 
a^C$.
\end{corollary}
Next, we show that, if we have a gen-rebut from an argument $a$ to an argument $b$, then there exists an argument $a'$ with only $a$ as a direct sub-argument, \suchthat $a'$ gen-rebuts $b$ on $ADSub(b)$:

\begin{proposition}\label{Prop:Post:OneStepAttacker}
Let \AS be an AS and let $a,b$ be arguments \suchthat $a$ gen-rebuts $b$.
Then there exists $a'$ of the form $a':a\rightarrow\lnot\bigwedge ADSub(b)^C$.
\end{proposition}
\begin{proof}
Since $a$ is gen-rebutting $b$ by assumption, we know that $a^C=\lnot\bigwedge\{\phi_0,...,\phi_m\}$ for $\{\phi_0,...,\phi_m\}\subseteq Sub(b)^C$.
We only have to show that $\lnot\bigwedge\{\phi_0,...,\phi_m\}\entails\lnot\bigwedge ADSub(b)^C$.
Towards a contradiction, assume that this is not the case.
Then there exists an interpretation $I$ \suchthat $I\isModelOf\lnot\bigwedge\{\phi_0,...,\phi_m\}$ and $I\isModelOf\bigwedge ADSub(b)^C$.
By Proposition~\ref{Prop:Post:ADSubImpliesSubConcs} we can infer that $I\isModelOf\bigwedge\{\phi_0,...,\phi_m\}$ also holds, contradicting $I\isModelOf\lnot\bigwedge\{\phi_0,...,\phi_m\}$.
Therefore the argument $a'$ can be constructed as claimed and gen-rebuts $b$ on $ADSub(b)^C$.
\end{proof}
\subsubsection{Closure under sub-arguments of preferred labelings}
The next step will be to show that preferred labelings are closed under sub-arguments, which will be very helpful in our later proofs.
For this, we will first introduce the notion of a \emph{propagated labeling}, which can be obtained from a labeling $L$ and an argument $a$ by adding $a$ to $in(L)$ and propagating the effect of this acceptance throughout the JSBAF.
For this propagation, we need to ensure two things:
Fristly, if there is a support $(S,b)$ where all arguments in $S$ are accepted, then $b$ is also accepted.
Secondly, if there is a support $(S,b)$ where $b$ is rejected and all arguments in $S$ except for one are accepted, this last remaining argument is rejected.
Note that we will not refer to preferences between arguments in our definition.
Instead, we will later use this method of propagation only on admissible labelings.
We will show that under certain conditions, admissibility can be retained with this construction, even though we do not explicitly account for preferences between arguments.
\begin{definition}\label{Def:Postulates:PropagatedLabeling}
Let \JSBAF be a JSBAF, let $L$ be a labeling of $\mathcal{J}$ and let $a\in\mathcal{A}$ be an argument.
We define the \emph{propagated labeling} of $L$ and $a$, denoted $L_{a}$, as follows:
\begin{align*}
	I_0&=
	in(L)\cup\{a\}\\
	I_{k+1}&=I_k\cup\{b\in\mathcal{A}\mid\exists(S,b)\in{\Rightarrow},S\subseteq I_k\}\\
	in(L_{a})&=\bigcup\limits_{k\geq 0} I_k\\
	O_0&=\{b\in\mathcal{A}\mid\exists (c,b)\in{\rightarrow},c\in in(L_a)\}\\
	O_{k+1}&=O_k\cup\{b\in\mathcal{A}\mid\exists (S,c)\in{\Rightarrow},b\in S,\\
	&\hspace{1.5cm}S\setminus\{b\}\subseteq in(L_a),c\in O_k\}\\
	out(L_a)&=\bigcup\limits_{k\geq 0} O_k\\
	undec(L_a)&=\mathcal{A}\setminus\big(in(L_a)\cup out(L_a)\big)
\end{align*}
\end{definition}
Note that, because supports $(S,b)\in{\Rightarrow}$ are based on the (transitive) entailment $S^C\entails b^C$, for every argument in $I_k$ for $k\geq 2$, there exists an \enquote{equivalent} argument in $I_1$.
This is easy to see with the following proof:
\begin{proposition}\label{Prop:Post:PropagatedLabelingOneStepPropagation}
Let \JSBAF be a JSBAF, let $L$ be labeling of $\mathcal{J}$, let $a\in\mathcal{A}$ be an argument and let $L_a$ be the propagated labeling of $L$ and $a$ as constructed in Definition~\ref{Def:Postulates:PropagatedLabeling}.
If $x\in I_k\setminus I_0$, then there exists $x'\in I_1\setminus I_0$ \suchthat $x^C=x'^C$ and $ADSub(x)=ADSub(x')$.
\end{proposition}
\begin{proof}
We show the claim via structured induction over $n\in\mathbb{N}$ for $x\in I_n\setminus I_0$.
Induction start, $n=1$:
In this case the claim trivially holds.
Induction step, $n\rightarrow n+1$:
Let $x$ be of the form $x:x_0,\dots,x_m\rightarrow x^C$.
By the induction hypothesis, there is $x_i'$ for each $0\leq i\leq m$ \suchthat $x_i'\in I_1\setminus I_0$, $x_i^C=x_i'^C$ and $ADSub(x_i)=ADSub(x_i')$.
Let $x_i'$ be of the form $x_i':x_{i_1}',\dots,x_{i_k}'\rightarrow x_i'^C$.
We define the sets $S_i$ as $S_i=\{x_{i_1}',\dots,x_{i_k}'\}$ and the set $S'$ as $S'=\bigcup\limits_{0\leq i\leq m} S_i$.
By transitivity of $\entails$, we have $S'^C\entails x^C$, \ie there exists the argument $x'$ of the form $x':x_1'',\dots,x_l''\rightarrow x^C$, where $\{x_1'',\dots x_l''\}=S'$.
Obviously we have $x^C=x'^C$.
By the induction hypothesis, we have $ADSub(x_i)=ADSub(x'_i)$ for each $x_i$.
From this, we can infer that $ADSub(x)=ADSub(x')$ also holds.
By the induction hypothesis we also have $x'_i\in I_1\setminus I_0$ for each $x'_i$.
From this we can infer that $S'\subseteq in(L)\cup\{a\}$ must hold.
By construction of $L_a$, we now have $x'\in I_1\setminus I_0$ as required.
\end{proof}
Next, we show that propagated labelings will be admissible if three conditions are met:
First, we need to start with an already admissible labeling $L$.
Secondly, for the argument $a$ that we choose as the \enquote{starting point} for our propagation, we need $L(a)\neq\Out$ and for all attackers $b$ of $a$ we need to have $L(b)=\Out$.
Lastly, we need to ensure that there does not exist a support chain $\mathcal{C}=\big{\{}(S_0,b_0),\dots,(S_n,b_n)\big{\}}$ starting at $a\in S_0$ \suchthat we have $S_i\setminus\{b_{i-1}\}\subseteq in(L)$, while there exists $(c,b_n)\in{\rightarrow}$ with $L(c)\neq\Out$.
We will later ensure that these conditions are always satisfied when creating a propagated labeling $L_a$.
To make the actual proof more accessible, we have split it into several parts.
We begin by showing that $L_a$ is a labeling:
\begin{proposition}\label{Prop:Post:PropagatedLabelingIsLabeling}
Let \JSBAF be a JSBAF, let $L\in adm(\mathcal{J})$ be an admissible labeling of $\mathcal{J}$ and let $a\in\mathcal{A}$ be an argument.
Furthermore, assume that $L(a)\neq\Out$ and that for all attackers $b$ of $a$, we have $L(b)=\Out$.
Lastly, assume that there does not exist a support chain $\big{\{}(S_0,b_0),...,(S_n,b_n)\big{\}}\subseteq{\Rightarrow}$ and an attack $(c,b_n)\in{\rightarrow}$ with $a\in S_0$, $L(c)\neq\Out$, $S_0\setminus\{a\}\subseteq in(L)$ and $S_i\setminus\{b_{i-1}\}\subseteq in(L)$ for $0<i\leq n$.
Then $L_a$ is a labeling.
\end{proposition}
\begin{proof}
It is easy to see that each argument gets \emph{some} label and that $undec(L_a)\cap in(L_a)=undec(L_a)\cap out(L_a)=\emptyset$.
Therefore, we have left to show that $in(L_a)\cap out(L_a)=\emptyset$ also holds.
Towards a contradiction, assume that there is an argument $x\in out(L_a)\cap in(L_a)$.
We show the contradiction by induction over $n\in\mathbb{N}$ for $x\in O_n$.
Induction start, $n=0$:
This means there exists some $(y,x)\in{\rightarrow}$ with $y\in in(L_a)$.
We make two case distinctions, based on the origins of $x\in in(L_a)$ and $y\in in(L_y)$:
We have $x\in in(L_a)$ either because $x\in I_0$ or because $x\in I_k$ for some $k\geq 1$.
Similarly, we have $y\in in(L_a)$ either because $y\in I_0$ or because $y\in in(I_k)$ for some $k\geq 1$.
Assume first that $x\in I_0$ and $y\in I_0$ holds:
If we have $x\in I_0$ because $x\in in(L)$, then we can use admissibility of $L$ to infer that $L(y)=\Out$.
On the other hand, if we have $x\in I_0$ because $x=a$, then we can use our assumption that all attackers of $a$ are labeled \Out in $L$ to infer $L(y)=\Out$.
Either way, we have $y\in out(L)\cap I_0$, a contradiction.
Now assume that $x\in I_0$ and $y\in I_k$ for some $k\geq 1$:
As we have dealt with the case $y\in I_0$ above, we assume $y\in I_k\setminus I_0$.
By Proposition~\ref{Prop:Post:PropagatedLabelingOneStepPropagation} 
there exists an argument $y'\in I_1\setminus I_0$ \suchthat $y'^C=y^C$ and $ADSub(y')=ADSub(y)$.
Because $y'\in I_1\setminus I_0$, we can infer from the construction of $I_{1}$ that there exists a support $(S,y')\in{\Rightarrow}$ with $a\in S$ and $S\setminus \{a\}\subseteq in(L)$.
Similar to the case of $y$ above, we have $L(y')=\Out$, because all attackers of $x$ are labeled \Out in $L$ and $y'$ attacks $x$.
Now we have the support $(S,y')$ with $L(y')=\Out$ and $S\setminus\{a\}\subseteq in(L)$.
Because $L$ was an admissible labeling, we can now infer that $L(a)=\Out$ must hold, contradicting our assumption $L(a)\neq\Out$.
Next, let us consider the cases where $x\in I_k$ for some $k\geq 1$.
We assume $x\not\in I_0$ as we have dealt with this case before.
We again use Proposition~\ref{Prop:Post:PropagatedLabelingOneStepPropagation} to infer that there exists an argument $x'\in I_1\setminus I_0$ \suchthat $ADSub(x)=ADSub(x')$ and $x^C=x'^C$.
Next, consider the attack $(y,x)$:
If $(y,x)\in{\rightarrow}$ because $y$ undercuts $x$, then $y$ also undercuts $x'$.
On the other hand, if $y$ gen-rebuts $x$, then we can use Proposition~\ref{Prop:Post:OneStepAttacker} to infer the existence of an argument $\widetilde{y}:y\rightarrow\lnot\bigwedge ADSub(x^C)$.
Because $ADSub(x)=ADSub(x')$, we now also infer $(\widetilde{y},x')\in{\rightarrow}$.
Now, let us consider the origin of $y\in in(L_a)$ again:
We either have $y\in I_0$ or $y\in I_k$ for some $k\geq 1$.
Assume first that $y\in I_0$, \ie $y\in in(L)$ or $y=a$.
Note that, because $x'\in I_1\setminus I_0$, there must exist a support $(S_x,x')\in{\Rightarrow}$ \suchthat $a\in S_x$ and $S_x\setminus\{a\}\subseteq in(L)$.
By our assumption on the support chains starting at $a$, for this support chain $\big{\{}(S_x,x')\big{\}}\subseteq{\Rightarrow}$, we must have $L(y)=\Out$ (if $y$ undercuts $x$) or $L(\widetilde{y})=\Out$ (if $y$ gen-rebuts $x$), which in turn implies $L(y)=\Out$ by admissibility of $L$ and construction of $\widetilde{y}$.
Either way, we can infer $L(y)=\Out$, which contradicts $y\in in(L)\cup\{a\}$.
Now, assume that $y\in I_k\setminus I_0$ for some $k\geq 1$.
We use Proposition~\ref{Prop:Post:PropagatedLabelingOneStepPropagation} again to infer:
If $y$ undercuts $x$, then there exists an argument $y'\in I_1\setminus I_0$ \suchthat $y'^C=y^C$ and $ADSub(y')=ADSub(y)$ and if $y$ gen-rebuts $x$, then there exists an argument $\widetilde{y}'\in I_1\setminus I_0$ \suchthat $\widetilde{y}'^C=\widetilde{y}^C$ and $ADSub(\widetilde{y}')=ADSub(\widetilde{y})$.
Note that, because $y'\in I_1\setminus I_0$ (respectively $\widetilde{y}'\in I_1\setminus I_0$), we can infer that there exists a support $(S_y,y')\in{\Rightarrow}$ (respectively $(S_y,\widetilde{y}')\in{\Rightarrow}$) with $a\in S_y$ and $S_y\setminus \{a\}\subseteq in(L)$.
Because $ADSub(y')=ADsub(y)$ (respectively $ADSub(\widetilde{y}')=ADSub(\widetilde{y})$), and $y'^C=y^C$ (respectively $\widetilde{y}'^C=\widetilde{y}^C$), we can again infer that $(y',x')\in{\rightarrow}$ (respectively $(\widetilde{y}',x')\in{\rightarrow}$) holds.
By the assumptions made on the support chains starting at $a$, we can now again infer that $L(y')=\Out$ (respectively $L(\widetilde{y}')=\Out$).
However, for the supporting set $S_y$ we can now use the admissibility of $L$ and the fact that $S_y\setminus\{a\}\subseteq in(L)$ to infer $L(a)=\Out$, which contradicts our original assumption $L(a)\neq\Out$.
This concludes the induction start.
Induction step, $n\rightsquigarrow n+1$:
By the induction hypothesis we have $O_n\cap in(L_a)=\emptyset$, therefore we restrict ourselves to $O_{n+1}\setminus O_n$.
Assume that $x\in O_{n+1}\setminus O_n$.
This means there exists some support $(S,c)\in{\Rightarrow}$, with $x\in S$, $S\setminus \{x\}\subseteq in(L_a)$ and $c\in O_n$.
By assumption, $x\in in(L_a)$ also holds.
Now we can infer that there must be some $k\in\mathbb{N}$ for which we have $S\subseteq I_k$.
This implies $c\in I_{k+1}\subseteq in(L_a)$ by construction of $in(L_a)$.
Now $c\in O_n\cap in(L_a)$, which contradicts the induction hypothesis.
\end{proof}
Next, we show that arguments accepted in $L_a$ are also legally accepted:
\begin{proposition}\label{Prop:Post:PropagatedLabelingLegallyIn}
Let \JSBAF be a JSBAF, let $L\in adm(\mathcal{J})$ be an admissible labeling of $\mathcal{J}$ and let $a\in\mathcal{A}$ be an argument.
Furthermore, assume that $L(a)\neq\Out$ and that for all attackers $b$ of $a$, we have $L(b)=\Out$.
Lastly, assume that there does not exist a support chain $\big{\{}(S_0,b_0),...,(S_n,b_n)\big{\}}\subseteq{\Rightarrow}$ and an attack $(c,b_n)\in{\rightarrow}$ with $a\in S_0$, $L(c)\neq\Out$, $S_0\setminus\{a\}\subseteq in(L)$ and $S_i\setminus\{b_{i-1}\}\subseteq in(L)$ for $0<i\leq n$.
If $x\in in(L_a)$, then $x$ is legally \In \wrt $L_a$.
\end{proposition}
\begin{proof}
First, let us consider the attackers of $x$:
Assume that $x\in in(L_a)$ and $(y,x)\in{\rightarrow}$.
If $x\in I_0$, then we either have $x\in in(L)$ or $x=a$.
In both cases, we know $L(y)=\Out$.
Since $in(L)\subseteq in(L_a)$, we can infer that $L_a(y)=\Out$ holds due to the same attack or support chain that lead to $L(y)=\Out$.
Now assume $x\in I_k\setminus I_0$ for some $k\geq 1$.
By Proposition~\ref{Prop:Post:PropagatedLabelingOneStepPropagation}, there exists $x'\in I_1\setminus I_0$ \suchthat $x'^C=x^C$ and $ADSub(x')=ADSub(x)$.
As for the proof that showed $L_a$ is a labeling, we again want to point out that $x'\in I_1\setminus I_0$ implies the existence of a support $(S,x')\in{\Rightarrow}$ with $a\in S$ and $S\setminus\{a\}\subseteq in(L)$.
Now for the attack $(y,x)\in{\rightarrow}$:
If $y$ undercuts $x$, then $(y,x')\in{\rightarrow}$ also holds.
On the other hand, if $y$ gen-rebuts $x$, then we can again use Proposition~\ref{Prop:Post:OneStepAttacker} to infer the existence of an argument $\widetilde{y}:y\rightarrow\lnot\bigwedge ADSub(x)^C$ and because $ADSub(x)=ADSub(x')$ and $DR(x)=DR(x')$, we have $(\widetilde{y},x')\in{\rightarrow}$.
Since $x\in I_1$ and by the assumptions made on support-chains starting at $a$, we can infer:
If $y$ undercuts $x$, then $L(y)=\Out$ must thold.
On the other hand, if $y$ gen-rebuts $x$, then $L(\widetilde{y})=\Out$ must hold and by admissibility of $L$ this implies $L(y)=\Out$.
In both cases, we have $L(y)=\Out$ and can again infer $L_a(y)=\Out$ due to the same attack or support chain that lead to $L(y)=\Out$.
Lastly, for the supports $(S,b)\in{\Rightarrow}$ with $x\in S$:
By construction of $in(L_a)$, we cannot have $S\subseteq in(L_a)$ while $b\not\in in(L_a)$.
Furthermore, if $|S\setminus in(L_a)|=1$ and $L_a(b)=\Out$, then $S\setminus in(L_a)\subseteq out(L_a)$ by construction of $out(L_a)$.
From this we can infer that, if $x\preceq c$ for all $c\in S\setminus\{x\}$, then one of the items $1.a$ to $1.c$ of Definition~\ref{Def:JSBAF:LegalLabeling} holds, as required.
\end{proof}

Before we continue with the actual proof, we show a small auxiliary statement which tells us that, if we have an admissible labeling $L$ and two arguments $a$ and $a'$, where the defeasible rules and the conclusions of sub-arguments of $a'$ are also defeasible rules and conclusions of sub-arguments of $a$, then $L(a')=\Out$ implies $L(a)=\Out$:
\begin{proposition}\label{Prop:Post:AlternativeReasonOut}
Let \JSBAF be a JSBAF and let $L\in adm(\mathcal{J})$ be an admissible labeling of $\mathcal{J}$.
Furthermore, let $a,a'\in\mathcal{A}$ be two arguments \suchthat $DR(a')\subseteq DR(a)$ and $ADSub(a')^C\subseteq ADSub(a)^C$.
Then $L(a')=\Out$ implies $L(a)=\Out$.
\end{proposition}
\begin{proof}
We go through the possible cases for why $a'$ is rejected in $L$:
Suppose first that we have $L(a')=\Out$ due to an attack $(b,a')\in{\rightarrow}$ with $L(b)=\In$.
If this attack is the result of an undercut, then we can use $DR(a')\subseteq DR(a)$ to infer that $b$ also undercuts $a$, meaning $(b,a')\in{\rightarrow}$ also holds.
By admissiblity of $L$, this implies $L(a)=\Out$.
Next, suppose this attack is the result of a gen-rebut.
By Proposition~\ref{Prop:Post:OneStepAttacker}, there exists an argument $b'$ which is of the form $b':b\rightarrow\lnot\bigwedge ADSub(a')^C$.
By admissibility of $L$, we have $L(b')=\In$.
Since $ADSub(a')^C\subseteq ADSub(a)^C$, we can infer that $b'$ gen-rebuts $a$.
Because $(b,a')\in{\rightarrow}$ by assumption, we have $b\not\prec a'$.
Since $DR(b)=DR(b')$, while $DR(a')\subseteq DR(a)$, we can infer that $b'\not\prec a$ holds.
Now $(b',a)\in{\rightarrow}$ and by admissibility of $L$, we have $L(a)=\Out$.
Lastly, assume that $L(a')=\Out$ due to a support chain $\mathcal{C}=\big{\{}(S_0,b_0),\dots,(S_n,b_n)\big{\}}$ and an attack $(c,b_n)\in{\rightarrow}$ with $a'\in S_0$, $L(c)=\In$, $b_i\in out(L)$ for $0\leq i\leq n$, $S_0\setminus\{a'\}\subseteq in(L)$ and $S_i\setminus\{b_{i-1}\}\subseteq in(L)$ for $0<i\leq n$.
Let $S=\big(S_0\setminus\{a'\}\big)\cup\bigcup\limits_{0<i\leq n}\big(S_i\setminus\{b_{i-1}\}\big)$, \ie $S$ contains all the \In-labeled arguments of the support-chain $\mathcal{C}$.
First, assume that the attack $(c,b_n)\in{\rightarrow}$ is the result of an undercut.
Then this attack must also be directed towards one of the arguments in $S\cup\{a\}$.
By assumption, $L(c)=\In$, $L$ is an admissible lableing and $S\subseteq in(L)$.
From this we can infer that $(c,a)\in{\rightarrow}$ must hold.
By admissibility of $L$ this implies $L(a)=\Out$.
Next, assume that $(c,b_n)\in{\rightarrow}$ is the result of a gen-rebut.
We use Proposition~\ref{Prop:Post:OneStepAttacker} to infer that there exists an attack $(c',b_n)\in{\rightarrow}$, where $c'$ is of the form $c':c\rightarrow\lnot\bigwedge ADSub(b_n)^C$.
Now we construct the argument $b'$ as follows: $b':a,b_0',\dots,b_m'\rightarrow\bigwedge\{a,b_0',\dots,b_m'\}^C$, where $\{b_0',\dots,b_m'\}=S$.
We want to point out three facts for the argument $b'$:
Firstly, we know that $\{b_0',\dots,b_m'\}=S\subseteq in(L)$ holds.
Secondly, from $ADSub(a')^C\subseteq ADSub(a)^C$, we can infer $ADSub(b_n)^C\subseteq ADSub(b')^C$.
This means $c'$ gen-rebuts $b'$.
Thirdly, we have $DR(b_n)\subseteq DR(b')$.
By assumption, $(c,b_n)\in{\rightarrow}$ is the result of a gen-rebut.
This means we have $c'\not\prec b_n$ (since $DR(c)=DR(c')$), which implies $c'\not\prec b'$.
Now we have $(c',b')\in{\rightarrow}$.
By admissibility of $L$ and by construction of $c'$, we have $L(c')=\In$.
This means $b'$ is legally \Out \wrt $L$, and by admissibility of $L$ we can infer $L(b')=\Out$.
Since $\{b_0',\dots,b_m'\}=S\subseteq in(L)$, this implies $a\in out(L)$, otherwise none of the arguments in $S$ would be legally \In \wrt $L$.
This proves the claim.
\end{proof}
Now we can continue with our actual proof of the admissibility of $L_a$.
In the next step, we show that an argument is rejected in $L_a$ iff it is legally rejected:
\begin{proposition}\label{Prop:Post:PropagatedLabelingLegallyOut}
Let \JSBAF be a JSBAF, let $L\in adm(\mathcal{J})$ be an admissible labeling of $\mathcal{J}$ and let $a\in\mathcal{A}$ be an argument.
Furthermore, assume that $L(a)\neq\Out$ and that for all attackers $b$ of $a$, we have $L(b)=\Out$.
Lastly, assume that there does not exist a support chain $\big{\{}(S_0,b_0),...,(S_n,b_n)\big{\}}\subseteq{\Rightarrow}$ and an attack $(c,b_n)\in{\rightarrow}$ with $a\in S_0$, $L(c)\neq\Out$, $S_0\setminus\{a\}\subseteq in(L)$ and $S_i\setminus\{b_{i-1}\}\subseteq in(L)$ for $0<i\leq n$.
We have $x\in out(L_a)$ iff $x$ is legally \Out \wrt $L_a$.
\end{proposition}
\begin{proof}
$\leftarrow$:
First, assume that $x$ is legally \Out \wrt $L_a$ due to an attack.
Then $x\in O_0\subseteq out(L_a)$.
Next, assume $x$ is legally \Out \wrt $L_a$ due to a support $(S,b)$ with $x\in S$ and $b\in out(L_a)$.
By construction of $L_a$, there needs to exist some $n\in\mathbb{N}$ \suchthat $b\in O_n$.
Now $x\in O_{n+1}\subseteq out(L_a)$ as required.
$\rightarrow$:
We need to show that, if $x\in out(L_a)$, then $x$ is legally \Out \wrt $L_a$.
Note that, if $x\in O_0$, then this is trivial.
We therefore assume $x\not\in O_0$.
By the construction of $O_{k+1}$ from $O_k$, it is clear that, if $x\in O_{k+1}\setminus O_k$, then there exists a support chain $\mathcal{C}=\big{\{}(S_0,b_0),\dots,(S_n,b_n)\big{\}}\subseteq{\Rightarrow}$ and an attack $(y,b_n)\in{\rightarrow}$, \suchthat $x\in S_0$, $L_a(y)=\In$, for $0\leq i\leq n$ we have $b_i\in out(L_a)$, for $0<i\leq n$ we have $S_i\setminus\{b_i-1\}\subseteq in(L_a)$ and for $S_0$ we have $S_0\setminus\{x\}\subseteq in(L_a)$ (i.e. items $2.a$ to $2.d$ of Definition~\ref{Def:JSBAF:LegalLabeling} all hold).
However, from the construction of $O_{k+1}$ it is not clear that $x\preceq d$ for all $d\in S\setminus\{x\}$ holds (i.e. item $2.e$ of Definition~\ref{Def:JSBAF:LegalLabeling}) and in fact this is not necessarily the case.
Thus, for the cases where item $2.e$ of Definition~\ref{Def:JSBAF:LegalLabeling} does not hold, we will show that there is an \enquote{alternative reason} for $x$ to be legally \Out \wrt $L_a$.
Suppose that $x\in O_{k+1}\setminus O_k$ due to the support chain $\mathcal{C}=\big{\{}(S_0,b_0),\dots,(S_n,b_n)\big{\}}$ and an attack $(y,b_n)\in{\rightarrow}$ as described above.
Furthermore, assume that there is some $d\in S_0\setminus\{x\}$ \suchthat $d\prec x$ (\ie because of $d$, the support chain $\mathcal{C}$ is no reason for $x$ to be considered legally \Out \wrt $L_a$).
Let $S=\big(S_0\setminus\{x\}\big)\cup\bigcup\limits_{0<i\leq n}\big(S_i\setminus\{b_{i-1}\}\big)$, \ie $S$ contains all arguments labeled \In in the support chain $\mathcal{C}$.
We now claim that, for $S=\{z_0,\dots,z_m\}$, we have $ADSub(x)^C\entails\lnot\bigwedge\big(ADSub(y)\cup ADSub(z_0)\cup\dots\cup ADSub(z_m)\big)^C$.
To see that this claim holds, assume towards a contradiction that it does not.
Then there exists an interpretation $I$ \suchthat $I\isModelOf ADSub(x)^C$, $I\isModelOf ADSub(z_i)^C$ for each $z_i\in S$ and $I\isModelOf ADSub(y)^C$.
Let $y^C=\lnot\bigwedge\Gamma$ for some $\Gamma\subseteq Sub(b_n)^C$.
From $I\isModelOf ADSub(x)^C$ and $I\isModelOf ADSub(z_i)^C$ for each $z_i\in S$, we can infer that $I\isModelOf\bigwedge\Gamma$ holds.
However, from $I\isModelOf ADSub(y)^C$, we can infer $I\isModelOf\lnot\bigwedge\Gamma$, a contradiction.
Next, we construct the arguments $y':y,z_0,\dots,z_m\rightarrow\bigwedge\{y,z_0,\dots,z_m\}^C$ and $x':x_0,\dots,x_k\rightarrow\lnot\bigwedge\Delta$, where $\{x_0,\dots x_k\}=ADSub(x)$ and $\Delta=\big(ADSub(y)\cup ADSub(z_0)\cup\dots\cup ADSub(z_m)\big)^C$.
Now $x'$ gen-rebuts $y'$.
Note that, because $d\prec x$ by assumption, there exists $r_d\in DR(d)\subseteq DR(y')$ \suchthat for all $r_x\in DR(x)=DR(x')$, $r_d\leq_r r_x$.
This implies $y'\preceq_{ewl} x'$, therefore $x'\not\prec_{ewl} y'$ and thus $(x',y')\in{\rightarrow}$.
Because $y\in in(L_a)$ and $z_0,\dots,z_m\subseteq in(L_a)$, we can infer that $y'\in in(L_a)$ holds by construction of $L_a$.
This implies that either $y'\in I_0$ or, by Proposition~\ref{Prop:Post:PropagatedLabelingOneStepPropagation}, there exists $y''\in I_1\setminus I_0$ \suchthat $ADSub(y')=ADSub(y'')$ and $y'^C=y''^C$.
In the first case we have $(x',y')\in{\rightarrow}$ and can infer $L(x')=\Out$ from $y'\in I_0$.
In the second case we have $(x',y'')\in{\rightarrow}$ because $ADSub(y')=ADSub(y'')$.
Furthermore, we can use $y''\in I_1\setminus I_0$ to infer that there exists a support $(S'',y'')\in{\Rightarrow}$ with $a\in S''$ and $S''\setminus\{a\}\subseteq in(L)$.
Now the support chain $\mathcal{C}''=\big{\{}(S'',y'')\big{\}}\subseteq{\Rightarrow}$ is a support chain of which we know by assumption that all attackers of $y''$ are labeled \Out in $L$.
With this we can again infer $L(x')=\Out$.
By construction of $x'$, we obviously have $DR(x')\subseteq DR(x)$ and $ADSub(x')^C\subseteq ADSub(x)^C$.
This means we can apply Proposition~\ref{Prop:Post:AlternativeReasonOut} to infer that $L(x)=\Out$ must thold.
By assumption, $L$ is an admissible labeling.
This means $x$ is legally \Out \wrt $L$, either due to an attack or due to a support chain.
Since $in(L)\subseteq in(L_a)$, we can now infer that $x$ is legally \Out \wrt $L_a$ due to the same attack or support chain.
This finishes the proof.
\end{proof}
Finally, we are ready to combine the previous propositions to show that $L_a$ is an admissible labeling:
\begin{proposition}\label{Prop:Post:PropagatedLabelingIsAdmissible}
Let \JSBAF be a JSBAF, let $L\in adm(\mathcal{J})$ be an admissible labeling of $\mathcal{J}$ and let $a\in\mathcal{A}$ be an argument.
Furthermore, assume that $L(a)\neq\Out$ and that for all attackers $b$ of $a$, we have $L(b)=\Out$.
Lastly, assume that there does not exist a support chain $\big{\{}(S_0,b_0),...,(S_n,b_n)\big{\}}\subseteq{\Rightarrow}$ and an attack $(c,b_n)\in{\rightarrow}$ with $a\in S_0$, $L(c)\neq\Out$, $S_0\setminus\{a\}\subseteq in(L)$ and $S_i\setminus\{b_{i-1}\}\subseteq in(L)$ for $0<i\leq n$.
Then the propagated labeling $L_a$ is an admissible labeling.
\end{proposition}
\begin{proof}
By Proposition~\ref{Prop:Post:PropagatedLabelingIsLabeling} we know that $L_a$ is a labeling.
By Proposition~\ref{Prop:Post:PropagatedLabelingLegallyIn} we know that $x\in in(L_a)$ implies $x$ is legally \In \wrt $L_a$.
By Proposition~\ref{Prop:Post:PropagatedLabelingLegallyOut} we know that $x\in out(L_a)$ iff $x$ is legally \Out \wrt $L_a$.
Lastly, because $L\in adm(\mathcal{J})$, it is clear from the construction of $L_a$ that $STR_{\mathcal{J}}\subseteq in(L_a)$ holds.
We conclude that $L_a$ is an admissible labeling.
\end{proof}
Next, we use propagated labelings to show that preferred labelings are closed under sub-arguments.
\begin{lemma}
\label{Lem:Post:PrefClosedUnderSubArgs}
Let \JSBAF be a JSBAF and let $L\in pr(\mathcal{J})$ be a preferred labeling of $\mathcal{J}$.
Then for all $a\in in(L)$, we have $Sub(a)\subseteq in(L)$.
\end{lemma}
\begin{proof}
Towards a contradiction, suppose that the claim does not hold.
Then there is $a'\in Sub(a)\setminus\{a\}$ \suchthat $a'\not\in in(L)$.
We claim that all attackers of $a'$ are labeled \Out by $L$ and that there does not exist a support chain $\mathcal{C}=\big{\{}(S_0,b_0),...,(S_n,b_n)\big{\}}$ and an attack $(c,b_n)\in{\rightarrow}$ \suchthat $a'\in S_0$, $L(c)\neq \Out$, $S_0\setminus\{a'\}\subseteq in(L)$ and $S_i\setminus\{b_{i-1}\}\subseteq in(L)$ for $0<i\leq n$.
If this holds, then by Proposition~\ref{Prop:Post:PropagatedLabelingIsAdmissible} the propagated labeling $L_{a'}$ is an admissible labeling.
Since $a'\not\in in(L)$ by assumption, this implies $in(L)\subset in(L_{a'})$, contradicting the assumption that $L$ is a preferred labeling.
Now for the attackers of $a'$ and the support chains $\mathcal{C}=\big{\{}(S_0,b_0),...,(S_n,b_n)\big{\}}$ with $a'\in S_0$:
From the way arguments and attacks are constructed in an AS and from the translation of an AS into a JSBAF, it is clear that any attacker of $a'$ is also an attacker of $a$.
Since $a\in in(L)$ by assumption, we can infer that $a$ is legally \In, meaning all attackers of $a$ (and therefore all attackers of $a'$) are labeled \Out.
Next, for the support chains starting at $a'$:
Towards a contradiction, assume that there is a support chain $\mathcal{C}=\big{\{}(S_0,b_0),...,(S_n,b_n)\big{\}}$ and an attack $(c,b_n)\in{\rightarrow}$ \suchthat $a'\in S_0$, $L(c)\neq \Out$, $S_0\setminus\{a'\}\subseteq in(L)$ and $S_i\setminus\{b_{i-1}\}\subseteq in(L)$ for $0<i\leq n$.
Note that, if $c$ is undercutting $b_n$, then it is also undercutting $a'$ or some argument $x\in\bigcup\limits_{0\leq i\leq n} \big(S_i\cap in(L)\big)$.
As we have argued above, if the undercut is directed towards $a'$, then it is also directed towards $a$.
Since $a$ is legally \In but $L(c)\neq\Out$, this assumption leads to a contradiction.
We therefore assume that $c$ is not undercutting $a'$ but some other argument $x\in\bigcup\limits_{0\leq i\leq n} \big(S_i\cap in(L)\big)$.
However, if $c$ is undercutting such an argument $x$, then $L(c)\neq\Out$ implies $x\not\in in(L)$ by admissibility of $L$.
This is also a contradiction, therefore we assume that $c$ is not undercutting $b_n$, which means $c$ must be gen-rebutting $b_n$.
By Proposition~\ref{Prop:Post:OneStepAttacker}, there now exists $c'\in \mathcal{A}$ which is of the form $c':c\rightarrow \lnot\bigwedge ADSub(b_n)^C$.
This means $c'$ also gen-rebuts $b_n$.
For the argument $c$, we have $c\not\prec b_n$ because $(c,b_n)\in{\rightarrow}$ by assumption and because this attack results from a gen-rebut.
Since $DR(c)=DR(c')$, we can infer that $c'\not\prec b_n$ also holds, meaning $(c',b_n)\in{\rightarrow}$.
Note that because $c\in in(L)$, we have $c'\in in(L)$ by admissibility of $L$.
Now, for each $0\leq i\leq n$ take $S'_i=S_i\cap in(L)$ and create the set $S'=\{a\}\cup\bigcup\limits_{0\leq i\leq n}S'_i$.
We have $S'^C\entails\bigwedge S'^C$, therefore there exists a support $(S',d)$ for an argument $d$ with $d^C=\bigwedge S'^C$.
Note that for all arguments $x\in S'$, we have $x\in in(L)$, thus by admissibility of $L$ we also have $d\in in(L)$.
Now consider again the attacker $c'$:
As before, we can use $Sub(a')^C\subseteq Sub(a)^C$ to infer that, because $c'$ is gen-rebutting $b_n$, $c'$ is also gen-rebutting $d$.
Since $L(d)=\In$, we can infer $L(c')=\Out$ by admissibility of $L$.
Now $c$ is legally \Out \wrt $L$ and by admissibility of $L$ we have $L(c)=\Out$.
This contradicts our assumption $L(c)\neq\Out$.
\end{proof}
\subsubsection{Reduced versions of arguments}
With the results from the previous sub-section, we are now equipped to consider the second key-insight mentioned in our proof overview: \emph{Reduced verisons of arguments}.
We will show that -- when considering preferred labelings -- the labels of arguments and their reduced versions somewhat coincide.
\begin{definition}\label{Def:DAOM:ReducedArgument}
Let \JSBAFOne,\JSBAFTwo be two JSBAF which are syntactically disjoint.
Furthermore, let \JSBAFPlus be the union of $\mathcal{J}_1$ and $\mathcal{J}_2$.
Lastly, let $a,a'\in\mathcal{A}_+$ be arguments \suchthat $Atoms(a^C)\subseteq Atoms(AS_i)$ for $i\in\{1,2\}$.
The argument $a'$ is a \emph{reduced version} of $a$ \wrt $AS_i$ iff all of the following are satisfied:
\begin{itemize}
	\item{
	$a'^C=a^C$
	}
	\item{
	$Atoms(a')\subseteq Atoms(AS_i)$
	}
	\item{
	For all $r\in R_s^{AX_i}\cup R_{d_i}$, there is $b\in Sub(a)$ with $TR(b)=r$ iff there is $b'\in Sub(a')$ with $TR(b')=r$.
	}
	\item{
	If $b'\in Sub(a')$ with $b'^C=\phi$, then there is $b\in Sub(a)$ with $b^C=\phi$.
	}
\end{itemize}
\end{definition}
The following is clear from the definition of the atoms of an argument:
\begin{corollary}\label{Cor:Post:ReducedInLeftRight}
Let \JSBAFOne,\JSBAFTwo be two JSBAF which are syntactically disjoint.
Furthermore, let \JSBAFPlus be a union of $\mathcal{J}_1$ and $\mathcal{J}_2$.
Lastly, let $a\in\mathcal{A}_+$ be an argument \suchthat $Atoms(a^C)\subseteq Atoms(AS_i)$ ($i\in\{1,2\}$).
If $a'$ is the reduced version of $a$ \wrt $AS_i$, then $a'\in\mathcal{A}_i$.
\end{corollary}
Next, we introduce a handy way to describe those arguments that we can reach by traversing the support-relation as far back as possible.
We will call these arguments \emph{crucial sub-arguments}.
\begin{definition}
Let \ASOne and \ASTwo be two AS \suchthat $AS_1||AS_2$ and let \ASPlus be a union of $AS_1$ and $AS_2$.
Furthermore, let \JSBAFOne, \JSBAFTwo and \JSBAFPlus be the JSBAFs corresponding to $AS_1$, $AS_2$ and $AS_+$ respectively.
The mapping $CSub:\mathcal{A}_+\rightarrow 2^{\mathcal{A}_+}$ maps arguments to their \emph{crucial sub-arguments} as follows:
\begin{itemize}
	\item{
	If $TR(a)\in R_s^{AX_+}\cup R_{d_+}$, then $CSub(a)=\{a\}$.
	}
	\item{
	If $TR(a)\in R_s^{\vDash}$, then we define:\\
	\begin{align*}
	CSub(a)=	\big{\{}&
	a'\in Sub(a)\mid	TR(a')\in R_s^{AX_+}\cup R_{d_+}\\
	&\textnormal{and }\mathcal{C}=\big{\{}(S_0,b_0),\dots,(S_m,a)\big{\}}\subseteq{\Rightarrow_+}\\
	&\textnormal{is a support chain with }a'\in S_0\big{\}}
	\end{align*}
	}
\end{itemize}
We define the \emph{restriction} of $CSub$ to $AS_i$ ($i\in\{1,2\}$) as:\\
$CSub(a)\!\!\mid\!\!_{AS_i}=\big{\{}a'\in CSub(a)\mid Atoms(a'^C)\subseteq Atoms(AS_i)\big{\}}$.
\end{definition}
We will somewhat abuse the above notation for sets of arguments, \ie for $S\subseteq\mathcal{A}(AS_+)$ we define $CSub(S)=\bigcup\limits_{a\in S} CSub(a)$.
By transitivity of $\entails$ and because the support-relation between arguments corresponds to the application of rules from $R_s^{\vDash}$, the following is easy to see:
\begin{corollary}\label{Cor:Post:CSubImpliesConc}
Let \AS be an AS.
For any $a\in\mathcal{A}(AS)$, we have $CSub(a)^C\entails a^C$.
\end{corollary}
We will now show that, for a consistent argument $a$, there always exists a reduced verison $a'$ of $a$.
We start with the following auxilliary statement regarding the logical languages that we consider:
\begin{proposition}\label{Prop:Post:ReducedArgumentsHelper}
Let $\Gamma=\{\phi_0,...,\phi_m,\phi\}$ and $\Delta=\{\psi_0,...,\psi_n\}$ be syntactically disjoint sets of formulas.
Furthermore, let $\{\psi_0,...\psi_n\}$ be satisfiable.
If we have $\{\phi_0,...,\phi_m,\psi_0,...,\psi_n\}\entails\phi$, then $\{\phi_0,...,\phi_m\}\entails\phi$ also holds.
\end{proposition}
\begin{proof}
Towards a contradiction, assume that the claim does not hold.
Then there exists an interpretation $I_1$ \suchthat $I_1\isModelOf\phi_i$ for all $\phi_i\in\{\phi_0,...,\phi_m\}$ but $I_1\not\vDash_{M}\phi$.
This means $I_1\isModelOf\lnot\phi$.
By assumption, the set $\{\psi_0,...,\psi_n\}$ is satisfiable, \ie there exists an interpretation $I_2$ \suchthat $I_2\isModelOf\psi_j$ for all of these $\psi_j$.
Because $\Gamma$ and $\Delta$ are syntactically disjoint, there now exists an interpretation $I$ \suchthat $I\isModelOf\phi_i$ for all $\phi_i\in\{\phi_0,...,\phi_m\}$, $I\isModelOf\psi_j$ for all $\psi_j\in\{\psi_0,...,\psi_n\}$ and $I\isModelOf\lnot\phi$.
However, because $\{\phi_0,...,\phi_m,\psi_0,...,\psi_n\}\entails\phi$ by assumption, we now also have $I\isModelOf\phi$, a contradiction to $I\isModelOf\lnot\phi$.
\end{proof}
Now for the existence of reduced versions of arguments:
\begin{proposition}\label{Prop:Post:ReducedArgsExistence}
Let \JSBAFOne,\JSBAFTwo be two JSBAF which are syntactically disjoint.
Furthermore, let \JSBAFPlus be a union of $\mathcal{J}_1$ and $\mathcal{J}_2$.
Lastly, let $a\in\mathcal{A}_+$ be a consistent argument \suchthat $Atoms(a^C)\subseteq Atoms(AS_i)$ ($i\in\{1,2\}$).
Then there exists an argument $a'$ \suchthat $a'$ is the reduced version of $a$ \wrt $AS_i$.
\end{proposition}
\begin{proof}
We show the claim by structured incution over the construction of $a$.
Induction start:
Let $a$ be of the form $a:\rightsquigarrow\phi$.
By assumption, $Atoms(a^C)\subseteq Atoms(AS_i)$.
Therefore, $a$ itself is a reduced version of $a$ \wrt $AS_i$.
Induction step:
Let $a$ be of the form $a:a_0,\dots,a_m\rightsquigarrow\phi$.
First, we make a distinction regarding the top-rule of $a$:
We disregard the case $TR(a)\in R_s^{AX_+}$, as it was covered by the induction start.
If $TR(a)\in R_{d}$, we can infer from $Atoms(a^C)\subseteq Atoms(AS_i)$, that $TR(a)\in R_{d_i}$ must hold.
Now for each $0\leq h\leq m$, we have $Atoms(a_h^C)\subseteq Atoms(AS_i)$.
By the induction hypothesis, there exists reduced versions $a'_h$ (\wrt $AS_i$) for each of these arguments $a_h$.
Because $a_h^C=a_h'^C$, we can use these arguments to construct $a'$ as $a':a'_0,\dots a'_m\Rightarrow\phi$.
It is clear that $a'$ satisfies the conditions of Definition~\ref{Def:DAOM:ReducedArgument}.
Now suppose that $TR(a)\in R_s^{\vDash}$ holds and let $X=CSub(a)$.
Note that $X^C\entails \phi$.
We partition the set $X$ according to the top-rules of the arguments $x\in X$:
Let $X=X_i\cup X_j$ with $X_i$ containing precisely those $x$ for which we have $TR(x)\in R_s^{AX_i}\cup R_{d_i}$ and $X_j$ containing precisely those $x$ for which we have $TR(x)\in R_s^{AX_j}\cup R_{d_j}$.
Since $a$ is consistent by assumption, the set $X_j^C$ is satisfiable.
Since $Atoms(a)\subseteq Atoms(AS_i)$ by assumption, we can now use Proposition~\ref{Prop:Post:ReducedArgumentsHelper} to infer $X_i^C\entails\phi$.
Now we construct a set of arguments $Y=\{y_0,\dots,y_l\}$ from $X_i\cup X_j$ as follows:
If $x\in X_i$, then we add a reduced version of $x$ \wrt $AS_i$ to $Y$.
Such a reduced version exists by the induction hypothesis.
On the other hand, if $x\in X_j$, then we collect all arguments $z\in Sub(x)$ for which we have $TR(z)\in R_s^{AX_i}\cup R_{d_i}$ and add their reduced versions \wrt $AS_i$ to $Y$.
Again, these reduced versions exist by the induction hypothesis.
We now claim that the argument $a':y_0,\dots y_l\rightarrow\phi$ exists and satisfies the conditions of Definition~\ref{Def:DAOM:ReducedArgument}.
We first argue for the existence of this argument.
For this, we have to show $Y^C\entails\phi$.
Note that by construction of $Y$, we have $X_i^C\subseteq Y^C$.
As we have argued above, $X_i^C\entails\phi$, thus monotonicity of ${\entails}$ yields $Y^C\entails\phi$ as required.
Next, we argue that $a'$ satisfies the conditions of Definition~\ref{Def:DAOM:ReducedArgument}:
We trivially have $a^C=a'^C$.
By construction of $Y$, we only added reduced arguments \wrt $AS_i$ to $Y$, thus we have $Atoms(a')\subseteq Atoms(AS_i)$.
From the construction of $Y$, it is clear that for all $r\in R_s^{AX_i}\cup R_{d_i}$ for which we have some $b'\in Sub(a')$ with $TR(b')=r$, there is some $b\in Sub(a)$ with $TR(b)=r$.
Next, let $r\in R_s^{AX_i}\cup R_{d_i}$ \suchthat there is some $b\in Sub(a)$ with $TR(b)=r$.
By construction of $X$, we can infer that $b\in Sub(x)$ for some $x\in X_i$ or some $x\in X_j$ must hold.
In the first case, we added a reduced version of $x$ to $Y$.
In the second case, we collected all $z\in Sub(x)$ for which we have $TR(z)\in R_s^{AX_i}\cup R_{d_i}$ and added their reduced versions $z'$ to $Y$.
In both cases, we can infer from the induction hypothesis, that there exists some $b'\in Sub(a')$ with $TR(b')=r$.
Lastly, because we only added reduced versions of arguments to $Y$ and because $a'^C=a^C$, it is clear that if $b'\in Sub(a')$ with $b'^C=\phi$, then there is some $b\in Sub(a)$ \suchthat $b^C=\phi$.
This finishes the proof.
\end{proof}
Now we are ready to show the main result regarding reduced versions of arguments:
For any argument $a$ and its reduced version $a'$, if $a$ is accepted in a preferred labeling, then $a'$ is also accepted and if $a'$ is rejected, then $a$ is also rejected.
\begin{lemma}\label{Lem:Post:ReducedArgumentsLabelImplication}
Let \JSBAFOne,\JSBAFTwo be two JSBAF which are syntactically disjoint.
Furthermore, let \JSBAFPlus be the union of $\mathcal{J}_1$ and $\mathcal{J}_2$.
Let $a\in\mathcal{A}_+$ be a consistent argument \suchthat $Atoms(a^C)\subseteq Atoms(AS_i)$ ($i\in\{1,2\}$) and let $a'$ be a reduced version of $a$ \wrt $AS_i$.
If $L\in pr(\mathcal{J}_+)$ and $a$ is legally \In \wrt $L$, then $L(a')=\In$.
If $L$ is a labeling of $\mathcal{J}_+$ \suchthat for all $x\in\mathcal{A}_+$, $x\in out(L)$ iff $x$ is legally \Out \wrt $L$, then $a'\in out(L)$ implies $a\in out(L)$.
\end{lemma}
\begin{proof}
We begin by showing that any attacker of $a'$ is also an attacker of $a$ and that any support-chain $\mathcal{C'}=\big{\{}(S'_0,b'_0),\dots,(S'_n,b'_n)\big{\}}$ with $a'\in S'_0$ can be \enquote{recreated} with $a$ instead of $a'$.
First, assume that there is $(b,a')\in{\rightarrow_+}$.
Then $b$ either undercuts or gen-rebuts $a'$.
If $b$ undercuts $a'$, then we can use $DR(a')\subseteq DR(a)$ to infer that $b$ also undercuts $a$, meaning $(b,a)\in{\rightarrow_+}$.
If $b$ gen-rebuts $a'$, then we have $b^C=\lnot\bigwedge\Gamma$ with $\Gamma\subseteq Sub(a')^C\subseteq Sub(a)^C$.
Because $(b,a')\in{\rightarrow_+}$, we have $b\not\prec_{ewl} a'$, thus either $b\not\preceq_{ewl} a'$ or $a'\preceq_{ewl} b$ must hold.
If $b\not\preceq_{ewl} a'$, then for all $r_b\in DR(b)$ there exists $r_{a'}\in DR(a')\subseteq DR(a)$, \suchthat $r_b\not\leq_r^+ r_{a'}$.
On the other hand, if $a'\preceq_{ewl} b$, then there exists $r_{a'}\in DR(a')\subseteq DR(a)$ \suchthat for all $r_b\in DR(b)$, $r_{a'}\leq_r^+ r_b$.
In both cases, we can infer $b\not\prec_{ewl} a$ and thus $(b,a)\in{\rightarrow_+}$, as required.
Next, assume that there is a support chain $\mathcal{C'}=\big{\{}(S'_0,b'_0),\dots,(S'_n,b'_n)\big{\}}$ with $a'\in S'_0$.
Let $S'_0=\{c_0,\dots,a',\dots,c_m\}$ and $S_0=\{c_0,\dots,a,\dots,c_m\}$.
We have $S_0'^C\entails b_0'^C$ and because $a'^C=a^C$, we must also have $S_0^C\entails b_0'^C$.
Therefore, there exists the argument $b_0:c_0,\dots,a,\dots,c_m\rightarrow b_0'^C$ and the support $(S_0,b_0)$.
With the same reasoning, we can create a support chain $\mathcal{C}=\big{\{}(S_0,b_0),\dots,(S_n,b_n)\big{\}}$ with $a\in S_0$, $b_k'^C=b_k^C$ and $S_k=(S_k'\setminus\{b_{k-1}'\})\cup\{b_{k-1}\}$ for all $0<k\leq n$.
Now we can use $DR(a')\subseteq DR(a)$ to infer the following:
If $d\in S_0'\setminus\{a'\}$ \suchthat $a'\preceq_{ewl} d$, then there is some $r_a\in DR(a')\subseteq DR(a)$ \suchthat for all $r_d\in DR(d)$, we have $r_a\leq_r^+ r_d$.
This, in turn, implies we also have $a\preceq_{ewl} d$.
With the same reasoning, for any $0<k\leq n$, if we have $d\in S_k'\setminus\{b_{k-1}'\}$ \suchthat $b_{k-1}'\preceq_{ewl} d$, then $b_{k-1}\preceq_{ewl} d$ also holds.
Furthermore, because $a'$ is a reduced version of $a$ by assumption, we have $DR(b_n')\subseteq DR(b_n)$ and $Sub(b_n')^C\subseteq Sub(b_n)^C$.
Now, let $(d,b_n')\in{\rightarrow_+}$ be an attack towards $b_n'$.
Similar to the case $(b,a')\in{\rightarrow_+}$ above, we can infer that $(d,b_n)\in{\rightarrow_+}$ must also hold:
If $d$ undercuts $b_n'$, then it must do so on an argument $x\in CSub(b_n')$ and because $a'$ is a reduced version of $a$, we can infer that $d$ also undercuts $b_n$, meaning $(d,b_n)\in{\rightarrow_+}$.
On the other hand, if $d$ gen-rebuts $b_n'$, then $d^C=\lnot\bigwedge\Gamma$ and $d\not\prec_{ewl} b_n'$, for $\Gamma\subseteq Sub(b_n')^C\subseteq Sub(b_n)^C$.
Analogous to before, we can show that $d\not\prec_{ewl} b_n$ also holds, meaning we again have $(d,b_n)\in{\rightarrow_+}$.
Therefore, we can essentially \enquote{recreate} the support chains $\mathcal{C}'$ starting at $a'$ as support chains $\mathcal{C}$ which start at $a$.
Now for the actual claim:
Suppose first that $L$ is a labeling of $\mathcal{J}_+$ for which we have that, $x\in out(L)$ iff $x$ is legally \Out \wrt $L$.
Assume that $a'$ is legally \Out \wrt $L$.
Then this must be due to an attack or a support.
If there is an attack $(d,a')\in{\rightarrow_+}$ with $d\in in(L)$, then we have argued above that $(d,a)\in{\rightarrow_+}$ also holds, therefore $a$ is legally \Out \wrt $L$.
On the other hand, if there is a support chain $\mathcal{C'}=\big{\{}(S_0',b_0'),\dots,(S_n',b_n')\big{\}}$ \suchthat $a'\in S_0'$ is legally \Out \wrt $L$, then we have argued above that there exists a support chain $\mathcal{C}=\big{\{}(S_0,b_0),\dots,(S_n,b_n)\big{\}}$ with $a\in S_0$.
Note that in the support chain $\mathcal{C}$, all arguments in the supporting sets $S_k$ are the same as in the supporting sets $S_k'$ in the support chain $\mathcal{C}'$, except for the newly created arguments $b_k$.
More precisely, we have $S_0'\setminus\{a'\}=S_0\setminus\{a\}$ and for each $0<k\leq n$ we have $S_k'\setminus\{b_{k-1}'\}=S_k\setminus\{b_{k-1}\}$.
Since $a'$ is legally \Out \wrt $L$ by assumption, we have $L(b_l')=\Out$ for each $0\leq l\leq n$.
Now we can make an inductive argument to show that $L(b_l)=\Out$ also holds:
For the induction start, we note that $L(b_n')=\Out$ due to an attack $(c,b_n')\in{\rightarrow_+}$ with $L(c)=\In$.
We have argued above that this means $(c,b_n)\in{\rightarrow_+}$ also holds, therefore $b_n$ is legally \Out \wrt $L$ and by our assumption this implies $L(b_n)=\Out$.
For the induction step, we first note that, since each $b_k'\in\{b_0',\dots,b_{n-1}'\}$ is labeled \Out, they must be legally \Out \wrt $L$ by our assumption.
From this we can infer that $b_k'\preceq d$ for all $d\in S_{k+1}'\setminus\{b_k'\}$ and by the way JSBAFs are constructed from their AS counterparts, we can infer $b_k'\preceq_{ewl} d$.
As we have argued above, this implies $b_k\preceq_{ewl} d$ for all $d\in S_{k+1}\setminus\{b_k'\}$, therefore $b_k\preceq d$ also holds for these arguments $d$.
Since $S_{k+1}'\setminus\{b_k'\}=S_{k+1}\setminus\{b_k\}$, we now have that each $b_k$ is legally \Out \wrt $L$ and therefore $L(b_k)=\Out$ also holds.
In particular, this means $L(b_0)\in\Out$.
Since $S_0'\setminus\{a'\}=S_0\setminus\{a\}$ and since $a'\preceq d$ for all $d\in S_0'\setminus\{a'\}$, we can now also infer that $a$ is legally \Out \wrt $L$, as required.
Now suppose that $L\in pr(\mathcal{J}_+)$ and that $a$ is legally \In \wrt $L$.
We will show that for the reduced argument $a'$, the prerequisites for Proposition~\ref{Prop:Post:PropagatedLabelingIsAdmissible} are satisfied.
If this holds, then we can use Proposition~\ref{Prop:Post:PropagatedLabelingIsAdmissible} to infer that $L_{a'}$ is an admissible labeling.
Since $L$ is a preferred labeling, we can then infer that $L(a')=\In$ holds as required, since $L(a')\neq\In$ would contradict the maximality (\wrt set-inclusion of accepted arguments) of $L$.
Thus we only need to show that the prerequisites for Proposition~\ref{Prop:Post:PropagatedLabelingIsAdmissible} are satisfied.
First of all, we can use the above proof to infer that $L(a')\neq\Out$ must hold:
If $L(a')=\Out$, then by admissibility of $L$, $a'$ is legally \Out \wrt $L$.
By our argumentation above, this implies $a$ is legally \Out \wrt $L$, contradicting $a\in in(L)$.
Next, for the attackers of $a'$:
Towards a contradiction, suppose that there is an attack $(d,a')\in{\rightarrow}$ with $L(d)\neq\Out$.
We have argued above that this implies $(d,a)\in{\rightarrow_+}$.
Now $a$ is not legally \In \wrt $L$, a contradiction.
Lastly, for the support chains starting at $a'$:
Towards a contradiction, assume that there is a support chain $\mathcal{C'}=\big{\{}(S_0',b_0'),\dots,(S_n',b_n')\big{\}}$ and an attack $(d,b_n')\in{\rightarrow_+}$ with $a'\in S_0'$, $L(d)\neq \Out$, $S_0'\setminus\{a'\}\subseteq in(L)$ and $S_k'\setminus\{b_{k-1}'\}\subseteq in(L)$ for all $0<k\leq n$.
As we have argued above, there now also exists a support chain $\mathcal{C}=\big{\{}(S_0,b_0),\dots,(S_n,b_n)\big{\}}$ and an attack $(d,b_n)$ with $a\in S_0$, $L(d)\neq \Out$, $S_0\setminus\{a\}\subseteq in(L)$ and $S_k\setminus\{b_{k-1}\}\subseteq in(L)$ for all $0<k\leq n$.
Now none of the $b_k\in\{b_0,\dots,b_n\}$ are legally \In \wrt $L$, and by admissiblity of $L$ we can infer that none of them are actually labeled \In.
For the support $(S_0,b_0)$ we now have $L(b_0)\neq\In$ but $S_0\subseteq in(L)$, which means none of the arguments in $S_0$ are legally \In \wrt $L$.
This contradicts the admissibility of $L$ and finishes the proof.
\end{proof}

\subsubsection{Interactions between argumentation systems}
In this section, we will consider the first key insight mentioned in our proof overview: The four \enquote{edge cases} depicted in Illustration~\ref{Illus:Post:ASOneVsASTwoVsASPlus}.
We start with an auxiliary proposition which tells us that, if the conclusion of an argument $a\in \mathcal{A}_i$ contains only atoms of $AS_j$, then either $a$ is inconsistent or its conclusion is a tautology.
\begin{proposition}\label{Prop:Post:StrictConclusionWOutKnownAtomsInconsistent}
Let \ASOne and \ASTwo be two AS \suchthat $AS_1||AS_2$.
Furthermore, let $a\in\mathcal{A}_i$ be an argument with $TR(a)\in R_{s}^{\vDash}$.
If $Atoms(a^C)\subseteq Atoms(AS_j)$, then either $a$ is inconsistent or $\phi$ is a tautology.
\end{proposition}
\begin{proof}
Let $a$ be of the form $a:a_0,\dots,a_m\rightarrow\phi$.
Towards a contradiction, suppose that the claim does not hold, meaning we have $Atoms(\phi)\subseteq Atoms(AS_j)$, $a\in\mathcal{A}_i$ is consistent and $\phi$ is not a tautology.
Let $\Gamma=ADSub(a)^C$.
By Corollary~\ref{Cor:Post:ADSubImpliesConc}, we know that $\Gamma\entails\phi$.
Because $a\in\mathcal{A}_i$, we know that for each $\psi\in\Gamma$, we have $Atoms(\psi)\subseteq Atoms(AS_i)$.
By $AS_1|| AS_2$, we can now infer that $\phi$ and $\Gamma$ must be syntactically disjoint.
Since $a$ is consistent by assumption, there is an interpretation $I_1$ \suchthat $I_1\isModelOf\psi$ for each $\psi\in\Gamma$.
Because we also assumed that $\phi$ is not a tautology, there exists an interpretation $I_2$ \suchthat $I_2\isModelOf -\phi$.
Because $\phi$ and $\Gamma$ are disjoint, there now exists an interpretation $I$ \suchthat $I\isModelOf\psi$ for each $\psi\in\Gamma$, but $I\isModelOf -\phi$ also holds.
This contradicts $\Gamma\entails\phi$, therefore the claim must hold.
\end{proof}
Now we are ready to consider the first case of Illustration~\ref{Illus:Post:ASOneVsASTwoVsASPlus}.
Using the proposition above, we will show that, for any syntactically disjoint $AS_1$ and $AS_2$ as well as their union $AS_+$, if an argument $a\in\mathcal{A}_i$ attacks an argument $b\in\mathcal{A}_j$, then either $a$ or $b$ is rejected by any admissible labeling of the corresponding JSBAFs.
\begin{proposition}\label{Prop:Post:ClownCaseOne}
Let \ASOne, \ASTwo be two $AS$ \suchthat $AS_1|| AS_2$ and let $AS_+$ be their union.
Furthermore, let \JSBAFOne, \JSBAFTwo and \JSBAFPlus be the JSBAFs corresponding to $AS_1$, $AS_2$ and $AS_+$.
Lastly, for $i,j\in\{1,2\}$ with $i\neq j$, let $a\in\mathcal{A}_i$ and $b\in\mathcal{A}_j$ be two arguments.
If $(a,b)\in{\rightarrow_+}$, then there exists $c\in\mathcal{A}_i\cap\mathcal{A}_j\cap\mathcal{A}_+$ \suchthat $c$ is of the form $c:\rightarrow\phi$, $TR(c)\in R_s^{\vDash}$ and either $(c,a)\in{\rightarrow_i}$ or $(c,b)\in{\rightarrow_j}$.
\end{proposition}
\begin{proof}
By $(a,b)\in{\rightarrow_+}$, we know that $a$ is either undercutting or gen-rebutting $b$.
We first show the undercut case.

\Generality, we assume that $a$ undercuts $b$ on $b$ itself.
Let $a$ and $b$ be of the form $a:a_0,\dots,a_m\rightsquigarrow\phi$ and $b:b_0,\dots,b_k\Rightarrow\phi$ with $TR(b)=r\in R_{d_j}$ and $-\phi=n_j(r)$.
Because $TR(b)\in R_{d_j}$, we have $Atoms(\phi)\subseteq Atoms(AS_j)$.
Thus we know that $TR(a)\in R_s^{\vDash}$ must hold, otherwise the syntactic disjointness of $AS_1$ and $AS_2$ would be violated.
By Proposition~\ref{Prop:Post:StrictConclusionWOutKnownAtomsInconsistent}, we can now infer that either $\phi$ is a tautology or $a$ is inconsistent.
If $a$ is inconsistent, then there is $\Gamma\subseteq Sub(a)^C$ \suchthat $\Gamma$ is unsatisfiable.
Now $\lnot\bigwedge\Gamma$ is a tautology and there exists an argument $\overline{a}\in\mathcal{A}_i\cap\mathcal{A}_j\cap\mathcal{A}_+$ of the form $\overline{a}:\rightarrow\lnot\bigwedge\Gamma$ which gen-rebuts $a$.
Since $\overline{a}$ is a strict argument, we have $\overline{a}\not\prec a$, therefore $(\overline{a},a)\in{\rightarrow_i}$, as claimed.
Next, assume that $a$ is consistent.
Then $\phi$ must be a tautology.
This means there exists an argument $\overline{b}\in\mathcal{A}_i\cap\mathcal{A}_j\cap\mathcal{A}_+$ of the form $\overline{b}:\rightarrow\phi$ which undercuts $b$, meaning $(\overline{b},b)\in{\rightarrow}_j$.
This concludes the undercut case.

Now for the gen-rebut case:
If $a$ gen-rebuts $b$, then $b$ is defeasible and there is $\Gamma\subseteq Sub(b)^C$ \suchthat $a^C=\phi=\lnot\bigwedge\Gamma$.
We make a distinction \wrt $TR(a)$.
In the first case, assume that $TR(a)\in R_{d_i}\cup R_s^{AX_i}$.
Then for all $\psi\in\Gamma$, we have $Atoms(\psi)\subseteq Atoms(AS_i)$.
Thus for all $b'\in Sub(b)$ with $b'^C=\psi'\in\Gamma$, we must have $TR(b')\in R_s^{\vDash}$, otherwise syntactic disjointness of $AS_1$ and $AS_2$ would be violated again.
Now we can use Proposition~\ref{Prop:Post:StrictConclusionWOutKnownAtomsInconsistent} again to infer that for each of these sub-arguments of $b$, either $b'$ is inconsistent or $\psi'$ a tautology.
If any of the $b'$ are inconsistent, then $b$ is also inconsistent and we can construct the argument $\overline{b}\in\mathcal{A}_i\cap\mathcal{A}_j\cap\mathcal{A}_+$ as before.
We therefore assume that all $b'$ are consistent.
But then all $\psi'\in\Gamma$ are tautologies, which means $\bigwedge\Gamma$ is also a tautology.
Now $\lnot\bigwedge\Gamma=a^C$ is unsatisfiable, thus $a$ is inconsistent and we can construct the argument $\overline{a}\in\mathcal{A}_i\cap\mathcal{A}_j\cap\mathcal{A}_+$ as before.
Now for the second case, \ie $TR(a)\in R_s^{\vDash}$.
Let $\Delta_a=ADSub(a)^C$ and $\Delta_b=ADSub(b)^C$ be the conclusions of the axiomatic and defeasible sub-arguments of $a$ and $b$ respectively.
By Proposition~\ref{Prop:Post:ADSubImpliesSubConcs}, we have $\Delta_a\entails \lnot\bigwedge\Gamma$ and $\Delta_b\entails \bigwedge\Gamma$.
Because $AS_1$ and $AS_2$ are syntactically disjoint, $\Delta_a$ and $\Delta_b$ must also be syntactically disjoint.
Now, towards a contradiction, assume that $a$ and $b$ are both consistent.
Then in particular $\Delta_a$ and $\Delta_b$ are satisfiable, \ie there exists interpretations $I_a,I_b$ \suchthat $I_a\isModelOf\phi_a$ for all $\phi_a\in\Delta_a$ and $I_b\isModelOf\phi_b$ for all $\phi_b\in\Delta_b$.
Because $\Delta_a$ and $\Delta_b$ are syntactically disjoint, there now exists an interpretation $I$ \suchthat $I\isModelOf\phi$ for all $\phi\in\Delta_a\cup\Delta_b$.
Because we have $\Delta_a\entails \lnot\bigwedge\Gamma$ and $\Delta_b\entails\bigwedge\Gamma$, we can now infer that $I\isModelOf\lnot\bigwedge\Gamma$ and $I\isModelOf \bigwedge\Gamma$ must hold, a contradiction.
We conclude that either $a$ or $b$ be need to be inconsistent and we can again construct one of the arguments $\overline{a}$, $\overline{b}$ as before.
\end{proof}
Before we continue with our next results, we need one more definition, which will be used to adequately describe those arguments that correspond to the argument $b$ in Illustration~\ref{Illus:Post:ASOneVsASTwoVsASPlus}.
More precisely, we use the term \emph{minimal $\mathcal{A}_+$-argument} for those arguments $b\in\mathcal{A}_+\setminus(\mathcal{A}_i\cup\mathcal{A}_j)$, which have crucial sub-arguments in $\mathcal{A}_i$ and $\mathcal{A}_j$, while not having any crucial sub-arguments in $\mathcal{A}_+\setminus(\mathcal{A}_i\cup\mathcal{A}_j)$.
\begin{definition}\label{Def:ASPIC:MinimalASPlusArguments}
Let \ASOne, \ASTwo be two $AS$ \suchthat $AS_1|| AS_2$ and let $AS_+$ be their union.
Furthermore, let $a\in \mathcal{A}_+$ be an argument.
We say that $a$ is a \emph{minimal $\mathcal{A}_+$-argument}, iff $CSub(a)\!\!\mid\!\!_{AS_i}\neq\emptyset$, $CSub(a)\!\!\mid\!\!_{AS_j}\neq\emptyset$ and $CSub(a)\subseteq \mathcal{A}_i\cup \mathcal{A}_j$.
\end{definition}
Now we prove three propositions which, respectively, cover cases two, three and four of Illustration~\ref{Illus:Post:ASOneVsASTwoVsASPlus}.
For now, we only show that if these cases exist, then we can infer the existence of arguments and attacks which are wholly contained in either $\mathcal{J}_1$ or $\mathcal{J}_2$.
Afterwards, we will combine all these cases to infer that a preferred labeling of $\mathcal{J}_i$ can be \enquote{extended} to a preferred labeling of $\mathcal{J}_+$ and a preferred labeling of $\mathcal{J}_+$ can be \enquote{reduced} to a preferred labeling of $\mathcal{J}_i$.
\begin{proposition}\label{Prop:Post:ClownCaseTwo}
Let \ASOne, \ASTwo be two $AS$ \suchthat $AS_1|| AS_2$ and let $AS_+$ be their union.
Let \JSBAFOne, \JSBAFTwo and \JSBAFPlus be the JSBAFs corresponding to $AS_1$, $AS_2$ and $AS_+$ respectively.
Furthermore, let $a,b\in \mathcal{A}_+$ be arguments \suchthat $a$ is a consistent minimal $\mathcal{A}_+$-argument.
Lastly, let $(a,b)\in {\rightarrow_+}$.
If $b\in \mathcal{A}_i$, then there exists an argument $a'\in \mathcal{A}_i$ \suchthat $(a',b)\in {\rightarrow_i}$ and $CSub(a)\!\!\mid\!\!_{AS_i}=CSub(a')$.
\end{proposition}
\begin{proof}
Throughout this proof, let $\Gamma_i=CSub(a)\!\!\mid\!\!_{AS_i}$ and $\Gamma_j=CSub(a)\!\!\mid\!\!_{AS_j}$.
Note that this means $\Gamma_i\subseteq \mathcal{A}_i$ and $\Gamma_j\subseteq \mathcal{A}_j$ because $a$ is a minimal $\mathcal{A}_+$-argument.
First, assume that $(a,b)\in {\rightarrow_+}$ is the result of an undercut.
Then we know $Atoms(a^C)\subseteq Atoms(AS_i)$.
Furthermore, $Atoms(\Gamma_i^C)\subseteq Atoms(AS_i)$ and $Atoms(\Gamma_j^C)\subseteq Atoms(AS_j)$, \ie $\Gamma_i^C\cup \{a^C\}$ and $\Gamma_j^C$ are syntactically disjoint sets of formulas.
By assumption, we also know that $\Gamma_i^C\cup \{a^C\}$ as well as $\Gamma_j^C$ are consistent sets of formulas.
Now we can use Proposition~\ref{Prop:Post:ReducedArgumentsHelper} to infer $\Gamma_i^C\entails a^C$.
Thus we can construct the argument $a'\in\mathcal{A}_i$ as follows:
$a':a_0,\dots,a_m\rightarrow a^C$, where $\{a_0,...,a_m\}=\Gamma_i$.
Clearly, $a'$ undercuts $b$, therefore $(a',b)\in{\rightarrow_i}$ as required.
Next, assume that $(a,b)\in {\rightarrow_+}$ is the result of a gen-rebut.
Let $\Delta=\{\phi_0,...,\phi_m\}$ and $a^C=\lnot\bigwedge\Delta$.
By Proposition~\ref{Prop:Post:ADSubImpliesSubConcs}, we know that $ADSub(b)^C\entails\bigwedge\Delta$.
We claim that $\Gamma_i^C\entails\lnot\bigwedge ADSub(b)^C$ also holds.
Towards a contradiction, assume that this is not the case.
Then there exists an interpretation $I_1$ \suchthat $I_1\isModelOf\Gamma_i^C$ and $I_1\isModelOf\bigwedge ADSub(b)^C$.
By assumption $a$ is consistent, thus there also exists an interpretation $I_2$ \suchthat $I_2\isModelOf\Gamma_j^C$.
Note that $Atoms(ADSub(b)^C)\cup Atoms(\Gamma_i^C)\subseteq Atoms(AS_i)$ since $b\in\mathcal{A}_i$, while $Atoms(\Gamma_j^C)\subseteq Atoms(AS_j)$.
Now there exists an interpretation $I$ \suchthat $I\isModelOf\Gamma_i^C\cup\Gamma_j^C$, \ie $I\isModelOf\lnot\bigwedge\Delta$ and $I\isModelOf ADSub(b)^C$, \ie $I\isModelOf\bigwedge\Delta$.
Obviously, this is a contradiction.
We infer $\Gamma_i^C\entails\lnot\bigwedge ADSub(b)^C$.
With this, we can construct the argument $a'\in\mathcal{A}_i$ as follows:
$a':a_0,\dots,a_m\rightarrow\lnot\bigwedge ADSub(b)^C$, where $\{a_0,\dots,a_m\}=\Gamma_i$.
Because $(a,b)\in{\rightarrow_+}$, we know that $a\not\prec b$.
Thus we must have either $a\not\preceq b$ or $b\preceq a$.
In the first case, we can infer that for all $r_a\in DR(a)\supseteq DR(a')$, there exists $r_b\in DR(b)$ \suchthat $r_a\not\leq r_b$.
In the second case, we can infer that there exists $r_b\in DR(b)$ \suchthat for all $r_a\in DR(a)\supseteq DR(a')$, we have $r_b\leq r_a$.
In both cases we can infer $a'\not\prec b$, therefore $(a',b)\in{\rightarrow_i}$ as required.
\end{proof}
\begin{proposition}\label{Prop:Post:ClownCaseThree}
Let \ASOne, \ASTwo be two $AS$ \suchthat $AS_1|| AS_2$ and let $AS_+$ be their union.
Let \JSBAFOne, \JSBAFTwo and \JSBAFPlus be the JSBAFs corresponding to $AS_1$, $AS_2$ and $AS_+$ respectively.
Furthermore, let $a,b\in \mathcal{A}_+$ be consistent arguments, \suchthat $a\in\mathcal{A}_i$ and $b$ is a minimal $AS_+$-argument ($i\in\{1,2\}$).
Lastly, let $(a,b)\in {\rightarrow_+}$.
Then either there exist arguments $a',b'\in \mathcal{A}_i$, \suchthat $CSub(a')=CSub(a)$, $CSub(b')=CSub(b)\!\!\mid\!\!_{AS_i}$ and $(a',b')\in {\rightarrow_i}$, or there exists arguments $a',b'\in \mathcal{A}_i$ \suchthat $CSub(a')=CSub(a)$, $CSub(b')=ADSub(b)\!\!\mid\!\!_{AS_i}$ and either $(a',b')\in{\rightarrow_i}$ or $(b',a')\in{\rightarrow_i}$.
\end{proposition}
\begin{proof}
First, assume that $(a,b)\in {\rightarrow_+}$ is the result of an undercut.
Then this undercut must be directed towards an argument in $CSub(b)\!\!\mid\!\!_{AS_i}$ or in $CSub(b)\!\!\mid\!\!_{AS_j}$.
In the second case, we have $Atoms(a^C)\subseteq Atoms(AS_j)$.
With the same argumentation as in the proof of Proposition~\ref{Prop:Post:ClownCaseOne}, we can show that this implies either $a$ or $b$ need to be inconsistent.
Because this contradicts our assumptions, we infer that, if $a$ undercuts $b$, then it must do so on some $b'\in CSub(B)\!\!\mid\!\!_{AS_i}$.
Now, for $CSub(b)\!\!\mid\!\!_{AS_i}=\{b_0,...,b_m\}$, let $b'\in\mathcal{A}_i$ be of the form $b':b_0,...,b_m\rightarrow\bigwedge\{b_0,...,b_m\}^C$.
Obviously we have $CSub(b')=CSub(b)\!\!\mid\!\!_{AS_i}$.
Now $a$ undercuts $b'$, meaning $(a,b')\in{\rightarrow_i}$ as required.
Now for the gen-rebut case:
Let $a^C=\lnot\bigwedge\Gamma$.
We claim that $CSub(a)^C\entails\lnot\bigwedge\big(ADSub(b)\!\!\mid\!\!_{AS_i}\big)^C$.
Towards a contradiction, assume that this is not true.
Then there exists an interpretation $I_1$ \suchthat $I_1\isModelOf CSub(a)^C$ and $I_1\isModelOf\bigwedge\big(ADSub(b)\!\!\mid\!\!_{AS_i}\big)^C$.
Because $b$ was consistent by assumption, there also exists an interpretation $I_2$ for which we have $I_2\isModelOf\big(ADSub(b)\!\!\mid\!\!_{AS_j}\big)^C$.
Note that $CSub(a)^C\cup \big(ADSub(b)\!\!\mid\!\!_{AS_i}\big)^C$ and $\big(ADSub(b)\!\!\mid\!\!_{AS_j}\big)^C$ are syntactically disjoint sets of formulas.
Now there exists an interpretation $I$ for which we have both $I\isModelOf CSub(a)^C$ and $I\isModelOf\big(ADSub(b)\!\!\mid\!\!_{AS_i}\big)^C\cup\big(ADSub(b)\!\!\mid\!\!_{AS_j}\big)^C$, \ie $I\isModelOf ADSub(b)^C$.
By Corollary~\ref{Cor:Post:CSubImpliesConc} we can infer that $I\isModelOf\lnot\bigwedge\Gamma$ and by Proposition~\ref{Prop:Post:ADSubImpliesSubConcs} we can infer $I\isModelOf\bigwedge\Gamma$.
Obviously, this is a contradiction, therefore $CSub(a)^C\entails\lnot\bigwedge\big(ADSub(b)\!\!\mid\!\!_{AS_i}\big)^C$ must hold.
Now we construct the arguments $a',b',b''\in\mathcal{A}_i$ as follows:
For $CSub(a)=\{a_0,\dots,a_m\}$, let $a':a_0,\dots,a_m\rightarrow\lnot\bigwedge\big(ADSub(b)\!\!\mid\!\!_{AS_i}\big)^C$, and for $ADSub(b)\!\!\mid\!\!_{AS_i}=\{b_0,\dots,b_n\}$, let $b':b_0,\dots,b_n\rightarrow\bigwedge\{b_0,\dots,b_n\}^C$ and $b'':b'\rightarrow\lnot\lnot\bigwedge\{b_0,\dots,b_n\}^C$.
Clearly, $a'$ gen-rebuts $b'$ and $b''$ gen-rebuts $a'$.
Now, if $a'\not\prec b'$, then $(a',b')\in{\rightarrow_i}$ and the claim holds.
Thus we assume that $a'\prec b'$.
Because $DR(b')=DR(b'')$, this also implies $a'\prec b''$, \ie $a'\preceq b''$ and $b''\not\preceq a'$.
From this we infer $b''\not\prec a'$, therefore $(b'',a')\in{\rightarrow_i}$ as required.
\end{proof}
\begin{proposition}\label{Prop:Post:ClownCaseFour}
Let \ASOne, \ASTwo be two $AS$ \suchthat $AS_1|| AS_2$ and let $AS_+$ be their union.
Let \JSBAFOne, \JSBAFTwo and \JSBAFPlus be the JSBAFs corresponding to $AS_1$, $AS_2$ and $AS_+$ respectively.
Furthermore, let $a,b\in \mathcal{A}_+$ be arguments, \suchthat both $a$ and $b$ are consistent minimal $\mathcal{A}_+$-arguments.
Lastly, let $(a,b)\in {\rightarrow_+}$.
Then one of the following four cases holds:
\begin{enumerate}
	\item{
	There are $a',b'\in \mathcal{A}_i$ \suchthat $CSub(a)\!\!\mid\!\!_{AS_i}=CSub(a')$, $CSub(b)\!\!\mid\!\!_{AS_i}=CSub(b')$ and $(a',b')\in {\rightarrow_i}$.
	}
	\item{
	There are $a',b'\in \mathcal{A}_j$ \suchthat $CSub(a)\!\!\mid\!\!_{AS_j}=CSub(a')$, $CSub(b)\!\!\mid\!\!_{AS_j}=CSub(b')$ and $(a',b')\in {\rightarrow_j}$.
	}
	\item{
	There are $a',b'\in\mathcal{A}_i$ \suchthat $CSub(a)\!\!\mid\!\!_{AS_i}=CSub(a')$, $ADSub(b)\!\!\mid\!\!_{AS_i}=CSub(b')$ and either $(a',b')\in {\rightarrow_i}$ or $(b',a')\in{\rightarrow_i}$.
	}
	\item{
	There are $a',b'\in \mathcal{A}_j$ \suchthat $CSub(a)\!\!\mid\!\!_{AS_j}=CSub(a')$, $ADSub(b)\!\!\mid\!\!_{AS_j}=CSub(b')$ and either $(a',b')\in {\rightarrow_j}$ or $(b',a')\in{\rightarrow_j}$.
	}
\end{enumerate}
\end{proposition}
\begin{proof}
Throughout this proof, let $\Gamma_{a_i}=CSub(a)\!\!\mid\!\!_{AS_i}$, $\Gamma_{a_j}=CSub(a)\!\!\mid\!\!_{AS_j}$ and let $\Gamma_{b_i}=CSub(b)\!\!\mid\!\!_{AS_i}$, $\Gamma_{b_j}=CSub(b)\!\!\mid\!\!_{AS_j}$.
First, assume that $(a,b)\in {\rightarrow_+}$ is the result of an undercut.
Then either $Atoms(a^C)\subseteq Atoms(AS_i)$ or $Atoms(a^C)\subseteq Atoms(AS_j)$.
Because $CSub(b)=CSub(b)\!\!\mid\!\!_{AS_i}\cup CSub(b)\!\!\mid_{AS_j}$, we know that this undercut needs to be directed towards an argument which is either in $\Gamma_{b_i}$ or in $\Gamma_{b_j}$.
Note that we have $\Gamma_{b_i}\subseteq \mathcal{A}_i$ and $\Gamma_{b_j}\subseteq\mathcal{A}_j$.
Now we can use Proposition~\ref{Prop:Post:ReducedArgumentsHelper} to infer that, if $Atoms(a^C)\subseteq Atoms(AS_i)$, then $\Gamma_{a_i}^C\entails a^C$ and if $Atoms(a^C)\subseteq Atoms(AS_j)$, then $\Gamma_{a_j}^C\entails a^C$.
In the first case, we can construct the arguments $a',b'\in\mathcal{A}_i$ which are of the form $a':a_0,\dots,a_m\rightarrow a^C$ and $b':b_0,\dots,b_n\rightarrow\bigwedge\{b_0,\dots,b_n\}^C$, where $\{a_0,\dots,a_m\}=\Gamma_{a_i}$ and $\{b_0,\dots,b_n\}=\Gamma_{b_i}$.
Then $CSub(a)\!\!\mid\!\!_{AS_i}=CSub(a')$, $CSub(b)\!\!\mid\!\!_{AS_i}=CSub(b')$ and $(a',b')\in{\rightarrow_i}$, as required by item one of the claim. 
In the second case, we can construct the arguments $a'',b''\in\mathcal{A}_j$ which are of the form $a'':a_0,\dots,a_m\rightarrow a^C$ and $b'':b_0,\dots,b_n\rightarrow\bigwedge\{b_0,...,b_n\}^C$, where $\{a_0,\dots,a_m\}=\Gamma_{a_j}$ and $\{b_0,\dots,b_n\}=\Gamma_{b_j}$.
Then $CSub(a)\!\!\mid\!\!_{AS_j}=CSub(a'')$, $CSub(b)\!\!\mid\!\!_{AS_j}=CSub(b'')$ and $(a'',b'')\in {\rightarrow_j}$, as required by item two of the claim.
Now for the gen-rebut case.
Let $\Delta=\{\phi_0,...,\phi_m\}$ and $a^C=\lnot\bigwedge\Delta$.
We claim that either $\Gamma_{a_i}^C\entails\lnot\bigwedge \big(ADSub(b)\!\!\mid\!\!_{AS_i}\big)^C$ or $\Gamma_{a_j}^C\entails\lnot\bigwedge \big(ADSub(b)\!\!\mid\!\!_{AS_j}\big)^C$ must hold.
Towards a contradiction, assume that this is not the case.
Then there exists two interpretations, $I_1,I_2$ \suchthat $I_1\isModelOf\Gamma_{a_i}^C$, $I_1\isModelOf\bigwedge \big(ADSub(b)\!\!\mid\!\!_{AS_i}\big)^C$ and $I_2\isModelOf\Gamma_{a_j}^C$, $I_2\isModelOf\bigwedge \big(ADSub(b)\!\!\mid\!\!_{AS_j}\big)^C$.
Note that $\Gamma_{a_i}^C\cup \big(ADSub(b)\!\!\mid\!\!_{AS_i}\big)^C$ and $\Gamma_{a_j}^C\cup \big(ADSub(b)\!\!\mid\!\!_{AS_j}\big)^C$ are syntactically disjoint sets of formulas.
Now there exists an interpretation $I$ \suchthat $I\isModelOf \Gamma_{a_i}\cup\Gamma_{a_j}$, \ie $I\isModelOf\lnot\bigwedge\Delta$ and $I\isModelOf \big(ADSub(b)\!\!\mid\!\!_{AS_i}\big)^C\cup \big(ADSub(b)\!\!\mid\!\!_{AS_j}\big)^C$, \ie $I\isModelOf\bigwedge\Delta$.
Obviously, this is a contradiction.
We infer that either $\Gamma_{a_i}^C\entails\lnot\bigwedge \big(ADSub(b)\!\!\mid\!\!_{AS_i}\big)^C$ or $\Gamma_{a_j}^C\entails\lnot\bigwedge \big(ADSub(b)\!\!\mid\!\!_{AS_j}\big)^C$ must hold.
In the first case, we can construct the arguments $a',b',b''\in\mathcal{A}_i$ which are of the form $a':a_0,\dots,a_m\rightarrow\lnot\bigwedge \{b_0,\dots,b_n\}^C$, $b':b_0,\dots,b_n\rightarrow\bigwedge\{b_0,\dots,b_n\}^C$ and $b'':b'\rightarrow\lnot\lnot\bigwedge\{b_0,\dots,b_n\}^C$, where $\{a_0,\dots,a_m\}=\Gamma_{a_i}$ and $\{b_0,\dots,b_n\}=ADSub(b)\!\!\mid\!\!_{AS_i}$.
Then $a'$ gen-rebuts $b'$ and $b''$ gen-rebuts $a'$.
Analogous to the proof of the gen-rebut case in Proposition~\ref{Prop:Post:ClownCaseTwo}, we either have $(a',b')\in {\rightarrow_i}$ (if $a'\not\prec b'$), or we have $(b'',a')\in{\rightarrow_i}$ (if $a'\prec b'$).
Either way, item three of the claim is satisfied.
In the second case, we can construct the arguments $a',b',b''\in\mathcal{A}_j$ which are of the form $a':a_0,\dots,a_m\rightarrow\lnot\bigwedge\{b_0,\dots,b_n\}^C$, $b':b_0,\dots,b_n\rightarrow\bigwedge\{b_0,\dots,b_n\}^C$ and $b'':b'\rightarrow\lnot\lnot\bigwedge\{b_0,\dots,b_n\}^C$, where $\{a_0,\dots,a_m\}=\Gamma_{a_j}$ and $\{b_0,\dots,b_n\}=ADSub(b)\!\!\mid\!\!_{AS_j}$.
Again, we either have $(a',b')\in{\rightarrow_j}$ or $(b'',a')\in{\rightarrow_j}$, as required by item four of the claim.
\end{proof}
\subsubsection{Non-Interference postulate}
Finally, we are ready to show that, for any two syntactically disjoint AS $AS_1$, $AS_2$ and their union $AS_+$, for the corresponding JSBAFs $\mathcal{J}_1$, $\mathcal{J}_2$ and $\mathcal{J}_+$, any preferred labeling $L_1$ of $\mathcal{J}_1$ can be turned into a preferred labeling $L_+$ of $\mathcal{J}_+$ and any preferred labeling $L_+$ of $\mathcal{J}_+$ can be reduced to a preferred labeling $L_1$ of $\mathcal{J}_1$.
For this, we will show three sub-results:
First, we take a preferred labeling $L$ of $\mathcal{J}_1$ and show that by \enquote{adding} to $in(L)$ only the strict arguments of $\mathcal{J}_2$, we can turn $L$ into an admissible labeling of $\mathcal{J}_+$.
Next, we will show that a preferred labeling $L_+$ of $\mathcal{J}_+$ can be turned into an admissible labeling $L_1$ of $\mathcal{J}_1$ by simply \enquote{removing} all arguments of $\mathcal{A}_+\setminus\mathcal{A}_1$ from $in(L_+)$, $out(L_+)$ and $undec(L_+)$.
Lastly, we will then show that, if there is any admissible labeling $L_1'$ of $\mathcal{J}_1$ \suchthat $in(L_1) \subseteq in(L_1')$, we can \enquote{add back} the arguments of $in(L_+)$ that were removed in the previous step and create another admissible labeling from that.
By combining the first and second of these sub-results, we will be able to show that $\mathcal{C}_{\sigma}(AS_1)\!\!\mid\!\!_{Atoms(AS_1)}\subseteq\mathcal{C}_{\sigma}(AS_+)\!\!\mid\!\!_{Atoms(AS_1)}$ holds, while combining the second and third of these sub-results will prove that $\mathcal{C}_{\sigma}(AS_1)\!\!\mid\!\!_{Atoms(AS_1)}\supseteq\mathcal{C}_{\sigma}(AS_+)\!\!\mid\!\!_{Atoms(AS_1)}$ holds.
The proof for turning $L_1$ into $L_+$ works by first creating an admissible labeling $L_+^{adm}$ of $\mathcal{J}_+$ \suchthat $in(L_1)\subseteq in(L_+^{adm})$.
To do this, we use the following definition:
\begin{definition}\label{Def:Post:CombinedMinimalLabeling}
Let \ASOne, \ASTwo be two $AS$ \suchthat $AS_1|| AS_2$ and let $AS_+$ be their union.
Furthermore, let \JSBAFOne, \JSBAFTwo and \JSBAFPlus be the JSBAF's corresponding to $AS_1$, $AS_2$ and $AS_+$ respectively.
Lastly, for $i,j\in\{1,2\}$ with $i\neq j$, let $L_i\in pr(\mathcal{J}_i)$ and let $SIM_{\mathcal{J}_j}$ be the strict including minimal labeling of $\mathcal{J}_j$.
We construct the \emph{combined minimal labeling} of $L_i$ and $SIM_{\mathcal{J}_j}$, denoted $L_+^{min}$, as follows:
\begin{align*}
	I_0&
	=in(L_i)\cup in(SIM_{\mathcal{J}_j})\\
	I_{k+1}&
	=I_k\cup\{a\in\mathcal{A}_+\mid\exists (S,a)\in{\Rightarrow}_+,S\subseteq I_k\}\\
	in(L_+^{min})&
	=\bigcup\limits_{k\geq 0} I_k\\
	O_0&
	=\{a\in\mathcal{A}_+\mid\exists(b,a)\in{\rightarrow}_+,b\in in(L_+^{min})\}\\
	O_{k+1}&
	=O_k\cup\{a\in\mathcal{A}_+\mid\exists(S,b)\in{\Rightarrow}_+,a\in S,\\
	&\hspace{2.93cm}S\setminus\{a\}\subseteq in(L_+^{min}),\\
	&\hspace{2.93cm}b\in O_k\}\\
	out(L_+^{min})&
	=\bigcup\limits_{k\geq 0} O_k\\
	undec(L_+^{min})&
	=\mathcal{A}_+\setminus\big{(}in(L_+^{min})\cup out(L_+^{min})\big{)}
\end{align*}	
\end{definition}
Note that the only arguments in $\mathcal{A}_+\setminus(\mathcal{A}_i\cup\mathcal{A}_j)$ that are accepted in $L_+^{min}$, are minimal $\mathcal{A}_+$-arguments.
Furthermore, no \enquote{new} arguments in $\mathcal{A}_i$ and $\mathcal{A}_j$ are accepted by $L_+^{min}$, \ie if $a\in in(L_+^{min})\cap\mathcal{A}_i$, then $a\in in(L_i)$ and if $a\in in(L_+^{min})\cap\mathcal{A}_j$, then $a\in SIM_{\mathcal{J}_j}$.
We now prove that $L_+^{min}$ is indeed an admissible labeling.
Later, we will use this result to show that any preferred labeling of $\mathcal{J}_i$ can be turned into a preferred labeling of $\mathcal{J}_+$.
As usual, we have divided this proof into several parts to make it more accessible.
\begin{proposition}\label{Prop:Post:CombinedMinimalLabelingIsLabeling}
Let \ASOne, \ASTwo be two $AS$ \suchthat $AS_1|| AS_2$ and let $AS_+$ be their union.
Furthermore, let \JSBAFOne, \JSBAFTwo and \JSBAFPlus be the JSBAF's corresponding to $AS_1$, $AS_2$ and $AS_+$ respectively.
Lastly, for $i,j\in\{1,2\}$ with $i\neq j$, for any preferred labeling $L_i\in pr(\mathcal{J}_i)$ and for $SIM_{\mathcal{J}_j}$ being the strict including minimal labeling of $\mathcal{J}_j$, let $L_+^{min}$ be the combined minimal labeling of $L_i$ and $SIM_{\mathcal{J}_j}$.
Then $L_+^{min}$ is a labeling.
\end{proposition}
\begin{proof}
We begin by arguing that $L_+^{min}$ is a labeling.
Similar to the proof of Proposition~\ref{Prop:Post:PropagatedLabelingIsLabeling}, the actual proof proceeds by induction over the construction of $out(L_+^{min})$.
However, since it is clear from the construction of $L_+^{min}$ that we only need to ensure $in(L_+^{min})\cap out(L_+^{min})=\emptyset$ and since the induction step of this proof is trivial, we focus here on the induction start.
That is, we argue why there cannot be some $a\in in(L_+^{min})\cap O_0$:
Towards a contradiction, suppose that this does not hold, \ie there are arguments $a,b\in in(L_+^{min})$ \suchthat $(b,a)\in{\rightarrow}_+$.
We first note that both $a$ and $b$ are consistent arguments.
If they are contained in $in(L_i)$ or in $in(SIM_{\mathcal{J}_j})$, then this is easy to see because these are admissible labelings.
If they are contained in $\mathcal{A}_+\setminus(\mathcal{A}_i\cup\mathcal{A}_j)$, then we can make a simple model-theoretic argument to show the consistency:
For $c\in\{a,b\}$, let $c\in\mathcal{A}_+\setminus(\mathcal{A}_i\cup\mathcal{A}_j)$.
Towards a contradiction, assume that $c$ is inconsistent, \ie there is $\Gamma\subseteq Sub(c)^C$ \suchthat $\Gamma$ is unsatisfiable.
Take $ADSub(c)\!\!\mid\!\!_{AS_i}$ and $ADSub(c)\!\!\mid\!\!_{AS_j}$.
Note that these sets of arguments are contained in $in(L_i)$ and $in(SIM_{\mathcal{J}_j})$ respectively.
Because $L_i$ and $SIM_{\mathcal{J}_j}$ are admissible labelings, we can now infer that both $\big(ADSub(c)\!\!\mid\!\!_{AS_i}\big)^C$ and $\big(ADSub(c)\!\!\mid\!\!_{AS_j}\big)^C$ are satisfiable, \ie there exists interpretations $I_1\isModelOf\bigwedge\big(ADSub(c)\!\!\mid\!\!_{AS_i}\big)^C$ and $I_2\isModelOf\bigwedge\big(ADSub(c)\!\!\mid\!\!_{AS_j}\big)^C$.
Note that these are sets of syntactically disjoint formulas.
Now there exists an interpretation $I$ \suchthat $I\isModelOf\big(ADSub(c)\!\!\mid\!\!_{AS_i}\big)^C$ and $I\isModelOf\big(ADSub(c)\!\!\mid\!\!_{AS_i}\big)^C$.
By Proposition~\ref{Prop:Post:ADSubImpliesSubConcs} we can now infer that $\Gamma$ is satisfied by $I$, contradicting our assumptions.
Now for the actual proof:
Because we used the admissible labelings $L_i$ and $SIM_{\mathcal{J}_j}$ as a starting point for the construction of $in(L_+^{min})$, we know that we cannot have $a,b\in\mathcal{A}_i$ or $a,b\in\mathcal{A}_j$.
By Proposition~\ref{Prop:Post:ClownCaseOne} (which corresponds to case one of Illustration~\ref{Illus:Post:ASOneVsASTwoVsASPlus}), we can also infer that we cannot have $a\in\mathcal{A}_i$ and $b\in\mathcal{A}_j$ or $a\in\mathcal{A}_j$ and $b\in\mathcal{A}_i$, as this would mean either $a$ or $b$ is attacked by a strict argument and therefore not accepted by $L_i$ or $SIM_{\mathcal{J}_j}$ respectively.
This means at least one of the arguments $a$ and $b$ needs to be from $\mathcal{A}_+\setminus(\mathcal{A}_i\cup\mathcal{A}_j)$.
Suppose first that we have $a\in\mathcal{A}_i$ and $b\in\mathcal{A}_+\setminus(\mathcal{A}_i\cup\mathcal{A}_j)$ (this corresponds to case two of Illustration~\ref{Illus:Post:ASOneVsASTwoVsASPlus}).
Then $a\in in(L_i)$ by construction of $L_+^{min}$.
From Proposition~\ref{Prop:Post:ClownCaseTwo}, we can infer that there exists an argument $b'\in \mathcal{A}_i$ \suchthat $(b',a)\in {\rightarrow_i}$ and $CSub(b)\!\!\mid\!\!_{AS_i}=CSub(b')$.
By construction of $L_+^{min}$, we have $CSub(b')\subseteq in(L_i)$ and by admissibility of $L_i$, we can infer that $b'\in in(L_i)$ must also hold.
Now $(b',a)\in{\rightarrow_i}$ contradicts the admissibility of $L_i$.
The same argument can be made for the case $a\in\mathcal{A}_j$ and $b\in\mathcal{A}_+\setminus(\mathcal{A}_i\cup\mathcal{A}_j)$.
Next, suppose that we have $a\in\mathcal{A}_+\setminus(\mathcal{A}_i\cup\mathcal{A}_j)$ and $b\in\mathcal{A}_i$ (this corresponds to case three of Illustration~\ref{Illus:Post:ASOneVsASTwoVsASPlus}).
Note again that this means $b\in in(L_i)$ by construction of $L_+^{min}$.
From Proposition~\ref{Prop:Post:ClownCaseThree}, we can infer that there exist arguments $b',a'\in \mathcal{A}_i$, \suchthat $CSub(b')=CSub(b)$, $CSub(a')=CSub(a)\!\!\mid\!\!_{AS_i}$ and $(b',a')\in {\rightarrow_i}$, or that there exists arguments $b',a'\in \mathcal{A}_i$ \suchthat $CSub(b')=CSub(b)$, $CSub(a')=ADSub(a)\!\!\mid\!\!_{AS_i}$ and either $(b',a')\in{\rightarrow_i}$ or $(a',b')\in{\rightarrow_i}$.
By construction of $in(L_+^{min})$, we must have $CSub(a)\!\!\mid\!\!_{AS_i}\subseteq in(L_i)$.
Furthermore, since $L_i$ was a preferred labeling, we can infer from Lemma~\ref{Lem:Post:PrefClosedUnderSubArgs} that $CSub(b)\subseteq in(L_i)$.
Similarly, we can use Lemma~\ref{Lem:Post:PrefClosedUnderSubArgs} and $CSub(a)\!\!\mid\!\!_{AS_i}\subseteq in(L_i)$ to infer that $ADSub(a)\!\!\mid\!\!_{AS_i}\subseteq in(L_i)$ also holds.
Now we can use admissibility of $L_i$ to infer $L_i(a')=L_i(b')=\In$.
However, now $(b',a')\in{\rightarrow}_i$ and $(a',b')\in{\rightarrow_i}$ both contradict the admissibility of $L_i$.
The same argument holds for the case that $a\in\mathcal{A}_+\setminus(\mathcal{A}_i\cup\mathcal{A}_j)$ and $b\in\mathcal{A}_j$ (note that $SIM_{\mathcal{J}_j}$ is by definition closed under sub-arguments).
Lastly, let us suppose that we have $a,b\in\mathcal{A}_+\setminus(\mathcal{A}_i\cup\mathcal{A}_j)$ (this corresponds to case four of Illustration~\ref{Illus:Post:ASOneVsASTwoVsASPlus}).
The argumentation in this case is similar to the one that corresponds to case three of Illustration~\ref{Illus:Post:ASOneVsASTwoVsASPlus}.
However, now we need to use Proposition~\ref{Prop:Post:ClownCaseFour} in order to infer that there are arguments $a',b'\in\mathcal{A}_i$ or $a',b'\in\mathcal{A}_j$ for which we have $L_i(a')=L_i(b')=\In$ or $SIM_{\mathcal{J}_j}(a')=SIM_{\mathcal{J}_j}(b')=\In$.
Either way, we can again infer a contradiction, either for $L_i$ being admissible or for $SIM_{\mathcal{J}_j}$ being admissible.
As all possible cases lead to a contradiction, we conclude that $in(L_+^{min})\cap O_0=\emptyset$ must hold, just as we claimed.
\end{proof}

\begin{proposition}\label{Prop:Post:CombinedMinimalLabelingLegallyOut}
Let \ASOne, \ASTwo be two $AS$ \suchthat $AS_1|| AS_2$ and let $AS_+$ be their union.
Furthermore, let \JSBAFOne, \JSBAFTwo and \JSBAFPlus be the JSBAF's corresponding to $AS_1$, $AS_2$ and $AS_+$ respectively.
Lastly, for $i,j\in\{1,2\}$ with $i\neq j$, for any preferred labeling $L_i\in pr(\mathcal{J}_i)$ and for $SIM_{\mathcal{J}_j}$ being the strict including minimal labeling of $\mathcal{J}_j$, let $L_+^{min}$ be the combined minimal labeling of $L_i$ and $SIM_{\mathcal{J}_j}$.
We have $a\in out(L_+^{min})$, iff $a$ is legally \Out \wrt $L_+^{min}$.
\end{proposition}
\begin{proof}
From the construction of $out(L_+^{min})$, it is easy to see that, if $a$ is legally \Out \wrt $L_+^{min}$, then $a\in out(L_+^{min})$.
We therefore focus here on the other direction, \ie showing that if $a\in out(L_+^{min})$, then $a$ is legally \Out \wrt $L_+^{min}$.
Similar to the proof of Proposition~\ref{Prop:Post:PropagatedLabelingLegallyOut}, it is not immediately clear from the construction of $out(L_+^{min})$, that this holds.
More precisely, consider the case that we have $x\in O_{k+1}\setminus O_k$.
Then there must exist some support chain $\mathcal{C}=\big{\{}(S_0,b_0),\dots,(S_n,b_n)\big{\}}\subseteq{\Rightarrow}$ \suchthat $x\in S_0$, for all $0<k\leq n$ we have $S_k\setminus\{b_{k-1}\}\subseteq in(L_+^{min})$, for $S_0$ we have $S_0\setminus\{x\}\subseteq in(L_+^{min})$ and there exists an attack $(y,b_n)\in{\rightarrow}$ with $L_+^{min}(y)=\In$.
However, it is not necessarily the case that we have $x\preceq_+ d$ for all $d\in S_0\setminus\{x\}$.
Assume that there is $d\in S_0\setminus\{x\}$ \suchthat $d\prec_+ x$.
As in the proof of Proposition~\ref{Prop:Post:PropagatedLabelingLegallyOut}, we will show that this implies the existence of an \enquote{alternative reason} for why $x$ is legally labeled \Out \wrt $L_+^{min}$.

We begin by constructing some helper-arguments:
Let $S=\big(S_0\setminus\{x\}\big)\cup\bigcup\limits_{0<k\leq n}\big(S_k\setminus\{b_{k-1}\}\big)$, \ie $S$ contains all arguments labeled \In along the support chain $\mathcal{C}$.
Analogous to the proof of Proposition~\ref{Prop:Post:PropagatedLabelingLegallyOut}, for $\{z_0,\dots,z_m\}=S$ we have $ADSub(x)^C\entails\lnot\bigwedge\big(ADSub(y)\cup ADSub(z_0)\cup\dots\cup ADSub(z_m)\big)^C$.
Now we construct the arguments $y':y,z_0,\dots,z_m\rightarrow\bigwedge\{y,z_0,\dots,z_m\}^C$ and $x':x_0,\dots,x_k\rightarrow\lnot\bigwedge\Delta$, where $\{x_0,\dots,x_k\}=ADSub(x)$ and $\Delta=\big(ADSub(y)\cup ADSub(z_0)\cup\dots\cup ADSub(z_m)\big)^C$.
Clearly, $x'$ gen-rebuts $y'$.
Furthermore, because we have $d\prec_+ x$, there must exist $r_d\in DR(d)\subseteq DR(y')$ \suchthat for all $r_x\in DR(x)=DR(x')$, we have $r_d\leq_r^+ r_x$.
This implies $y'\preceq_+ x'$, therefore $x'\not\prec_+ y'$ and thus $(x',y')\in{\rightarrow}_+$.
Note that by construction of $L_+^{min}$, $y'\in in(L_+^{min})$ holds.
Next, we make a case distinction based on the origins of $x',y'$:
Suppose first that we have $x',y'\in \mathcal{A}_i$.
By construction of $in(L_+^{min})$, we can infer that $y'\in in(L_i)$ must hold and by admissibility of $L_i$, this implies $x'\in out(L_i)$.
By construction of $x'$, we obviously have $DR(x')\subseteq DR(x)$ and $ADSub(x')^C\subseteq ADSub(x)^C$.
By assumption, $L_i$ is a preferred labeling, which means it is admissible.
Now we can apply Proposition~\ref{Prop:Post:AlternativeReasonOut} to infer $x\in out(L_i)$, which implies that $x$ is legally \Out \wrt $L_i$ either due to an attack or due to a support chain.
Because $in(L_i)\subseteq in(L_+^{min})$ we can now infer that $x$ is legally \Out \wrt $L_+^{min}$ as required.
The case $x',y'\in\mathcal{A}_j$ can be proven analogously.
Next, assume that we have $x'\in \mathcal{A}_j$ and $y'\in\mathcal{A}_i$ (this corresponds to case one of Illustration~\ref{Illus:Post:ASOneVsASTwoVsASPlus}).
By Proposition~\ref{Prop:Post:ClownCaseOne}, we can now infer that either $x'$ or $y'$ are inconsistent.
Similar to before, we can infer from the construction of $L_+^{min}$ that $y'\in in(L_i)$ must hold.
Since $L_i$ is an admissible labeling, $y'$ must be consistent, therefore $x'$ must be inconsistent.
This means there exists a strict argument $\overline{x'}$ and an attack $(\overline{x'},x)\in{\rightarrow_j}$.
By Proposition~\ref{Prop:Post:OneStepAttacker}, we can use this strict argument to construct a strict argument which is attacking $x$.
This means $x$ is legally \Out \wrt $L_+^{min}$ as required.
The case $x'\in \mathcal{A}_i$ and $y'\in\mathcal{A}_j$ can be proven analogously.
Next, assume that we have $y'\in\mathcal{A}_+\setminus(\mathcal{A}_i\cup\mathcal{A}_j)$, while $x'\in\mathcal{A}_i$ holds (this corresponds to case three of Illustration~\ref{Illus:Post:ASOneVsASTwoVsASPlus}).
Note that by construction of $in(L_+^{min})$, this implies that $y'$ is a minimal $\mathcal{A}_+$-argument.
Furthermore, we want to point out that we can assume both $y'$ and $x'$ are consistent:
For $y'$ we can make a model-theoretic argument similar as in the proof of Proposition~\ref{Prop:Post:CombinedMinimalLabelingIsLabeling}, while if $x'$ is inconsistent then we can immediately infer that $x$ must be inconsistent as well.
By Proposition~\ref{Prop:Post:ClownCaseThree}, we can now have two cases:
\begin{enumerate}
	\item{There exists arguments $x'',y''\in\mathcal{A}_i$ \suchthat $CSub(x'')=CSub(x')$, $CSub(y'')=CSub(y')\!\!\mid\!\!_{AS_i}$ and $(x'',y'')\in{\rightarrow_i}$.}
	\item{There exist arguments $x'',y''\in\mathcal{A}_i$ \suchthat $CSub(x'')=CSub(x')$, $CSub(y'')=ADSub(y')\!\!\mid\!\!_{AS_i}$ and either $(x'',y'')\in{\rightarrow_i}$, or $(y'',x'')\in{\rightarrow_i}$}
\end{enumerate}
In the first case, we can infer from $y'\in in(L_+^{min})$ that $CSub(y')\!\!\mid\!\!_{AS_i}\subseteq in(L_i)$ must hold.
Because $L_i$ is an admissible labeling, we can now infer $y''\in in(L_i)$, which implies $x''\in out(L_i)$.
We want to point out that by construction of $x'$ and $x''$, we have $ADSub(x)=ADSub(x')=ADSub(x'')$ and $DR(x)=DR(x')=DR(x'')$.
In particular, this means $DR(x'')\subseteq DR(x)$ and $ADSub(x'')^C\subseteq ADSub(x)^C$.
Since $L_i$ is an admissible labeling, we can now apply Proposition~\ref{Prop:Post:AlternativeReasonOut} again to infer that $L_i(x)=\Out$ must hold.
This implies that there is an attack or a support chain \suchthat $x$ is legally \Out \wrt $L_i$.
By $in(L_i)\subseteq in(L_+^{max})$ we can now infer that $x$ is legally \Out \wrt $L_+^{min}$ as required.
In the second case, we first use Lemma~\ref{Lem:Post:PrefClosedUnderSubArgs} (and the fact that $SIM_{\mathcal{J}_j}$ is by definition closed under sub-arguments) to infer that $ADSub(y')\!\!\mid\!\!_{AS_i}\subseteq in(L_i)$ holds.
This, in turn, implies $y''\in in(L_i)$.
Now, if $(y'',x'')\in{\rightarrow_i}$, we can use the argument $y''$ as a reason for why $x''$ is legally \Out \wrt $L_i$.
On the other hand, if $(x'',y'')\in{\rightarrow_i}$, we can use the admissibility of $L_i$ to infer that $x''$ must be legally \Out \wrt $L_i$.
Either way, we can again use the reason for why $x''$ is legally \Out \wrt $L_i$ to construct a reason for why $x$ is legally \Out \wrt $L_i$.
By $in(L_i)\subseteq in(L_+^{min})$ we can infer that $x$ is legally \Out \wrt $L_+^{min}$ as required.
The case $y'\in \mathcal{A}_+\setminus(\mathcal{A}_i\cup\mathcal{A}_j)$ and $x'\in\mathcal{A}_j$ can be proven analogously.
For our last few cases, suppose that we have $x'\in\mathcal{A}_+\setminus(\mathcal{A}_i\cup\mathcal{A}_j)$.
Note that $x'$ doesn't necessarily need to be a minimal $\mathcal{A}_+$-argument.
Thus, we first construct such a minimal $\mathcal{A}_+$-argument so that we can apply either Proposition~\ref{Prop:Post:ClownCaseTwo} or Proposition~\ref{Prop:Post:ClownCaseFour}:
Remember that $x'$ is of the form $x':x_0,\dots,x_k\rightarrow\lnot\bigwedge\Delta$, where for each $0\leq l\leq k$, we have $x_l\in ADSub(x)$.
This implies $Atoms(x_l^C)\subseteq Atoms(AS_i)$ or $Atoms(x_l^C)\subseteq Atoms(AS_j)$ for each of these arguments.
Note that, if any of the arguments $x_l$ is inconsistent, then $x$ is also inconsistent.
This implies $x$ is attacked by a strict argument and thus legally \Out \wrt $L_+^{min}$ because of this strict attacker.
We therefore assume that all the arguments $x_l$ are consistent.
Now we can use Proposition~\ref{Prop:Post:ReducedArgsExistence} to infer that for each $0\leq l\leq k$, there exists an argument $x_l'$ which is the reduced version of $x_l$ either \wrt $AS_i$ or \wrt $AS_j$ (depending on whether $Atoms(x_l^C)\subseteq Atoms(AS_i)$ or $Atoms(x_l^C)\subseteq Atoms(AS_j)$).
With these reduced versions, we can construct the argument $x''$ which is of the form $x'':x_0',\dots x_k'\rightarrow\lnot\bigwedge\Delta$.
Note that we have $DR(x'')\subseteq DR(x')=DR(x)$, which implies $(x'',y')\in{\rightarrow_+}$.
Furthermore, $x''$ is a minimal $\mathcal{A}_+$-argument and by construction of $x'$ and $x''$, we have $ADSub(x'')^C\subseteq ADSub(x)^C$.
Lastly, we can assume that $x''$ is consistent, since otherwise $x$ would be inconsistent and legally \Out \wrt $L_+^{min}$ as required.
Now, assume that we have $y'\in\mathcal{A}_i$ (this corresponds to case two of Illustration~\ref{Illus:Post:ASOneVsASTwoVsASPlus}).
By construction of $y'$, this implies $y'\in in(L_i)\subseteq in(L_+^{min})$.
Furthermore, we can infer that $y'$ is consistent.
Because $x''$ is a (consistent) minimal $\mathcal{A}_+$-argument, we can now use Proposition~\ref{Prop:Post:ClownCaseTwo} to infer that there exists an argument $x'''\in\mathcal{A}_i$ \suchthat $CSub(x'')\!\!\mid\!\!_{AS_i}=CSub(x''')$ and $(x''',y')\in{\rightarrow_i}$.
By admissibility of $L_i$, we can now infer that $x'''\in out(L_i)$ must hold.
This implies $x'''$ is legally \Out \wrt $L_i$, which means there is either an attack $(u,x''')\in{\rightarrow_i}$ or a support chain $\mathcal{C}=\big{\{}(S_0,b_0),\dots,(S_n,b_n)\big{\}}$ giving us a reason for why $x'''$ is legally \Out \wrt $L_i$.
Note that, by $CSub(x'')\!\!\mid\!\!_{AS_i}=CSub(x''')$, we have $DR(x''')\subseteq DR(x'')\subseteq DR(x)$ and $ADSub(x''')^C\subseteq ADSub(x'')^C\subseteq ADSub(x)^C$.
Suppose first that we have an attack $(u,x''')\in{\rightarrow_i}$ with $u\in in(L_i)$.
Then $u$ is either undercutting or gen-rebutting $x'''$.
If $u$ is undercutting $x'''$, then we can use $DR(x''')\subseteq DR(x)$ to infer that $(u,x)\in{\rightarrow_+}$ also holds.
If $u$ is gen-rebutting $x'''$, then we can again use Proposition~\ref{Prop:Post:OneStepAttacker} as well as $ADSub(x''')^C\subseteq ADSub(x)^C$ and $DR(x''')\subseteq DR(x)$ to infer that there exists $u'\in in(L_i)$ \suchthat $(u',x)\in{\rightarrow_+}$.
Either way, there exists an attacker of $x$ which is labeled \In in $L_i$.
By $in(L_i)\subseteq in(L_+^{min})$ we can infer that $x$ is legally \Out \wrt $L_+^{min}$ as required.
Now suppose that $x'''$ is legally \Out \wrt $L_i$ due to a support chain $\mathcal{C}=\big{\{}(S_0,b_0),\dots,(S_n,b_n)\big{\}}$ and an attack $(c,b_n)\in{\rightarrow_i}$ with $x'''\in S_0$ and $L_i(c)=\In$.
Let $S=S_0\setminus\{x'''\}\cup\bigcup\limits_{0<k\leq n}S_k\setminus\{b_{k-1}\}$ be the set of all arguments labeled \In in this support chain.
If $(c,b_n)\in{\rightarrow_i}$ is the result of an undercut, we can use $S\subseteq in(L_i)$ to infer that $c$ must be undercutting $x'''$.
By $DR(x''')\subseteq DR(x)$, we can infer that $c$ is also undercutting $x$ and since $in(L_i)\subseteq in(L_+^{max})$, this means $x$ is legally \Out \wrt $L_+^{min}$ as required.
Now suppose that $(c,b_n)\in{\rightarrow_i}$ is the result of a gen-rebut.
We construct the argument $b'\in\mathcal{A}_+\setminus(\mathcal{A}_i\cup\mathcal{A}_j)$ which is of the form $b':x,d_0,\dots d_m\rightarrow\bigwedge\{x,d_0,\dots,d_m\}^C$, for $\{d_0,\dots,d_m\}=S$.
Furthermore, we use Proposition~\ref{Prop:Post:OneStepAttacker} to construct the argument $c'$ which is of the form $c':c\rightarrow\lnot\bigwedge ADSub(b_n)^C$.
Note that we have $ADSub(b_n)^C\subseteq ADSub(b')^C$ because $ADSub(x''')^C\subseteq ADSub(x)^C$.
This means $c'$ is also gen-rebutting $b'$.
Because $(c,b_n)\in{\rightarrow_i}$ is the result of a gen-rebut, we have $c\not\prec_i b_n$.
Since $DR(x''')\subseteq DR(x)$, we can now infer that $c\not\prec_+ b'$ also holds.
Because $DR(c)=DR(c')$ this means $(c',b')\in{\rightarrow_+}$.
By admissibility of $L_i$, we can infer $c'\in in(L_i)\subseteq in(L_+^{min})$.
By construction of $L_+^{min}$ we now have $b'\in O_0\subseteq out(L_+^{min})$.
We can make a simple inductive argument to show that we have $x'''\preceq_i d$ for all $d\in S$.
This means for all $d\in S$, there is $r_x\in DR(x''')$ \suchthat for all $r_d\in DR(d)$, $r_x\leq_i r_d$.
Because $DR(x''')\subseteq DR(x)$ and because ${\leq_{r_i}}\subseteq {\leq_+}$ by construction of $AS_+$ from $AS_i$ and $AS_j$, we now have $x\preceq_+ d$ for all $d\in S$.
Since $S\subseteq in(L_i)\subseteq in(L_+^{min})$, we can now infer that $x$ is legally \Out \wrt $L_+^{min}$ due to the support chain $\mathcal{C}'=\big{\{}(S\cup\{x\},b')\big{\}}$ and the attack $(c',b')\in{\rightarrow_+}$.
The case $y'\in\mathcal{A}_j$ can be proven analogously.
Lastly, suppose that we have $y'\in\mathcal{A}_+\setminus(\mathcal{A}_i\cup\mathcal{A}_j)$ (this corresponds to case four of Illustration~\ref{Illus:Post:ASOneVsASTwoVsASPlus}).
Then both $x'$ and $y'$ are minimal $\mathcal{A}_+$-arguments.
Note that we can argue for the consistency of $y'$ and $x'$ as before:
For $y'$ we can make a model-theoretic argument for its consistency, while inconsistency of $x'$ implies that $x$ is legally \Out \wrt $L_+^{min}$.
Now we can use Proposition~\ref{Prop:Post:ClownCaseFour} to infer that one of the following four cases must hold:
\begin{enumerate}
	\item{
	There are $x'',y''\in \mathcal{A}_i$ \suchthat $CSub(x')\!\!\mid\!\!_{AS_i}=CSub(x'')$, $CSub(y')\!\!\mid\!\!_{AS_i}=CSub(y'')$ and $(x'',y'')\in {\rightarrow_i}$.
	}
	\item{
	There are $x'',y''\in \mathcal{A}_j$ \suchthat $CSub(x')\!\!\mid\!\!_{AS_j}=CSub(x'')$, $CSub(y')\!\!\mid\!\!_{AS_j}=CSub(y'')$ and $(x'',y'')\in {\rightarrow_j}$.
	}
	\item{
	There are $x'',y''\in\mathcal{A}_i$ \suchthat $CSub(x')\!\!\mid\!\!_{AS_i}=CSub(x'')$, $ADSub(y')\!\!\mid\!\!_{AS_i}=CSub(y'')$ and either $(x'',y'')\in {\rightarrow_i}$ or $(y'',x'')\in{\rightarrow_i}$.
	}
	\item{
	There are $x'',y''\in \mathcal{A}_j$ \suchthat $CSub(x')\!\!\mid\!\!_{AS_j}=CSub(x'')$, $ADSub(y')\!\!\mid\!\!_{AS_j}=CSub(y'')$ and either $(x'',y'')\in {\rightarrow_j}$ or $(y'',x'')\in{\rightarrow_j}$.
	}
\end{enumerate}
With the same reasoning that we used to prove the cases which corresponded to cases two and three of Illustration~\ref{Illus:Post:ASOneVsASTwoVsASPlus}, we can show that in all of the four cases described above, $x''$ needs to be legally \Out \wrt $L_i$ or $L_j$ (depending on the specific case).
Because we have $DR(x'')\subseteq DR(x)$, $ADSub(x'')^C\subseteq ADSub(x)^C$ and $\big(in(L_i)\cup in(SIM_{\mathcal{J}_j})\big)\subseteq in(L_+^{min})$, we can use this to infer that $x$ is legally \Out \wrt $L_+^{min}$, similar to the cases corresponding to cases two and three of Illustration~\ref{Illus:Post:ASOneVsASTwoVsASPlus}.
We conclude that regardless of the origins of $x'$ and $y'$, if $x\in out(L_+^{min})$, then $x$ is legally \Out \wrt $L_+^{min}$.
\end{proof}

\begin{proposition}\label{Prop:Post:CombinedMinimalLabelingLegallyIn}
Let \ASOne, \ASTwo be two $AS$ \suchthat $AS_1|| AS_2$ and let $AS_+$ be their union.
Furthermore, let \JSBAFOne, \JSBAFTwo and \JSBAFPlus be the JSBAF's corresponding to $AS_1$, $AS_2$ and $AS_+$ respectively.
Lastly, for $i,j\in\{1,2\}$ with $i\neq j$, for any preferred labeling $L_i\in pr(\mathcal{J}_i)$ and for $SIM_{\mathcal{J}_j}$ being the strict including minimal labeling of $\mathcal{J}_j$, let $L_+^{min}$ be the combined minimal labeling of $L_i$ and $SIM_{\mathcal{J}_j}$.
If $a\in in(L_+^{min})$, then $a$ is legally \In \wrt $L_+^{min}$.
\end{proposition}
\begin{proof}
Let $a\in in(L_+^{min})$ be an accepted argument.
We first consider the supports $(S,b)\in{\Rightarrow_+}$ where $a\in S$ holds:
By construction of $in(L_+^{min})$, if $S\subseteq in(L_+^{min})$, then $b\in in(L_+^{min})$ also holds.
Furthermore, by construction of $out(L_+^{min})$, if $|S\setminus in(L_+^{min})|=1$ while $b\in out(L_+^{min})$, then $S\setminus in(L_+^{min})\in out(L_+^{min})$ also holds.
Therefore, the conditions of items $1.a$ to $1.c$ of Definition~\ref{Def:JSBAF:LegalLabeling} are always satisfied and we only need to show that all attackers of $a$ are labeled \Out in $L_+^{min}$.
Note that, when we argued that $L_+^{min}$ is a labeling, we have already shown that, if $a\in in(L_+^{min})$ and there is $(b,a)\in{\rightarrow_+}$, then $L_+^{min}(b)\neq\In$.
Thus, we only have left to prove that $L_+^{min}(b)\neq\Undec$.
Towards a contradiction, suppose that this does not hold.
We proceed in a similar matter as for the proof that showed $L_+^{min}$ is a labeling by considering all possible cases for the origins of $a$ and $b$.
First, suppose that $a,b\in\mathcal{A}_i$ (or $a,b\in\mathcal{A}_j$).
Then $b$ is legally \Out \wrt $L_i$ (or $SIM_{\mathcal{J}_j}$ respectively) and because $in(L_i)\cup in(SIM_{\mathcal{J}_j})\subseteq in(L_+^{min})$, we can infer that $b$ is legally \Out \wrt $L_+^{min}$.
By Proposition~\ref{Prop:Post:CombinedMinimalLabelingLegallyOut} this implies $L_+^{min}(b)=\Out$, contradicting our assumption $L_+^{min}(b)=\Undec$.
Next, suppose that $a\in\mathcal{A}_i$ and $b\in\mathcal{A}_j$ (or $a\in\mathcal{A}_j$ and $b\in\mathcal{A}_i$), which corresponds to case one of Illustration~\ref{Illus:Post:ASOneVsASTwoVsASPlus}.
Then by Proposition~\ref{Prop:Post:ClownCaseOne}, either $a$ or $b$ are legally \Out \wrt $L_+^{min}$, which again means either $L_+^{min}(a)=\Out$ or $L_+^{min}(b)=\Out$, contradicting our assumptions.
Now suppose that $a\in\mathcal{A}_i$, while $b\in\mathcal{A}_+\setminus(\mathcal{A}_i\cup\mathcal{A}_j)$, which corresponds to case two of Illustration~\ref{Illus:Post:ASOneVsASTwoVsASPlus} (we leave out the case $a\in\mathcal{A}_j$, while $b\in\mathcal{A}_+\setminus(\mathcal{A}_i\cup\mathcal{A}_j)$, as it can be proven analogous).
From $(b,a)\in{\rightarrow_+}$, we can infer that $b$ undercuts or gen-rebuts $a$.
In the first case, we have $Atoms(b^C)\subseteq Atoms(AS_i)$.
In the second case, we can use Proposition~\ref{Prop:Post:OneStepAttacker} to infer that there exists $b'\in\mathcal{A}_+$ which is of the form $b':b\rightarrow\lnot\bigwedge ADSub(a)^C$.
Note that we have $Atoms(b'^C)\subseteq Atoms(AS_i)$.
Now, regardless of whether the attack $(b,a)$ was the result of an undercut or a gen-rebut, there is some $(x,a)\in{\rightarrow_+}$ \suchthat $Atoms(x)\subseteq Atoms(AS_i)$.
By Proposition~\ref{Prop:Post:ReducedArgsExistence}, we can infer that there exists a reduced version of $x$ \wrt $AS_i$, \ie some $x'\in \mathcal{A}_i$ for which we have $x^C=x'^C$ and $DR(x')\subseteq DR(x)$.
Now $x'$ undercuts or gen-rebuts $a$.
If $x'$ gen-rebuts $a$, then we note that $(x,a)\in{\rightarrow_+}$ means $x\not\prec a$, which in turn implies $x'\not\prec a$.
Thus we can infer that $(x',a)\in{\rightarrow_i}$ must hold, regardless of whether $x'$ undercuts or gen-rebuts $a$.
Since $L_i$ was an admissible labeling and $a\in in(L_i)$, we can infer that $L_i(x')=\Out$ holds and because $in(L_i)\subseteq in(L_+^{min})$, we must also have $L_+^{min}(x')=\Out$ by construction of $L_+^{min}$.
Now we can use Lemma~\ref{Lem:Post:ReducedArgumentsLabelImplication} and Proposition~\ref{Prop:Post:CombinedMinimalLabelingLegallyOut}, to infer that $L_+^{min}(x)=\Out$ also holds, contradicting our assumption that $a$ is attacked by an argument which is labeled \Undec.
This proves case two of Illustration~\ref{Illus:Post:ASOneVsASTwoVsASPlus}.
Now, suppose that we have $a\in\mathcal{A}_+\setminus(\mathcal{A}_i\cup\mathcal{A}_j)$ and $b\in\mathcal{A}_i$, which corresponds to case three of Illustration~\ref{Illus:Post:ASOneVsASTwoVsASPlus} (we again leave out the case $a\in\mathcal{A}_+\setminus(\mathcal{A}_i\cup\mathcal{A}_j)$ and $b\in\mathcal{A}_j$, as it can be proven analogous).
Suppose that $(b,a)\in{\rightarrow_+}$ is the result of an undercut.
Then this attack must also be directed towards some $x\in CSub(a)$ and by construction of $in(L_+^{min})$, we have $x\in\mathcal{A}_i$ or $x\in\mathcal{A}_j$.
In the first case, we can again infer that $x$ is legally \Out \wrt $L_i$ by admissibility of $L_i$ and by construction of $L_+^{min}$, while in the second case either $x$ or $b$ must be legally \Out \wrt $SIM_{\mathcal{J}_j}$ or $L_i$ respectively.
All of these options contradict our assumptions for the labels of $a$ and $b$.
Now suppose that $(b,a)\in{\rightarrow_+}$ is the result of a gen-rebut.
By Proposition~\ref{Prop:Post:OneStepAttacker}, we can again infer that there must be an argument $b'\in\mathcal{A}_i$ of the form $b':b\rightarrow\lnot\bigwedge ADSub(a)^C$.
Since $DR(b)=DR(b')$, it is clear that $(b',a)\in{\rightarrow_+}$.
We can make a simple model-theoretic argument to show that $\big(ADSub(a)\!\!\mid\!\!_{AS_j}\big)^C\cup\{b'^C\}\entails\lnot\bigwedge \big(ADSub(a)\!\!\mid\!\!_{AS_i}\big)^C$.
With this, we can construct the argument $c\in\mathcal{A}_+\setminus(\mathcal{A}_i\cup\mathcal{A}_j)$ and the arguments $a',a''\in\mathcal{A}_i$, as follows: $c:a_{j_0},\dots,a_{j_m},b'\rightarrow\lnot\bigwedge \big(ADSub(a)\!\!\mid\!\!_{AS_i}\big)^C$, $a':a_{i_0},\dots,a_{i_n}\rightarrow\bigwedge \big(ADSub(a)\!\!\mid\!\!_{AS_i}\big)^C$ and $a'':a'\rightarrow\lnot\lnot\bigwedge \big(ADSub(a)\!\!\mid\!\!_{AS_i}\big)^C$, where $\{a_{j_0},\dots,a_{j_m}\}=ADSub(a)\!\!\mid\!\!_{AS_j}$ and $\{a_{i_0},\dots,a_{i_n}\}=ADSub(a)\!\!\mid\!\!_{AS_i}$.
Note that we have $ADSub(a)\!\!\mid\!\!_{AS_i}\subseteq in(L_i)$ by Lemma~\ref{Lem:Post:PrefClosedUnderSubArgs} and $ADSub(a)\!\!\mid\!\!_{AS_j}\subseteq in(SIM_{\mathcal{J}_j})$ by definition of the strict including minimal labeling.
Clearly, $c$ gen-rebuts $a'$ and $a''$ gen-rebuts $c$.
It is also clear that either $(c,a')\in{\rightarrow_+}$ or $(a'',c)\in{\rightarrow_+}$ must hold.
Suppose first that we have $(c,a')\in{\rightarrow_+}$.
From $ADSub(a)\!\!\mid\!\!_{AS_i}\subseteq in(L_i)$, we can infer $a'\in in(L_i)\subseteq in(L_+^{min})$.
Now the attack $(c,a')$ corresponds to case two of Illustration~\ref{Illus:Post:ASOneVsASTwoVsASPlus}.
We have argued above, that in this case, $a'\in in(L_+^{min})$ implies $c\in out(L_+^{min})$.
Since $ADSub(a)\!\!\mid\!\!_{AS_j}=\{a_{j_0},\dots,a_{j_m}\}\subseteq in(L_+^{min})$, we can infer from the construction of $out(L_+^{min})$ that we must have $b'\in out(L_+^{min})$, which also implies $b\in out(L_+^{min})$.
However, now $b\in out(L_+^{min})$ contradicts our assumption $L_+^{min}(b)=\Undec$.
Now suppose that we have $(a'',c)\in{\rightarrow_+}$.
From $ADSub(a)\!\!\mid\!\!_{AS_i}=\{a_{i_0},\dots,a_{i_n}\}\subseteq in(L_i)$, we can infer that $a''\in in(L_i)\subseteq in(L_+^{min})$ must hold.
Now $c$ is legally \Out \wrt $L_+^{min}$ and by Proposition~\ref{Prop:Post:CombinedMinimalLabelingLegallyOut}, this implies $L_+^{min}(c)=\Out$, which in turn implies that $b'$ and $b$ are legally \Out.
Again, we can now infer $L_+^{min}(b)=\Out$, which contradicts our assumption $L_+^{min}(b)=\Undec$.
This proves case three of Illustration~\ref{Illus:Post:ASOneVsASTwoVsASPlus}.
Finally, suppose that we have $a,b\in\mathcal{A}_+\setminus(\mathcal{A}_i\cup\mathcal{A}_j)$, which corresponds to case four of Illustration~\ref{Illus:Post:ASOneVsASTwoVsASPlus}. 
First, let us assume that $(b,a)\in{\rightarrow_+}$ is the result of an undercut.
Then this attack must also be directed towards some $a'\in CSub(a)\in\mathcal{A}_i\cup\mathcal{A}_j$.
This corresponds to case two of Illustration~\ref{Illus:Post:ASOneVsASTwoVsASPlus}.
By the construction of $in(L_+^{min})$, we can infer $a'\in in(L_+^{min})$ from $a\in in(L_+^{min})$.
We have argued above that this implies $L_+^{min}(b)=\Out$, contradicting the assumption $L_+^{min}(b)=\Undec$.
Now suppose that $(b,a)\in¸{\rightarrow_+}$ is the result of a gen-rebut.
The proof for this case proceeds analogous for that of case three of Illustration~\ref{Illus:Post:ASOneVsASTwoVsASPlus}, which we have proven above:
We construct the arguments $b':b\rightarrow\lnot\bigwedge ADSub(a)^C$ (by Proposition~\ref{Prop:Post:OneStepAttacker}), and with the sets $ADSub(a)\!\!\mid\!\!_{AS_j}=\{a_{j_0},\dots,a_{j_n}\}$ and $ADSub(a)\!\!\mid\!\!_{AS_i}=\{a_{i_0},\dots,a_{i_m}\}$ we construct the argugment $c:a_{j_0},\dots,a_{j_m},b'\rightarrow\lnot \big(ADSub(a)\!\!\mid\!\!_{AS_i}\big)^C$, the argument $a':a_{i_0},\dots,a_{i_n}\rightarrow\bigwedge\big(ADSub(a)\!\!\mid\!\!_{AS_i}\big)^C$ and the argument $a'':a'\rightarrow\lnot\lnot\bigwedge \big(ADSub(a)\!\!\mid\!\!_{AS_i}\big)^C$.
Note that $c\in\mathcal{A}_+\setminus(\mathcal{A}_i\cup\mathcal{A}_j)$, while $a',a''\in\mathcal{A}_i$.
Furthermore, we have $a',a''\in in(L_i)\subseteq in(L_+^{min})$ by Lemma~\ref{Lem:Post:PrefClosedUnderSubArgs} and by admissibility of $L_i$.
We again either have $(c,a')\in{\rightarrow_+}$ or $(a'',c)\in{\rightarrow_+}$.
If $(c,a')\in{\rightarrow_+}$, then this corresponds to case two of Illustration~\ref{Illus:Post:ASOneVsASTwoVsASPlus}, for which we have argued that $L_+^{min}(c)=\Out$, which -- together with  $a_{j_0},\dots,a_{j_n}\subseteq in(SIM_{\mathcal{J}_j})\subseteq in(L_+^{min})$ -- implies that $b'$ and $b$ are legally \Out, which means $L_+^{min}(b)=\Out$.
On the other hand, if $(a'',c)\in{\rightarrow_+}$, then we can infer that $c$ is legally \Out \wrt $L_+^{min}$, which again implies $b'$ and $b$ are legally \Out, thus $L_+^{min}(b)=\Out$ as required.
We conclude:
If $a\in in(L_+^{min})$, then $a$ is legally \In \wrt $L_+^{min}$.
\end{proof}

\begin{proposition}\label{Prop:Post:CombinedMinimalLabelingAdmissble}
Let \ASOne, \ASTwo be two $AS$ \suchthat $AS_1|| AS_2$ and let $AS_+$ be their union.
Furthermore, let \JSBAFOne, \JSBAFTwo and \JSBAFPlus be the JSBAF's corresponding to $AS_1$, $AS_2$ and $AS_+$ respectively.
Lastly, for $i,j\in\{1,2\}$ with $i\neq j$, for any preferred labeling $L_i\in pr(\mathcal{J}_i)$ and for $SIM_{\mathcal{J}_j}$ being the strict including minimal labeling of $\mathcal{J}_j$, let $L_+^{min}$ be the combined minimal labeling of $L_i$ and $SIM_{\mathcal{J}_j}$.
Then $L_+^{min}$ is an admissible labeling of $\mathcal{J}_+$.
\end{proposition}
\begin{proof}
By Proposition~\ref{Prop:Post:CombinedMinimalLabelingIsLabeling} we know that $L_+^{min}$ is a labeling.
By Proposition~\ref{Prop:Post:CombinedMinimalLabelingLegallyOut} we know that $a\in out(L_+^{min})$ iff $a$ is legally \Out \wrt $L_+^{min}$.
By Proposition~\ref{Prop:Post:CombinedMinimalLabelingLegallyIn} we know that $a\in in(L_+^{min})$ implies $a$ is legally \In \wrt $L_+^{min}$.
Lastly, it is clear from the definition of $in(L_+^{min})$ that $STR_{\mathcal{J}_+}\subseteq in(L_+^{min})$ holds.
\end{proof}
With this, we have shown that any preferred labeling of $\mathcal{J}_i$ can be turned into an admissible labeling of $\mathcal{J}_+$.
For the next step, we will show that any preferred labeling of $\mathcal{J}_+$ can be turned into an admissible labeling of $\mathcal{J}_i$.
We begin by defining the \emph{restriction} of a labeling:
\begin{definition}
Let \ASOne, \ASTwo be two $AS$ \suchthat $AS_1|| AS_2$ and let $AS_+$ be their union.
Let \JSBAFOne, \JSBAFTwo and \JSBAFPlus be the JSBAFs corresponding to $AS_1$, $AS_2$ and $AS_+$ respectively.
Furthermore, let $L_+$ be a labeling of $\mathcal{J}_+$.
For any set of arguments $\mathcal{A}\subseteq\mathcal{A}_+$, we define the \emph{restriction} of $L_+$ to $\mathcal{A}$, denoted $L_+\!\!\mid\!\!_{\mathcal{A}}$, as the following labeling:
$in(L_+\!\!\mid\!\!_{\mathcal{A}})=in(L_+)\cap \mathcal{A}$, $out(L_+\!\!\mid\!\!_{\mathcal{A}})=out(L_+)\cap \mathcal{A}$ and $undec(L_+\!\!\mid\!\!_{\mathcal{A}})=undec(L_+)\cap \mathcal{A}$.
\end{definition}
\begin{proposition}\label{Prop:Post:ASPlusToASI}
Let \ASOne, \ASTwo be two $AS$ \suchthat $AS_1|| AS_2$ and let $AS_+$ be their union.
Let \JSBAFOne, \JSBAFTwo and \JSBAFPlus be the JSBAFs corresponding to $AS_1$, $AS_2$ and $AS_+$ respectively.
Furthermore, let $L_+\in pr(\mathcal{J}_+)$ be a preferred labeling of $\mathcal{J}_+$ and for any $i\in\{1,2\}$, let $L_i=L_+\!\!\mid\!\!_{\mathcal{A}_i}$ be the restriction of $L_+$ to $\mathcal{A}_i$.
Then $L_i\in adm(\mathcal{J}_i)$.
\end{proposition}
\begin{proof}
It is clear that $L_i$ is a labeling (\ie that every argument receives exactly one label), that all strict arguments are accepted in $L_i$, that $a\in in(L_i)$ implies $a$ is legally \In \wrt $L_i$ and that, if $a$ is legally \Out \wrt $L_i$, then $L_i(a)=\Out$.
Therefore, we only need to show that $a\in out(L_i)$ implies $a$ is legally \Out \wrt $L_i$.
The case is clear if $L_+(a)=\Out$ due to an attacker $b\in \mathcal{A}_i$ or a support chain $\mathcal{C}=\{(S_0,b_0),...(S_n,b_n)\}\subseteq{\Rightarrow_i}$ with $a\in S_0$.
We therefore assume that these cases do not hold.
First, suppose that $L_+(a)=\Out$ because of an attack $(b,a)\in{\rightarrow_+}$, where $b\in \mathcal{A}_+\setminus \mathcal{A}_i$ and $L_+(b)=\In$.
We make a case distinction for the origins of $b$:
If $b\in \mathcal{A}_j$, then we know by Proposition~\ref{Prop:Post:ClownCaseOne} that either $a$ or $b$ are legally \Out due to a strict argument.
By assumption, $L_+(b)=\In$ and $L_+$ was a preferred labeling, therefore we must have that $a$ is legally \Out due to a strict argument $\overline{a}$.
This argument $\overline{a}$ is also present in $\mathcal{A}_i$, therefore $a$ is legally \Out \wrt $L_i$ as required.
Next, suppose that $L_+(a)=\Out$ due to some attacker $b\in\mathcal{A}_+\setminus(\mathcal{A}_i\cup\mathcal{A}_j)$ with $L_+(b)=\In$.
We make a distinction regarding for the reason of $(b,a)\in{\rightarrow_+}$:
If this attack results from an undercut, then we have $Atoms(b^C)\subseteq Atoms(AS_i)$.
Let $b'$ be a reduced version of $b$ \wrt $AS_i$.
Then $b'\in\mathcal{A}_i$ and $b^C=b'^C$, \ie $b'$ undercuts $a$.
By Lemma~\ref{Lem:Post:ReducedArgumentsLabelImplication} we can infer that $L_+(b')=L_i(b')=\In$, which means $a$ is legally \Out \wrt $L_i$ as required.
Now assume $(b,a)\in{\rightarrow_+}$ is the result of a gen-rebut.%
By Proposition~\ref{Prop:Post:OneStepAttacker} we infer that there exists an argument $b'\in\mathcal{A}_+$ of the form $b':b\rightarrow\lnot\bigwedge ADSub(a)^C$ and an attack $(b',a)\in{\rightarrow_+}$.
Now we take the sets $\{b_{i_0},\dots,b_{i_m}\}=CSub(b')\!\!\mid\!\!_{AS_i}=CSub(b)\!\!\mid\!\!_{AS_i}$ and $\{b_{j_0},\dots,b_{j_n}\}=CSub(b')\!\!\mid\!\!_{AS_j}=CSub(b)\!\!\mid\!\!_{AS_j}$.
Note that $Atoms(b_k^C)\subseteq Atoms(AS_i)$ for $0\leq k\leq m$ and $Atoms(b_l^C)\subseteq Atoms(AS_j)$ for $0\leq l\leq n$, while $Atoms(b'^C)\subseteq Atoms(AS_i)$.
Furthermore, note that $\{b_{i_0},\dots,b_{i_m}\}^C\cup\{b_{j_0},\dots,b_{j_n}\}^C\entails b'^C$.
By Proposition~\ref{Prop:Post:ReducedArgumentsHelper}, we can now infer that $\{b_0,...,b_m\}^C\entails\lnot\bigwedge ADSub(a)^C$.
Next, let $b_{i_0}',\dots,b_{i_m}'\in\mathcal{A}_i$ be reduced versions of $b_{i_0},\dots,b_{i_m}$ \wrt $AS_i$.
We construct the argument $\widetilde{b}\in\mathcal{A}_i$ as follows:
$\widetilde{b}:b_{i_0}',\dots,b_{i_m}'\rightarrow\lnot\bigwedge ADSub(a)^C$.
It is clear that $\widetilde{b}$ gen-rebuts $a$.
Since each $b_{i_k}'$ is a reduced version of some $b_{i_k}\in CSub(b)$, we have $DR(b_{i_k}')\subseteq DR(b)$, therefore $DR(\widetilde{b})\subseteq DR(b)$.
By assumption, we have $(b,a)\in{\rightarrow_+}$ as the result of a gen-rebut, therefore $b\not\prec a$.
Now $\widetilde{b}\not\prec a$ also holds and therefore $(\widetilde{b},a)\in{\rightarrow_i}$.
Since $b\in in(L_+)$ by assumption, we can use Lemma~\ref{Lem:Post:PrefClosedUnderSubArgs} to infer $\{b_{i_0},\dots b_{i_m}\}\subseteq in(L_+)$.
Now we can use Lemma~\ref{Lem:Post:ReducedArgumentsLabelImplication} to infer that $\{b_{i_0}',\dots,b_{i_m}'\}\subseteq in(L_+)$ holds.
By admissibility of $L_+$ we now have $\widetilde{b}\in in(L_+)$ and by construction of $L_i$ from $L_+$ we have $\widetilde{b}\in in(L_i)$.
Now $(\widetilde{b},a)\in{\rightarrow_i}$ and $L_i(\widetilde{b})=\In$, thus $a$ is legally \Out \wrt $L_i$ as required.
Lastly, assume that $L_+(a)=\Out$ because of a support chain $\mathcal{C}=\big{\{}(S_0,b_0),\dots,(S_n,b_n)\big{\}}\subseteq{\Rightarrow_+}$, with $a\in S_0$, $(c,b_n)\in{\rightarrow_+}$, $L_+(c)=\In$, $L_+(b_k)=\Out$ for $0\leq k\leq n$, $S_0\setminus\{a\}\subseteq in(L_+)$, $S_k\setminus\{b_{k-1}\}\subseteq in(L_+)$ for $0<k\leq n$ and $a\preceq d$ for all $d\in S_0\setminus\{a\}$.
We will show that this implies that there is an attack $(\widetilde{c},a)\in{\rightarrow_+}$ with $L_+(\widetilde{c})=\In$.
As we have argued above, this implies $a$ is legally \Out \wrt $L_i$.
We show the existance of such an argument $\widetilde{c}$ via induction over $n\in\mathbb{N}$ for $n$ being the length of the support chain $\mathcal{C}$.
Induction start: $n=1$, \ie $\mathcal{C}=\big{\{}(S_0,b_0)\big{\}}$.
Let $S_0=\{a,a_0,\dots,a_m\}$.
Suppose first that the attack $(c,b_0)\in{\rightarrow_+}$ is the result of an undercut.
Then this attack must also be directed towards one of the arguments in $S_0$.
By $S_0\setminus\{a\}\subseteq in(L_+)$, we can infer that $c$ must be undercutting $a$.
Now we have $(c,a)\in{\rightarrow_+}$ and $L_+(c)=\In$ as required.
Next, suppose that $(c,b_0)\in{\rightarrow_+}$ is the result of a gen-rebut.
We first use Proposition~\ref{Prop:Post:OneStepAttacker} to infer that there exists $c'$ of the form $c':c\rightarrow\lnot\bigwedge  ADSub(b_n)^C$.
It is easy to see that we have $\{c'\}^C\cup\big(\bigcup\limits_{0\leq k\leq m} ADSub(a_k)\big)^C\entails\lnot\bigwedge  ADSub(a)^C$.
With this, we construct the argument $c''$ as $c'':c_0,\dots,c_l\rightarrow\lnot\bigwedge  ADSub(a)^C$, where $\{c_0,\dots,c_l\}=\{c'\}\cup\bigcup\limits_{0\leq k\leq m} ADSub(a_k)$.
Clearly, $c''$ gen-rebuts $a$.
Furthermore, we have $c'\in in(L_+)$ by admissibility of $L_+$ and -- since the arguments $a_0,\dots,a_m$ are labeled \In by assumption -- $\big(\bigcup\limits_{0\leq k\leq m}ADSub(a_k)\big)\subseteq in(L_+)$ by Lemma~\ref{Lem:Post:PrefClosedUnderSubArgs}.
By admissibility of $L_+$ we can now infer that $L_+(c'')=\In$ also holds.

Now, take $r_a\in DR(a)$ \suchthat $r_a$ is minimal (\wrt $\leq_r^+$) among all rules in $DR(a)$.
Because $r_+$ is a total pre-order and because $a\preceq d$ for all $d\in S_0\setminus\{a\}$, we can infer that for all $a_l\in\{a_0,\dots,a_m\}$, for all $r_{a_l}\in DR(a_l)$, we have $r_a\leq_r^+ r_{a_l}$.
We now claim that for all $r_c\in DR(c)$, $r_a\leq r_c$ must hold.
To see this, assume towards a contradiction that there is some $r_c\in DR(c)$ \suchthat $r_c<_r^+ r_a$ holds.
By transitivity of $\leq_r^+$, this implies $r_c<_r^+ r$ for all $r\in DR(b_0)$.
This means $c\preceq b_n$.
Because $(c,b_n)$ is the result of a gen-rebut, we have $c\not\prec b_n$, \ie either $c\not\preceq b_n$ or $b_n\preceq c$.
Obviously $c\not\preceq b_n$ contradicts $c\preceq b_n$, therefore we assume $b_n\preceq c$.
This means there exists some $r_b\in DR(b_n)$ \suchthat for all $r_c'\in DR(c)$, $r_b\leq_r^+ r_c'$ holds.
Now for $r_c'=r_c$ in particular we can infer $r_c<_r^+ r_b\leq_r^+ r_c$, a contradiction.
With this, we have shown that for all $r_c\in DR(c)$, $r_a\leq_r^+ r_c$ holds.
Now we can infer that for all $r\in DR(c'')$, we have $r_a\leq_r^+ r$.
This, in turn, implies $a\preceq c''$, therefore $c''\not\prec a$ and $(c'',a)\in{\rightarrow_+}$ as required.
Induction step: $n\rightarrow n+1$, \ie $\mathcal{C}=\big{\{}(S_0,b_0),\dots,(S_n,b_n),(S_{n+1},b_{n+1})\big{\}}\subseteq{\Rightarrow_+}$.
Now $\mathcal{C}'=\big{\{}(S_1,b_1),\dots,(S_n,b_n),(S_{n+1},b_{n+1})\big{\}}\subseteq{\Rightarrow_+}$ is a support chain of length $n$.
By the induction hypothesis, there now exists an attack $(\widetilde{c},b_0)\in{\rightarrow_+}$ with $L_+(\widetilde{c})=\In$.
Since $\big{\{}(S_0,b_0)\big{\}}$ is a support-chain of lengths $1\leq n$, we can now use the induction hypothesis again to infer that there is an attack $(\widetilde{c},a)\in{\rightarrow_+}$ for which we have $L_+(\widetilde{c})=\In$.
With this, we have shown that, if $a$ is legally \Out \wrt $L_+$ due to a support chain, then there also exists an attacker of $a$ which is labeled \In in $L_+$.
As we have argued above, this implies that $a$ is legally \Out \wrt $L_i$ as required.
\end{proof}
Next, we will show that, when a labeling $L_+$ of $\mathcal{J}_+$ is reduced to a labeling $L_i$ of $\mathcal{J}_i$, this reduced labeling is not only admissible, but even preferred.
%
%
For this, we first introduce a \emph{combined labeling} similar to the combined minimal labeling of Definition~\ref{Def:Post:CombinedMinimalLabeling}:
\begin{definition}\label{Def:Post:CombinedLabeling}
Let \ASOne, \ASTwo be two $AS$ \suchthat $AS_1|| AS_2$ and let $AS_+$ be their union.
Let \JSBAFOne, \JSBAFTwo and \JSBAFPlus be the JSBAFs corresponding to $AS_1$, $AS_2$ and $AS_+$ respectively.
Furthermore, let $L_+\in pr(\mathcal{J}_+)$ be a preferred labeling of $\mathcal{J}_+$.
Lastly, for any $i\in\{1,2\}$, let $L_i=L_+\!\!\mid\!\!_{\mathcal{A}_i}$ be the restriction of $L_+$ to $\mathcal{A}_i$ and let $L\in pr(\mathcal{J}_i)$ be a preferred labeling of $\mathcal{J}_i$ \suchthat $in(L_i)\subseteq in(L)$.
We define the \emph{combined labeling} of $L$ and $L_+$, denoted as $L_{i+}$, as follows:
\begin{align*}
	I_{i}&
	=in(L)\\
	I_{old}&
	=in(L_+)\backslash \mathcal{A}_i\\
	I_{new_0}&
	=\big{\{}a\in \mathcal{A}_+\setminus (\mathcal{A}_i\cup \mathcal{A}_j)\mid CSub(a)\subseteq I_i\cup I_{old}\big{\}}\\
	I_{new_{k+1}}&
	=\big{\{}a\in \mathcal{A}_+\setminus (\mathcal{A}_i\cup \mathcal{A}_j)\mid CSub(a)\subseteq I_{new_{k}}\big{\}}\\
	&\hspace{0.5cm}\cup I_{new_{k}}\\
	in(L_{i+})&
	=I_i\cup I_{old}\cup \bigcup\limits_{k\geq 0} I_{new_k}\\
	O_0&
	=\{a\in \mathcal{A}_+\mid\exists (b,a)\in Att_+\textnormal{ with }b\in in(L_{i+})\}\\
	O_{k+1}&
	=\big{\{}a\in \mathcal{A}_+\mid\exists (S,b)\in {\Rightarrow_+}\textnormal{ with }a\in S,\\
	&\hspace{0.6cm}S\setminus\{a\}\subseteq in(L_{i+})\textnormal{and }b\in O_{k}\big{\}}\cup O_k\\
	out(L_{i+})&
	=\bigcup\limits_{k\geq 0}O_k\\
	undec(L_{i+})&
	=\mathcal{A}_+\setminus \big(in(L_{i+})\cup out(L_{i+})\big)\\
\end{align*}	
\end{definition}
Intuitively, the combined labeling of $L_i$ and $L_+$ can be obtained by labeling all arguments in either $L_i$ or $L_+$ as \In and then iteratively accepting all arguments in $\mathcal{A}_+\setminus(\mathcal{A}_i\cup\mathcal{A}_j)$ which need to be accepted to retain admissibility.
We will now show that this construction yields an admissible labeling of $\mathcal{J}_+$ again.
As ususal, we have divided the proof into several parts to make it more accesible:
\begin{proposition}\label{Prop:Post:CombinedLabelingIsLabeling}
Let \ASOne, \ASTwo be two $AS$ \suchthat $AS_1|| AS_2$ and let $AS_+$ be their union.
Let \JSBAFOne, \JSBAFTwo and \JSBAFPlus be the JSBAFs corresponding to $AS_1$, $AS_2$ and $AS_+$ respectively.
Furthermore, let $L_+\in pr(\mathcal{J}_+)$ be a preferred labeling of $\mathcal{J}_+$.
Lastly, for any $i\in\{1,2\}$, let $L_i=L_+\!\!\mid\!\!_{\mathcal{A}_i}$ be the restriction of $L_+$ to $\mathcal{A}_i$ and let $L\in pr(\mathcal{J}_i)$ be a preferred labeling of $\mathcal{J}_i$ \suchthat $in(L_i)\subseteq in(L)$.
Then the combined labeling of $L$ and $L_+$ is a labeling.
\end{proposition}
\begin{proof}
Let $L_{i+}$ be the combined labeling of $L$ and $L_+$.
Similar to the proofs of Proposition~\ref{Prop:Post:PropagatedLabelingIsLabeling} and Proposition~\ref{Prop:Post:CombinedMinimalLabelingIsLabeling}, the actual proof proceeds via induction over the construction of $out(L_{i+})$.
However, it is clear that we only need to ensure $in(L_{i+})\cap out(L_{i+})=\emptyset$.
Furthermore, the induction step of this proof is again trivial.
We therefore only focus on the induction start, \ie we show that there cannot be some $a\in in(L_{i+})\cap O_0$.
Towards a contradiction, suppose that this does not hold.
Then there are $a,b\in in(L_{i+})$ \suchthat $(b,a)\in{\rightarrow_+}$.
Note that both $a$ and $b$ are consistent arguments.
If $a,b\in in(L)$ or $a,b\in I_{old}$, then this is easy to see by the admissibility of $L$ and $L_+$.
In the case that we have $c\in\{a,b\}$ with $c\in I_{new_k}\setminus (I_i\cup I_{old})$ for some $k\geq 0$, then we can make a model theoretic argument to show the consistency of $c$:
Let $\{c_{i_0},\dots,c_{i_k}\}=CSub(c)\cap in(L)$ and let $\{c_{old_0},\dots,c_{old_l}\}=CSub(c)\cap in(L_+)\setminus\mathcal{A}_i$.
First, we construct the set $\mathcal{C}_{new}=\bigcup\limits_{c_i\in \{c_{i_0},\dots,c_{i_k}\}}ADSub(c_i)$.
Since $L$ is a preferred labeling, we can use Lemma~\ref{Lem:Post:PrefClosedUnderSubArgs} to infer that $\mathcal{C}_{new}\subseteq in(L)$ holds.
Next, for each argument $c_{old}\in \{c_{old_0},\dots,c_{old_l}\}$, take $ADSub(c_{old})=\{\widehat{c}_0,\dots,\widehat{c}_p\}$ and let $AD_{c_{old}}=\{\widehat{c}_0',\dots,\widehat{c}_p'\}$ be their reduced versions \wrt $AS_i$ or $AS_j$ (depending on the respecitve top-rules).
By
Lemma~\ref{Lem:Post:PrefClosedUnderSubArgs} we can infer that $ADSub(c_{old})\subseteq in(L_+)$ holds and by Lemma~\ref{Lem:Post:ReducedArgumentsLabelImplication}  this implies $AD_{c_{old}}\subseteq in(L_+)$.
Now, take $\mathcal{C}_{old_i}=\big(\bigcup\limits_{c_{old}\in\{c_{old_0},\dots,c_{old_l}\}}AD_{c_{old}}\big)\cap\mathcal{A}_i$ and $\mathcal{C}_{old_j}=\big(\bigcup\limits_{c_{old}\in\{c_{old_0},\dots,c_{old_l}\}}AD_{c_{old}}\}\big)\cap\mathcal{A}_j$ (note that reduced versions of arguments are contained in either $\mathcal{A}_i$ or $\mathcal{A}_j$ by Corollary~\ref{Cor:Post:ReducedInLeftRight}).
By construction of $L$ from $L_+$, we have $in(L_+)\cap\mathcal{A}_i\subseteq in(L)$, therefore $\mathcal{C}_{old_i}\cup\mathcal{C}_{new}\subseteq in(L)$.
By admissibility of $L$, we can now infer that $(\mathcal{C}_{old_i}\cup\mathcal{C}_{new})^C$ is satisfiable.
Similarly, we can use the admissibility of $L_+$ to infer that $\mathcal{C}_{old_j}^C$ is satisfiable.
By construction, $(\mathcal{C}_{old_i}\cup\mathcal{C}_{new})^C$ and $\mathcal{C}_{old_j}^C$ are syntactically disjoint sets of formulas, therefore there exists an interpretation $I$ \suchthat $I\isModelOf\bigwedge\big(\mathcal{C}_{new}\cup \mathcal{C}_{old_i}\cup C_{old_j}\big)^C$.
It is clear that we have $\big(\mathcal{C}_{new}\cup C_{old_i}\cup C_{old_j}\big)^C=ADSub(c)^C$.
Now let $\Gamma\subseteq Sub(c)^C$.
By Proposition~\ref{Prop:Post:ADSubImpliesSubConcs} we can infer that $\big(\mathcal{C}_{new}\cup\mathcal{C}_{old_i}\cup\mathcal{C}_{old_j}\big)^C\entails\Gamma$ holds, therefore $I$ is a model of $\Gamma$.
We conclude that each $\Gamma\subseteq Sub(c)^C$ is satisfiable, therefore $c$ is consistent.
Now for the actual proof:
Similar to the proof for Proposition~\ref{Prop:Post:CombinedMinimalLabelingIsLabeling}, we will show the claim by going through the different cases for the origins of $a$ and $b$.
Note that for any $c\in\{a,b\}$, if we have $c\not\in(\mathcal{A}_i\cup\mathcal{A}_j)$, we can infer from the construction of $in(L_{i+})$ that $CSub(c)\subseteq in(L)\cup in(L_+)$ must hold.
Because both $L$ and $L_+$ are preferred labelings, this also implies $ADSub(c)\subseteq in(L)\cup in(L_+)$ by Lemma~\ref{Lem:Post:PrefClosedUnderSubArgs}.
Now for the different cases:
Because $L$ and $L_+$ are admissible labelings, it is clear that $a,b\in \mathcal{A}_i$, or $a,b\in \mathcal{A}_j$ is a contradiction.
By Proposition~\ref{Prop:Post:ClownCaseOne}, we can also infer that $a\in\mathcal{A}_i$ while $b\in\mathcal{A}_j$, and $a\in\mathcal{A}_j$ while $b\in \mathcal{A}_i$ both contradict the admissibility of $L$ and $L_+$.
These two options correspond to case one of Illustration~\ref{Illus:Post:ASOneVsASTwoVsASPlus}.
Next, suppose that we have $a\in \mathcal{A}_i$ and $b\in \mathcal{A}_+\setminus(\mathcal{A}_i\cup\mathcal{A}_j)$, which corresponds to case two of Illustration~\ref{Illus:Post:ASOneVsASTwoVsASPlus} (the case $a\in \mathcal{A}_j$ and $b\in \mathcal{A}_+\setminus(\mathcal{A}_i\cup\mathcal{A}_j)$ can be proven analogous).
Note that $b$ is not necessarily a minimal $\mathcal{A}_+$-argument, as it might be the case that there is some $b'\in CSub(b)\cap\big(\mathcal{A}_+\setminus(\mathcal{A}_i\cup\mathcal{A}_j)\big)$ \suchthat $b'$ is not a minimal $\mathcal{A}_+$-argument.
However, we can construct an \enquote{alternative} version of $b$ which is a minimal $\mathcal{A}_+$-argument:
Let $\{b_{i_0},\dots,b_{i_k}\}=Csub(b)\cap\mathcal{A}_i$, $\{b_{j_0},\dots,b_{j_l}\}=CSub(b)\cap\mathcal{A}_j$ and $\{b_{+_0},\dots,b_{+_m}\}=CSub(b)\cap\big(\mathcal{A}_+\setminus(\mathcal{A}_i\cup\mathcal{A}_j)\big)$.
For each $b_+\in\{b_{+_0},\dots,b_{+_m}\}$, let $b_+'$ be a reduced version of $b_+$ \wrt $AS_i$ or $AS_j$ (depending on $TR(b_+)$).
As we have mentioned above, we have $b_+\in in(L_+)$, therefore we can use Lemma~\ref{Lem:Post:ReducedArgumentsLabelImplication} to infer that $b_+'\in in(L_+)$ also holds.
Note that each $b_+'$ is either in $\mathcal{A}_i$ or in $\mathcal{A}_j$ by Corollary~\ref{Cor:Post:ReducedInLeftRight}.
By construction of $in(L_{i+})$ this also means that, if $b_+'\in\mathcal{A}_i$, then $L(b_+')=\In$.
Now we use $b_+^C=b_+'^C$, to construct a minimal $\mathcal{A}_+$ argument $b'$ which is of the form $b':b_{i_0},\dots,b_{i_k},b_{j_0},\dots,b_{j_l},b_{+_0}',\dots,b_{+_m}'\rightarrow b^C$.
Obviously $b'$ undercuts or gen-rebuts $a$ (depending on whether $(b,a)\in{\rightarrow_+}$ was the result of an undercut or a gen-rebut).
In the first case, we can immediately infer $(b',a)\in{\rightarrow_+}$, while in the second case we can use $DR(b')\subseteq DR(b)$ to infer that $b'\not\prec a$ and therefore again $(b',a)\in{\rightarrow_+}$.
Either way, we now have that $b'$ is a  minimal $\mathcal{A}_+$-argument which is attacking $a$.
By Proposition~\ref{Prop:Post:ClownCaseTwo}, this implies that there is $b''\in\mathcal{A}_i$ \suchthat $(b'',a)\in{\rightarrow_i}$ and $CSub(b')\!\!\mid\!\!_{AS_i}=CSub(b'')$.
Because $b\in in(L_{i+})$ by assumption and because $L$ is closed under sub-arguments by Lemma~\ref{Lem:Post:PrefClosedUnderSubArgs}, we can infer $\{b_{i_0},\dots,b_{i_k}\}\subseteq in(L)$.
As we have argued above, we also have $\big(\{b_{+_0}',\dots,b_{+_m}'\}\cap\mathcal{A}_i\big)\subseteq in(L_+)\cap\mathcal{A}_i\subseteq in(L)$.
Because $L$ was a preferred labeling, we can therefore use Lemma~\ref{Lem:Post:PrefClosedUnderSubArgs} again to infer that $CSub(b'')\subseteq in(L)$ holds.
By admissibility of $L$ we now have $b''\in in(L)$, contradicting $a\in in(L_{i+})\cap\mathcal{A}_i$.
Next, suppose that we have $a\in\mathcal{A}_+\setminus(\mathcal{A}_i\cup\mathcal{A}_j)$ while $b\in\mathcal{A}_i$ (the case of $a\in\mathcal{A}_+\setminus(\mathcal{A}_i\cup\mathcal{A}_j)$ while $b\in\mathcal{A}_j$ can be proven analogously).
Again, it is not necessarily the case that $a$ is a minimal $\mathcal{A}_+$-argument, but we can construct an \enquote{alternative} version of $a$:
Let $\{a_0,\dots,a_m\}=ADSub(a)$.
As we have argued above, we know that $\{a_0,\dots,a_m\}\subseteq in(L)\cup in(L_+)$ holds.
Let $a_0',\dots,a_m'$ be reduced versions of  $a_0,\dots,a_m$ (\wrt $AS_i$ or $AS_j$, depending on the respective top-rules).
By Lemma~\ref{Lem:Post:ReducedArgumentsLabelImplication} we can infer that $a_0',\dots a_m'\subseteq in(L)\cup in(L_+)$ also holds.
By Corollary~\ref{Cor:Post:ReducedInLeftRight} we know that these arguments $a_0',\dots a_m'$ are contained in $\mathcal{A}_i\cup\mathcal{A}_j$.
Using Proposition~\ref{Prop:Post:ADSubImpliesSubConcs},  we can now construct the minimal $\mathcal{A}_+$-argument $a'$ as follows: $a':a_0',\dots a_m'\rightarrow a^C$.
Note that, since we used $ADSub(a)$ for the construction of $a'$, we have $DR(a)=DR(a')$.
Now, if $(b,a)\in{\rightarrow_+}$ is the result of an undercut, then we can immediately infer that $(b,a')\in{\rightarrow_+}$ also holds.
On the other hand, if this attack results from a gen-rebut, then we can use Proposition~\ref{Prop:Post:OneStepAttacker} to construct the argument $b'$ of the form $b':b\rightarrow\lnot\bigwedge ADSub(a)^C$.
Since $ADSub(a)^C=ADSub(a')^C$, we can infer $(b',a')\in{\rightarrow_+}$.
Either way, there exists an argument $\widetilde{b}\in in(L)$ \suchthat $(\widetilde{b},a')\in{\rightarrow_+}$.
By Proposition~\ref{Prop:Post:ClownCaseThree}, we can now infer that one of the following two cases must hold:
\begin{enumerate}
	\item{There are $a'',\widetilde{b}'\in\mathcal{A}_i$ \suchthat $CSub(\widetilde{b}')=CSub(\widetilde{b})$, $CSub(a'')=CSub(a')\!\!\mid\!\!_{AS_i}$ and $(\widetilde{b}',a'')\in{\rightarrow_i}$.}
	\item{There are $a'',\widetilde{b}'\in\mathcal{A}_i$ \suchthat $CSub(a'')=ADSub(a')\!\!\mid\!\!_{AS_i}$, $CSub(\widetilde{b}')=CSub(\widetilde{b})$ and either $(a'',\widetilde{b}')\in{\rightarrow_i}$ or $(\widetilde{b}',a'')\in{\rightarrow_i}$.}
\end{enumerate}
By construction of $a''$, we have $CSub(a'')=\big(\{a_0',\dots,a_m'\}\cap\mathcal{A}_i\big)\subseteq in(L)$ and by admissibility of $L$ this means $a''\in in(L)$ must also hold.
Furthermore, we can use Lemma~\ref{Lem:Post:PrefClosedUnderSubArgs} and $b\in in(L)$ to infer that $CSub(\widetilde{b}')=CSub(b)\subseteq in(L)$ must also hold.
By admissibility of $L$, this also implies $\widetilde{b}'\in in(L)$.
Now both $(a'',\widetilde{b}')\in{\rightarrow_i}$ and $(\widetilde{b}',a'')\in{\rightarrow_i}$ contradict the admissibility of $L$.
Lastly, suppose that we have $a,b\in\mathcal{A}_+\setminus(\mathcal{A}_i\cup\mathcal{A}_j)$, which corresponds to case four of Illustration~\ref{Illus:Post:ASOneVsASTwoVsASPlus}.
Analogous to the case above, we first construct minimal $\mathcal{A}_+$-arguments $a',b'$ from $a$ and $b$ which are of the form $a':a_0',\dots,a_k'\rightarrow a^C$ and $b':b_0',\dots,b_l'\rightarrow b^C$, where $a_0',\dots,a_k'$ are reduced versions of the arguments in $\{a_0,\dots,a_k\}=ADSub(a)$ and $b_0',\dots,b_l'$ are reduced versions of the arguments in $\{b_0,\dots,b_l\}=ADSub(b)$.
Again, we can infer that $\{a_0',\dots,a_k'\}\subseteq in(L)\cup in(L_+)$ as well as $\{b_0',\dots,b_l'\}\subseteq in(L)\cup in(L_+)$ holds.
Furthermore, we can again infer that $(b',a')\in{\rightarrow_+}$ must hold (possibly utilizing Proposition~\ref{Prop:Post:OneStepAttacker} as well as $DR(b)=DR(b')$ and $DR(a)=DR(a')$, depending on whether or not $(b,a)\in{\rightarrow_+}$ is the result of an undercut or a gen-rebut).
Then we can apply Proposition~\ref{Prop:Post:ClownCaseFour} to infer for the arguments $a',b'$ that one of the following four cases must hold:
\begin{enumerate}
	\item{There are $a'',b''\in\mathcal{A}_i$ \suchthat $CSub(a')\!\!\mid\!\!_{AS_i}=CSub(a'')$, $CSub(b')\!\!\mid\!\!_{AS_i}=CSub(b')$ and $(b'',a'')\in{\rightarrow_i}$.}
	\item{There are $a'',b''\in\mathcal{A}_j$ \suchthat $CSub(a')\!\!\mid\!\!_{AS_j}=CSub(a'')$, $CSub(b')\!\!\mid\!\!_{AS_j}=CSub(b')$ and $(b'',a'')\in{\rightarrow_j}$}
	\item{There are $a'',b''\in\mathcal{A}_i$ \suchthat $ADSub(a')\!\!\mid\!\!_{AS_i}=CSub(a'')$, $CSub(b')\!\!\mid\!\!_{AS_i}=CSub(b'')$ and either $(a'',b'')\in{\rightarrow_i}$, or $(b'',a'')\in{\rightarrow_i}$.}
	\item{There are $a'',b''\in\mathcal{A}_j$ \suchthat $ADSub(a')\!\!\mid\!\!_{AS_j}=CSub(a'')$, $CSub(b')\!\!\mid\!\!_{AS_j}=CSub(b'')$ and either $(a'',b'')\in{\rightarrow_j}$, or $(b'',a'')\in{\rightarrow_j}$.}
\end{enumerate}
In each of these four cases we can infer a contradiction, since we either have $a'',b''\in in(L)$ (cases one and three) or $a'',b''\in in(L_+)$ (cases two and four).
With this, we have shown that each possible case for $a,b\in in(L_+)\cap O_0$ leads to a contradiction, as required.
\end{proof}
\begin{proposition}\label{Prop:Post:CombinedLabelingLegallyOut}
Let \ASOne, \ASTwo be two $AS$ \suchthat $AS_1|| AS_2$ and let $AS_+$ be their union.
Let \JSBAFOne, \JSBAFTwo and \JSBAFPlus be the JSBAFs corresponding to $AS_1$, $AS_2$ and $AS_+$ respectively.
Furthermore, let $L_+\in pr(\mathcal{J}_+)$ be a preferred labeling of $\mathcal{J}_+$.
Lastly, for any $i\in\{1,2\}$, let $L_i=L_+\!\!\mid\!\!_{\mathcal{A}_i}$ be the restriction of $L_+$ to $\mathcal{A}_i$, let $L\in pr(\mathcal{J}_i)$ be a preferred labeling of $\mathcal{J}_i$ \suchthat $in(L_i)\subseteq in(L)$ and let $L_{i+}$ be the combined labeling of $L$ and $L_+$.
We have $a\in out(L_{i+})$ iff $a$ is legally \Out \wrt $L_{i+}$.
\end{proposition}
\begin{proof}
This proof is very similar to that of Proposition~\ref{Prop:Post:CombinedMinimalLabelingLegallyOut}.
We will therefore focus on the differences here.
As in the proof for Proposition~\ref{Prop:Post:CombinedMinimalLabelingLegallyOut}, the main task is to find an \enquote{alternative reason} for an argument $a$ to be legally \Out in the case that $a$ is labeled \Out in the construction of $L_{i+}$ because of a support chain $\mathcal{C}=\big{\{}(S_0,b_0),\dots,(S_n,b_n)\big{\}}\subseteq{\Rightarrow_+}$ and an attack $(c,b_n)\in{\rightarrow_+}$, where $a\in S_0$, for all $0<k\leq n$ we have $S_{k}\setminus\{b_{k-1}\}\subseteq in(L_{i+})$, for $S_0$ we have $S_0\setminus\{a\}\subseteq in(L_{i+})$ and there is $d\in S_0\setminus\{a\}$ \suchthat $d\prec a$.
The proof idea is also the same as for Proposition~\ref{Prop:Post:CombinedMinimalLabelingLegallyOut}:
Let $S=S_0\setminus\{a\}\cup\bigcup\limits_{0<k\leq n}\big(S_k\setminus\{b_{k-1}\}\big)=\{d_0,\dots,d_m\}$, let $c'$ be an argument of the form $c':c,d_0,\dots,d_m\rightarrow\bigwedge\{c,d_0,\dots,d_m\}^C$ and for the set $ADSub(a)=\{a_0,\dots,a_l\}$, let $a'$ be of the form $a':a_0,\dots,a_l\rightarrow\lnot\bigwedge\big(ADSub(c)\cup ADSub(d_0)\cup\dots\cup ADSub(d_m)\big)^C$.
Note that by construction of $L_{i+}$, this implies $c'\in in(L_{i+})$, from which we can also infer $CSub(c')\subseteq in(L)\cup in(L_+)$ and -- by Lemma~\ref{Lem:Post:PrefClosedUnderSubArgs} -- $ADSub(c')\subseteq in(L)\cup in(L_+)$.
We have $(a',c')\in{\rightarrow_+}$ and we use this as a starting point to infer that there either exists an attack or a support chain which satisfies the conditions for $a$ to be legally \Out \wrt $L_{i+}$.
To do this, we go through the possible cases for the origins of $a'$ and $c'$.
The cases for $a',c'\in\mathcal{A}_i$ (or $a',c'\in\mathcal{A}_j$) and $a'\in\mathcal{A}_i$ while $c'\in\mathcal{A}_j$ (or $a'\in\mathcal{A}_j$ while $c'\in\mathcal{A}_i$) are identical to those of Proposition~\ref{Prop:Post:CombinedMinimalLabelingLegallyOut}.
We therefore skip these cases and assume that we have $a'\in\mathcal{A}_i$ and $c'\in\mathcal{A}_+\setminus(\mathcal{A}_i\cup\mathcal{A}_j)$, which corresponds to case three of Illustration~\ref{Illus:Post:ASOneVsASTwoVsASPlus} (the case $a'\in\mathcal{A}_j$ and $c'\in\mathcal{A}_+\setminus(\mathcal{A}_i\cup\mathcal{A}_j)$ can be proven analogous).
Note that $c'$ is not necessarily a minimal $\mathcal{A}_+$-argument.
We therefore utilize a similar technique as in the Proof of Propsotion~\ref{Prop:Post:CombinedLabelingIsLabeling} and construct an \enquote{equivalent} argument to $c'$ as follows:
Let $ADSub(c')=\{c_0,\dots,c_p\}$ and let
$c_0',\dots,c_p'$ be the reduced versions of these arguments (\wrt $AS_i$ or $AS_j$, depending on the respective top-rules).
Then we construct the minimal $\mathcal{A}_+$-argument $c''$ as follows: $c'':c_0',\dots,c_p'\rightarrow c'^C$.
Note that by using $ADSub(c')$ as a starting point, we have $DR(c')=DR(c'')$.
Furthermore, since $\{c_0,\dots,c_p\}\subseteq in(L)\cup in(L_+)$, we can use Lemma~\ref{Lem:Post:ReducedArgumentsLabelImplication} to infer that $\{c_0',\dots,c_p'\}\subseteq in(L)\cup in(L_+)$ also holds.
By construction of $in(L_{i+})$, this also implies $c''\in in(L_{i+})$.
Lastly, we want to point out that both $a'$ and $c''$ are consistent arguments.
For $a'$, we can assume consistency since we could otherwise directly construct an attack towards the original argument $a$, while for $c''$ we can assume consistency since otherwise we would be able to construct an attack from a strict argument towards $c'$, which would contradict $L_{i+}(c')=\In$ by construction of $L_{i+}$ and Proposition~\ref{Prop:Post:CombinedLabelingIsLabeling}.
Now we can use Proposition~\ref{Prop:Post:ClownCaseThree} to infer that one of the following two cases must hold:
\begin{enumerate}
	\item{There are $\widetilde{a},\widetilde{c}\in\mathcal{A}_i$ \suchthat $CSub(\widetilde{a})=CSub(a')$, $CSub(\widetilde{c})=CSub(c'')\!\!\mid\!\!_{AS_i}$ and $(\widetilde{a},\widetilde{c})\in{\rightarrow_i}$.}
	\item{There are $\widetilde{a},\widetilde{c}\in\mathcal{A}_i$ \suchthat $CSub(\widetilde{a})=CSub(a')$, $CSub(\widetilde{c})=ADSub(c'')\!\!\mid\!\!_{AS_i}$ and either $(\widetilde{a},\widetilde{c})\in{\rightarrow_i}$ or $(\widetilde{c},\widetilde{a})\in{\rightarrow_i}$.}
\end{enumerate}
Since $\{c_0',\dots,c_p'\}\subseteq in(L)\cup in(L_+)$ and since $c_0',\dots,c_p'$ are the direct sub-arguments of $c''$, we can use Lemma~\ref{Lem:Post:PrefClosedUnderSubArgs} and the admissibility of $L$ to infer that in either of the two cases, $\widetilde{c}\in in(L)$ must hold.
This implies $\widetilde{a}\in out(L)$ by admissibility of $L$.
By construction of $a'$ and $\widetilde{a}$ we have $DR(\widetilde{a})\subseteq DR(a')=DR(a)$ as well as $ADSub(\widetilde{a})^C\subseteq ADSub(a')^C=ADSub(a)^C$.
Since $a'\in\mathcal{A}_i$ by assumption and since $L$ is admissible, we can now use Proposition~\ref{Prop:Post:AlternativeReasonOut} to infer $a'\in out(L)$.
This means $a'$ is legally \Out \wrt $L$ due to an attack or a support chain.
Because $DR(a')=DR(a)$ and $ADSub(a')^C=ADSub(a)^C$, we can use the same attacker or support chain to infer that $a$ is legally \Out \wrt $L_{i+}$ (possibly with the help of Proposition~\ref{Prop:Post:OneStepAttacker}).
The case that $a'\in\mathcal{A}_+\setminus(\mathcal{A}_i\cup\mathcal{A}_j$ while $c'\in\mathcal{A}_i$ (or $c'\in\mathcal{A}_j$) is identical to that of Proposition~\ref{Prop:Post:CombinedMinimalLabelingLegallyOut} and we skip it here.
This takes care of case two of Illustration~\ref{Illus:Post:ASOneVsASTwoVsASPlus}.
For the fourth and last case of Illustration~\ref{Illus:Post:ASOneVsASTwoVsASPlus}, \ie both $a'$ and $c'$ are arguments in $\mathcal{A}_+\setminus(\mathcal{A}_i\cup\mathcal{A}_j)$, we first need to construct minimal $\mathcal{A}_+$-arguments $a''$ and $c''$ as follows:
The argument $c''$ is constructed as we did above, \ie for $ADSub(c')=\{c_0,\dots,c_p\}$ we use reduced versions $c_0',\dots,c_p'$ to construct $c''$ as $c'':c_0',\dots,c_p'\rightarrow c'^C$.
Analogously, for $ADSub(a')=\{a_0,\dots,a_l\}$, we use their reduced versions $a_0',\dots,a_l'$ an construct $a''$ as $a'':a_0',\dots,a_l'\rightarrow a'^C$.
Note that we have $\{c_0,\dots,c_p\}\subseteq in(L)\cup in(L_+)$.
By Lemma~\ref{Lem:Post:ReducedArgumentsLabelImplication} this implies $\{c_0',\dots,c_p'\}\subseteq in(L)\cup in(L_+)$, which in turn implies $c''\in in(L_{i+})$ by construction of $L_{i+}$.
Since both $a''$ and $c''$ are minimal $\mathcal{A}_+$-arguments, we can now employ Proposition~\ref{Prop:Post:ClownCaseFour} to infer that one of the following four cases holds:
\begin{enumerate}
	\item{There are $\widetilde{a},\widetilde{c}\in\mathcal{A}_i$ \suchthat $CSub(\widetilde{a})=CSub(a'')\!\!\mid\!\!_{AS_i}$, $CSub(\widetilde{c})=CSub(c'')\!\!\mid\!\!_{AS_i}$ and $(\widetilde{a},\widetilde{c})\in{\rightarrow_i}$.}
	\item{There are $\widetilde{a},\widetilde{c}\in\mathcal{A}_j$ \suchthat $CSub(\widetilde{a})=CSub(a'')\!\!\mid\!\!_{AS_j}$, $CSub(\widetilde{c})=CSub(c'')\!\!\mid\!\!_{AS_j}$ and $(\widetilde{a},\widetilde{c})\in{\rightarrow_j}$.}
	\item{There are $\widetilde{a},\widetilde{c}\in\mathcal{A}_i$ \suchthat $CSub(\widetilde{a})=CSub(a'')\!\!\mid\!\!_{AS_i}$, $CSub(\widetilde{c})=ADSub(c'')\!\!\mid\!\!_{AS_i}$ and either $(\widetilde{a},\widetilde{c})\in{\rightarrow_i}$, or $(\widetilde{c},\widetilde{a})\in{\rightarrow_i}$.}
	\item{There are $\widetilde{a},\widetilde{c}\in\mathcal{A}_j$ \suchthat $CSub(\widetilde{a})=CSub(a'')\!\!\mid\!\!_{AS_j}$, $CSub(\widetilde{c})=ADSub(c'')\!\!\mid\!\!_{AS_j}$ and either $(\widetilde{a},\widetilde{c})\in{\rightarrow_j}$, or $(\widetilde{c},\widetilde{a})\in{\rightarrow_j}$.}
\end{enumerate}
Since $\{c_0',\dots,c_p'\}\subseteq in(L)\cup in(L_+)$, we can use Lemma~\ref{Lem:Post:PrefClosedUnderSubArgs} to infer that $CSub(c'')\subseteq in(L)\cup in(L_+)$ and $ADSub(c'')\subseteq in(L)\cup in(L_+)$ holds, which in turn implies $\widetilde{c}\in in(L)$ (if $\widetilde{c}\in\mathcal{A}_i$) or $\widetilde{c}\in in(L_+)$ (if $\widetilde{c}\in\mathcal{A}_j$) by admissibility of $L$ and $L_+$.
In the first case, we can infer $\widetilde{a}\in out(L)$ by the admissibility of $L$, while in the second case we can infer $\widetilde{a}\in out(L_+)$ by the admissibility of $L_+$.
Either way, we can use the fact that $\widetilde{a}$ is legally \Out \wrt $L$ or $L_+$ and apply the same reasoning as before to infer that $a$ is legally \Out \wrt $L_{i+}$.
Similar to our previous argumentations, this is possible because we have $DR(\widetilde{a})\subseteq DR(a'')\subseteq DR(a')=DR(a)$ as well as $ADSub(\widetilde{a})^C\subseteq ADSub(a'')^C\subseteq ADSub(a')^C=ADSub(a)^C$.
This finishes the proof.
\end{proof}
\begin{proposition}\label{Prop:Post:CombinedLabelingLegallyIn}
Let \ASOne, \ASTwo be two $AS$ \suchthat $AS_1|| AS_2$ and let $AS_+$ be their union.
Let \JSBAFOne, \JSBAFTwo and \JSBAFPlus be the JSBAFs corresponding to $AS_1$, $AS_2$ and $AS_+$ respectively.
Furthermore, let $L_+\in pr(\mathcal{J}_+)$ be a preferred labeling of $\mathcal{J}_+$.
Lastly, for any $i\in\{1,2\}$, let $L_i=L_+\!\!\mid\!\!_{\mathcal{A}_i}$ be the restriction of $L_+$ to $\mathcal{A}_i$, let $L\in pr(\mathcal{J}_i)$ be a preferred labeling of $\mathcal{J}_i$ \suchthat $in(L_i)\subseteq in(L)$ and let $L_{i+}$ be the combined labeling of $L$ and $L_+$.
If $a\in in(L_{i+})$, then $a$ is legally \In \wrt $L_{i+}$.
\end{proposition}
\begin{proof}
Let $a\in in(L_{i+})$.
Similar to the proof of Proposition~\ref{Prop:Post:CombinedMinimalLabelingLegallyIn}, it is clear from the construction of $in(L_{i+})$ that, for any support $(S,c)\in{\Rightarrow_+}$ with $a\in S$ and $a\preceq b$ for all $b\in S\setminus\{a\}$, one of the items $1.a-1.c$ of Definition~\ref{Def:JSBAF:LegalLabeling} holds.
Thus, we only need to ensure that all attackers of $a$ are labeled \Out in $L_{i+}$.
Let $(b,a)\in{\rightarrow_+}$.
Towards a contradiction, assume that $L_{i+}(b)\neq\Out$.
Note that, if $L_{i+}(b)=\In$, then $L_{i+}(a)=\Out$ by construction of $L_{i+}$, which contradicts $L_{i+}(a)=\In$ by Proposition~\ref{Prop:Post:CombinedLabelingIsLabeling}.
Therefore, we assume that $L_{i+}(b)=\Undec$ holds.
As usual, we go through the possible cases for the origins of $a$ and $b$ to derive a contradiction in each of them.
Note that we can assume that both $a$ and $b$ are consistent, since they would otherwise be attacked by a strict argument, which again contradicts our assumptions on their labels (any strict argument $x$ is contained in $\mathcal{A}_i$ and accepted by the admissible labeling $L$ which implies $x\in in(L_{i+})$ and by construction of $out(L_{i+})$ any argument $y$ attacked by $x$ is labeled \Out in $L_{i+}$).
It is obvious that $a,b\in\mathcal{A}_i$ or $a,b\in\mathcal{A}_j$ contradicts the admissibility of $L$ or $L_+$.
Furthermore, if $a\in\mathcal{A}_i$ and $b\in\mathcal{A}_j$, then this corresponds to case one of Illustration~\ref{Illus:Post:ASOneVsASTwoVsASPlus} and we can utilize Proposition~\ref{Prop:Post:ClownCaseOne} to infer that either $a$ or $b$ are inconsistent.
Again, this inconsistency would contradict our assumptions on the labels of $a$ and $b$.
The case $a\in\mathcal{A}_j$ and $b\in\mathcal{A}_i$ can be proven analogous.
Next, suppose that we have $a\in\mathcal{A}_i$ and $b\in\mathcal{A}_+\setminus(\mathcal{A}_i\cup\mathcal{A}_j)$, which corresponds to case two of Illustration~\ref{Illus:Post:ASOneVsASTwoVsASPlus}.
If $(b,a)\in{\rightarrow_+}$ is the result of an undercut, then this implies $Atoms(b^C)\subseteq Atoms(AS_i)$.
By  Proposition~\ref{Prop:Post:ReducedArgsExistence} there now exists an argument $b'$ which is a reduced version of $b$ \wrt $AS_i$.
By Corollary~\ref{Cor:Post:ReducedInLeftRight} we have $b'\in \mathcal{A}_i$, therefore admissibility of $L$ yields $b'\in out(L)$, which also implies $b'\in out(L_{i+})$ by $in(L)\subseteq in(L_{i+})$.
By Proposition~\ref{Prop:Post:CombinedLabelingLegallyOut} we know that $x\in out(L_{i+})$ iff $x$ is legally \Out \wrt $L_{i+}$, which means we can apply Lemma~\ref{Lem:Post:ReducedArgumentsLabelImplication} to infer $b\in out(L_{i+})$, contradicting our assumption $b\in undec(L_{i+})$.
Now suppose that $(b,a)\in{\rightarrow_+}$ is the result of a gen-rebut.
Then we can use Proposition~\ref{Prop:Post:OneStepAttacker} to construct $b'$ which is of the form $b':b\rightarrow\lnot\bigwedge ADSub(a)^C$.
By Proposition~\ref{Prop:Post:ReducedArgsExistence} there exists an argument $b''$ whcih is a reduced version of $b'$ \wrt $AS_i$.
We have $DR(b'')\subseteq DR(b')$, thus we can infer $(b'',a)\in{\rightarrow_i}$.
Now admissibility of $L$ implies $b''\in out(L)$.
As in the undercut-case we can now use Proposition~\ref{Prop:Post:CombinedLabelingLegallyOut} in order to apply Lemma~\ref{Lem:Post:ReducedArgumentsLabelImplication} and infer that $b'\in out(L_{i+})$ must hold.
This implies $b\in out(L_{i+})$ by construction of $L_{i+}$, which contradicts our assumption on $L_{i+}(b)$.
The case $a\in\mathcal{A}_j$ and $b\in\mathcal{A}_+\setminus(\mathcal{A}_i\cup\mathcal{A}_j)$ can be proven analogous.

Next, suppose that we have $a\in\mathcal{A}_+\setminus(\mathcal{A}_i\cup\mathcal{A}_j)$ while $b\in\mathcal{A}_i$.
This corresponds to case three of Illustration~\ref{Illus:Post:ASOneVsASTwoVsASPlus}.
If $(b,a)\in{\rightarrow_+}$ is the result of an undercut, then this attack must also be directed towards some $a'\in CSub(a)\subseteq in(L)\cup in(L_+)$.
If $a'\in\mathcal{A}_i$, then we can use the admissibility of $L$ to infer that $b\in out(L)$ holds and by construction of $L_{i+}$, this implies $b\in out(L_{i+})$, contradicting our assumption $b\in undec(L_{i+})$.
On the other hand, if $a'\in\mathcal{A}_+\setminus\mathcal{A}_i$, then we must have $b\in out(L_+)$ by admissibility of $L_+$, which also means $b\in out(L_+\!\!\mid\!\!_{\mathcal{A}_i})$.
By Proposition~\ref{Prop:Post:ASPlusToASI}, $L_+\!\!\mid\!\!_{\mathcal{A}_i}$ is an admissible labeling, therefore there must exist an attack or support chain in $\mathcal{J}_i$, \suchthat $b$ is labeled \Out because of this attack or support chain..
By assumption, we have $in(L_+\!\!\mid\!\!_{\mathcal{A}_i})\subseteq in(L)\subseteq in(L_{i+})$, therefore $b\in out(L_{i+})$, contradicting our assumption $b\in undec(L_{i+})$.
Now suppose that $(b,a)\in{\rightarrow_+}$ is the result of a gen-rebut.
Note that $a$ is not necessarily a minimal $\mathcal{A}_+$-argument.
Similar to the proof of Proposition~\ref{Prop:Post:CombinedLabelingLegallyOut}, we therefore again use $ADSub(a)=\{a_0,\dots,a_m\}$ and their reduced versions $a_0',\dots,a_m'$ to construct $a':a_0',\dots,a_m'\rightarrow a^C$, which is a minimal $\mathcal{A}_+$-argument.
By Lemma~\ref{Lem:Post:PrefClosedUnderSubArgs} we have $ADsub(a)\subseteq in(L)\cup in(L_+)$ and by Lemma~\ref{Lem:Post:ReducedArgumentsLabelImplication} this implies $a_0',\dots,a_m'\subseteq in(L)\cup in(L_+)$.
By construciton of $L_{i+}$ we now have $a'\in in(L_{i+})$.
Note that $DR(a)=DR(a')$, therefore $(b,a')\in{\rightarrow_+}$.
Now we can apply Proposition~\ref{Prop:Post:ClownCaseThree} to infer that one of the following cases must hold:
\begin{enumerate}
	\item{There are $\widetilde{a},\widetilde{b}\in\mathcal{A}_i$ \suchthat $CSub(\widetilde{a})=CSub(a')\!\!\mid\!\!_{AS_i}$, $CSub(\widetilde{b})=CSub(b)$ and $(\widetilde{b},\widetilde{a})\in{\rightarrow_i}$.}
	\item{There are $\widetilde{a},\widetilde{b}\in\mathcal{A}_i$ \suchthat $CSub(\widetilde{a})=ADSub(a')\!\!\mid\!\!_{AS_i}$, $CSub(\widetilde{b})=CSub(b)$ and either $(\widetilde{b},\widetilde{a})\in{\rightarrow_i}$, or $(\widetilde{a},\widetilde{b})\in{\rightarrow_i}$.}
\end{enumerate}
Since $a'\in in(L_{i+})$, we have $CSub(a')\subseteq in(L)\cup in(L_+)$ by construction of $L_{i+}$ and by Lemma~\ref{Lem:Post:PrefClosedUnderSubArgs} we have $ADSub(a')\subseteq in(L)\cup in(L_+)$, which implies $\widetilde{a}\in in(L)$ by admissibility of $L$.
Now, again by admissibility of $L$, we have $\widetilde{b}\in out(L)$.
Note that $CSub(\widetilde{b})=CSub(b)$ implies $ADSub(\widetilde{b})=ADSub(b)$.
Since $\widetilde{b},b\in\mathcal{A}_i$ and since $L$ is admissible, we can now utilize Proposition~\ref{Prop:Post:AlternativeReasonOut} to infer $L(b)=\Out$.
By admissibility of $L$ this means $b$ is labeled \Out in $L$ either due to an attack or due to a support chain.
By $in(L)\subseteq in(L_{i+})$, we can infer that $b\in out(L_{i+})$ due to the same attacker or support chain.
This contradicts our assumption $L\in undec(L_{i+})$.
The case $a\in\mathcal{A}_+\setminus(\mathcal{A}_i\cup\mathcal{A}_j)$ while $b\in\mathcal{A}_j$ can be proven similarly.
Finally, suppose that $a,b\in\mathcal{A}_+\setminus(\mathcal{A}_i\cup\mathcal{A}_j)$ holds.
Again, we start by assuming $(b,a)\in{\rightarrow_+}$ is the result of an undercut.
Then this attack is directed towards some $x\in CSub(a)\subseteq in(L)\cup in(L_+)$.
If $x\in in(L_+)$, then we can infer $b\in out(L_+)$ by admissibility of $L_+$, which implies $b\in out(L_{i+})$ by construction of $in(L_{i+})$ from $in(L_+)$.
On the other hand, if $x\in in(L)$, then by Proposition~\ref{Prop:Post:ReducedArgsExistence} there exists $b'$, which is a reduced version of $b$ \wrt $AS_i$.
We have $b'\in\mathcal{A}_i$ by Corollary~\ref{Cor:Post:ReducedInLeftRight}, therefore $b'\in out(L)$ by admissibility of $L$.
By construction of $in(L_{i+})$ from $in(L)$, this implies $b'\in out(L_{i+})$.
As before, we can now utilize Proposition~\ref{Prop:Post:CombinedMinimalLabelingLegallyOut} to apply Lemma~\ref{Lem:Post:ReducedArgumentsLabelImplication}, and infer that $b\in out(L_{i+})$ also holds.
This contradicts our assumption $b\in undec(L_{i+})$.
Lastly, assume that $(b,a)\in{\rightarrow_+}$ is the result of a gen-rebut.
As before, we use $ADSub(a)=\{a_0,\dots,a_k\}$ and the reduced versions $a_0',\dots,a_k'$ of these arguments to construct the minimal $\mathcal{A}_+$-argument $a'$ which is of the form $a':a_0',\dots,a_m'\rightarrow a^C$.
Similarly, we use $ADSub(b)=\{b_0,\dots,b_l\}$ and the reduced versions $b_0',\dots,b_l'$ of these arguments to construct the minimal $\mathcal{A}_+$-argument $b'$ which is of the form $b':b_0',\dots,b_m'\rightarrow b^C$.
Because $DR(a)=DR(a')$ and $DR(b)=DR(b')$, we can infer $(b',a')\in{\rightarrow_+}$.
Similar to before, we can also infer $a'\in in(L_{i+})$ by Lemma~\ref{Lem:Post:PrefClosedUnderSubArgs}, Lemma~\ref{Lem:Post:ReducedArgumentsLabelImplication} and by construction of $in(L_{i+})$.
Now we can apply Proposition~\ref{Prop:Post:ClownCaseFour} to infer that one of the following four cases holds:
\begin{enumerate}
	\item{There are $\widetilde{a},\widetilde{b}\in\mathcal{A}_i$ \suchthat $CSub(\widetilde{a})=CSub(a')\!\!\mid\!\!_{AS_i}$, $CSub(\widetilde{b})=CSub(b')\!\!\mid\!\!_{AS_i}$ and $(\widetilde{b},\widetilde{a})\in{\rightarrow_i}$.}
	\item{There are $\widetilde{a},\widetilde{b}\in\mathcal{A}_j$ \suchthat $CSub(\widetilde{a})=CSub(a')\!\!\mid\!\!_{AS_j}$, $CSub(\widetilde{b})=CSub(b')\!\!\mid\!\!_{AS_j}$ and $(\widetilde{b},\widetilde{a})\in{\rightarrow_j}$.}
	\item{There are $\widetilde{a},\widetilde{b}\in\mathcal{A}_i$ \suchthat $CSub(\widetilde{a})=ADSub(a')\!\!\mid\!\!_{AS_i}$, $CSub(\widetilde{b})=CSub(b')\!\!\mid\!\!_{AS_i}$ and either $(\widetilde{b},\widetilde{a})\in{\rightarrow_i}$, or $(\widetilde{a},\widetilde{b})\in{\rightarrow_i}$.}
		\item{There are $\widetilde{a},\widetilde{b}\in\mathcal{A}_j$ \suchthat $CSub(\widetilde{a})=ADSub(a')\!\!\mid\!\!_{AS_j}$, $CSub(\widetilde{b})=CSub(b')\!\!\mid\!\!_{AS_j}$ and either $(\widetilde{b},\widetilde{a})\in{\rightarrow_j}$, or $(\widetilde{a},\widetilde{b})\in{\rightarrow_j}$.}
\end{enumerate}
Similar to before, we can infer by the  construction of $L_{i+}$ and by Lemma~\ref{Lem:Post:PrefClosedUnderSubArgs} that $a'\in in(L_{i+})$ implies $\widetilde{a}\in in(L)$ (in cases 1 and 3) or $\widetilde{a}\in in(L_+)$ (in cases 2 and 4).
Either way, this again implies $\widetilde{b}\in out(L)$ (in cases 1 and 3) or $\widetilde{b}\in out(L_+)$ (in cases 2 and 4) by admissibility of $L$ and $L_+$.
Note that, by construction of $\widetilde{b}$ and $b'$, we have $DR(\widetilde{b})\subseteq DR(b')=DR(b)$ as well as $ADSub(\widetilde{b})^C\subseteq ADSub(b')^C=ADSub(b)^C$.
We only show the case of $\widetilde{b}\in out(L)$ here, as the case for $\widetilde{b}\in out(L_+)$ can be proven analogously:
Assume that we have $\widetilde{b}\in out(L)$ because of an attack $(c,\widetilde{b})\in{\rightarrow_i}$ with $c\in in(L)$.
Regardless of whether this attack is the result of an undercut or a gen-rebut, we can use $DR(\widetilde{b})\subseteq DR(b)$, $ADSub(\widetilde{b})^C\subseteq ADSub(b)^C$, Proposition~\ref{Prop:Post:OneStepAttacker} and the admissibility of $L$ to infer that there is an attack $(c',b)\in{\rightarrow_+}$ from an attacker $c'\in in(L)\subseteq in(L_{i+})$.
By construction of $L_{i+}$ this implies $b\in out(L_{i+})$, contradicting our assumption $b\in undec(L_{i+})$.
Lastly, suppose that $\widetilde{b}\in out(L)$ because of a support chain $\mathcal{C}=\big{\{}(S_0,b_0),\dots,(S_n,b_n)\big{\}}\subseteq{\Rightarrow_i}$ and an attack $(c,b_n)\in{\rightarrow_i}$ with $\widetilde{b}\in S_0$ and $c\in in(L)$.
Similar to our proof for Proposition~\ref{Prop:Post:CombinedMinimalLabelingLegallyOut}, we can now use the set $S=\bigcup\limits_{0\leq k\leq n}S_k\cap in(L)$ to create a new support chain $\mathcal{C}'=\big{\{}(S\cup\{b\},b')\big{\}}\subseteq{\Rightarrow_+}$.
If $(c,b_n)\in{\rightarrow_i}$ is the result of an undercut we can infer from $S\subseteq in(L)$ that $(c,\widetilde{b})\in{\rightarrow_i}$ holds.
By $DR(\widetilde{b})\subseteq DR(b)$ we then have $(c,b)\in{\rightarrow_+}$ and by $c\in in(L)\subseteq in(L_{i+})$ we can infer that $L_{i+}(b)=\Out$, contradicting our assumption $L_{i+}(b)=\Undec$.
On the other hand, if $(c,b_n)\in{\rightarrow_i}$ is the result of a gen-rebut, then we can use Proposition~\ref{Prop:Post:OneStepAttacker} to construct an argument $c'$ which is of the form $c':c\rightarrow\lnot\bigwedge ADSub(b_n)^C$.
Because $DR(c')=DR(c)$, $DR(\widetilde{b})\subseteq DR(b)$ and $ADSub(\widetilde{b})^C\subseteq ADSub(b)^C$, we can infer that we also have $(c',b')\in{\rightarrow_+}$.
By construction of $L_{i+}$, we have $c'\in in(L_{i+})$.
By construciton of $out(L_{i+})$, we now have $b'\in O_0\subseteq out(L_{i+})$ and -- because $S\subseteq in(L)\subseteq in(L_{i+})$ -- $b\in O_1\subseteq out(L_{i+})$.
This contradicts our assumption $b\in undec(L_{i+})$ and finishes the proof.
\end{proof}
\begin{proposition}\label{Prop:Post:CombinedLabelingAdmissible}
Let \ASOne, \ASTwo be two $AS$ \suchthat $AS_1|| AS_2$ and let $AS_+$ be their union.
Let \JSBAFOne, \JSBAFTwo and \JSBAFPlus be the JSBAFs corresponding to $AS_1$, $AS_2$ and $AS_+$ respectively.
Furthermore, let $L_+\in pr(\mathcal{J}_+)$ be a preferred labeling of $\mathcal{J}_+$.
Lastly, for any $i\in\{1,2\}$, let $L_i=L_+\!\!\mid\!\!_{\mathcal{A}_i}$ be the restriction of $L_+$ to $\mathcal{A}_i$, let $L\in pr(\mathcal{J}_i)$ be a preferred labeling of $\mathcal{J}_i$ \suchthat $in(L_i)\subseteq in(L)$ and let $L_{i+}$ be the combined labeling of $L$ and $L_+$.
Then $L_{i+}\in adm(\mathcal{J}_+)$.
\end{proposition}
\begin{proof}
By Proposition~\ref{Prop:Post:CombinedLabelingIsLabeling} we know that $L_{i+}$ is a labeling of $\mathcal{J}_+$.
By Proposition~\ref{Prop:Post:CombinedLabelingLegallyOut} we know that $a\in out(L_{i+})$ iff $a$ is legally \Out \wrt $L_{i+}$.
By Proposition~\ref{Prop:Post:CombinedLabelingLegallyIn} we know that any $a\in in(L_{i+})$ is legally \In \wrt $L_{i+}$.
By construction of $in(L_{i+})$, we have $in(L)\subseteq in(L_{i+})$, as well as $in(L_+)\subseteq in(L_{i+})$.
By admissibility of $L$ and $L_+$ we can infer that $STR_{\mathcal{J}_+}\subseteq in(L_{i+})$ holds.
We conclude that $L_{i+}$ is an admissible labeling of $\mathcal{J}_+$.
\end{proof}
With this, we are finally ready to prove Theorem~\ref{Theo:Post:NonInterCrashResist}.
We begin by proving that preferred semantics of \DAOM satisfies non-interference:
\begin{lemma}\label{Lem:Post:NonInterference}
Preferred semantics of \DAOM satisfies non-interference.
\end{lemma}
\begin{proof}
Let \ASOne and \ASTwo be two AS \suchthat $AS_1||AS_2$.
Furthermore, let \ASPlus be the union of $AS_1$ and $AS_2$.
We need to show that for each $i\in\{1,2\}$, $\mathcal{C}_{pr}(AS_i)\!\!\mid\!\!_{Atoms(AS_i)}=\mathcal{C}_{pr}(AS_+)\!\!\mid\!\!_{Atoms_{(AS_i)}}$ holds.
Let \JSBAFOne, \JSBAFTwo be the JSBAFs corresponding to $AS_1$ and $AS_2$ and let \JSBAFPlus be the JSBAFs corresponding to $AS_+$.
Lastly, let $j\in\{1,2\}$ with $i\neq j$.
$\subseteq$:
Let $E_i\in pr(AS_i)$ be a preferred extension of $AS_i$ and let $L_i\in pr(\mathcal{J}_i)$ be the preferred labeling of $\mathcal{J}_i$ corresponding to $E_i$.
We first show that there is a preferred extension $E_+\in pr(AS_+)$ \suchthat $E_i=E_+\cap \mathcal{A}(AS_i)$:
Let $SIM_j$ be the strict including minimal labeling of $\mathcal{J}_j$ and let $L_+$ be the combined minimal labeling of $L_i$ and $SIM_j$.
By Proposition~\ref{Prop:Post:CombinedMinimalLabelingAdmissble} we know that $L_+$ is an admissible labeling of $\mathcal{J}_+$.
That means there is a preferred labeling $L_{pr+}\in pr(\mathcal{J}_+)$ \suchthat $in(L_i)\subseteq in(L_+)\subseteq in(L_{pr+})$.
Towards a contradiction, assume that $in(L_i)\subset in(L_{pr+})\!\!\mid\!\!_{\mathcal{A}(AS_i)}$.
Take $L'_i=L_{pr+}\!\!\mid\!\!_{\mathcal{A}(AS_i)}$.
By Proposition~\ref{Prop:Post:ASPlusToASI}, we know that $L'_i$ is an admissible labeling of $\mathcal{J}_i$.
It is clear that $in(L_i)\subseteq in(L_i')$ holds.
Now, if we have $in(L_i)\subset in(L_{pr+})\!\!\mid\!\!_{\mathcal{A}(AS_i)}$, then this means $in(L_i)\subset in(L'_i)$, which contradicts that $L_i$ is a preferred labeling.
Therefore, we must have $in(L_i)=in(L_{pr+})\!\!\mid\!\!_{\mathcal{A}(AS_i)}$.
We conclude that for $E_i\in pr(AS_i)$ there exists $E_+\in pr(AS_+)$ \suchthat $E_i=E_+\cap \mathcal{A}(AS_i)$.
Now, let $\Gamma_i=E_i^C$, $\Gamma_+=E_+^C$ be the sets of conclusions of $E_i$ and $E_+$.
Clearly, we have $\Gamma_i\!\!\mid\!\!_{Atoms(AS_i)}\subseteq\Gamma_+\!\!\mid\!\!_{Atoms(AS_i)}$.
Towards a contradiction, assume that we have $\Gamma_i\!\!\mid\!\!_{Atoms(AS_i)}\subset\Gamma_+\!\!\mid\!\!_{Atoms(AS_i)}$.
Let $\phi\in\Gamma_+\!\!\mid\!\!_{Atoms(AS_i)}$ \suchthat $\phi\not\in\Gamma_i\!\!\mid\!\!_{Atoms(AS_i)}$ and let $a$ be the corresponding argument, \ie $a^C=\phi$.
Then we have $Atoms(a^C)\subseteq Atoms(AS_i)$.
By Proposition~\ref{Prop:Post:ReducedArgsExistence}, there exists a reduced version $a'$ of $a$.
By Corollary~\ref{Cor:Post:ReducedInLeftRight}, $a'\in \mathcal{A}(AS_i)$ and by Lemma~\ref{Lem:Post:ReducedArgumentsLabelImplication}, we can infer that $a'\in E_i$.
Since $a^C=a'^C$, we infer $\phi\in\Gamma_i\!\!\mid\!\!_{Atoms(AS_i)}$, contradicting our assumption.
With this, we conclude that $\mathcal{C}_{pr}(AS_i)\!\!\mid\!\!_{Atoms(AS_i)}\subseteq \mathcal{C}_{pr}(AS_+)\!\!\mid\!\!_{Atoms(AS_i)}$ holds.
$\supseteq$:
Take $E_+\in pr(AS_+)$ and let $L_+$ be the corresponding labeling of $\mathcal{J}_+$.
Let $L_i=L_+\!\!\mid\!\!_{\mathcal{A}_i}$ be the restriction of $L_+$ to $\mathcal{A}_i$.
By Proposition~\ref{Prop:Post:ASPlusToASI}, we know that $L_i$ is an admissible labeling of $\mathcal{J}_i$.
Towards a contradiction, assume that $L_i$ is not a preferred labeling.
Then there exists $L\in pr(\mathcal{J}_i)$ \suchthat $in(L_i)\subset in(L)$.
Now, let $L_{i+}$ be the combined labeling of $L$ and $L_+$.
By Proposition~\ref{Prop:Post:CombinedLabelingAdmissible}, we know that $L_{i+}$ is an admissible labeling of $\mathcal{J}_+$.
Since $in(L_i)\subset in(L)$, we now have $in(L_+)\subset in(L_+')$, contradicting that $L_+$ is an admissible labeling.
We conclude that for $E_+\in pr(AS_+)$, there exists $E_i\in pr(AS_i)$ \suchthat $E_i=E_+\cap\mathcal{A})(AS_i)$.
Now, take $\Gamma_+=E_+^C$ and $\Gamma_i=E_i^C$.
Clearly, we have $\Gamma_i\!\!\mid\!\!_{Atoms(AS_i)}\subseteq\Gamma_+\!\!\mid\!\!_{Atoms(AS_i)}$.
Towards a contradiction, assume $\Gamma_i\!\!\mid\!\!_{Atoms(AS_i)}\subset\Gamma_+\!\!\mid\!\!_{Atoms(AS_i)}$.
Let $\phi\in\Gamma_+\!\!\mid\!\!_{Atoms(AS_i)}$ \suchthat we have $\phi\not\in\Gamma_i\!\!\mid\!\!_{Atoms(AS_i)}$ and let $a$ be the corresponding argument, \ie $a^C=\phi$.
By Proposition~\ref{Prop:Post:ReducedArgsExistence} there exists a reduced version $a'$ of $a$ and by Corollary~\ref{Cor:Post:ReducedInLeftRight} we know that $a'\in\mathcal{A}_i$ holds.
By Lemma~\ref{Lem:Post:ReducedArgumentsLabelImplication} we can infer that $a'\in E_i$ holds.
Now $\phi\in\Gamma_i\!\!\mid\!\!_{Atoms(AS_i)}$, contradicting our assumption.
We conclude that $\mathcal{C}_{pr}(AS_i)\!\!\mid\!\!_{Atoms(AS_i)}\supseteq \mathcal{C}_{pr}(AS_1\uplus AS_2)\!\!\mid\!\!_{Atoms(AS_i)}$ holds.
This finishes the proof.
\end{proof}
By using Proposition~\ref{Prop:ASPIC:DAOMNonTrivial}, Proposition~\ref{Prop:ASPIC:NonInterferenceImpliesCrashResistance} and Lemma~\ref{Lem:Post:NonInterference}, we can now infer that \DAOM satisfies crash-resistance under preferred semantics:
\begin{corollary}
\DAOM satisfies crash-resistance under preferred semantics.
\end{corollary}

\end{technicalReportOnly}

\begin{technicalReportOnly}
\section{Grounded Semantics}
\label{Sec:Grounded}
In this section, we give a definition of a grounded semantics for JSBAFs.
We need to point out that this semantics was developed somewhat independently from the preferred semantics for JSBAFs.
In particular, the notion of preferences between arguments is not considered in our grounded semantics.
We therefore begin this section by giving an overview of some definitions and results for JSBAFs without preferences.
For brevity, we do not include any proofs for these preliminary results.
However, the actual proofs are very similar to those for JSBAFs with preferences, which are contained in Section~\ref{Sec:JSBAF}.

\subsection{Preliminaries}
We define JSBAFs without preferences as follows:
\begin{definition}\label{Def:Ground:Triple}
A JSBAF is a triple \JSBAFGR, with:
\begin{itemize}
	\item {$Args$, the set of \emph{arguments}}
	\item {$Att\subseteq Args\times Args$, the set of \emph{attacks} between arguments}
	\item {$Supp\subseteq 2^{Args}\times Args$, the set of \emph{supports} between arguments}
\end{itemize}
We denote the set of all possible JSBAFs without preferences as $\mathfrak{J}$.
\end{definition}
Note that we use the notations $Args$, $Att$ and $Supp$ when talking about JSBAFs without preferences, as opposed to $\mathcal{A}$, ${\rightarrow}$ and ${\Rightarrow}$ when talking about JSBAFs with preferences.
We do so to avoid confusion between the two.
Furthermore, throughout this section we will denote arguments by capital letters $A,B,C,\dots$ to distinguish these JSBAFs from those considered in Section~\ref{Sec:JSBAF}.
In the context of JSBAFs without preferences, strict arguments and chains of supports are defined analogous to the case of JSBAFs with supports (see Section~\ref{Sec:JSBAF}).
Given a chain of supports $\mathcal{C}=\{(S_0,B_0),...,(S_n,B_n)\}$ and some $A\in S_0$, we will sometimes say that $\mathcal{C}$ \emph{starts at $A$} and \emph{ends at $B_n$} or that $\mathcal{C}$ \emph{starts with $(S_0,B_0)$}.
If $\mathcal{C}=\{(S_0,B_0)\}$ we call $\mathcal{C}$ \emph{trivial}.
We use the following additional notation regarding chains of supports:
\begin{definition}\label{Def:Ground:SuppPathAndLastStrictAncestor}
Let \JSBAFGR be some JSBAF and let $A,B\in Args$ be arguments of \J.
We say there is a \emph{support-path} of length $n\geq 1$ from $A$ to $B$ in $Supp$ iff there is a chain of supports $\mathcal{C}\subseteq Supp$ which starts at $A$ and ends at $B$.\\
The \emph{support ancestors} of $B$ in $Supp$, denoted $SA(B)$, are defined as:
$SA(B)=\{A\in Args \mid \text{there is a support-path from }A\text{ to } B\text{ in } Supp\}$.
The \emph{support children} of $A$ in $Supp$, denoted $SC(A)$, are defined as:
$SC(A)=\{B\in Args \mid \text{there is a support-path from }A\text{ to } B\text{ in } Supp\}$.
\end{definition}
Similar to the case of JSBAFs with preferences, we only consider JSBAFs that can be constructed on the basis of an argumentation system (while ignoring preferences between arguments).
To this end, we use the following definition:
\begin{definition}\label{Def:Ground:JDAMinus}
Let \JSBAFGR be some JSBAF.
We say that $\mathcal{J}$ \emph{is constructible according to \DAOM} iff the following conditions hold:
\begin{enumerate}
	\item {For any $A\in Args$, we have $A\not\in SA(A)$.}
	\item {Let $S,S'\subseteq Args$ and let $B\in Args$.
	If $(S,B)\in Supp$ and $(S',B)\in Supp$, then $S=S'$.}	
	\item {For any $(S,B)\in Supp$, we have $|S|<\infty$.}
	\item {If $A\in STR_{\mathcal{J}}$, then $A$ is unattacked in $\mathcal{J}$.}
\end{enumerate}
We denote by $\mathfrak{J}_{DA^{\ominus}}\subseteq\mathfrak{J}$ the set of all JSBAFs \emph{constructible according to \DAOM}.
\end{definition}
As for the case of JSBAFs with supports, we use a labeling based approach for our semantics, with the possible labels identical to those we used for JSBAFs with preferences (see Section~\ref{Sec:JSBAF}).
We define legal labelings in the context of JSBAFs without preferences by utilizing the following auxiliary definition:
\begin{definition}\label{Def:Ground:LabelOrdering}
Let $\mathbb{L}$ be a multiset of labels and $l$ a label. Then $\mathbb{L}\leq l$ iff any of the following conditions hold:
	\begin{enumerate}
		\item {$l=\In$}
		\item {$l=\Undec$ and either $\Undec\in\mathbb{L}$ or $\Out\in\mathbb{L}$}
		\item {$l=\Out$ and $\Out\in\mathbb{L}$ or $|\mathbb{L}(Undec)|\geq 2$.\footnote{Here, $|\mathbb{L}(Undec)|$ denotes the number of \Undec-labels in the multiset $\mathbb{L}$.}}
	\label{Def:Ground:LabelOrdering:ItemTwoUndecLeqOut}
	\end{enumerate}
\end{definition}
Note that with this definition, for $\mathbb{L}=\emptyset$ we have $\mathbb{L}\leq l$, iff $l=\In$.
Based on this ordering, we now define what it means for an argument to be \emph{legally} \In, \Out\ and \Undec:
\begin{definition}\label{Def:Ground:LegalLabeling}
Let \JSBAFGR\ be a JSBAF \suchthat \JDA, let  $A,B,C\in Args$ be arguments, $S\subseteq Args$ a set of arguments and let $Lab$ be a labeling of \J.
\begin{enumerate}
	 \item {$A$ is \emph{legally \In}\ \wrt $Lab$, iff
	 \begin{enumerate}
	 	\item {for all attackers $B$ of $A$, we have $Lab(B)=\Out$ \textbf{and}}
	 	\item {for every deductive support
$S\Rightarrow B$ \suchthat $A\in S$, we have that $\dot{ \big\{ } Lab(C)\mid C\in S\:\backslash\:\{A\}\dot{ \big\} } \leq Lab(B)$.\footnote{Here, $\dot{\{ }\:\:\dot{\} }$ denotes a multiset.}}
	\end{enumerate}}
	\item $A$ is \emph{legally \Out}\ \wrt $Lab$, iff
	\begin{enumerate}
		\item {there is some attacker $B$ of $A$, \suchthat $Lab(B)=\In$ \textbf{or}}
		\item {there is a chain of supports $\mathcal{C}=\{(S_0,B_0),...,(S_n,B_n)\}\subseteq Supp$ which starts at $A$ and ends at some argument $B_n$ \suchthat there exists an attack $(C,B_n)\in Att$ with $Lab(C)=\In$ and for each $(S_i,B_i)\in\mathcal{C}$ the following conditions hold:
		\begin{itemize}
			\item{$B_i\in out(Lab)$ and}
			\item{$|S_i\cap in(Lab)|=|S_i|-1$}
		\end{itemize}}\label{Def:Ground:LegalLabeling:ItemOut}
	\end{enumerate}
	\item $A$ is \emph{legally \Undec}\ iff 
	\begin{enumerate}
		\item $A$ is not \emph{legally \In}\ \textbf{and}
		\item $A$ is not \emph{legally \Out}.
	\end{enumerate} 
	\item {$Lab$ is a \emph{legal labeling} of \J, iff every argument is \emph{legally labeled} \wrt $Lab$.}
\end{enumerate}
\end{definition}
For brevity, we will sometimes write a labeling $Lab$ as the tuple $\big(in(Lab),out(Lab),undec(Lab)\big)$.
Admissible labelings, preferred labelings and the strict including minimal labeling (SIM) of JSBAFs without preferences are defined analogous to the case of JSBAFs with preferences (see Section~\ref{Sec:JSBAF}).
The following statements can be proven similar as for the case of JSBAFs with preferences:
\begin{proposition}\label{Theo:JSBAF:MSIisAdmissible}
Let \JSBAF be a JSBAF \suchthat \JDA and let $SIM_{\mathcal{J}}$ be the \SIM labeling of \J.
Then $SIM_{\mathcal{J}}$ is an \emph{admissible} labeling.
\end{proposition}
\begin{proposition}\label{Cor:JSBAF:admNonEmpty}
Let \JSBAFGR be a JSBAF \suchthat \JDA.
Then $adm(\mathcal{J})\neq\emptyset$.
\end{proposition}

\subsection{Motivation}

\subsubsection{Introductory example}

It turns out that defining a grounded semantics for our JSBAFs in such a way that they intuitively correspond to the ideas of the standard grounded semantics of abstract argumentation is not that straight forward.
One property of the grounded semantics in abstract argumentation is that there is no admissible labeling $L_{adm}$ for which $in(L_{adm})$ attacks the accepted arguments of the grounded labeling $L_{gr}$.\footnote{
To see that this is the case, consider the characteristic function of abstract argumentation, defined in~\cite[Definition 16]{Dung1995:OnTheAcceptabilityOfArguments}, which maps a set of argument $S$ to the set of all arguments defended by $S$.
Towards a contradiction, assume that there is an admissible labeling $L_{adm}$ \suchthat $in(L_{adm})$ attacks $in(L_{gr})$ for $L_{gr}$ being the grounded labeling.
By iterating over the characteristic function with starting point $in(L_{adm})$ until a fixpoint $in(L_{cmp})$ is reached, one can find a complete labeling $L_{cmp}$.
In particular, $L_{cmp}$ is conflict-free.
Since the characteristic function is monotone, we have $in(L_{adm})\subseteq in(L_{cmp})$.
Since the grounded labeling is the smallest (\wrt sub-set inclusion of accepted arguments) complete labeling, we also have $in(L_{gr})\subseteq in(L_{cmp})$.
Now conflict-freeness of $L_{cmp}$ is violated, a contradiction.}
In other words, a grounded labeling accepts only arguments that can never be rejected by an admissible labeling.
In particular, all preferred labelings agree on the arguments accepted in the grounded labeling.\footnote{Since every preferred labeling is a complete labeling and the grounded labeling is the smallest (\wrt sub-set inclusion of accepted arguments) complete labeling.}
At the same time, the grounded semantics is still a very skeptical semantics.

Consider the following as an introductory example, in order to see how these ideas can be translated into the realm of JSBAFs.

\begin{example}\label{Exmp:JSBAF:GRIntroduction}
JSBAF $\mathcal{J}_2$
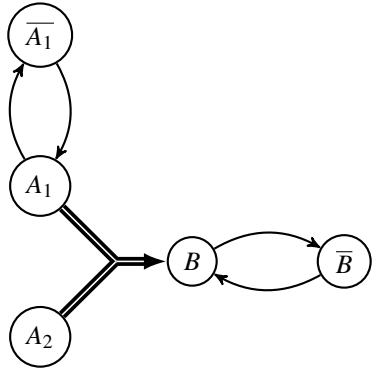
\leavevmode
\begin{center}
\begin{tikzpicture}
	\node[args] at (-2, 3) (NotAOne) {$\overline{A_1}$};
	\node[args] at (-2, 1) (AOne) {$A_1$};
	\node[args] at (-2, -1) (ATwo) {$A_2$};
	\node[args] at (0, 0) (B) {$B$};
	\node[args] at (2, 0) (NotB) {$\overline{B}$};
	\node[suppNode] at (-1, 0) (SuppB) {};
	
	\path[thick, ->, bend left]
		(NotAOne) edge (AOne)
		(AOne) edge (NotAOne)
		(NotB) edge (B)
		(B) edge (NotB);
		
		
	\draw[multiLine] (AOne) to (SuppB);
	\draw[multiLine] (ATwo) to (SuppB);
	\draw[multiLine, -\tip] (SuppB) to (B);
\end{tikzpicture}
\end{center}
\end{example}

Suppose for a moment that we are not dealing with JSBAFs but with an abstract argumentation framework which does not contain the support $(\{A_1,A_2\},B)$.
Then the grounded labeling of this framework would correspond to the labeling $Lab_{1}$, for which we have $in(Lab_{1})=\{A_2\}$, $out(Lab_{1})=\emptyset$ $undec(Lab_{1})=\{A_1,\overline{A_1},b,\overline{B}\}$.
In particular, in the case of abstract argumentation, there could never be a reason to reject $A_2$ -- after all, this argument is unattacked.
However, in our JSBAFs we actually can have a reason to reject $A_2$, namely the the following labeling, which we define to be $Lab_{2}$: $in(Lab_{2})=\{A_1,\overline{B}\}$, $out(Lab_{2})=\{\overline{A_1},A_2,B\}$, $undec(Lab_{2})=\emptyset$.
This is because accepting $\overline{B}$ means we have to reject $B$, while simultaneously accepting $A_1$ means we have to make sure that $A_2$ is rejected.\footnote{Remember again that supports correspond to the application of strict rules, meaning that if we accept both $A_1$ and $A_2$ while $B$ is rejected, the closure postulate would be violated.}
This tells us that in the example above, $A_2$ should actually not be accepted in the grounded labeling for JSBAFs.

As a first (naive) attempt at defining our grounded semantics, assume that we simply define a grounded labeling as the smallest (\wrt sub-set inclusion of accepted arguments) admissible labeling which accepts every argument that is legally \In.
In the case of $\mathcal{J}_2$, the smallest admissible labeling is the the strict including minimal labeling, which we denote here as $SIM_{\mathcal{J}_2}$.
We have $in(SIM_{\mathcal{J}_2})=out(SIM_{\mathcal{J}_2})=\emptyset$ and $undec(SIM_{\mathcal{J}_2})=\{A_1,\overline{A_1},A_2,B,\overline{B}\}$.
Lets see if $SIM_{\mathcal{J}_2}$ is already a grounded labeling according to this approach:
Note that none of the arguments $A_1,\overline{A_1},B$ and $\overline{B}$ are legally \In \wrt $SIM_{\mathcal{J}_2}$, because each of them is attacked by at least one argument which is not labeled \Out.
However, the argument $A_2$ is legally \In \wrt $SIM_{\mathcal{J}_2}$, because it is unattacked and because $\dot{ \big\{ } SIM_{\mathcal{J}_2}(A_1) \dot{ \big\} } \leq SIM_{\mathcal{J}_2}(B)$ holds.
Thus, according to our naive approach, $SIM_{\mathcal{J}_2}$ would not actually be a grounded labeling because we would have to accept $A_2$.
On the other hand, the following labeling which we denote as $Lab_{3}$, would be a grounded labeling according to this first attempt: $in(Lab_{3})=\{A_2\}$, $out(Lab_{3})=\emptyset$, $undec(Lab_{3})=\{A_1,\overline{A_1},B,\overline{B}\}$.
The problem with this naive approach now lies in the fact that we actually have a preferred labeling in which $A_2$ is rejected, as stated above.
Therefore it seems unintuitive to consider $Lab_{3}$ as a grounded labeling.

What this first approach tells us, is that when we want to check if a specific labeling $Lab$ should be a grounded labeling, we cannot confine ourselves just to $Lab$.
Rather, we have to take a \emph{broader} approach and consider a variety of labelings.
Furthermore, when trying to determine if a specific argument $A$ should be accepted in the grounded labeling, it is not just the label of $A$ and its attackers that we have to think about.
Instead, for any support $(S,B)$ where $A$ is contained in the supporting set $S$, we also have to take into account the labels of all the other arguments that are part of $S$ as well as the label of $B$.

\subsubsection{Important ideas for grounded semantics}

Now, how should we go about finding a grounded labeling in general?
In abstract argumentation, one way is to start with the labeling that labels all arguments as \Undec and then iteratively accepting arguments that can never be rejected (by iterating over the characteristic function) until a fixpoint is reached.
As we have done above, in the case of JSBAFs the most natural starting point for this procedure is SIM.
Thus we could start with SIM and then go about checking arguments $A$ one after the other to see if there exists a reason for why they could be rejected.
If we don't find such a reason, we accept $A$ and continue until we don't find any more arguments that we need to accept.

An obvious reason for rejection would be an attack from an argument $C$ to $A$ \suchthat $C$ is not labeled \Out.
However, another reason could be that there is a support $(S,B)$ with $A\in S$ and an admissible labeling in which we cannot accept $A$ because of this support.
Then this support would give us a valid reason for why $A$ should not be contained in the grounded labeling.
If none of these conditions are satisfied (\ie all attackers are \Out and we cannot find such a support), we are in a way \emph{forced} to accept $A$ in the grounded labeling.

However, when considering a specific labeling $Lab$ and an argument $A$ that is part of a support $(S,B)$, not every admissible labeling in which we cannot legally accept $A$ should count as a valid reason to dismiss the idea of accepting $A$ in the ground labeling.
To see this, consider this following, slightly modified example:

\begin{example}\label{Exmp:JSBAF:GRIntroductionAltered}
JSBAF $\mathcal{J}_3$
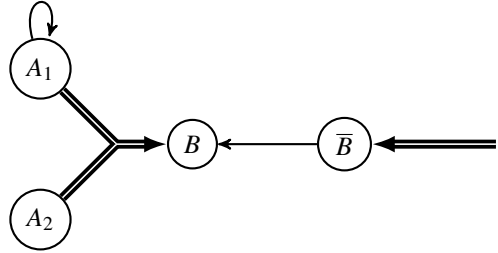
\leavevmode
\begin{center}
\begin{tikzpicture}
	\node[args] at (-2, 1) (AOne) {$A_1$};
	\node[args] at (-2, -1) (ATwo) {$A_2$};
	\node[args] at (0, 0) (B) {$B$};
	\node[args] at (2, 0) (NotB) {$\overline{B}$};
	\node[suppNode] at (-1, 0) (SuppB) {};
	\node[suppNode] at (4, 0) (SuppNotB) {};
	
	\path[thick, ->]
		(NotB) edge (B);
		
	\path[thick, ->, loop above]
		(AOne) edge (AOne);
		
		
	\draw[multiLine] (AOne) to (SuppB);
	\draw[multiLine] (ATwo) to (SuppB);
	\draw[multiLine, -\tip] (SuppB) to (B);
	\draw[multiLine, -\tip] (SuppNotB) to (NotB);
\end{tikzpicture}
\end{center}
\end{example}

Let us suppose that we again start with the strict including minimal labeling of $\mathcal{J}_3$ and want to check if $A_2$ should be accepted in the grounded labeling of this framework.
We have $SIM_{\mathcal{J}_3}=\big(\{\overline{B}\},\{B\},\{A_1,A_2\}\big)$, meaning we start with a labeling where $\overline{B}$ is accepted (because it is strict) and $B$ is rejected.
Now $A_2$ is not legally \In \wrt $SIM_{\mathcal{J}_3}$, since the label $SIM_{\mathcal{J}_3}(B)=\Out$ requires us to either reject an argument in the supporting set $\{A_1,A_2\}$ or to leave both of these arguments as \Undec.
However, since $A_1$ is a self-attacker, we can never accept it.
On the other hand, if we were to accept $A_2$, we could still satisfy $\dot{ \big\{ } SIM_{\mathcal{J}_3}(A_1)\}\dot{ \big\} } \leq SIM_{\mathcal{J}_3}(B)$ by labeling $A_1$ as \Out.
In fact, even though $A_2$ is not legally \In \wrt $SIM$ in $\mathcal{J}_3$, there now does not seem to be a valid reason anymore to ever reject $A_2$.
Because we can never accept $A_1$, we will always have an argument in the supporting set $\{A_1,A_2\}$ which can be labeled \Out in order to satisfy the support-condition for $A_2$.
This leads us an important observation:
When we have a labeling $Lab$ and are considering a (possibly different) labeling $Lab'$ in order to see if we can find a reason to not accept $A$, we can allow ourselves to \emph{alter} $Lab'$.

Naturally, there should be some restrictions placed on this idea.
Most importantly, we should not be allowed to change the labeling of arguments to which we have already committed ourselves.
That is, we should not be allowed to change the label of arguments already labeled \In or \Out, but restrict our changes only to the arguments labeled \Undec.
We capture this in the following definition:

\begin{definition}\label{Def:JSBAF:ExtendedLabeling}
Let \JSBAFGR be a JSBAF and let $Lab,Lab'$ be two labelings of $\mathcal{J}$.
We say that $Lab'$ \emph{extends} $Lab$ iff $in(Lab)\subseteq in(Lab')$ and $out(Lab)\subseteq out(Lab')$.
\end{definition}

Another restriction we should put on the idea of altering the labelings that we consider is that we should not allow ourselves to \enquote{cheat} by labeling the supported argument \In (in cases where this is possible while still adhering to admissibility).
After all, for any support $(S,B)$, accepting $B$ means this support can never give us a reason to reject any of the arguments in the supporting set $S$.\footnote{This restriction might seem a bit harsh now. For example, consider the case that we are currently investigating a labeling $Lab$ in which $Lab(B)=\Undec$ while $B$ is unattacked and doesn't support anything, \ie there will never be a valid reason to reject $B$.
Then depriving us of the option to label $B$ as \In does not seem reasonable.
However, when iteratively accepting arguments, we can always accept $B$ first and get back to $A\in S$ later.
When we are then reconsidering the status of $A$ at a later time, this support will not be a reason for the rejection of $A$ anymore.}

The next question now is, which labelings we should even consider when searching for a valid reason to not accept an argument in the grounded labeling.
As we have stated above, the fundamental problem we have in $\mathcal{J}_2$ of Example~\ref{Exmp:JSBAF:GRIntroduction} is that, whenever $B$ is labeled \Undec or \Out, we have to make sure that not both $A_1$ and $A_2$ are labeled \In.
Thus we need to check at least the labelings where $B$ is \Undec or \Out.
This motivates the following definition that we need for our grounded semantics:

\begin{definition}\label{Def:JSBAF:MoreInformativeLabel}
Let $l,l'\in\{\In,\Out,\Undec\}$ be two labels.
We say that the label $l'$ is \emph{more informative} than the label $l$, denoted $l'\geq_{p}l$ iff $l=\Undec$ or $l'=l$.
\end{definition}

Note that with this definition, for every label $l\in\{\In,\Out,\Undec\}$ we have $l\geq_{p} l$ and for every three labels $l_1,l_2,l_3$, if $l_1\geq_{p} l_2\geq_{p} l_3$, then $l_1\geq_{p} l_3$ also holds.

To see how we will be utilizing this definition, imagine again that we are iteratively labeling arguments and we are interested in whether or not we should accept an argument $A$ that is part of a supporting set $S$ for an argument $B$, \ie we have a support $(S,B)$ with $A\in S$.
We want to know if there exists an admissible labeling, in which the support $(S,B)$ gives us a reason to reject $A$.
Therefore, when we have a current labeling $Lab$, we are going to check all admissible labelings $Lab'$, for which we have $Lab'(B)\geq_p Lab(B)$.
If our current labeling $Lab$ labels $B$ as \Undec, then this means we also need to check those labelings where $B$ is labeled \In or \Out.
On the other hand, if our current labeling already labels $B$ as \Out (\In), then we don't need to additionally check those labelings where $B$ is labeled \Undec or \In (\Out).
After all, when we are iteratively labeling arguments and we have already committed to labeling $B$ as \In or \Out, then this label will not be changed anymore.

So far, we can summarize our findings as follows:
A natural starting point to look for a grounded labeling is the most skeptical labeling of all -- the strict including minimal labeling.
When considering a particular labeling $Lab$ and checking if an argument $A$ should be accepted according to the grounded semantics, we not only need to check the attackers of $A$ but also the supports $(S,B)$ that $A$ is a part of.
Out of these supports, we can dismiss those where $Lab(B)=\In$ holds.
We then need to check every labeling $Lab'$ where $Lab'(B)$ is more informative than $Lab(B)$.
What we need to ensure is that for all of those labelings $Lab'$, we can still find a labeling $Lab''$ which extends $Lab'$ and where $A$ can be legally accepted.
However, in our extension of $Lab'$ we should not change the label of any arguments that are already accepted or rejected and we should not change the label of $B$.

With this, we are almost ready to give our most important definition for the grounded semantics.
However, there is one more caveat we need to consider, namely that of JSBAFs with an infinite amount of arguments.
Suppose we have a JSBAF that only consists of one argument which is the starting point for an infinitely long chain of supports:

\begin{example}\label{Exmp:JSBAF:GRInfinity}
JSBAF $\mathcal{J}_4$
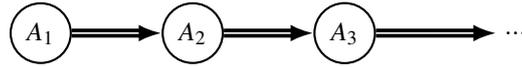
\leavevmode
\begin{center}
\begin{tikzpicture}
	\node[args] at (0, 0) (AOne) {$A_1$};
	\node[args] at (2, 0) (ATwo) {$A_2$};
	\node[args] at (4, 0) (AThree) {$A_3$};
	\node[suppNode] at (6, 0) (Dot) {};
	\draw[draw=white] (6.25, 0) circle (0.5) node {$...$};
		
	\draw[multiLine, -\tip] (AOne) to (ATwo);
	\draw[multiLine, -\tip] (ATwo) to (AThree);
	\draw[multiLine, -\tip] (AThree) to (Dot);
\end{tikzpicture}
\end{center}
\end{example}

Since there are no strict arguments in this JSBAF, the strict including minimal labeling doesn't accept any arguments.
However, in the absence of any attacks to any of the $A_i$, there does not seem to exist a valid reason to ever reject any of these arguments.
In fact, is seems most natural to say that the grounded labeling for this particular JSBAF should just accept all the (infintely many) arguments in this chain.\footnote{This is most evident when remembering that supports in our JSBAFs will ultimately correspond to the application of strict rules, which we want to base on the entailment relation of some underlying logic.
Thus a JSBAF like that would correspond to something like an infinite chain of entailments $\phi\vdash\phi\vdash\phi\vdash...$.}
On the other hand, when applying our previous ideas, we are at an impasse:
Suppose that we consider the strict including minimal labeling of this JSBAF and want to check if there is some argument $A_i$ which we need to accept.
Then in particular we need to check if we can alter the strict including minimal labeling itself in such a way, that $A_i$ is legally \In without altering the label of $A_{i+1}$.
However, as $A_i$ is the only argument in the support for $A_{i+1}$ and since $A_{i+1}$ has the label \Undec, we cannot make this modification.

To remedy this, we will use one last idea for our grounded labelings.
Consider some $A_i$ from $\mathcal{J}_4$:
Since there are no attackers for this particular $A_i$ or any $A_j$ that is a support-child of $A_i$, no chain of arguments starting at the support $\big(\{A_i\},A_{i+1})$ can ever be a reason for us to label $A_i$ as \Out.
In this sense, the support $\big(\{A_i\},A_{i+1})$ is \emph{safe} for $A_i$ and we don't actually need to consider it when checking if $A_i$ should be accepted in the grounded labeling.
This idea is captured in the following definition:

\begin{definition}\label{Def:JSBAF:SafeSupport}
Let \JSBAFGR be a JSBAF \suchthat \JDA.
Furthermore, let $A\in Args$ be an argument and $(S,B)\in Supp$ a support with $A\in S$.
Lastly, let $Lab$ be a labeling of $\mathcal{J}$.
We say that $(S,B)$ is \emph{safe for $A$ in $Lab$} iff for all chains of supports $\mathcal{C}$ that start with $(S,B)$, we have that for all $(S',B')\in\mathcal{C}$ and all $(D,B')\in Att$, $Lab(D)=\Out$.
We denote by $SS_{Lab}(A)$ the set of all safe supports for $A$ in $Lab$, that is $SS_{Lab}(A)=\big{\{} (S,B)\in Supp\mid A\in S\text{ and }(S,B)\text{ is safe for }A\text{ in }Lab\big{\}}$.
\end{definition}

With this, we are finally ready to define the property that an argument needs to satisfy in order for us to be \emph{forced} to accept it in the grounded labeling:

\begin{definition}\label{Def:JSBAF:ForcedIn}
Let \JSBAFGR be a JSBAF \suchthat \JDA.
Furthermore, let $A\in Args$ be an argument and $Lab$ be a labeling of $\mathcal{J}$.
We say $A$ is \emph{forced \In} \wrt $Lab$ iff all of the following conditions hold:
\begin{enumerate}
	\item{For any argument $B\in Args$ with $(B,A)\in Att$, we have $Lab(B)=\Out$.}
	\item{For every support $(S,B)\in Supp$ with $A\in S$ and $Lab(B)\neq \In$, one of the following conditions hold:
	\begin{enumerate}
		\item{For every admissible labeling $Lab'$ with $Lab'(B)\geq_{p}Lab(B)$	there exists an admissible labeling $Lab''$ which extends $Lab'$ \suchthat $A$ is legally \In \wrt $Lab''$ and $Lab''(B)=Lab'(B)$, or}
		\item{$(S,B)$ is safe for $A$ in $Lab$.}			
	\end{enumerate}}
\end{enumerate}
For any labeling $Lab$ of $\mathcal{J}$, we denote by $FI(Lab)$ the \emph{set of all arguments} of $\mathcal{J}$ which are forced \In \wrt $Lab$.
\end{definition}

Based on this notion of \emph{forced \In}, we now define a family of labelings that will be the basis for our grounded semantics:

\begin{definition}\label{Def:JSBAF:GRCompleteLabeling}
Let \JSBAFGR\ be a JSBAF \suchthat \JDA and $Lab$ a labeling of $\mathcal{J}$.
We say that $Lab$ is a \emph{ground-complete} labeling iff both of the following conditions hold:
\begin{enumerate}
	\item{$Lab\in adm(\mathcal{J})$}
	\item{If $A\in Args$ is forced \In \wrt $Lab$, then $Lab(A)=\In$.}
\end{enumerate}
We denote the \emph{set of all ground-complete labelings} of $\mathcal{J}$ as $grcmp(\mathcal{J})$.
\end{definition}

Finally, we are ready to define our grounded semantics.
A grounded labeling is simply a minimal (\wrt set-inclusion of accepted arguments) ground-complete labeling:

\begin{definition}\label{Def:JSBAF:GR}
Let \JSBAFGR\ be a JSBAF \suchthat \JDA and $Lab$ a labeling of $\mathcal{J}$.
We say that $Lab$ is a \emph{grounded labeling} iff both of the following conditions hold:
\begin{enumerate}
	\item{$Lab\in grcmp(\mathcal{J})$}
	\item{There is no $Lab'\in grcmp(\mathcal{J})$ with $in(Lab')\subset in(Lab)$.}
\end{enumerate}
The set of all grounded labelings of $\mathcal{J}$ is denoted by $gr(\mathcal{J})$.
\end{definition}

We give a detailed example of the grounded labeling in Section~\ref{Sec:JSBAF:GroundedConstruction} and only want to state the grounded labelings of the examples above for completeness:
We have $gr(\mathcal{J}_2)=\Big{\{}\big(\emptyset,\emptyset,\{A_1,\overline{A_1},A_2,B,\overline{B}\}\big)\Big{\}}$ and $gr(\mathcal{J}_3)=\Big{\{}\big(\{\overline{B},A_2\},\{A_1,B\},\emptyset\big)\Big{\}}$, while $gr(\mathcal{J}_4)=\Big{\{}\big(\{A_i\mid i\geq 1\},\emptyset,\emptyset\big)\Big{\}}$.

\subsection{Grounded Construction}\label{Sec:JSBAF:GroundedConstruction}

\subsubsection{Definition}

One useful property of the grounded semantics for abstract argumentation is that it only contains a single labeling.
For the grounded labelings as we have defined them, it is not immediately clear that this is the case.
However, we will show in this section that this property is actually satisfied.
We begin by turning the idea of iteratively labeling arguments -- which we have used to explain our motivation for the definition of a grounded labeling -- into an algorithmic approach for constructing a labeling for a JSBAF.
This approach will use transfinite recursion and sequences of potentially transfinite length in order to be applicable in the case of JSBAFs with an infinite amount of arguments.
We will discuss in Example~\ref{Exmp:JSBAF:GRAlgorithm} why this approach is necessary.

\begin{definition}\label{Def:JSBAF:GRConstruction}
Let \JSBAFGR\ be a JSBAF \suchthat \JDA and let $SIM$ be the \SIM labeling of $\mathcal{J}$.
A sequence $GC=(Lab_0,...,Lab_{\alpha})$ is called a \emph{grounded construction} iff:
\begin{itemize}
	\item{$Lab_0=SIM$,}
	\item{$FI(Lab_{\alpha})\backslash in(Lab_{\alpha})=\emptyset$ and}
	\item{for $0\leq \beta <\alpha$ there is some $A\in FI(Lab_{\beta})\backslash in(Lab_{\beta})$ \suchthat:
	\begin{align*}
in(Lab_{\beta+1})=&
in(Lab_\beta)\cup\{A\}\cup\\
&\bigcup\limits_{(S,B)\in SS_{Lab_{\beta}}(A)}\big(\{B\}\cup SC(B)\big)\\
O_{\beta+1}^0=&
\{D\in Args\mid (C,D)\in Att,\\
&\hspace{0.2cm}C\in in(Lab_{\beta+1})\}\\
O_{\beta+1}^{j+1}=&
O_{\beta+1}^{j}\cup\{D\in Args\mid (S,B)\in Supp,\\
&\hspace{1.4cm}D\in S,B\in O_{\beta+1}^{j}\textnormal{ and}\\
&\hspace{1.4cm}S\backslash\{D\}\subseteq in(Lab_{\beta+1})\}\\
out(Lab_{\beta+1})&
=\bigcup\limits_{j\geq 0}O_{\beta+1}^j\\
undec(Lab_{\beta+1})&
=Args\backslash\big(in(Lab_{\beta+1})\cup out(Lab_{\beta+1})\big)
	\end{align*}	}
	\item{$in(Lab_{\delta})
=\bigcup\limits_{\alpha < \delta} in(Lab_{\alpha})\\
out(Lab_{\delta})=\bigcup\limits_{\alpha < \delta} out(Lab_{\alpha})\\
undec(Lab_{\delta})=Args\backslash\big( in(Lab_{\delta})\cup out(Lab_{\delta})\big)$}
\end{itemize}
We call $Lab_{\alpha}$ the \emph{result} of the grounded construction $GC$.
\end{definition}

Essentially, our construction works as follows:
We start with the strict including minimal labeling.
In each step, we accept all arguments that we have previously accepted ($in(Lab_\beta)\subseteq in(Lab_{\beta+1})$).
We take a single argument $A$ which is forced \In \wrt to the previous labeling and add $A$ to the accepted arguments ($\{A\}\subseteq in(Lab_{\beta+1})$).
Furthermore, for each support $(S,B)$ which is safe for $A$ \wrt the previous labeling, we add each support-child of $B$, as well as $B$ itself to the set of accepted arguments ($\bigcup\limits_{(S,B)\in SS_{Lab_{\beta}}(A)}\big(\{B\}\cup SC(B)\big)\subseteq in(Lab_{\beta+1})$).
Afterwards, we compute the effect that accepting these arguments has on the resulting framework.
For this, we begin by rejecting all arguments that are attacked by an accepted argument (via the set $O_{\beta+1}^0$).
Then we reject all arguments $B$ for which we have a chain of supports that satisfies the conditions of Definition~\ref{Def:Ground:LegalLabeling} (via the sets $O_{\beta+1}^{j+1}$).
Lastly, we label all remaining arguments as \Undec.

\subsubsection{Example}

To see how this approach works in practice, we give an example below.
Note that this JSBAF is supposed to contain an infinitely long chain of supports starting, with $\big(\{D_1\},D_2\big)$ as well as an infinitely long chain of attacks starting at $(E_1,E_2)$.
We have indicated this by adding $...$ below the arguments $D_4$ and $E_4$.
Furthermore, suppose that for each $E_i$ we have $(E_i,H)\in Att$ iff $(E_{i+1},H)\not\in Att$, beginning with the chain of arguments $E_1$ to $E_4$ that is shown.

\begin{example}\label{Exmp:JSBAF:GRAlgorithm}
JSBAF $\mathcal{J}_5$
\leavevmode
\begin{center}
\begin{tikzpicture}
	\node[args] at (2, -2) (A) {$A$};
	\node[suppNode] at (2, -4) (SuppA) {};
	
	\node[args] at (2, 0) (B) {$B$};
	\node[suppNode] at (3, 0) (SuppB) {};

	\node[args] at (4, 1) (C) {$C$};
	
	\node[args] at (4, -1) (DOne) {$D_1$};
	\node[args] at (4, -3) (DTwo) {$D_2$};
	\node[args] at (4, -5) (DThree) {$D_3$};
	\node[args] at (4, -7) (DFour) {$D_4$};
	\draw[draw=white] (4, -8.25) circle (0.5) node {$...$};
	\node[suppNode] at (4, -8) (Dot) {};
	
	\node[args] at (6, -1) (EOne) {$E_1$};
	\node[args] at (6, -3) (ETwo) {$E_2$};
	\node[args] at (6, -5) (EThree) {$E_3$};
	\node[args] at (6, -7) (EFour) {$E_4$};
	\draw[draw=white] (6, -8.25) circle (0.5) node {$...$};
	\node[suppNode] at (6, -8) (DotTwo) {};
	
	\node[args] at (8, -1) (G) {$G$};
	\node[suppNode] at (8, 0) (SuppG) {};
	\node[args] at (7, 1) (F) {$F$};
	\node[args] at (9, 1) (I) {$I$};
	
	\node[args] at (8, -3) (H) {$H$};
	
	\path[thick, ->]
		(A) edge (B)
		(DOne) edge (EOne)
		(EOne) edge (ETwo)
		(ETwo) edge (EThree)
		(EThree) edge (EFour)
		(EFour) edge (DotTwo)
		(EOne) edge (G)
		(EOne) edge (H)
		(EThree) edge (H);
		
	\path[thick, ->, loop above]
		(C) edge (C);
		
	\draw[multiLine, -\tip] (SuppA) to (A);
	\draw[multiLine, -\tip] (SuppB) to (B);
	\draw[multiLine] (C) to (SuppB);
	\draw[multiLine] (DOne) to (SuppB);
	\draw[multiLine, -\tip] (DOne) to (DTwo);
	\draw[multiLine, -\tip] (DTwo) to (DThree);
	\draw[multiLine, -\tip] (DThree) to (DFour);
	\draw[multiLine, -\tip] (DFour) to (Dot);
	\draw[multiLine, -\tip] (SuppG) to (G);
	\draw[multiLine] (F) to (SuppG);
	\draw[multiLine] (I) to (SuppG);
\end{tikzpicture}
\end{center}
\end{example}

Now for our grounded construction.
We begin with the strict including minimal labeling of $\mathcal{J}_5$.
It is easy to see that this is the labeling $Lab_0=\big(\{A\},\{B\},\{C,D_1,...,E_1,...,F,G,H,I\}\big)$.
Now let us consider the arguments in this framework to see which of them is \emph{forced \In} \wrt $Lab_0$.
The arguments $C,G$ and $H$ as well as all the $E_i$ are attacked by some argument not labeled \Out, therefore they cannot be forced \In.
Thus the only arguments that we really have to consider are $F,I$ and each of the $D_i$.

Lets begin with the argument $F$.
Here we need to check the support $\big(\{F,I\},G\big)$, because $Lab_0(G)\neq\In$.
Note that this is not a safe support, since $G$ is attacked by $E_1$ and $E_1$ is not labeled \Out in $Lab_0$.
We have to check all admissible labelings $Lab'$ for which we have that $Lab'(G)$ is more informative than $Lab_0(G)$.
Since $Lab_0(G)=\Undec$ this essentially means we have to check all cases where $Lab'(G)=\Out$, or $Lab'(G)=\Undec$.
Technically, we also have to check the cases where $Lab'(G)=\In$, but as we said in our motivation, in these cases the support $\big(\{F,I\},G\big)$ can never give us a reason to not accept $F$.
Since we are only interested in knowing if there \emph{is} a reason to not accept $F$, we can skip these labelings.
Note that there exists the admissible labeling $Lab'=\big(\{A,I\},\{B\},\{C,D_1,...,E_1,...,F,G,H\}\big)$, which is exactly like $Lab_0$ except for the fact that we have accepted $I$.
We have that $F$ is not legally \In \wrt $Lab'$, thus $F$ is not forced \In \wrt $Lab_0$.
Obviously, we can make a similar argument for the case of $I$.

Now let us consider an argument $D_i$ in the infinite chain of supports.
For example, lets consider $D_3$.
Here we only have to check one support, namely $\big(\{D_3\},D_{4}\big)$.
Note that none of the arguments $D_i$ are attacked by any argument, therefore this support $\big(\{D_3\},D_{4}\big)$ is actually a safe support for $D_3$ \wrt $Lab_0$.
This means we have found our first argument that is legally \In \wrt $Lab_0$.
Obviously, with the same argumentation we actually have that all arguments $D_i$ for $i\geq 2$ are forced \In \wrt $Lab_0$.
However, for the sake of the example, lets stick with $D_3$.
Thus we choose $D_3$ as the argument in $FI(Lab_0)\backslash in(Lab_0)$ which we want to accept in the first actual step of our grounded construction.
This gives us our next labeling, $Lab_1=\big(\{A,D_3,D_4,...\},\{B\},\{C,D_1,D_2,E_1,...,F,G,H,I\}\big)$.
Note that, because we chose $D_3$ as the argument we want to label as \In, we also have to consider all chains of supports $\mathcal{C}$ which start at the safe support $\big(\{D_3\},D_{4}\big)$ and for each $(S,B)\in\mathcal{C}$ we have to label $B$ \In as well.
Therefore we also had to accept all arguments $D_i$ for $i\geq 4$ in addition to $D_3$.

Now let us check again if there are any more arguments that we need to accept according to our construction.
Similar to the case of $Lab_0$, for the arguments $F$ and $I$ we can easily turn $Lab_1$ into an admissible labeling where either $F$ or $I$ cannot be accepted.
This means neither $F$ nor $I$ are forced \In \wrt $Lab_1$.
Next, lets consider $D_1$.
Here we have to check two supports, namely $\big(\{D_1\},D_2\big)$ and $\big(\{C,D_1\},B\big)$, since both $B$ and $D_2$ are not labeled \In.
First, we point out that, similar to the case of $D_3$ in $Lab_0$, we actually have that the support $\big(\{D_1\},D_2\big)$ is safe for $D_1$ in $Lab_1$.
Therefore, we don't have to consider any other admissible labelings for the case of this support.
However, the same does not hold for the support $\big(\{C,D_1\},B\big)$, since $B$ is attacked by $A$ and $A$ is labeled \In by $Lab_1$.
This means we have to check if we can satisfy item 2.a of Definition~\ref{Def:JSBAF:ForcedIn} in order to find out if $D_1$ is forced \In \wrt $Lab_1$.

Assume that we have some admissible labeling $Lab'$ for which we have that $Lab'(B)\geq_p Lab(B)$.
Since $Lab(B)=\Out$, this means $Lab'(B)=\Out$ must also hold.
The question is now, if we can extend $Lab'$ to turn it into an admissible labeling $Lab''$ \suchthat $D_1$ is legally \In \wrt $Lab''$ and $Lab''(B)=\Out$ still holds.
To check if $D_1$ is legally \In, we only have to consider the supports that $D_1$ is a part of, \ie $\big(\{D_1\},D_2\big)$ and $\big(\{C,D_1\},B\big)$.
Note that the only restriction we have when considering if we can extend $Lab'$ to $Lab''$, is that the labeling of $B$ and the labeling of any argument which is \In or \Out cannot be changed.
However, if we have $Lab'(D_2)=\Undec$, then we are still allowed to extend $Lab'$ in such a way that $Lab''(D_2)=\In$ holds, so long as the labeling $Lab''$ is admissible in the end and we don't change the label of any argument labeled \In or \Out by $Lab'$.
Since we cannot have the case that $Lab'(D_i)=\Out$ for any $i\geq 2$ (because all these arguments are unattacked and thus can never be labeled \Out in any admissible labeling), this means we can simply extend $Lab'$ to $Lab''$ in such a way that $Lab''(D_i)=\In$ for any $i\geq 2$.
Then for the support $\big(\{D_1\},D_2\big)$, we have that $\dot{ \big\{ } \dot{ \big\} } \leq Lab''(D_2)$ is trivially satisfied.
So lets suppose that our extension of $Lab'$ is such that $Lab''(D_i)=\In$ for any $i\geq 2$.

Now for the second support that $D_1$ is a part of, namely $\big(\{C,D_1\},B\big)$.
Remember that we must adhere to $Lab''(B)=Lab'(B)=\Out$, \ie we are not allowed to change the label of $B$ itself.
However, if $Lab'(C)=\Undec$, then we can still change the label of $C$.
Note that $C$ is a self-attacker in $\mathcal{J}_5$, therefore we cannot actually have the case that $Lab'(C)=\In$.
Furthermore, it is easy to see that if $Lab'(C)=\Undec$, then we cannot have $Lab'(D_1)=\Out$.
Thus we can simply extend $Lab'$ to $Lab''$ in such a way that $Lab''(C)=\Out$ while $Lab''(D_1)=\In$ holds.
Then $C$ is legally \Out \wrt $Lab''$ (which is required for admissibility of $Lab''$) and we have $\dot{ \big\{ } Lab''(C) \dot{ \big\} } \leq Lab''(B)$, therefore $D$ is now legally \In \wrt $Lab''$.

Lastly, we have to ensure that no label of any other arguments labeled \In or \Out in $Lab'$ are changed by our extension to $Lab''$.
For this, we have to make sure that the effect of accepting $D_1$ does not cause any interference with $F,G,H,I$ or any of the $E_i$.
For $Lab'$, we note that we cannot have $Lab'(D_1)=\Out$, as this would require either $C$ to be labeled \In (so that we could reject $D_1$ because of the support $\big(\{C,D_1\},B\big)$) or $D_2$ to be labeled \Out.
However, as we have pointed out above, these cases can never occur.
This means for $Lab'$, we actually either have $Lab'(D_1)=\In$ or $Lab'(D_1)=\Undec$.
In the first case, the extension to $Lab''$ described in the previous paragraph does not actually change the label of $D_1$, so we don't need to check any of the other arguments $F,G,H,I$ or $E_i$ because $Lab'$ is already admissible.
In the second case, we note that $Lab'(E_1)=\Undec$ must hold, thus we can now simply label $E_1$ \Out in our extension $Lab''$ to achieve admissibility of $Lab''$.
After that, no further modifications to the labeling need to be made, since $E_1$ being \Undec means $Lab'(G)=Lab'(H)=\Undec$ as well as $Lab'(E_i)=\Undec$ for any $i\geq 2$. 

We conclude that, for any admissible labeling $Lab'$ for which we have $Lab'(B)=\Out$, we can find another admissible labeling $Lab''$ for which we have $Lab''(B)=\Out$, while $D_1$ is legally \In \wrt $Lab''$.
Therefore $D_1$ is forced \In \wrt $Lab_1$ and we can actually choose it as the argument for the next step of our grounded construction.
As before, we now also have to accept the argument $D_2$ because it is part of the safe support $\big(\{D_1\},D_2\big)$ of $D_1$.
This gives us the following labeling:
$Lab_2=\big(\{A,D_1,...\},\{B,C,E_1\},\{E_2,...,F,G,H,I\}\big)$.

Lets see if there are any more arguments in $FI(Lab_2)\backslash in(Lab_2)$ which we need to label \In according to our grounded construction.
The situation for the arguments $F$ and $I$ still has not changed.
However, for the argument $G$ we now actually have that all attackers of $G$, namely $E_1$, are labeled \Out in $Lab_2$.
Whats more, $G$ does not support any argument in $\mathcal{J}_5$, which means that $G$ is now forced \In.
We thus get our next labeling $Lab_3=\big(\{A,D_1,...,G\},\{B,C,E_1\},\{E_2,...,F,H,I\}\big)$.
After this step, we now finally have that both $F$ and $I$ are forced \In \wrt $Lab_3$.
This is because they are unattacked and the only support that they are a part of is $\big(\{F,I\},G\big)$, which we don't have to check because the supported argument $G$ is labeled \In.
We summarize the next two steps in our grounded construction into one and immediately state the labeling that we get from accepting both $F$ and $I$:
$Lab_{2}=\big(\{A,D_1,...,F,G,I\},\{B,C,E_1\},\{E_2,...,H\}\big)$.

Finally, let us consider the chain of attacks starting at $E_1$.
Because we have accepted $D_1$, we needed to reject $E_1$, meaning $E_2$ is now forced \In because its only attacker is \Out and because it does not support anything.
This means we would have to accept $E_2$ and label $E_3$ as \Out as a consequence.
Then we could continue with $E_4$ and so on and so on.
Since this chain of arguments is infinite, situations like these are the reason for why we actually needed to define our grounded construction via transfinite recursion.
After accepting all the arguments $E_i$ for $i$ being a multiple of two and rejecting all $E_j$ for $j$ not being a multiple of two, we arrive at the labeling $Lab_{\omega}=\big(\{A,D_1,...,E_2,E_4,...,F,G,I\},\{B,C,E_1,E_3,...\},\{H\}\big)$.
Finally, we only have one argument left to check, namely $H$.
All attackers of $H$ are now labeled \Out in $Lab_{\omega}$ and $H$ does not support any arguments.
Therefore, $H$ is forced \In \wrt $Lab_{\omega}$ and we arrive at the result of our grounded construction:
$Lab_{\omega+1}=\big(\{A,D_1,...,E_2,E_4,...,F,G,H,I\},\{B,C,E_1,E_3,...\},\emptyset\big)$.

Note that we only described a specific grounded construction for this example, while there are actually infinitely many possible grounded constructions for $\mathcal{J}_5$.
In particular, we could have chosen to not start with $D_3$, but with another argument in the chain of $D_i$'s.
Similarly, instead of accepting $G$ right after the step where we accepted $D_1$ and rejected $E_1$, we could have also first accepted (some) arguments in the infinite chain of $E_i$'s before accepting $F,G$ and $I$.
We will therefore show in the next section that the \emph{result} of each grounded construction is identical, namely the unique grounded labeling.
To see this, we will first show that each grounded construction ends with a labeling which is ground-complete and minimal (\wrt set-inclusion of accepted arguments) and then we will show that there is actually only one such minimal labeling that is ground-complete for each JSBAF.
This also immediately implies that there is a unique grounded labeling for each JSBAF.

\subsection{Uniqueness of the grounded labeling}\label{Sec:JSBAF:UniqueGrounded}

\subsubsection{Overview}

We begin by showing that the labeling produced by the algorithmic approach of Definition~\ref{Def:JSBAF:GRConstruction} is a ground-complete labeling.
This proof will mainly consist of showing that the constructed labeling is an admissible labeling.
Afterwards, we will prove the uniqueness of the grounded labeling by showing that the labeling we construct is contained in every ground-complete labeling.

For the proof of admissibility we will show via transfinite induction that, for a given grounded construction $GC=(Lab_0,...Lab_{\alpha})$, each of the labelings $Lab_\beta$ are admissible.
The induction start will be trivial, since the construction begins with $Lab_0=SIM$ of which we know that it is an admissible labeling.
To make the induction step more accessible, we prove three statements separately before combining them and proving the limit-case for the transfinite induction in Proposition~\ref{Prop:JSBAF:GRConstructionResultAdmissible}.
These three statements are the following:
\begin{enumerate}
	\item{$Lab_{\beta+1}$ is a labeling.}
	\item{If any argument $A$ is labeled \In by $Lab_{\beta+1}$, then $A$ is \emph{legally \In} \wrt $Lab_{\beta+1}$.}
	\item{For any argument $A$ we have $Lab_{\beta+1}(A)=\Out$ iff it is legally \Out \wrt $Lab_{\beta+1}$.}
\end{enumerate}

\subsubsection{Induction step}

We begin by proving that, if $Lab_\beta$ is an admissible labeling, then $Lab_{\beta+1}$ is a labeling.
In this proof, we will use several small auxiliary statements.
The first one tells us that, if $A$ is forced \In \wrt to $Lab_\beta$, then for any support $(S,B)$ where $A\in S$ and $S\backslash\{A\}\subseteq in(Lab_\beta)$, this support must either be safe for $A$ or $B$ must already be labeled \In in $Lab_\beta$.
This way we know that, when labeling $A$ as \In in the step $\beta+1$, we don't end up with a support $S$ where all arguments are labeled \In, while the supported argument $B$ is not \In.

\begin{proposition}\label{Prop:JSBAF:ForcedInHeadIsIn}
Let \JSBAFGR be a JSBAF \suchthat \JDA and let $Lab$ be an admissible labeling of $\mathcal{J}$.
Furthermore, let $(S,B)\in Supp$ be some support of $\mathcal{J}$ and $A\in S$ be some argument of the supporting set $S$.
If $A\in FI(Lab)$ and $S\backslash\{A\}\subseteq in(Lab)$, then $Lab(B)=\In$ or $(S,B)\in SS_{Lab}(A)$. 
\end{proposition}
\begin{proof}
We show that if $(S,B)\not\in SS_{Lab}(A)$, then $Lab(B)=\In$ holds.
Towards a contradiction, assume that this is not the case, \ie $(S,B)\not\in SS_{Lab}(A)$ and $Lab(B)\neq\In$.
As $A\in FI(Lab)$ by assumption, we know that either item 2.a or 2.b of Definition~\ref{Def:JSBAF:ForcedIn} must hold.
By assumption, item 2.b does not hold.
Note that $Lab$ itself is an admissible labeling for which we have $Lab(B)\geq_p Lab(B)$.
Thus there exists an admissible labeling $Lab'\in adm(\mathcal{J})$ that extends $Lab$ \suchthat $A$ is legally \In \wrt $Lab'$ and $Lab(B)=Lab'(B)$.
As $Lab'$ extends $Lab$ we know that $S\backslash\{A\}\subseteq in(Lab')$ must hold.
Now if $A$ is legally \In \wrt $Lab'$, then $Lab(B)\neq\In$ is a contradiction.
\end{proof}

The second auxiliary statement tells us that, if an argument $A$ is forced \In \wrt some admissible labeling $Lab_\beta$, then $A$ cannot be labeled \Out in $Lab_\beta$:

\begin{proposition}\label{Prop:JSBAF:ForcedInArgIsNotOut}
Let \JSBAFGR be a JSBAF \suchthat \JDA and let $Lab$ be an admissible labeling of $\mathcal{J}$.
If $A\in FI(Lab)$, then $Lab(A)\neq\Out$.
\end{proposition}
\begin{proof}
Suppose towards a contradiction that $Lab(A)=\Out$ holds.
Because $Lab$ is an admissible labeling, we know that either of the following two cases must hold:
\begin{itemize}
	\item{There is $B\in Args$ with $(B,A)\in Att$ and $Lab(B)=\In$, or}
	\item{there is a chain of supports $\big{\{}(S_0,B_0),...,(S_n,B_n)\big{\}}$ stat starts at $A$ \suchthat that there is an attack $(C,B_n)\in Att$ with $Lab(C)=\In$ and for each $0\leq i\leq n$ we have $B_i\in out(Lab)$ while $|S_i\cap in(Lab)|=|S_i|-1$.}
\end{itemize}
The first case cannot hold because $A$ is forced \In \wrt $Lab$, meaning all its attackers need to be labeled \Out in $Lab$.
Now for the second case:
Note that $(S_0,B_0)$ is a support with $A\in S_0$ \suchthat $B_0\neq\In$, while $Lab$ is itself a admissible labeling.
It is obvious that the support $(S_0,B_0)$ is not safe for $A$ in $Lab$ because of the attacker $C$.
Thus item 2.b of Definition~\ref{Def:JSBAF:ForcedIn} does not hold, which means item 2.a must hold.
Therefore, there needs to be an admissible labeling $Lab'$ which extends $Lab$ in such a way that $A$ is legally \In \wrt $Lab'$.
Because $Lab'$ needs to extend $Lab$, we know that $Lab'(B_0)=\Out$ and $S_0\backslash\{A\}\subseteq in(Lab')$ must hold.
Now for the support $(S_0,B_0)$ we have $A\in S_0$, while $\dot{ \big\{ } Lab(C)\mid C\in S_0\backslash\{A\}\dot{ \big\} } \not\leq Lab(B_0)$, therefore $A$ is not legally \In \wrt $Lab'$.
This contradicts $A$ being forced \In \wrt $Lab$.

We conclude that both cases lead to a contradiction, meaning we have $Lab(A)\neq\Out$.
\end{proof}

The third auxiliary statement tells us that, if an argument $A$ is labeled \Out in $Lab_\beta$, then it cannot be in the union $\bigcup\limits_{(S,B)\in SS_{Lab_\beta}(X)}\big(\{B\}\cup SC(B)\big)$ for $X$ being the argument that is added in the step $\beta+1$ of the grounded construction.

\begin{proposition}\label{Prop:JSBAF:SafeSuppArgumentNotOut}
Let \JSBAFGR be a JSBAF \suchthat \JDA.
Furthermore, let $A,B\in Args$ be two arguments and $Lab$ an admissible labeling of $\mathcal{J}$ \suchthat $Lab(A)=\Out$.
Then $A\not\in \bigcup\limits_{(S,C)\in SS_{Lab}(B)}\big(\{C\}\cup SC(C)\big)$.
\end{proposition}
\begin{proof}
Towards a contradiction, suppose that the claim does not hold.
Then there exists a support $(S,C)\in SS_{Lab}(B)$ \suchthat $A=C$ or $A\in SC(C)$.

First, we assume that $(S,A)\in SS_{Lab}(B)$.
Because $A\in out(Lab)$ by assumption, we know that there either is an attack $(D,A)\in Att$ with $Lab(D)=\In$ or there is a chain of supports $\big{\{}(S_0,B_0),...,(S_n,B_n) \big{\}}$ which starts at $A$ and for which there exists an attack $(D,B_n)\in Att$ with $Lab(D)=\In$.
In the first case, $\big{\{}(S,A)\big{\}}$ is a (trivial) chain of supports starting with $(S,A)$ \suchthat $A$ is attacked by an argument $D\in in(Lab)$.
Thus $(S,A)\in SS_{Lab}(B)$ is a contradiction according to the definition of safe supports in~\ref{Def:JSBAF:SafeSupport}.
In the second case, we can construct a new chain of supports starting with $(S,A)$, namely $\big{\{}(S,A),(S_0,B_0),...,(S_n,B_n)\big{\}}$, \suchthat $B_n$ is attacked by an argument $D\in in(Lab)$.
Again, this contradicts $(S,A)\in SS_{Lab}(B)$.

Next, assume that $A\in SC(C)$ for some safe support $(S,C)\in SS_{Lab}(B)$.
Then there exists a chain of supports $\mathcal{C}=\big{\{}(S,C),...,(S',A)\big{\}}$ that starts with $(S,C)$.
Because $Lab(A)=\Out$, there either is an attack $(D,A)\in Att$ with $Lab(D)=\In$ or there is a chain of supports $\big{\{}(S'_0,B'_0),...,(S'_m,B'_m)\big{\}}$ for which we have $A\in S'_0$ and an attack $(D,B'_m)\in Att$ with $Lab(D)=\In$.
In the first case, $\mathcal{C}$ itself is a chain of supports starting with $(S,C)$ which contradicts $(S,C)\in SS_{Lab}(B)$ and in the second case we can construct the chain of supports $\big{\{}(S,C),...,(S',A),(S'_0,B'_0),...,(S'_m,B'_m)\big{\}}$ which contradicts $(S,C)\in SS_{Lab}(B)$.
\end{proof}

Now for the prove that $Lab_{\beta+1}$ is a valid labeling:

\begin{proposition}\label{Prop:JSBAF:GRConstructionIsLabeling}
Let \JSBAFGR\ be a JSBAF \suchthat \JDA.
Furthermore, let $GC=(Lab_0,...,Lab_{\alpha})$ be a grounded construction and let $Lab_\beta$ be a labeling of this grounded construction \suchthat $Lab_\beta\in adm(\mathcal{J})$.
Then $Lab_{\beta+1}$ is a labeling.
\end{proposition}
\begin{proof}
Throughout this proof, let $X$ be the argument in $FI(Lab_\beta)$ that is chosen for this step of the grounded construction, \ie $Lab_{\beta+1}=Lab_\beta\cup\{X\}\cup\bigcup\limits_{(S,B)\in SS_{Lab_\beta}(X)}\big(\{B\}\cup SC(B)\big)$.

From the construction of $Lab_{\beta+1}$ it is clear that each $A\in Args$ receives \emph{some} label and that $undec(Lab_{\beta+1})\cap in(Lab_{\beta+1})=undec(Lab_{\beta+1})\cap out(Lab_{\beta+1})=\emptyset$.
It remains to be shown that $in(Lab_{\beta+1})\cap out(Lab_{\beta+1})=\emptyset$ also holds.
For this, we will argue via induction over $n\in\mathbb{N}$ that $in(Lab_{\beta+1})\cap O_{\beta+1}^{n}=\emptyset$.
Towards a contradiction, assume that this does not hold and let $A\in in(Lab_{\beta+1})\cap O_{\beta+1}^n$.

Induction start $n=0$:
Then there exists an argument $C\in in(Lab_{\beta+1})$ \suchthat $(C,A)\in Att$.
Because $A\in in(Lab_{\beta+1})$, we can have three cases:
Either $A\in in(Lab_\beta)$, $A=X$ (\ie $A$ was forced \In \wrt $Lab_\beta$), or $A\in\bigcup\limits_{(S,B)\in SS_{Lab}(X)}\big(\{B\}\cup SC(B)\big)$.
We first argue that each of these cases means $C\in out(Lab_\beta)$ must hold:
This is clear for the first and second case.
For the third case, assume first that there is a safe support $(S,B)\in SS_{Lab_\beta}(X)$ for $X$ in $Lab_\beta$ \suchthat $A=B$.
Then $\big{\{}(S,A)\big{\}}$ is a chain of supports that starts with $(S,A)$.
By definition of a safe support in~\ref{Def:JSBAF:SafeSupport}, all attackers of $A$ are \Out in $Lab_\beta$, \ie we have $Lab_\beta(C)=\Out$.
On the other hand, if there is a safe support $(S,B)\in SS_{Lab_\beta}(X)$ \suchthat $A\in SC(B)$, then we must have a chain of supports $\big{\{}(S,B),...,(S',A)\big{\}}$.
Again, by definition of a safe support this means all attackers of $A$ are \Out in $Lab_\beta$, \ie we have $Lab_\beta(C)=\Out$.

Now to see that $C$ cannot be in $Lab_{\beta+1}$:
$C\in Lab_{\beta+1}$ means we either have $C\in in(Lab_\beta)$, $C=X$ or $C\in\bigcup\limits_{(S,B)\in SS_{Lab}(X)}\big(\{B\}\cup SC(B)\big)$.
Because $Lab_\beta(C)=\Out$, we must have $Lab_\beta(C)\neq\In$.
By Proposition~\ref{Prop:JSBAF:ForcedInArgIsNotOut} we can also infer that $C\not\in FI(Lab_\beta)$, \ie $C\neq X$.
Lastly, by Proposition~\ref{Prop:JSBAF:SafeSuppArgumentNotOut} we know that $C\not\in\bigcup\limits_{(S,B)\in SS_{Lab_\beta}(X)}\big(\{B\}\cup SC(B)\big)$.
Thus all the possible cases for $C\in in(Lab_{\beta+1})$ lead to a contradiction and we infer $in(Lab_{\beta+1})\cap O_{\beta+1}^0=\emptyset$.

Induction step $n\rightarrow n+1$:
By IH we know $in(Lab_{\beta+1})\cap O_{\beta+1}^n=\emptyset$, thus we concentrate on $O'=O_{\beta+1}^{n+1}\backslash O_{\beta+1}^n$.
If $A\in O'$, then there must be a support $(S,B)\in Supp$ with $A\in S$, $S\backslash\{A\}\subseteq in(Lab_{\beta+1})$ and $B\in O_{\beta+1}^n$.
Because $A\in in(Lab_{\beta+1})$ by assumption, we can now infer $S\subseteq in(Lab_{\beta+1})$.
We argue over the possible cases for $S\subseteq in(Lab_\beta)\cup\{X\}\cup\bigcup\limits_{(S,B)\in SS_{Lab_\beta}(X)}\big(\{B\}\cup SC(B)\big)$ and show that each of them leads to a contradiction.
From here on, let $\mathcal{U}=\bigcup\limits_{(S',B')\in SS_{Lab_\beta}(X)}\big(\{B'\}\cup SC(B')\big)$.

If $S=\emptyset$, then $B$ is a strict argument, thus $B\in Lab_\beta\subseteq in(Lab_{\beta+1})$, contradicting our induction hypothesis IH.
If $S\subseteq in(Lab_\beta)$, then by admissibility of $Lab_\beta$ we have $B\in in(Lab_\beta)$.
Thus $B\in in(Lab_{\beta+1})$, again contradicting IH.
For the case $S\subseteq\{X\}$ we first note that $X$ was forced \In in $Lab_\beta$.
We ignore the case $Lab_\beta(B)=\In$ because this trivially contradicts IH again.
Thus for $Lab_\beta(B)\neq\In$, either $(S,B)$ was safe for $X$ in $Lab_\beta$ or item 2.a of Definition~\ref{Def:JSBAF:ForcedIn} was satisfied.
In the first case, we have $B\in \mathcal{U}$, therefore $B\in in(Lab_{\beta+1})$, which again contradicts IH.
In the second case, we note that $Lab_\beta$ itself was admissible, so there must have been an extension $Lab'$ of $Lab_\beta$ \suchthat $Lab'(B)\neq\In$ while $X$ was legally \In \wrt $Lab'$.
Clearly this cannot hold, thus we again have a contradiction.
Next, for the case that $S\subseteq in(Lab_\beta)\cup\{X\}$.
Then by Proposition~\ref{Prop:JSBAF:ForcedInHeadIsIn} we either have $Lab_\beta(B)=\In$ or $(S,B)\in SS_{Lab_\beta}(X)$.
Either way, we can infer $Lab_{\beta+1}(B)=\In$ which again contradicts IH.
Lastly, for the remaining cases we note that in each of them we have $S\cap \mathcal{U}\neq\emptyset$.
We show that this cannot hold.

Towards a contradiction, suppose there exists $D\in S\cap \mathcal{U}$.
Let $(S',B')\in SS_{Lab_\beta}(X)$ be the safe support for $X$ in $Lab_\beta$ for which we have $D\in \{B'\}\cup SC(B')$.
In particular, this means that there must be a chain of supports $\big{\{}(S',B'),...,(S'',D)\big{\}}$ starting at $X$.
Because $D\in S$, we can infer from the construction of $O_{\beta+1}^{n+1}$ that there exists a chain of supports $\big{\{}(S,B),...,(S''',B''')\big{\}}$ starting at $D$, \suchthat $(C,B''')\in Att$ and $C\in in(Lab_{\beta+1})$.
Now we can combine these two chains to construct a new chain of supports $\big{\{}(S',B'),...,(S'',D),(S,B),...,(S''',B''')\big{\}}$ with $(C,B''')\in Att$ and $Lab_{\beta+1}(C)=\In$.
We note that, because $(S',B')\in SS_{Lab_\beta}(X)$, by Definition of a safe support in~\ref{Def:JSBAF:SafeSupport}, we must have $Lab_\beta(C)=\Out$.
From $C\in in(Lab_{\beta+1})$ we infer $C\in in(Lab_\beta)\cup\{X\}\cup\mathcal{U}$.
Clearly, $C\in in(Lab_\beta)\cap out(Lab_\beta)$ cannot hold.
By Proposition~\ref{Prop:JSBAF:ForcedInArgIsNotOut} we also know that $C=X$ and $Lab_\beta(C)=\Out$ cannot hold.
Lastly, suppose that $C\in\mathcal{U}$.
Because $Lab_\beta(C)=\Out$, we know that there must have been an attack $(E,C)\in Att$ with $Lab_\beta(E)=\In$ or yet another chain of supports $\big{\{}(\widehat{S},\widehat{B}),...,(\widehat{S}',\widehat{B}')\big{\}}$ \suchthat $C\in\widehat{S}$ and there is $E\in in(Lab_\beta)$ with $(E,\widehat{B}')\in Att$.
Clearly, this means there cannot be a safe support $(\widehat{S}'',\widehat{B}'')\in SS_{Lab_\beta}(X)$ \suchthat $C\in \{\widehat{B''}\}\cup SC(\widehat{B}'')$ holds.

With this, we infer that $S\cap\mathcal{U}=\emptyset$ must hold, therefore all cases for $S\subseteq in(Lab_\beta)\cup\{X\}\cup\mathcal{U}$ have lead to a contradiction.
We conclude that $in(Lab_{\beta+1})\cap O_{\beta+1}^{n+1}=\emptyset$ must hold.
\end{proof}

The proof for $A\in out(Lab_{\beta+})$ iff $A$ is legally \Out \wrt $Lab_{\beta+1}$ will be omitted because this is clear from the construction of $Lab_{\beta+1}$.

\begin{proposition}\label{Prop:JSBAF:GRConstructionOutIffLegallyOut}
Let \JSBAFGR\ be a JSBAF \suchthat \JDA.
Furthermore, let $GC=(Lab_0,...,Lab_{\alpha})$ be a grounded construction and let $Lab_\beta$ be a labeling of this grounded construction \suchthat $Lab_\beta\in adm(\mathcal{J})$.
Then $A\in out(Lab_{\beta+1})$ iff $A$ is legally \Out \wrt $Lab_{\beta+1}$.
\end{proposition}

Next, we will show another auxiliary statement concerning the arguments labeled \Out by two different labelings $Lab,Lab'$ \suchthat $in(Lab)\subseteq in(Lab')$.

\begin{proposition}\label{Prop:JSBAF:HelperInAdmImpliesOutADM}
Let \JSBAFGR be a JSBAF \suchthat \JDA.
Furthermore, let $Lab,Lab'$ be two labelings of $\mathcal{J}$ \suchthat $Lab\in adm(\mathcal{J})$.
Lastly, suppose that if $A$ is legally \Out \wrt $Lab'$, then $A\in out(Lab')$.
If $in(Lab)\subseteq in(Lab')$, then we have $out(Lab)\subseteq out(Lab')$.
\end{proposition}
\begin{proof}
Assume $in(Lab)\subseteq in(lab')$ and let $A\in out(Lab)$.
Because $Lab$ is admissible by assumption, we know that $A$ is legally \Out \wrt $Lab$.
Thus there either is an attack $(B,A)\in Att$ with $B\in in(Lab)$, or there is a chain of supports $\big{\{}(S_0,B_0),...,(S_n,B_n)\big{\}}$ that starts at $A$, ends at $B_n$ and satisfies the conditions of Definition~\ref{Def:Ground:LegalLabeling}.
In the first case, we have that $A$ is legally \Out \wrt $Lab'$ and by assumption this implies $A\in out(Lab')$.
For the second case we can prove via induction over $n\in\mathbb{N}$, for $n$ being the length of the chain, that $A$ is again legally \Out \wrt $Lab'$, which implies $A\in out(Lab')$.
As this induction is trivial, we omit it here.
\end{proof}

\begin{corollary}\label{Corr:JSBAF:InAdmImpliesOutAdm}
Let \JSBAFGR be a JSBAF \suchthat \JDA.
Furthermore, let $Lab,Lab'\in adm(\mathcal{J})$.
If we have $in(Lab)\subseteq in(Lab')$, then $out(Lab)\subseteq out(Lab')$ also holds.
\end{corollary}

Lastly, we will show that any argument that is labeled \In by $Lab_{\beta+1}$ is also legally \In \wrt $Lab_{\beta+1}$.

\begin{proposition}\label{Prop:JSBAF:GRConstructionInIsLegallyIn}
Let \JSBAFGR\ be a JSBAF \suchthat \JDA.
Furthermore, let $GC=(Lab_0,...,Lab_{\alpha})$ be a grounded construction and let $Lab_\beta$ be a labeling of this grounded construction \suchthat $Lab_\beta\in adm(\mathcal{J})$.
If $A\in in(Lab_{\beta+1})$, then $A$ is legally \In \wrt $Lab_{\beta+1}$.
\end{proposition}
\begin{proof}
Let $A\in in(Lab_{\beta+1})$.
Towards a contradiction, assume that $A$ is not legally \In \wrt $Lab_{\beta+1}$.
First, assume that there is an attack $(B,A)\in Att$ with $Lab_{\beta+1}(B)\neq\Out$.
By our grounded construction we have $A\in Lab_\beta$, $A=X$ or $A\in \cup\bigcup\limits_{(S,B)\in SS_{Lab_\beta}(X)}\big(\{B\}\cup SC(B)\big)$ for $X\in FI(Lab_\beta)\backslash in(Lab_\beta)$.
In either of these cases, we know that all attackers $B$ of $A$ were labeled \Out in $Lab_\beta$.
By $in(Lab_\beta)\subseteq in(Lab_{\beta+1})$ and the same argumentation that was used in the proof for Proposition~\ref{Prop:JSBAF:HelperInAdmImpliesOutADM} we can infer that $B$ is legally \Out \wrt $Lab_{\beta+1}$.
By Proposition~\ref{Prop:JSBAF:GRConstructionOutIffLegallyOut} we have $Lab_{\beta+1}(B)=\Out$, contradicting the assumption $Lab_{\beta+1}(B)\neq\Out$.

Now take some support $(S,B)\in Supp$ with $A\in S$ and assume that $\dot{ \big\{ } Lab_{\beta+1}(C)\mid C\in S\backslash\{A\}\dot{ \big\} } \not\leq Lab_{\beta+1}(B)$.
We only need to consider the case where $S\subseteq in(Lab_{\beta+1})$ while $B\not\in in(Lab_{\beta+1})$ (we have excluded the case $|S\cap in(Lab_{\beta+1})|=|S|-1$ while $Lab_{\beta+1}(B)=\Out$ and $S\cap out(Lab_{\beta+1})=\emptyset$ by construction of $O_{\beta+1}^{n+1}$ whereas the cases $|S\cap in(Lab_{\beta+1})|<|S|-1$ while $Lab_{\beta+1}(B)=\Out$ and $|S\cap in(Lab_{\beta+1})|=|S|-1$ while $Lab_{\beta+1}(B)=\Undec$ are trivial).
Let $S\subseteq in(Lab_{\beta+1})$ and suppose $Lab_{\beta+1}(B)\neq\In$.
Now we again consider all the possible cases for $S\subseteq in(Lab_\beta)\cup\{X\}\cup\bigcup\limits_{(S',B')\in SS_{Lab_\beta}(X)}\big(\{B'\}\cup SC(B')\big)$.
If $S=\emptyset$ then $B$ is strict, thus $B\in in(Lab_\beta)\subseteq in(Lab_{\beta+1})$, contradicting our assumption.
If $S\subseteq in(Lab_\beta)$, then $B\in in(Lab_\beta)\subseteq in(Lab_{\beta+1})$ by admissibility of $Lab_\beta$, which again contradicts our assumption.
If $S\subseteq\{X\}$ or $S\subseteq in(Lab_\beta)\cup\{X\}$, then we know by Proposition~\ref{Prop:JSBAF:ForcedInHeadIsIn} that $B\in in(Lab_\beta)\subseteq in(Lab_{\beta+1})$ which again contradicts our assumption.
For the remaining cases we note that we have $\bigcup\limits_{(S',B')\in SS_{Lab_\beta}(X)}\big(\{B'\}\cup SC(B')\big)\cap S\neq\emptyset$.
However, this means there is $(S',B')\in SS_{Lab_\beta}(X)$ \suchthat $D\in \{B'\}\cup SC(B')$ and $D\in S$.
Now $B\in \{B'\}\cup SC(B')$ holds, meaning $Lab_{\beta+1}(B)=\In$, which again contradicts our assumption.

As all possible cases for $S$ have lead to a contradiction, we conclude that $S\subseteq in(Lab_{\beta+1})$ implies $B\in in(Lab_{\beta+1})$, as required.
\end{proof}

\subsubsection{Transfinite Induction}

Finally, we are ready to show that every labeling that we create during our grounded construction is an admissible labeling:

\begin{proposition}\label{Prop:JSBAF:GRConstructionResultAdmissible}
Let \JSBAFGR be a JSBAF \suchthat \JDA.
Furthermore, let $GC=\{Lab_0,...,Lab_{\alpha}\}$ be a grounded construction.
Then for all ordinals $\delta$, each $Lab_{\delta}$ is an admissible labeling.
\end{proposition}
\begin{proof}
We show the claim via transfinite induction.

Induction start $\delta=0$.
Then $Lab_{\delta}$ is the strict including minimal labeling of $\mathcal{J}$.
By Proposition~\ref{Theo:JSBAF:MSIisAdmissible} we know that $Lab_{\delta}$ is an admissible labeling.

Induction step $\delta\rightarrow\delta+1$:
By the induction hypothesis we know that $Lab_{\delta}$ was an admissible labeling.
By Propositions~\ref{Prop:JSBAF:GRConstructionIsLabeling}, \ref{Prop:JSBAF:GRConstructionOutIffLegallyOut}  and~\ref{Prop:JSBAF:GRConstructionInIsLegallyIn}, $Lab_{\delta+1}$ is an admissible labeling ($STR_{\mathcal{J}}\subseteq in(Lab_{\delta+1})$ is trivial).

Limit step $\alpha$:
From here on, we denote our induction hypothesis for the transfinite induction as IH$_{\alpha}$.
Note that $STR_{\mathcal{J}}\subseteq in(Lab_{\alpha})$ is again trivially satisfied.
We first show that $Lab_{\alpha}$ is a labeling.
By construction it is clear that all arguments receive at least one label and that $in(Lab_{\alpha})\cap undec(Lab_{\alpha})=out(Lab_{\alpha})\cap undec(Lab_{\alpha})=\emptyset$.
Thus we have left to show $in(Lab_{\alpha})\cap out(Lab_{\alpha})=\emptyset$.
Towards a contradiction, assume that this does not hold.
Then there are ordinals $\gamma<\alpha$ and $\delta<\alpha$ \suchthat $A\in in(Lab_{\gamma})$ and $A\in out(Lab_{\delta})$.
We either have $\gamma\leq\delta$ or $\delta\leq\gamma$.
In the first case, $A\in in(Lab_{\delta})\cap out(Lab_{\delta})$ by construction of $in(Lab_{\delta})$.
This contradicts IH$_{\alpha}$.
In the second case, we note that $A\in out(Lab_{\delta})$ means $A$ is legally \Out \wrt $Lab_{\delta}$.
Thus there is an attack $(B,A)\in Att$ with $B\in in(Lab_{\delta})$ or a chain of supports $\big{\{}(S_0,B_0),...,(S_n,B_n)\big{\}}$ which starts at $A$ and for which we have an attack $(C,B)\in Att$ with $C\in in(Lab_{\delta})$ while for $0<i<n$ $|S_i\cap in(Lab_{\delta})|=|S|-1$.
We know by construction of $in(Lab_{\gamma})$ that in either of those cases $A\in out(Lab_{\gamma})\cap in(Lab_{\gamma})$ holds, which again contradicts IH$_{\alpha}$.

Next, we show that $A\in in(Lab_{\alpha})$ implies that $A$ is legally \In \wrt $Lab_{\alpha}$.
Let $\gamma<\alpha$ \suchthat $Lab_{\gamma}(A)=\In$.
Then $A$ was legally \In \wrt $Lab_{\gamma}$, meaning all attackers of $A$ were labeled \Out in $Lab_{\gamma}$.
By construction of $Lab_{\alpha}$, we can infer that each of these attackers is labeled \Out in $Lab_{\alpha}$.
Now take some support $(S,B)$ with $A\in S$.
Towards a contradiction, suppose first that $S\subseteq in(Lab_{\alpha})$ while $B\not\in in(Lab_{\alpha})$.
Let $S=\{A,A_0,...,A_n\}$ and take $Lab_{\gamma},Lab_{\gamma_0},...,Lab_{\gamma_n}$ \suchthat $A\in in(Lab_{\gamma})$ and for each $0<i<n$ we have $A_i\in in(Lab_{\gamma_i})$.
Now let $\gamma_{max}$ be \suchthat $\gamma\leq\gamma_{max}$ and $\gamma_i\leq\gamma_{max}$ for each $0\leq i\leq n$.
Then $S\subseteq in(Lab_{\gamma_{max}})$, meaning $B\in in(Lab_{\gamma_{max}})\subseteq in(Lab_{\alpha})$, contradicting our assumption.
Next we note that if $|S\cap in(Lab_{\alpha})|\leq |S|-2$ and $Lab_{\alpha}(B)=\Out$ or $|S\cap in(Lab_{\alpha})|\leq |S|-1$ and $Lab_{\alpha}(B)=\Undec$, then $\dot{ \big\{ } Lab_{\alpha}(C)\mid C\in S\backslash\{A\}\dot{ \big\} } \leq Lab_{\alpha}(B)$ is trivially satisfied.
We therefore assume $S\backslash in(Lab_{\alpha})=\{A'\}$, $Lab_{\alpha}(B)=\Out$ and for $A'$ we have $Lab_{\alpha}(A')=\Undec$.
By construction of $Lab_{\alpha}$ there must exist $Lab_{\gamma}$ \suchthat $Lab_{\gamma}(B)=\Out$.
Similarly, for $S'=\{A_0,...,A_n\}$  there must exist $Lab_{\gamma_i}$ for each $0\leq i\leq n$ \suchthat $Lab_{\gamma_i}(A_i)=\In$.
Take $\gamma_{max}\in\{\gamma,\gamma_0,...,\gamma_n\}$ \suchthat $\gamma\leq\gamma_{max}$ and $\gamma_i\leq\gamma_{\max}$ for each $0\leq i\leq n$.
Then $S'\subseteq in(Lab_{\gamma_{max}})$ while $B\in out(Lab_{\gamma_{max}})$ by construction of $Lab_{\gamma_{max}}$.
Since $\gamma_{max}<\alpha$, $Lab_{\gamma_{max}}$ is admissible by IH$_{\alpha}$, therefore $Lab_{\gamma_{max}}(A')=Lab_{\alpha}(A')=\Out$.
Note that $A'=A$ cannot hold since we assumed $Lab_{\alpha}(A)=\In$.
Therefore there is an argument labeled \Out in $S$ \suchthat $\dot{ \big\{ } Lab_{\alpha}(C)\mid C\in S\backslash\{A\}\dot{ \big\} } \leq Lab_{\alpha}(B)$ is satisfied.

Lastly, we show $A\in out(Lab_{\alpha})$ iff $A$ is legally \Out \wrt $Lab_{\alpha}$.
First, assume that $A\in out(Lab_{\alpha})$.
Then by construction of $Lab_{\alpha}$ there is $\gamma<\alpha$ \suchthat $A\in out(Lab_{\gamma})$.
Since $Lab_{\gamma}$ was an admissible labeling by IH$_{\alpha}$, we again have an attack or a chain of supports \suchthat $A$ is legally \Out \wrt $Lab_{\gamma}$.
By construction of $Lab_{\alpha}$ we have $in(Lab_{\gamma})\subseteq in(Lab_{\alpha})$ and $out(Lab_{\gamma})\subseteq out(Lab_{\alpha})$.
Therefore we can infer that $A$ is legally \Out \wrt $Lab_{\alpha}$.

Now suppose that $A$ is legally \Out \wrt $Lab_{\alpha}$.
Assume first that there is an attack $(B,A)\in Att$ with $Lab_{\alpha}(B)=\In$.
Then there is $\gamma<\alpha$ \suchthat $Lab_{\gamma}(B)=\In$.
By IH$_{\alpha}$ we know that $Lab_{\gamma}$ was an admissible labeling, thus $A\in out(Lab_{\gamma})\subseteq out(Lab_{\alpha})$.
Now assume that there is a chain of supports $
\big{\{}(S_0,B_0),...,(S_n,B_n)\big{\}}$ \suchthat $A\in S_0$, there exists an attack $(C,B_n)\in Att$ with $C\in in(Lab_{\alpha})$ and for each $0<i<n$ we have $|S_i\cap Lab_{\alpha}|=|S_i|-1$ while $B_i\in out(Lab_{\alpha})$.
We show the claim via induction over $n\in\mathbb{N}$ for $n$ being the length of that chain.
Let $S'=S_0\backslash\{A\}$.
Induction start $n=1$:
By construction of $Lab_{\alpha}$ there exists a labeling $Lab_{\gamma}$ for each $X\in\{C\}\cup S'$ \suchthat $Lab_{\gamma}(X)=\In$.
Let $\mathbb{L}=\{Lab_{\gamma_0},...,Lab_{\gamma_m}\}$ be the set of all these labelings $Lab_{\gamma}$ and take $Lab_{\gamma_{max}}\in \mathbb{L}$ \suchthat $\gamma_{i}\leq \gamma_{\max}$ for each $0\leq i\leq m$.
By construction of $in(Lab_{\gamma_{max}})$ we know that $\{C\}\cup S'\subseteq in(Lab_{\gamma_{max}})$.
Thus $B\in out(Lab_{\gamma_{max}})$ and therefore $A\in out(Lab_{\gamma_{max}})$, meaning $A\in out(Lab_{\alpha})$ by construction of $Lab_{\alpha}$.
Induction step $n\rightarrow n+1$:
By IH we know that for the chain of supports $\big{\{}(S_1,B_1),...,(S_{n+1},B_{n+1})\big{\}}$ that starts at $B_0\in S_1$ and has length $n$, we have $B_0\in out(Lab_{\alpha})$.
By construction of $Lab_{\alpha}$ there is $\gamma<\alpha$ \suchthat $Lab_{\gamma}(B_0)=\Out$ holds.
Since $Lab_{\gamma}$ was an admissible labeling, we again know that there is an attack from an argument labeled \In or a chain of supports satisfying the conditions of Definition~\ref{Def:Ground:LegalLabeling}.
Again, for each $X\in S'=S_0\backslash\{A\}$, take a labeling $Lab_{\gamma'}$ \suchthat $Lab_{\gamma'}(X)=\In$ and let $\mathbb{L}=\{Lab_{\gamma'_0},...,Lab_{\gamma'_n}\}$ be the set of all these labelings.
Now take $\gamma_{max}\in\{\gamma,\gamma'_0,...,\gamma'_n\}$ \suchthat for $0\leq i\leq n$ we have $\gamma'_i\leq\gamma_{max}$ and $\gamma\leq\gamma_{max}$.
By construction of $Lab_{\gamma_{max}}$, we know $S'\subseteq in(Lab_{\gamma_{max}})$ and $B\in out(Lab_{\gamma_{max}})$, therefore $A$ must be legally \Out \wrt $Lab_{\gamma_{max}}$.
Since $Lab_{\gamma_{max}}$ was an admissible labeling we have $Lab_{\gamma_{max}}(A)=Lab_{\alpha}(A)=\Out$ as required.
\end{proof}

It is now easy to see that the result of a grounded construction is a ground-complete labeling:

\begin{proposition}
Let \JSBAFGR be a JSBAF \suchthat \JDA and let $Lab$ be the result of a grounded construction $GC=(Lab_0,...,Lab_{\alpha})$.
Then $Lab$ is a ground-complete labeling.
\end{proposition}
\begin{proof}
By Proposition~\ref{Prop:JSBAF:GRConstructionResultAdmissible} we know that $Lab=Lab_{\alpha}$ is a admissible labeling.
By Definition~\ref{Def:JSBAF:GRConstruction} we also know that $FI(Lab)\backslash in(Lab)=\emptyset$.
We can therefore infer that for every $A\in FI(Lab)$, $A\in in(Lab)$ must hold.
Thus $Lab$ is a ground-complete labeling according to Definition~\ref{Def:JSBAF:GRCompleteLabeling}.
\end{proof}

At this point, we quickly state that since our grounded construction results in \emph{some} ground-complete labeling, the set of all ground-complete labelings is non-empty (for this, we also note that the starting point of our grounded construction, $SIM$, always exists).
Therefore, there has to exist \emph{at least one} grounded labeling for each JSBAF:

\begin{corollary}\label{Cor:JSBAF:GRAlwaysExists}
Let \JSBAFGR be a JSBAF \suchthat \JDA.
Then $gr(\mathcal{J})\neq\emptyset$.
\end{corollary}

\subsubsection{Uniqueness of the grounded labeling}

Before we can show the uniqueness of grounded labelings, we need one more definition that will be used in the upcoming proof:

\begin{definition}
Let \JSBAFGR be a JSBAF \suchthat \JDA.
Furthermore, let $Lab_{\alpha}$ be the result of a grounded construction $GC$ and let $A\in in(Lab_{\alpha})\backslash STR_{\mathcal{J}}$ be a non-strict argument labeled \In by $Lab_{\alpha}$.
We define the \emph{computation step} of $A$ \wrt $GC$, denoted $STEP(A)$ as follows:
$STEP(A)=\beta$ where $GC=(Lab_0,...,Lab_{\beta},...,Lab_{\alpha})$ and $A\in in(Lab_{\beta})$ while there is no $\gamma<\beta$ \suchthat $A\in in(Lab_{\gamma})$.
\end{definition}

The key part of this proof is to show that during the grounded construction, every argument $A$ which is labeled \In in some computation step $\beta$, is also forced \In \wrt every \emph{other} ground-complete labeling.
Therefore all arguments that are labeled \In in $Lab_{\alpha}$ are also labeled \In in every other ground-complete labeling.

\begin{proposition}\label{Prop:JSBAF:GRIsUnique}
Let \JSBAFGR\ be a JSBAF \suchthat \JDA.
Furthermore, let $Lab$ be the result of some grounded construction $GC=(Lab_0,...,Lab_{\alpha})$ for $\mathcal{J}$.
Then $Lab$ is the unique grounded labeling of $\mathcal{J}$.
\end{proposition}
\begin{proof}
Let $L\in grcmp(\mathcal{J})$ be some ground-complete labeling of \J.
We first show that $in(Lab)\subseteq in(L)$ holds.

Towards a contradiction, assume $in(Lab)\not\subseteq in(L)$, \ie there is $A\in in(Lab)\backslash in(L)$.
Let $\beta=STEP(A)$ be the computation step of $A$ \wrt $GC$.
\Generality we assume that there does not exist an argument $A'\in in(Lab)\backslash in(L)$ \suchthat $STEP(A')< \beta$, \ie $A$ is the \enquote{first} argument which was labeled \In by $Lab$ and which is not labeled \In by $L$.
For this, we note that $in(SIM)\subseteq in(Lab)$ and $in(SIM)\subseteq in(L)$ because both of these labelings are admissible.
This means that, although we did not define $STEP$ for strict arguments, this is not relevant here because $A\not\in in(SIM)$ must hold.
Lastly, we note that for each labeling $Lab_{\gamma}$ \suchthat $\gamma < \beta$, we have $in(Lab_{\gamma})\subseteq in(L)$, because there is no $A'\in in(Lab)\backslash in(L)$ with $STEP(A')<\beta$.
By Corollary~\ref{Corr:JSBAF:InAdmImpliesOutAdm}, this means $out(Lab_{\gamma})\subseteq out(L)$ also holds.
We will now show that $A$ is forced \In \wrt $L$.

By construction of $Lab_{\beta}$, we know that there was $\gamma<\beta$ \suchthat $A\in FI(Lab_{\gamma})$ or there is an argument $X\in in(Lab_{\beta})$ \suchthat for $STEP(X)=\gamma$ there is $(S,B)\in SS_{Lab_{\gamma}}(X)$ with $A=B$ or $A\in SC(B)$.
First, let us consider the attackers of $A$ and their label in $L$.
Because we have $A\in FI(Lab_{\gamma})\cup \Big(\bigcup\limits_{(S,B)\in SS_{Lab_{\gamma}}(X)}\big(\{B\}\cup SC(B)\big)\Big)$, we can infer that for all $(C,A)\in Att$, $Lab_{\gamma}(C)=\Out$ must hold.
Because $out(Lab_{\gamma})\subseteq out(L)$, we know that $L(C)=\Out$ must also hold.
Therefore, all attackers of $A$ are labeled \Out by $L$.

Next, let us consider an arbitrary support $(S,B)\in Supp$ with $A\in S$ and $L(B)\neq \In$.
We first note that, if $(S,B)$ was safe for $A$ in $Lab_{\gamma}$, then we can use Corollary~\ref{Corr:JSBAF:InAdmImpliesOutAdm} to infer that $(S,B)$ is safe for $A$ in $L$.
Therefore, if $(S,B)$ is not safe for $A$ in $L$, then $(S,B)$ cannot be safe for $A$ in $Lab_{\gamma}$.
Now we assume that $(S,B)$ is not safe for $A$ in $L$ and show that this means for $(S,B)$, item 2.a of Definition~\ref{Def:JSBAF:ForcedIn} is satisfied.

Take an arbitrary admissible labeling $L'$ with $L'(B)\geq_p L(B)$.
From $L'(B)\geq_p L(B)$ we know that either $L(B)=\Undec$ or $L'(B)=L(B)$ must hold.
Suppose first that $L(B)=\Undec$.
Then by $in(Lab_{\gamma})\subseteq in(L)$ and $out(Lab_{\gamma})\subseteq out(L)$ we can infer that $Lab_{\gamma}(B)=\Undec$ must also hold.
Therefore we have $L'(B)\geq_{p}Lab_{\gamma}(B)$.
Now assume that $L'(B)=L(B)$ and let us consider the possible cases for $Lab(B)$.
The case that $L(B)=\Undec$ is analogous to above.
If $L(B)=\In$, then we must have either $Lab_{\gamma}(B)=\In$ or $Lab_{\gamma}(B)=\Undec$.
Similarly, if $Lab(B)=\Out$ we must have either $Lab_{\gamma}(B)=\Out$ or $Lab_{\gamma}(B)=\Undec$.
In all cases we again infer $L'(B)\geq_{p}Lab_{\gamma}(B)$.
By assumption, $A$ was forced \In \wrt $Lab_{\gamma}$ while $(S,B)$ was not safe for $A$ in $Lab_{\gamma}$.
By Definition~\ref{Def:JSBAF:ForcedIn}, we now infer that there is an admissible labeling $L''$ which extends $L'$ \suchthat $A$ is legally \In \wrt $Lab''$ and $Lab''(B)=L'(B)$.

We have just shown that for an arbitrary support $(S,B)$ with $A\in S$ and $L(B)\neq\In$, if $(S,B)$ is not safe for $A$ in $L$, then for every admissible labeling $L'$ with $L'(B)\geq_p L(B)$, we can find an admissible labeling $L''$ which extends $L'$ \suchthat $L''(B)=L'(B)$ and $A$ is legally \In \wrt $L''$.
By Definition~\ref{Def:JSBAF:ForcedIn}, $A$ is now forced \In \wrt $L$.
Since $L$ is a ground-complete labeling by assumption, we infer that $A\in in(L)$ must hold, a contradiction to our assumption $A\in in(Lab)\backslash in(L)$.

With this, we have shown that for every $L\in grcmp(\mathcal{J})$ and for $Lab$ being the result of a grounded construction, we have $in(Lab)\subseteq in(L)$.
In particular, this means that there is no $L\in grcmp(\mathcal{J})$ \suchthat $in(L)\subset in(Lab)$, \ie the result of our grounded construction is a grounded labeling according to Definition~\ref{Def:JSBAF:GR}.

Lastly, to show the uniqueness of the grounded labeling, assume towards a contradiction that there is some $L\in gr(\mathcal{J})$ with $Lab\neq L$.
By Definition~\ref{Def:JSBAF:GR}, $L\in grcmp(\mathcal{J})$.
We have argued above that $in(Lab)\subseteq in(L)$ must hold.
Because $L\in gr(\mathcal{J})$ we cannot have $in(Lab)\subset in(L)$, therefore $in(Lab)=in(L)$. 
By Corollary~\ref{Corr:JSBAF:InAdmImpliesOutAdm} this clearly implies $out(Lab)=out(L)$.
Because $Lab\neq L$ by assumption, it must now be the case that $undec(Lab)\neq undec(L)$.
Obviously this contradicts the fact that both $Lab$ and $L$ are labelings which assign a single label to every argument in $\mathcal{J}$.
We conclude that $Lab=L$ must hold, \ie $Lab$ is the \emph{unique} grounded labeling of $\mathcal{J}$.
\end{proof}

\end{technicalReportOnly}

\section{Future Work}\label{Sec:FutureWork}
\begin{conferencePaperOnly}
While the postulates we discuss have been established for grounded semantics, we thought it might be interesting to adapt our approach to grounded semantics as well.
In the technical report~\cite{cramer2026:SatisfyingRationalityPostulatesTechnicalReport} we present a definition of grounded semantics for JSBAFs. We can show that closure, direct and indirect consistency hold for this semantics, but the proofs for non-interference and crash-resistance are left for future work.
\end{conferencePaperOnly}
\begin{technicalReportOnly}
While the postulates we discuss have been established for grounded semantics, we thought it might be interesting to adapt our approach to grounded semantics as well. We can show that closure, direct and indirect consistency hold for this semantics, but the proofs for non-interference and crash-resistance are left for future work.
\end{technicalReportOnly}

Note that in ASPIC-style formalisms that apply restricted rebut, like ASPIC+, admissibility is required -- albeit not sufficient -- to satisfy the closure postulate (see~\cite{Caminada2017:RationalityPostulates}).
In \DAOM, the support-relation between arguments alone can be used to satisfy the closure postulate.
Thus non-admissibility based semantics also have a chance to satisfy the closure postulate.
It would be interesting to define JSBAF based analogues of naive-based semantics, like $CF2$ semantics, defined by Baroni~\etal~\cite{Baroni2005:SCCRecursiveness}, stage semantics, defined by Verheij~\cite{Verheij1996:stage}, stage2 semantics, defined by Dvo{\v{r}}{\'a}k and Gaggl~\cite{Dvorak2012:incorporating,Dvorak2014:stage}, and SCF2 semantics, defined by Cramer and van der Torre~\cite{Cramer2019:SCF2,Cramer2023:Argumentation}.

\section{Conclusion}\label{Sec:Conclusion}
Although various versions of ASPIC have been proposed in the literature, to the best of our knowledge, so far none of them satisfy all five of the rationality postulates defined by Caminada and Amgoud~\cite{Caminada2007:OnEvaluationOfArgumentationFormalisms} and Caminada~\etal~\cite{Caminada2012:SemiStableSemantics} in a credulous semantics like preferred, while considering both rebuttals and undercuts and while avoiding the downsides of restricted rebuttal. In this paper, we proposed \DAOM to address this issue.

\DAOM combines the notion of Joint Support Bipolar Argumentation Frameworks defined by Cramer~and~Bhadra~\cite{Cramer2020:DeductiveJointSupportForRationalUnrestrictedRebuttal} with the notion of gen-rebuttals of Heyninck~and~Straßer~\cite{Heyninck2017:RevisitingUnrestrictedRebutPreferences}.
We have given definitions of legal labelings and admissibility for JSBAFs (and by extension \DAOM) and used them to define a preferred semantics for JSBAFs.
This semantics intuitively correspond to the preferred semantics of abstract argumentation.
Furthermore, we have shown that with preferred semantics, \DAOM satisfies the rationality postulates of closure, direct consistency, indirect consistency, non-interference and crash-resistance.

\bibliography{literature}

@article{Dung1995:OnTheAcceptabilityOfArguments,
  author       = {Phan Minh Dung},
  title        = {On the Acceptability of Arguments and its Fundamental Role in Nonmonotonic
                  Reasoning, Logic Programming and n-Person Games},
  journal      = {Artif. Intell.},
  volume       = {77},
  number       = {2},
  pages        = {321--358},
  year         = {1995},
  doi          = {10.1016/0004-3702(94)00041-X},
  bibsource    = {dblp computer science bibliography, https://dblp.org}
}

@article{Caminada2017:RationalityPostulates,
  author       = {Martin Caminada},
  title        = {Rationality Postulates: Applying Argumentation Theory for Non-monotonic
                  Reasoning},
  journal      = {{FLAP}},
  volume       = {4},
  number       = {8},
  year         = {2017},
  url          = {http://www.collegepublications.co.uk/downloads/ifcolog00017.pdf},
  bibsource    = {dblp computer science bibliography, https://dblp.org}
}

@article{Caminada2007:OnEvaluationOfArgumentationFormalisms,
  author       = {Martin Caminada and
                  Leila Amgoud},
  title        = {On the evaluation of argumentation formalisms},
  journal      = {Artif. Intell.},
  volume       = {171},
  number       = {5-6},
  pages        = {286--310},
  year         = {2007},
  url          = {https://doi.org/10.1016/j.artint.2007.02.003},
  doi          = {10.1016/J.ARTINT.2007.02.003},
  bibsource    = {dblp computer science bibliography, https://dblp.org}
}

@inproceedings{Caminada2014:PreferencesUnrestrictedRebut,
  author       = {Martin Caminada and
                  Sanjay Modgil and
                  Nir Oren},
  editor       = {Simon Parsons and
                  Nir Oren and
                  Chris Reed and
                  Federico Cerutti},
  title        = {Preferences and Unrestricted Rebut},
  booktitle    = {Computational Models of Argument - Proceedings of {COMMA} 2014, Atholl
                  Palace Hotel, Scottish Highlands, UK, September 9-12, 2014},
  series       = {Frontiers in Artificial Intelligence and Applications},
  volume       = {266},
  pages        = {209--220},
  publisher    = {{IOS} Press},
  year         = {2014},
  url          = {https://doi.org/10.3233/978-1-61499-436-7-209},
  doi          = {10.3233/978-1-61499-436-7-209},
  bibsource    = {dblp computer science bibliography, https://dblp.org}
}

@book{ArgumentationInAI,
author = {Rahwan, Iyad and Simari, Guillermo R.},
title = {Argumentation in Artificial Intelligence},
year = {2009},
isbn = {0387981969},
publisher = {Springer Publishing Company, Incorporated},
edition = {1st}
}

@inproceedings{Heyninck2017:RevisitingUnrestrictedRebutPreferences,
  author       = {Jesse Heyninck and
                  Christian Stra{\ss}er},
  editor       = {Carles Sierra},
  title        = {Revisiting Unrestricted Rebut and Preferences in Structured Argumentation},
  booktitle    = {Proceedings of the Twenty-Sixth International Joint Conference on
                  Artificial Intelligence, {IJCAI} 2017, Melbourne, Australia, August
                  19-25, 2017},
  pages        = {1088--1092},
  publisher    = {ijcai.org},
  year         = {2017},
  url          = {https://doi.org/10.24963/ijcai.2017/151},
  doi          = {10.24963/IJCAI.2017/151},
  bibsource    = {dblp computer science bibliography, https://dblp.org}
}

@article{Prakken2010:AbstractFramework4ArgumentationWithStructureArguments,
  author       = {Henry Prakken},
  title        = {An abstract framework for argumentation with structured arguments},
  journal      = {Argument Comput.},
  volume       = {1},
  number       = {2},
  pages        = {93--124},
  year         = {2010},
  url          = {https://doi.org/10.1080/19462160903564592},
  doi          = {10.1080/19462160903564592},
  bibsource    = {dblp computer science bibliography, https://dblp.org}
}

@article{Caminada2012:SemiStableSemantics,
  author       = {Martin W. A. Caminada and
                  Walter Alexandre Carnielli and
                  Paul E. Dunne},
  title        = {Semi-stable semantics},
  journal      = {J. Log. Comput.},
  volume       = {22},
  number       = {5},
  pages        = {1207--1254},
  year         = {2012},
  url          = {https://doi.org/10.1093/logcom/exr033},
  doi          = {10.1093/LOGCOM/EXR033},
  bibsource    = {dblp computer science bibliography, https://dblp.org}
}

@article{Wu2015:ImplementingCrashResistanceAndNonInterferenceInLogicBasedArgumentation,
  author       = {Yining Wu and
                  Mikolaj Podlaszewski},
  title        = {Implementing crash-resistance and non-interference in logic-based
                  argumentation},
  journal      = {J. Log. Comput.},
  volume       = {25},
  number       = {2},
  pages        = {303--333},
  year         = {2015},
  url          = {https://doi.org/10.1093/logcom/exu017},
  doi          = {10.1093/LOGCOM/EXU017},
  bibsource    = {dblp computer science bibliography, https://dblp.org}
}

@inproceedings{Cramer2020:DeductiveJointSupportForRationalUnrestrictedRebuttal,
  author       = {Marcos Cramer and
                  Meghna Bhadra},
  editor       = {Henry Prakken and
                  Stefano Bistarelli and
                  Francesco Santini and
                  Carlo Taticchi},
  title        = {Deductive Joint Support for Rational Unrestricted Rebuttal},
  booktitle    = {Computational Models of Argument - Proceedings of {COMMA} 2020, Perugia,
                  Italy, September 4-11, 2020},
  series       = {Frontiers in Artificial Intelligence and Applications},
  volume       = {326},
  pages        = {147--158},
  publisher    = {{IOS} Press},
  year         = {2020},
  url          = {https://doi.org/10.3233/FAIA200500},
  doi          = {10.3233/FAIA200500},
  bibsource    = {dblp computer science bibliography, https://dblp.org}
}

@article{Heyninck2021:RationalityMaximalConsistentSetsForFragmentOfASPICPlus,
  author       = {Jesse Heyninck and
                  Christian Stra{\ss}er},
  title        = {Rationality and maximal consistent sets for a fragment of {ASPIC+} without undercut},
  journal      = {Argument Comput.},
  volume       = {12},
  number       = {1},
  pages        = {3--47},
  year         = {2021},
  url          = {https://doi.org/10.3233/AAC-200903},
  doi          = {10.3233/AAC-200903},
  bibsource    = {dblp computer science bibliography, https://dblp.org}
}

@techreport{Amgoud2006:FinalReviewReportFormalArgumentation,
author = {Amgoud, Leila and Bodenstaff, L. and Caminada, M. and McBurney, P. and Parsons, Simon and Prakken, Henry and Veenen, J. and Vreeswijk, Gerard},
year = {2006},
month = {01},
pages = {},
title = {Final Review and Report on Formal Argumentation System}
}

@article{Modgil2013:GeneralAccountOfArgumentationWithPreferences,
  author       = {Sanjay Modgil and
                  Henry Prakken},
  title        = {A general account of argumentation with preferences},
  journal      = {Artif. Intell.},
  volume       = {195},
  pages        = {361--397},
  year         = {2013},
  url          = {https://doi.org/10.1016/j.artint.2012.10.008},
  doi          = {10.1016/J.ARTINT.2012.10.008},
  bibsource    = {dblp computer science bibliography, https://dblp.org}
}

@article{Baroni2005:SCCRecursiveness,
  author       = {Pietro Baroni and
                  Massimiliano Giacomin and
                  Giovanni Guida},
  title        = {SCC-recursiveness: a general schema for argumentation semantics},
  journal      = {Artif. Intell.},
  volume       = {168},
  number       = {1-2},
  pages        = {162--210},
  year         = {2005},
  url          = {https://doi.org/10.1016/j.artint.2005.05.006},
  doi          = {10.1016/J.ARTINT.2005.05.006},
  bibsource    = {dblp computer science bibliography, https://dblp.org}
}

@article{Verheij1996:stage,
  title={Two approaches to dialectical argumentation: admissible sets and argumentation stages},
  author={Verheij, Bart},
  journal={Proc. NAIC},
  volume={96},
  pages={357--368},
  year={1996}
}

@inproceedings{Dvorak2012:incorporating,
  title={Incorporating stage semantics in the scc-recursive schema for argumentation semantics},
  author={Dvor{\'a}k, Wolfgang and Gaggl, Sarah Alice},
  booktitle={In Proceedings of the 14th International Workshop on Non-Monotonic Reasoning (NMR 2012)},
  year={2012}
}

@article{Dvorak2014:stage,
  title={Stage semantics and the SCC-recursive schema for argumentation semantics},
  author={Dvo{\v{r}}{\'a}k, Wolfgang and Gaggl, Sarah Alice},
  journal={Journal of Logic and Computation},
  volume={26},
  number={4},
  pages={1149--1202},
  year={2014},
  publisher={Dov Gabbay}
}

@inproceedings{Cramer2019:SCF2,
  title={SCF2-an argumentation semantics for rational human judgments on argument acceptability},
  author={Cramer, Marcos and van der Torre, Leon},
  booktitle={8th Workshop on Dynamics of Knowledge and Belief (DKB-2019) and 7th Workshop KI \& Kognition (KIK-2019)},
  year={2019}
}

@article{Cramer2023:Argumentation,
  title={An argumentation semantics for rational human evaluation of arguments},
  author={Cramer, Marcos and van der Torre, Leendert},
  journal={Frontiers in Artificial Intelligence},
  volume={6},
  pages={1045663},
  year={2023},
  publisher={Frontiers Media SA}
}
\bibliographystyle{vancouver}

\end{document}